%% file: cdm_theory.tex
\documentclass[11pt]{article}
\usepackage[margin = 1in]{geometry}
\usepackage[colorlinks=true, linkcolor=blue, citecolor=blue]{hyperref}
\usepackage{url}
\usepackage[utf8]{inputenc}
\usepackage{smile}
\usepackage{natbib}
\usepackage{enumitem}
\usepackage{xcolor}
\usepackage[toc,page,header]{appendix}
\usepackage{minitoc}

\newcommand{\diff}{\mathrm{d}}

\newcommand{\R}{\mathbb{R}}

\newcommand{\Z}{\mathbb{Z}}
\newcommand{\N}{\mathbb{N}}

\newcommand{\mH}{\mathcal{H}}

\newcommand{\abs}[1]{\left|#1\right|}
\newcommand{\set}[1]{{\left\{ #1 \right\}}}
\newcommand{\ReLU}{\sigma}
\newcommand{\void}{\text{\O}}
\newcommand{\subopt}{\texttt{SubOpt}}
\input{Sections/notations}

\title{Unveil Conditional Diffusion Models with Classifier-free Guidance: A Sharp Statistical Theory}

\author{
Hengyu Fu\thanks{Peking University. Email: \texttt{2100010881@stu.pku.edu.cn}}
\and
Zhuoran Yang\thanks{Yale University. Email: \texttt{zhuoran.yang@yale.edu}}
\and
Mengdi Wang\thanks{Princeton University. Email: \texttt{\{mengdiw,minshuochen\}@princeton.edu}}
\and
Minshuo Chen\footnotemark[3]
}

\begin{document}

\maketitle

\begin{abstract}
Conditional diffusion models serve as the foundation of modern image synthesis and find extensive application in fields like computational biology and reinforcement learning. In these applications, conditional diffusion models incorporate various conditional information, such as prompt input, to guide the sample generation towards desired properties. Despite the empirical success, theory of conditional diffusion models is largely missing. This paper bridges this gap by presenting a sharp statistical theory of distribution estimation using conditional diffusion models. Our analysis yields a sample complexity bound that adapts to the smoothness of the data distribution and matches the minimax lower bound. The key to our theoretical development lies in an approximation result for the conditional score function, which relies on a novel diffused Taylor approximation technique. Moreover, we demonstrate the utility of our statistical theory in elucidating the performance of conditional diffusion models across diverse applications, including model-based transition kernel estimation in reinforcement learning, solving inverse problems, and reward conditioned sample generation.
\end{abstract}

\doparttoc 
\faketableofcontents 
\part{} 

\input{Sections/Intro}
\input{Sections/Preliminary}

\input{Sections/Approximation}
\input{Sections/Estimation}
\input{Sections/Application}

\bibliography{ref.bib}
\bibliographystyle{plainnat}

\newpage
\appendix
\addcontentsline{toc}{section}{Appendix} 
\part{Appendix} 
\parttoc

\input{Sections/ReLU}
\input{Sections/AppendixEstimation}

\input{Sections/appendixTV}

\input{Sections/appendixclassfree}

\input{Sections/basicsReLU}

\end{document}

%% file: Sections/notations.tex
\newcommand{\Lt}{\cL^{\text{trunc}}}
\newcommand{\lt}{{\ell^{\text{trunc}}}}

\newcommand{\shat}{{\hat{\sb}}}
\newcommand{\indic}[1]{\mathbf{1}\left\{#1\right\}} 
\newcommand{\dmono}{\Phi_{\nb,\npb,\vb,\wb}}
\def\vb{{\mathbf v}}
\def\sb{{\mathbf s}}
\def \Rt{{\cR_{\star}^{\text{trunc}}}}
\def\lhat{{\hat{\ell}}}
\def\lstar{{\ell^{\star}}}
\def \wb{{\mathbf w}}
\def \nb{\mathbf{n}}
\def \npb{\mathbf{n'}}
\def \sstar{\sb^{\star}}
\def \epslow{\epsilon_{\text{low}}} 
\def \trelu{\text{ReLU}} 
\def \tTV{\text{TV}} 
\def \tKL{\text{KL}} 
\newcommand{\clip}[1]{\text{clip} \left(#1\right)}

%% file: Sections/Intro.tex
\section{Introduction}

Diffusion models constitute a class of generative models 
achieving state-of-the-art performance in generating realistic data in computer vision and audio applications \citep{song2019generative, dathathri2019plug, ho2020denoising, song2020score, kong2020diffwave, chen2020wavegrad, mittal2021symbolic, huang2022prodiff, jeong2021diff, ulhaq2022efficient, avrahami2022blended, kim2022diffusionclip, bansal2023universal}. 
The success of diffusion models are further extended in other domains, such as sequential data modeling \citep{alcaraz2022diffusion, tashiro2021csdi, tevet2022human, tian2023fast}, reinforcement learning \citep{pearce2023imitating, chi2023diffusion, hansen2023idql, reuss2023goal}, and life science \citep{cao2022high, chung2022mr, chung2022score, gungor2023adaptive, jing2022torsional, anand2022protein, lee2022proteinsgm, luo2022antigen, mei2022metal, waibel2022diffusion, ingraham2022illuminating, huang20223dlinker, schneuing2022structure, wu2022diffusion, gruver2023protein, weiss2023guided, xu2022geodiff, song2021solving}.

Diffusion models are widely appraised for their high-fidelity sample generation, yet the most fascinating feature is that they allow flexible input ``guidance'' to control the generation process --- an essential property that enables diffusion models for versatile real-world usage. For example, in image synthesis, diffusion models can generate images consistent with input prompts. In reinforcement learning, diffusion models can generate state-action trajectories of high rewards or satisfying safety constraints. To emphasize the dependence on guidance, diffusion models with guidance are termed Conditional Diffusion Models (CDMs).

In the continuous-time limit, CDMs couple two stochastic processes for sample generation. In the forward process, data points are corrupted by adding white noise with increasing variances. Then in the backward process, which can be seen as a time-reversal of the forward process, CDMs produce new samples by sequentially removing noise in the input. The backward process is accomplished by a so-called ``conditional score network'', which approximates the conditional score function -- gradient of the log conditional density function $\nabla \log p_t(\xb | \yb)$. Here $\xb$ is the sample, $\yb$ is the guidance, and $p_t$ is a diffused conditional density (see Section~\ref{sec:pre} for a precise definition). In this regard, the training of a CDM concentrates on obtaining a proper conditional score network.

Due to the introduction of the guidance $\yb$, the training of the conditional score network is different from standard score estimation methods in unconditional diffusion models. Classifier guidance is arguably the first method for training a conditional score network \citep{dhariwal2021diffusion}, which applies with discrete guidance $\yb$, such as class labels of images. Classifier guidance relies on training an external classifier for obtaining the conditional score function $\nabla \log p_t(\xb | \yb)$. The classifier is trained using noise-corrupted data produced by the forward process of CDMs. Consequently, the training can be difficult especially when a significant amount of noise is added to the clean data (corresponding to the later stage of the forward process). To mitigate the issue, classifier-free guidance is proposed to remove the external classifier and allow both discrete and continuous guidance \citep{ho2022classifier}. The idea is to introduce a mask signal to randomly ignore the guidance and unify the learning of conditional and unconditional score networks (a detailed description is deferred to Section~\ref{sec:pre}). Ever since its proposal, classifier-free guidance has become the benchmark method for different applications \citep{Meng2023CVPR, kornblith2023classifierfree}.

Despite the empirical success of CDMs trained with classifier-free guidance, theoretical underpinnings are largely lacking. In particular, the following fundamental questions about CDMs are curiously open:
\begin{center}
\it How do CDMs estimate the conditional score function with classifier-free guidance?\\
What are the corresponding statistical rates for conditional distribution estimation?
\end{center}
Recently, there is a growing body of works studying diffusion models and they provide valuable insights into diffusion models' ability to estimate data distributions \citep{oko2023diffusion, chen2023score, lee2022convergencea, lee2022convergenceb, chen2022sampling, benton2023linear, de2021diffusion, de2022convergence, wibisono2024optimal}. However, most of the study focuses on the unconditional diffusion models. It is noteworthy that \cite{yuan2023reward} consider the reward-directed CDMs and provide reward sub-optimality guarantees. Yet the corresponding analysis is tailored to scalar reward guidance in a semi-parametric setting, and the analysis does not cover the classifier-free guidance method.

In this paper, we answer the posted questions above by establishing the first set of theories of CDMs trained with classifier-free guidance. Specifically, we adopt a nonparametric statistics point of view: We assume H\"{o}lder regularity in the ground-truth conditional distribution and provide a sharp sample complexity bound of conditional distribution estimation. Our results are built upon a novel conditional score approximation theory, which develops a diffused Taylor approximation technique. Moreover, our statistical theory leads to theoretical insights into CDMs in diverse tasks, such as transition kernel estimation in model-based RL, solving inverse problems, and reward-conditioned sample generation. We summarize our contributions in the following.
\vspace{0.05in}
\begin{itemize}[leftmargin=0.15in, nosep]
\item We establish the first universal approximation theory of conditional score functions using neural networks in Theorem~\ref{thm::score approx}. To achieve a desired approximation error in the $L_2$ sense, we show that the network size scales adaptive to the smoothness of the data distribution. This result only requires the initial conditional data distribution to be H\"{o}lder continuous, indicating that the score function inherits the regularity of the data. Further, we establish an improved approximation result under an additional bounded H\"{o}lder norm assumption in Theorem~\ref{thm::score approx exp}. Built upon such approximation theories, we present optimal distribution estimation theory in later sections.
\vspace{0.05in}
\item We study using conditional diffusion models for distribution estimation, and provide sample complexity bounds in Theorem~\ref{thm::TVbound}. To facilitate the analysis, we establish a conditional score estimation result in Theorem~\ref{thm::Generalization}, when using the widely adopted classifier-free guidance method (see an introduction in Section~\ref{sec:pre}). The analysis in Theorem~\ref{thm::Generalization} is built upon a bias-variance trade-off in nonparametric statistics and further connects to Theorem~\ref{thm::TVbound} via Girsanov's theorem from stochastic processes. Our statistical rate in Theorem~\ref{thm::TVbound} matches its minimax lower bound (Proposition~\ref{prop::lower bound TV}). We also present statistical guarantees for the first time of applying conditional diffusion models to model-based reinforcement learning (Proposition~\ref{prop:transition_estimation}).

\vspace{0.05in}
\item We additionally establish theoretical foundations of conditional diffusion models for solving inverse problems and reward conditioned sample generation, demonstrating the utility of our established statistical theories. Specifically, we present sub-optimality bounds when generating high-reward samples in an offline setting (Proposition~\ref{thm::subopt}). We also provide error bounds for estimating the posterior mean given a measurement in linear inverse problems (Proposition~\ref{prop::inverse problem}). These results theoretically explain the performance of conditional diffusion models.
\end{itemize}

\subsection{Related Work}
This work contributes to the theory of diffusion models and develops the first set of theories of conditional diffusion models trained with classifier-free guidance. Existing results on diffusion models can be roughly categorized into two categories: 1) sampling theory assuming good score estimation; 2) approximation and statistical theories on score estimation and further distribution estimation. The two aspects are inner connected as we discuss as follows.

\paragraph{Sampling theory of diffusion models} Several recent sampling theories of diffusion models prove that the distribution generated by the backward process is close to the data distribution, as long as the score function is accurately estimated. The central contribution is a relationship between $\epsilon_{\rm dis}$ and $\epsilon_{\rm score}$, where $\epsilon_{\rm dis}$ is the distribution estimation error and $\epsilon_{\rm score}$ is the score estimation error. Specifically, \cite{de2021diffusion, albergo2023stochastic} establish upper bounds of $\epsilon_{\rm dis}$ using $\epsilon_{\rm score}$ for diffusion Schr\"{o}dinger bridges. The error $\epsilon_{\rm dis}$ is measured in the total variation distance and $\epsilon_{\rm score}$ is measured in the $L_\infty$ norm. More concrete bounds of $\epsilon_{\rm dis}$ are provided in \cite{block2020generative, lee2022convergencea, chen2022sampling, lee2022convergenceb, yingxi2022convergence}. These works specialize $\epsilon_{\rm score}$ to the $L_2$ error of the estimated score function, and $\epsilon_{\rm dis}$ to the total variation distance between the generated distribution and the data distribution. \cite{lee2022convergencea} require the data distribution satisfying a log-Sobolev inequality. Concurrent works \cite{chen2022sampling} and \cite{lee2022convergenceb} relax the log-Sobolev assumption on the data distribution to only having bounded moments.

It is worth mentioning that \cite{lee2022convergenceb} allow $\epsilon_{\rm score}$ to be time-dependent. Recently, \cite{chen2023restoration, chen2023probability, benton2023linear} largely enrich the study of sampling theory using diffusion models. Specifically, novel analyses based on Taylor expansions of the discretized backward process \citep{li2023diffusion} or localization method \citep{benton2023linear} are developed, which improve the upper bound on $\epsilon_{\rm dis}$. Furthermore, \cite{chen2023restoration} extend to DDIM sampling scheme, and \cite{chen2023probability} consider the probabilistic ODE backward sampling.

Besides Euclidean data, \cite{de2022convergence} made the first attempt to analyze diffusion models for learning low-dimensional manifold data. Assuming $\epsilon_{\rm score}$ is small under the $L_\infty$ norm (extension to the $L^2$ norm is also provided), \cite{de2022convergence} bound $\epsilon_{\rm dis}$ of diffusion models in terms of the Wasserstein distance. The obtained bound has an exponential dependence on the diameter of the data manifold. Moreover, \cite{montanari2023posterior} consider using diffusion processes to sample from noisy observations of symmetric spiked models and \cite{el2023sampling} study polynomial-time algorithms for sampling from Gibbs distributions based on diffusion processes. The construction of diffusion processes in \cite{montanari2023posterior, el2023sampling} leverages the idea of stochastic localization \citep{eldan2013thin, montanari2023sampling, chen2022localization, el2022information}.

\paragraph{Score approximation and estimation theory} The score approximation and estimation theory aim to prove the sample complexity bounds of score estimation, which complements the sampling theory. An early work \citep{block2020generative} provides a score estimation guarantee when the error is measured in the $L_2$ norm. Yet the bound depends on some unknown Rademacher complexity of the score network class. More recently, \cite{oko2023diffusion} and \cite{chen2023score} both establish score estimation theories from the nonparametric statistics point of view. \cite{oko2023diffusion} mainly focus on the Euclidean data, while \cite{chen2023score} study low-dimensional subspace data. \citet{wibisono2024optimal} leverage the empirical Bayes theory to study score estimation using kernel methods.

The statistical estimation theory in \cite{oko2023diffusion} and \cite{chen2023score} is established by a bias-variance trade-off analysis. Bounding the bias term relies on an approximation theory of the score function, which implies how to choose a proper score network class. \cite{oko2023diffusion} show the approximation theory by constructing a series of ``diffused basis'' functions. \cite{chen2023score} adopt a different approach and resort to local Taylor approximations. Both works leverage the smoothness of the score function and the approximation error depends on the data dimension. \cite{mei2023deep}, on the other hand, investigate score approximation theory in high-dimensional graphical models, where score approximation tends to be efficient in high dimensions, that is, the sample complexity may not increase with $d$. 

On the algorithmic side, we are aware of \cite{shah2023learning} studying score estimation in Gaussian mixture models. They provide convergence analysis of using gradient descent to minimize the score estimation loss. The algorithmic behavior can be characterized in two phases, where in the large-noise phase, gradient descent is approximated by power iteration, and in the small-noise phase, gradient descent is akin to the EM algorithm.

\paragraph{Distribution estimation theory} Distribution estimation theory of diffusion models is explored in \cite{song2020sliced} and \cite{liu2022let} from an asymptotic statistics point of view. These results do not provide an explicit sample complexity bound. Given the aforementioned sampling theory and score estimation theory, an end-to-end analysis of diffusion models for distribution estimation is established in \cite{oko2023diffusion} and \cite{chen2023score}. In Euclidean space, \cite{oko2023diffusion} show that diffusion models are minimax optimal in estimating distributions with Besov density functions. \cite{chen2023score} unveil the adaptivity of diffusion models to linear subspace data. Recently, \cite{yuan2023reward} study the distribution estimation of conditional diffusion models with scalar reward guidance.

\subsection*{Paper Organization} The rest of the paper is organized as follows: Section~\ref{sec:pre} reviews the score-based diffusion model along with its implementation in classifier-free guidance, and introduces basics on H\"older functions and ReLU neural networks. Section \ref{sec:cdm_approx} establishes the first approximation theory of conditional score functions using neural networks. Section \ref{sec:estimation} presents a distribution estimation theory built upon the score approximation theory in the previous section. We also study an application for transition kernel estimation in model-based reinforcement learning. Section \ref{sec::applications} presents extended applications for reward-directed sample generation and inverse problems.

\subsection*{Notation} 
We use bold normal font letters to denote vectors, e.g., $\xb \in \RR^d, \yb \in \RR^{d_y}$. $\norm{\xb}$ denotes the Euclidean norm of $\xb$. $\norm{\xb}_1 = \sum_{i=1}^d |x_i|$ denotes the $\ell_1$-norm of $\xb$, and $\norm{\xb}_{\infty} = \max_{i\in[d]} |x_i|$ denotes the $\ell_{\infty}$-norm of $\xb$. In describing the forward process of diffusion models, $\phi_t$ denotes the Gaussian transition kernel dependent on $t$. 

%% file: Sections/Preliminary.tex
\section{Preliminaries}\label{sec:pre}
We provide a brief introduction to conditional diffusion models (CDMs) with classifier-free guidance, H\"{o}lder functions, and score neural networks.

\paragraph{Diffusion process} 
Denote the initial conditional distribution as $P(X_0 = \xb|\yb)$ for $\xb \in \R^d$ given $\yb \in \R^{d_y}$. We consider adding noise progressively on $X_0$ only, which is described by a forward Ornstein–Uhlenbeck (OU) process,
\begin{align}\label{eq:forward}
\diff X_t = -\frac{1}{2} X_t \diff t + \diff W_t \quad \text{with} \quad X_0 \sim P(\cdot|\yb),
\end{align}
where $W_t$ is a Wiener process. In the infinite-time limit, $X_\infty$ follows a standard Gaussian distribution. At any finite time $t$, we denote $P_t(\cdot | \yb)$ as the marginal conditional distribution.

The forward process will terminate at a sufficiently large time $T$. To generate new samples, we reverse the time of \eqref{eq:forward} to obtain
\begin{align}\label{eq:backward}
\diff X_t^{\leftarrow} = \left[ \frac{1}{2} X_t^{\leftarrow} + \nabla \log p_{T-t}(X_t^{\leftarrow} | \yb)\right] \diff t + \diff \overline{W}_t \quad \text{with} \quad X_0^{\leftarrow} \sim P_T(\cdot|\yb),
\end{align}
where $\overline{W}_t$ is a time-reversed Wiener process and we use the arrow on $X$ to emphasize the backward process. 
The term $\nabla \log p_{T-t}(X_t^{\leftarrow} | \yb)$ is the conditional score function. Unfortunately, it is unknown and needs to be estimated using conditional score networks. 
We denote by $\shat(\xb,\yb,t)$ as such an estimator of the conditional score $\nabla \log p_t(\xb|\yb)$. Then the sample generation is described by the following backward SDE,
\begin{align}\label{eq:backward approx}
\diff \tilde{X}_t^{\leftarrow} = \left[ \frac{1}{2} \tilde{X}_t^{\leftarrow} + \shat(\tilde{X}_t,\yb,T-t)\right] \diff t + \diff \overline{W}_t \quad \text{with} \quad \tilde{X}_0^{\leftarrow} \sim {\sf N}(0, I).
\end{align}
The marginal distribution of $\tilde{X}_t^{\leftarrow}$ (conditioned on $\yb$) is written as $\tilde{P}
_{T-t}(\cdot |\yb)$.

\paragraph{Classifier-free guidance}
Classifier-free guidance, proposed in \cite{ho2022classifier}, is a widely adopted method for training $\shat(\xb,\yb,t)$. 
In specific, we learn both the conditional and unconditional score functions simultaneously, whose estimators are $\sb_1(\xb,\yb,t)$ and $\sb_2(\xb,t)$, respectively. To unify the notations, let $\tau \in \{\void, {\rm id}\}$ be a mask signal, where $\void$ means that we ignore the guidance $\yb$ and ${\rm id}$ means that we keep the guidance. According to the value of $\tau$, we consider the following two cases:
\begin{align}
\tau = {\rm id}: & \quad \int_{t_0}^T \frac{1}{T-t_0}\EE_{(\xb_0, \yb)} \left[\EE_{\xb' \sim {\sf N}(\alpha_t\xb_0, \sigma_t^2 I)} \left[\norm{\sbb_1(\xb', \yb, t) - \nabla_{\xb'} \log \phi_t(\xb' | \xb_0)}_2^2\right]\right] \diff t, \nonumber \\
\tau = \void: & \quad \int_{t_0}^T\frac{1}{T-t_0} \EE_{\xb_0} \left[\EE_{\xb' \sim {\sf N}(\alpha_t\xb_0, \sigma_t^2 I)} \left[\norm{\sbb_2(\xb', t) - \nabla_{\xb'} \log \phi_t(\xb' | \xb_0)}_2^2\right]\right] \diff t. \nonumber
\end{align}
Here $\phi_t$ is the Gaussian transition kernel of the forward process~\eqref{eq:forward}, i.e., $\nabla \log \phi_t(\xb' | \xb_0) = - (\xb' - \alpha_t \xb_0) / \sigma_t^2$ with $\alpha_t = e^{-t/2}$ and $\sigma_t^2 = 1 - e^{-t}$. We also note that $t_0$ is an early-stopping time to prevent the blow-up of score functions, which is commonly adopted in practice \citep{song2020improved, nichol2021improved}. 
As can be seen, when $\tau = \void$, the objective function reduces to that of score estimation in unconditional diffusion models.

Moreover, we unify these two cases by writing a tri-variate score function $ \sbb(\xb', \cdot, t)$ where the second argument is either $\void$ or $\yb$. 
We define the score estimator $\sb$ and its function class $\cF$ as
\begin{align*}
    \sb(\xb,\yb,t)=\begin{cases}
        \sb_1(\xb,\yb,t) & \text{if}~\yb \in \RR^{d_y}\\
        \sb_2(\xb,t) & \text{if}~\yb =\void
        \end{cases}
     ~~~\text{and}~~~\cF=\cF_1 \times \cF_2,
\end{align*}
where we recall that $\sb_1 \in \cF_1$ and $\sb_2 \in \cF_2$ are the conditional and unconditional score estimators, respectively. The function classes $\cF_1$ and $\cF_2$ are two ReLU neural networks (see \eqref{equ::ReLU archi}) with hyperparameters $(M_t, W, \kappa, L, K)$.
Then we have a unified objective for  classifier-free score estimation:
\begin{align}\label{eq:classifierfree_guidance}
\hat{\sbb} \in \argmin_{\sbb \in \cF} \int_{t_0}^T\frac{1}{T-t_0} \EE_{(\xb_0, \yb)} \left[\EE_{\tau, \xb' \sim {\sf N}(\alpha_t\xb_0, \sigma_t^2 I)} \left[\norm{\sbb(\xb', \tau \yb, t) - \nabla_{\xb'} \log \phi_t(\xb' | \xb_0)}_2^2\right]\right] \diff t,
\end{align}
where the inner expectation is taken with respect to $\tau \sim \textrm{Unif}\{ \void, {\rm id}\}$. We stick to the uniform prior on $\tau$ for simplicity, i.e., $\mathbb{P}(\tau = \void) = \PP(\tau = {\rm id}) = 0.5$. An extension to general mask rates causes no real difficulty.

In practice, \eqref{eq:classifierfree_guidance} is implemented using collected i.i.d. data points $\{(\xb_i, \yb_i)\}_{i=1}^n$, which essentially replaces the expectation over $(\xb_0, \yb)$ by its empirical counterpart. We denote a loss function
\begin{align}\label{equ::single empirical loss classifier}
\ell(\xb,\yb;\sb)=
\int_{t_0}^{T}\frac{1}{T-t_0}\EE_{\tau, \xb' \sim {\sf N}(\alpha_t\xb, \sigma_t^2 I)} \left[\norm{\sbb(\xb', \tau \yb, t) - \nabla_{\xb'} \log \phi_t(\xb' | \xb)}_2^2\right].
\end{align}
Note that we have assumed sufficient sampling on $\xb'$ and the mask signal $\tau$ in \eqref{equ::single empirical loss classifier}. Then classifier-free guidance is to minimize the following empirical risk
\begin{align}\label{eq:cdm_empirical_risk}
\argmin_{\sb \in \cF}~ \hat{\cL}(\sb) = \frac{1}{n} \sum_{i=1}^n \ell(\xb_i, \yb_i; \sb),
\end{align}
where we recall $n$ is the sample size. For future usage, we denote $\cL(\sbb)$ as the population risk function.

\paragraph{H\"{o}lder functions} H\"{o}lder functions are widely studied in nonparametric statistics \citep{gyorfi2006distribution, tsybakov2008introduction, wasserman2006all}. In the paper, we will focus on estimating distributions with a density in a H\"{o}lder ball.
\begin{definition}[H\"older norm]
Let $\beta = s+\gamma > 0$ be a degree of smoothness, where $s = \left\lfloor \beta \right \rfloor$ is an integer and $\gamma \in  \left[0,1\right)$.
For a function $f:\R^d \to \R$, its H\"older norm is defined as
\begin{align*}
   &\norm{f}_{\mH^\beta(\R^d)} :=\max_{\sbb: \norm{\sbb}_1 < s} \sup_{\xb}|\partial^{\sbb} f(\xb)|  +  \max_{\sbb: \norm{\sbb}_1 = s} \sup_{\xb\neq \zb} \frac{ \abs{ \partial^{\sbb} f(\xb) - \partial^{\sbb} f(\zb)} }{\norm{\xb - \zb}_{\infty}^{ \gamma } },
\end{align*}
where $\sbb$ is a multi-index. We say a function $f$ is $\beta$-H\"{o}lder, if and only if $\norm{f}_{\cH^{\beta}(\R^d)} < \infty$.
\end{definition}
We define a H\"{o}lder ball of radius $B > 0$ for some constant $B$ as
\begin{align*}
    \mH^{\beta}(\R^d, B) = \set{f: \R^d \rightarrow \R \middle\vert \norm{f}_{\mH^\beta(\R^d)} < B}.
\end{align*}
In the sequel, we will occasionally omit the domain $\R^d$, if it is clear from the context.

\paragraph{ReLU network architecture} We use neural networks to parameterize score functions. We consider the following class of  ReLU neural networks, denoted by  $\cF$:
\begin{align}
    \cF(M_t, W, \kappa, L, K):=\bigg \{& 
 \sbb(\xb,\yb,t)=(A_{L}\ReLU(\cdot) + \bbb_{L}) \circ \cdots
     \circ(A_{1}[\xb^\top,\yb^\top, t]^{\top} + \bbb_{1}): \nonumber\\
     &A_{i}\in \R^{d_i\times d_{i+1}}
     , ~\bbb_{i}\in \R^{d_{i+1}}, ~\max d_i \le  W, ~\sup_{\xb,\yb}\norm{\sb(\xb,\yb, t)}_{\infty} \le M_t, \nonumber\\
     &\max_{i}\|A_{i}\|_\infty \lor \|\bbb_{i}\|_\infty \leq \kappa, ~\sum_{i=1}^L (\|A_{i}\|_0+\|\bbb_{i}\|_0) 
     \le K
     \bigg \}. \label{equ::ReLU archi}
\end{align}
Here $\ReLU(\cdot)$ is the ReLU activation, $\norm{\cdot}_\infty$ is the maximal magnitude of entries and $\norm{\cdot}_0$ is the number of nonzero entries. 
The complexity of this network class is controlled by the number of layers, the number of neurons of each layer,  the magnitude of the network parameters, the number of nonzero parameters, and the magnitude of the neural network output.
We note that the output range $M_t$ is allowed to be dependent on the input $t$, and if we do not require a bounded output range, we will omit the parameter $M_t$.

%% file: Sections/Approximation.tex
\section{Conditional Score Approximation}\label{sec:cdm_approx}
The first step towards our statistical theory is to choose a proper score neural network for conditional score estimation. We establish an approximation theory of conditional score functions, where the rate of approximation is adaptive to the smoothness of the initial data distribution.

\subsection{Conditional Score Approximation}\label{sec::approx}

We impose the following light tail condition on the initial conditional data distribution $P(\cdot|\yb)$.
\begin{assumption}\label{assump:sub}
The conditional distribution has a density $p(\xb | \yb) \in \mH^{\beta}(\R^d\times[0,1]^{d_y}, B)$ for a H\"{o}lder index $\beta > 0$ and constant $B > 0$. Moreover, there exist positive constants $C_1,C_2$ such that for all $\yb\in[0,1]^{d_y}$, the density function $p(\xb | \yb) \leq C_1\exp(-C_2\norm{\xb}_2^2/2)$.
\end{assumption}
Assumption~\ref{assump:sub} encodes generic distributions with H\"{o}lder continuous densities. We consider bounded guidance $\yb \in [0, 1]^{d_y}$ for technical convenience only; the analysis can also be extended to the case where $\yb$ is unbounded with a light tail (see Appendix~\ref{sec::extend score approx} for details). The H\"{o}lder regularity is similar to the Besov regularity assumed in \cite{oko2023diffusion}, yet our light tail condition generalizes their bounded support condition.

On the other hand, Assumption~\ref{assump:sub} only concerns the regularity of the original data distribution. More importantly, it does not impose conditions on the induced conditional score function. This is substantially weaker than the Lipschitz score condition assumed in \cite{chen2022sampling, lee2022convergencea, lee2022convergenceb, chen2023score, yuan2023reward}.

The following theorem presents the approximation theory for using ReLU neural networks to approximate the conditional score.
\begin{theorem}\label{thm::score approx}
Suppose Assumption \ref{assump:sub} holds. For sufficiently large $N$ and constants $C_{\sigma}, C_{\alpha}>0$, by taking the early-stopping time $t_0=N^{-C_\sigma}$ and the terminal time $T=C_{\alpha}\log N$, there exists $\sb \in \cF(M_t, W, \kappa, L, K) $ such that for any $\yb \in [0,1]^{d_y}$ and $t \in [t_0,T]$, it holds that
\begin{align*}
    \int_{\R^{d}} \norm{\sb(\xb,\yb,t)-\nabla\log p_{t}(\xb|\yb)}^2_2 \cdot p_{t}(\xb|\yb) \diff \xb =\cO \left( \frac{B^2}{\sigma_t^4}\cdot  N^{-\frac{\beta}{d+d_y}}\cdot (\log N)^{d+\beta/2+1}\right).
\end{align*}
The hyperparameters in the ReLU neural network class $\cF$ satisfy
\begin{align*}
    & \hspace{0.4in} M_t = {\cO}\left(\sqrt{\log N}/\sigma^2_t\right),~
     W = {\cO}\left(N\log^7 N\right),\\
     & \kappa =\exp \left({\cO}(\log^4 N)\right),~
    L = {\cO}(\log^4 N),~
     K= {\cO}\left(N\log^9 N\right).
\end{align*}
where $\cO$ hides all other polynomial factors depending on $d, d_y, \beta, C_1$, $C_2,  C_\alpha$ and $C_\sigma$.
\end{theorem}
The proof is provided in Appendix~\ref{sec::proof score approx}. We note that the approximation theory also applies to unconditional score approximation, where we just need to set $d_y = 0$. 

\paragraph{Rate of approximation} Theorem~\ref{thm::score approx} establishes the rate of approximation to the conditional score function at time $t$. For a fixed network size $N$, the approximation error scales as $N^{-\frac{\beta}{d + d_y}}$, indicating a faster approximation when the initial data distribution has a higher order of smoothness. Meanwhile, we also observe that the approximation error increases as time $t$ decreases, which is due to the fact that the score blows up when $t$ approaches zero \citep{song2020improved, vahdat2021score}.

\paragraph{Relation to \cite{chen2023score} and \cite{yuan2023reward}} Both works establish approximation guarantees for Lipschitz continuous score functions. However, such Lipschitzness is not needed in our analysis. Instead, our approximation rate is adaptive to the H\"{o}lder smoothness of the initial conditional data distribution. This adaptivity is due to our novel constructive approximation of the score function. In particular, we write the score function as $\nabla \log p_t =\nabla p_t/p_t$ and propose \textbf{diffused local polynomials} to approximate $p_t$ and $\nabla p_t$ separately. Here $p_t$ and $\nabla p_t$ inherit the smoothness of the initial conditional distribution, without requiring any smoothness of the score function. See Section \ref{sec:proof_details} for more details.

In addition, the time $t$ counts as an additional dimension of the input of the score function and slows down the approximation in \cite{chen2023score, yuan2023reward}. Yet Theorem~\ref{thm::score approx} still takes $t$ as an input, the approximation is not affected by the augmented input dimension. The reason behind this is that the time $t$ enters the score function through the ratio $\alpha_t$ and the variance $\sigma_t$ of the Gaussian noise added to the clean data distribution. Both $\alpha_t$ and $\sigma_t$ are super smooth (infinitely differentiable) and therefore, very easy to approximate using neural networks. 

Theorem~\ref{thm::score approx} is the first approximation theory of conditional score functions with generic H\"{o}lder smooth data distributions. In the following analysis, we present a faster approximation result under a slightly stronger assumption, which further leads to a sharp distribution estimation guarantee in Section~\ref{sec:estimation}.
\begin{assumption}\label{assump::expdensity}
Let $C$ and $C_2$ be two positive constants and function $f \in \mH^{\beta}(\R^d\times [0,1]^{d_y}, B)$ for a constant radius $B$. We assume $f(\xb, \yb) \geq C$ for all $(\xb, \yb)$ and the conditional density function $p(\xb | \yb) = \exp(-C_2\norm{\xb}^2_2/2) \cdot f(\xb, \yb)$.
\end{assumption}
For a better interpretation, we can always write the conditional density function $p(\xb | \yb)$ in Assumption~\ref{assump:sub} as $p(\xb | \yb) = \exp(-C_2 \norm{\xb}_2^2 / 2) \cdot f(\xb, \yb)$. Clearly, $f(\xb, \yb)$ is H\"{o}lder continuous. In this regard, Assumption~\ref{assump::expdensity} only strengthens Assumption~\ref{assump:sub} by imposing lower and upper bounds on $f(\xb, \yb)$. The lower bound on $f(\xb, \yb)$ is often required for effective density estimation \citep{tsybakov2008introduction, wasserman2006all}. The upper bound enables the approximation of $f(\xb, \yb)$ in an extended region (see Section~\ref{sec:proof_details}). A fast approximation rate is presented in the following theorem.
\begin{theorem}\label{thm::score approx exp}
    Suppose Assumption \ref{assump::expdensity} holds. For sufficiently large $N$ and constants $C_{\sigma}, C_{\alpha}>0$, by taking early-stopping time $t_0=N^{-C_\sigma}$ and terminal time $T=C_{\alpha}\log N$, there exists $\sb \in \cF(M_t, W, \kappa, L, K) $ such that for all $\yb \in [0,1]^{d_y}$ and $t \in [t_0,T]$, it holds that
\begin{align*}
    \int_{\R^{d}} \norm{\sb(\xb,\yb,t)-\nabla\log p_{t}(\xb|\yb)}^2_2 \cdot p_{t}(\xb|\yb) ~\diff \xb =\cO \left( \frac{B^2}{\sigma_t^2}\cdot N^{-\frac{2\beta}{d+d_y}} \cdot (\log N)^{\beta+1}\right).
\end{align*}
The hyperparameters in the ReLU neural network class $\cF$ satisfy
\begin{align*}
    & \hspace{0.4in} M_t = \cO\left(\sqrt{\log N}/\sigma_t\right),~
     W = {\cO}\left(N\log^7 N\right),\\
     & \kappa =\exp \left({\cO}(\log^4 N)\right),~
    L = {\cO}(\log^4 N),~
     K= {\cO}\left(N\log^9 N\right).
\end{align*}
\end{theorem}
The proof is provided in Appendix \ref{sec::all proof score sub exp}. We discuss several interpretations.

\paragraph{Improved rate of approximation} The approximation rate here is $N^{-\frac{2\beta}{d + d_y}}$, which is substantially faster than Theorem~\ref{thm::score approx}. Further, we also have an improved dependence on $\sigma_t$ and $\log N$. These improvements are made possible by an intricate approximation of $f(\xb, \yb)$ in a shell region. See details in Section~\ref{sec:proof_details}.

\paragraph{Relation to \cite{oko2023diffusion}} A similar approximation rate is proved in \cite{oko2023diffusion} for Besov data distributions on a bounded domain, where a special boundary condition is needed to validate their approximation theory. Despite that Theorem~\ref{thm::score approx exp} allows conditional score approximation, the major difference in Theorem~\ref{thm::score approx exp} is that it only requires mild boundedness conditions on the conditional density function.

\paragraph{Extensions of Theorems \ref{thm::score approx} and \ref{thm::score approx exp}} We remark that our theory can also apply to the case in which $\yb\in\R^{d_y}$ instead of $\yb\in[0,1]^{d_y}$. Moreover, our theory naturally applies to the unconditioned score approximation (approximate $\nabla \log p_t(\xb)$) when we remove the conditional dependence of $\yb$. We refer the readers to Appendix \ref{sec::extend score approx} for more details about the extensions of our approximation theory. These extensions further enable wide applications of our theory in reinforcement learning and inverse problems.

\subsection{Proof Overview and Unraveling the Fast Rate}\label{sec:proof_details}
Here we introduce a unified analytical framework for proving Theorems~\ref{thm::score approx} and \ref{thm::score approx exp}. The key steps consist of a proper truncation of the data density function and domain, and a novel diffused Taylor polynomial approximation. More importantly, we discuss in detail how Assumption~\ref{assump::expdensity} leads to a fast approximation rate.

\paragraph{Unified Analytical Framework for Theorems~\ref{thm::score approx} and \ref{thm::score approx exp}}
To begin with, we rewrite the score function as
\begin{align*}
\nabla \log p_t(\xb | \yb) = \frac{\nabla p_t(\xb | \yb)}{p_t(\xb | \yb)},
\end{align*}
where we develop approximations to the numerator and denominator separately. Yet the construction of the approximations to the numerator and denominator is almost identical. In the following, we focus on the approximation of $p_t(\xb | \yb)$. We also demonstrate the idea in the left panel of Figure~\ref{fig:slow_fast_rate}.

\vspace{5pt} 
\noindent $\bullet$ \underline{\it Approximate numerator and denominator}. Following the forward process \eqref{eq:forward} of conditional diffusion models, we have
\begin{align}\label{equ:: sketch p_t integration}
p_t(\xb|\yb)&=\int_{\RR^{d}}p(\zb|\yb)\frac{1}{\sigma_t^{d}(2\pi)^{d/2}} \exp\left(-\frac{\norm{\alpha_t\zb-\xb}^2}{2\sigma_t^2}\right)\diff \zb.
\end{align}
Recall that the initial conditional density function $p(\zb | \yb)$ is H\"{o}lder continuous. To approximate $p_t(\xb | \yb)$, a na\"{i}ve idea is to use a Taylor polynomial $h^{\rm density}_{\rm Taylor}(\zb, \yb)$ to approximate $p(\zb | \yb)$. This leads to an approximator in the form of
\begin{align*}
\int_{\RR^{d}} h^{\rm density}_{\rm Taylor}(\zb, \yb) \frac{1}{\sigma_t^{d}(2\pi)^{d/2}} \exp\left(-\frac{\norm{\alpha_t\zb-\xb}^2}{2\sigma_t^2}\right)\diff \zb.
\end{align*}
Examining the display above, we encounter two caveats:
\begin{enumerate}[leftmargin = 0.2in]
\item Since the data domain is unbounded,  it can be difficult to uniformly approximate the  conditional density $p(\zb | \yb)$ using $h^{\rm density}_{\rm Taylor}$;
\item Although the Taylor polynomial $h^{\rm density}_{\rm Taylor}(\zb, \yb)$ can be implemented using a neural network, the integration over $\zb$ is prohibitively difficult to handle. Moreover, the exponential function and the time $t$ dependence make the approximation more obscure.
\end{enumerate}
To address the first challenge, we devise a proper truncation on the data domain. Specifically, for any time $t$, we truncate the data domain by an $\ell_\infty$-ball of radius $R$, that is, we denote $\cD_1 = \set{\zb:\norm{\zb}_{\infty}\le R}$ and only ensure $h^{\rm density}_{\rm Taylor}$ approximates $p(\zb | \yb)$ on $\cD_1$ for any $\yb$. Such a domain truncation is reasonable when the conditional density function has a light tail. In other words, the truncation error is well controllable when the radius $R$ is sufficiently large (see details in Lemma~\ref{lemma::truncation x}).

For the second challenge, we propose {\bf diffused local polynomials} suitable for approximation of $p_t(\xb | \yb)$. Let $h_{\rm Taylor}^{\rm kernel}(\zb, \xb, t)$ be a Taylor polynomial for approximating the exponential transition kernel in \eqref{equ:: sketch p_t integration}. Then we define
\begin{align*}
\textsf{Diffused-local-poly}(\xb, \yb, t) = \int_{\RR^{d}} \frac{1}{\sigma_t^{d}(2\pi)^{d/2}} h^{\rm density}_{\rm Taylor}(\zb, \yb) h^{\rm kernel}_{\rm Taylor}(\zb, \xb, t) \diff \zb.
\end{align*}
We note that diffused local polynomials resemble the same formulation of $p_t(\xb | \yb)$, while they enjoy a critical advantage: As the product $h^{\rm density}_{\rm Taylor}(\zb, \yb) h^{\rm kernel}_{\rm Taylor}(\zb, \xb, t)$ is again a polynomial, whose integration is explicitly computable and consequently allows a direct neural network implementation. We remark that the time $t$ enters the diffused local polynomials only through the two quantities $\sigma_t$ and $\alpha_t$ in the Gaussian kernel, which are both super smooth and very easy to approximate. We acknowledge that diffused local polynomials are inspired by the analysis in \cite{oko2023diffusion}. A similar approximation scheme utilizing diffused local polynomials can be applied to $\nabla p_t(\xb | \yb)$ in the numerator.

\vspace{5pt}

\noindent $\bullet$ \underline{\it Use a fraction to approximate the score function}. We approximate the score function by the fraction $\nabla \log p_t(\xb | \yb) = \nabla p_t(\xb | \yb) / p_t(\xb | \yb)$, however, there is an additional caveat: $p_t(\xb | \yb)$ can be arbitrarily small so that the reciprocal $1/p_t(\xb | \yb)$ can explode to infinity. The reason behind this exploding issue is that the initial data distribution fails to have good coverage uniformly. 
That is, the density of the initial data distribution can be small (or even zero)  in some areas.
As a result, estimating the density in these regions is fundamentally difficult \citep{tsybakov2008introduction}.

Here we introduce a threshold $\epsilon_{\rm low}$ to alleviate the exploding reciprocal issue. The idea is to replace the denominator in \eqref{equ:: sketch p_t integration} by $\max\{p_t(\xb | \yb), \epsilon_{\rm low}\}$. We choose a proper $\epsilon_{\rm low}$ balancing two criteria: 1) $\epsilon_{\rm low}$ should not be too small so that $1/\epsilon_{\rm low}$ is controlled; 2) $\epsilon_{\rm low}$ should not be too large to deviate heavily from the original score function. 
As we will show in Lemma~\ref{lemma::truncation p}, the choice of $\epsilon_{\rm low}$ depends on the tail behavior of the conditional distribution $p_t(\xb | \yb)$. 
We remark that truncating $p_t(\xb | \yb)$ at $\epsilon_{\rm low}$ inevitably compromises the approximation efficiency, which leaves room for improvement in Theorem~\ref{thm::score approx exp}.

To this end, it remains to implement the previous constructions by a neural network, where we leverage the universal approximation ability of ReLU networks.

\paragraph{Unraveling the fast rate}
We further discuss how Assumption~\ref{assump::expdensity} enables a fast approximation rate. Under Assumption~\ref{assump::expdensity}, substituting $p(\zb | \yb) = f(\zb, \yb) \exp(-C_2\norm{\zb}_2^2 / 2)$ into \eqref{equ:: sketch p_t integration}, by some algebraic manipulation, we have
\begin{align}
p_t(\xb|\yb)
        &=\frac{1}{(\alpha_t^2+C_2\sigma_t^2)^{d/2}}\exp\left(\frac{-C_2\norm{\xb}^2_2}{2(\alpha_t^2+C_2\sigma_t^2)}\right)\nonumber\\
        &\hspace{2cm}\cdot \underbrace{\int f(\zb,\yb) \frac{(\alpha_t^2+C_2\sigma_t^2)^{d/2}}{(2\pi)^{d/2}\sigma_t^d}\exp\left(-\frac{\norm{\zb-\alpha_t\xb/(\alpha_t^2+C_2\sigma_t^2)}^2}{2\sigma_{t}^2/(\alpha_t^2+C_2\sigma_t^2)}\right)\diff \zb}_{h(\xb,\yb,t)}. \label{eq:fast_pt}
\end{align}
We observe that $f(\zb, \yb)$ has two benign properties: 1) it is lower bounded away from zero, suggesting homogeneous spatial coverage of the data distribution; 2) it has H\"{o}lder regularity with a bounded H\"{o}lder norm. Denoting the integral in \eqref{eq:fast_pt} as $h(\xb, \yb, t)$, we immediately deduce that $h(\xb, \yb, t)$ is bounded away from $0$. Equation~\eqref{eq:fast_pt} also suggests that
\begin{align*}
\nabla \log p_t(\xb | \yb) = \frac{-C_2\xb}{(\alpha_t^2+C_2\sigma_t^2)} + \frac{\nabla_{\xb} h(\xb, \yb, t)}{h(\xb, \yb, t)}.
\end{align*}
Thus, it suffices to approximate $\frac{\nabla_{\xb} h(\xb, \yb, t)}{h(\xb, \yb, t)}$ using the analytical framework introduced in the previous paragraphs. Notably, we do not need to truncate $h(\xb, \yb, t)$ to prevent the exploding of $1/ h(\xb, \yb, t)$, which saves the truncation error and leads to fast approximation. We provide a side-by-side comparison between the approximation schemes in Theorems~\ref{thm::score approx} and \ref{thm::score approx exp} in Figure~\ref{fig:slow_fast_rate}.
\begin{figure}[!htb]
    \centering
    \includegraphics[width = 0.975\textwidth]{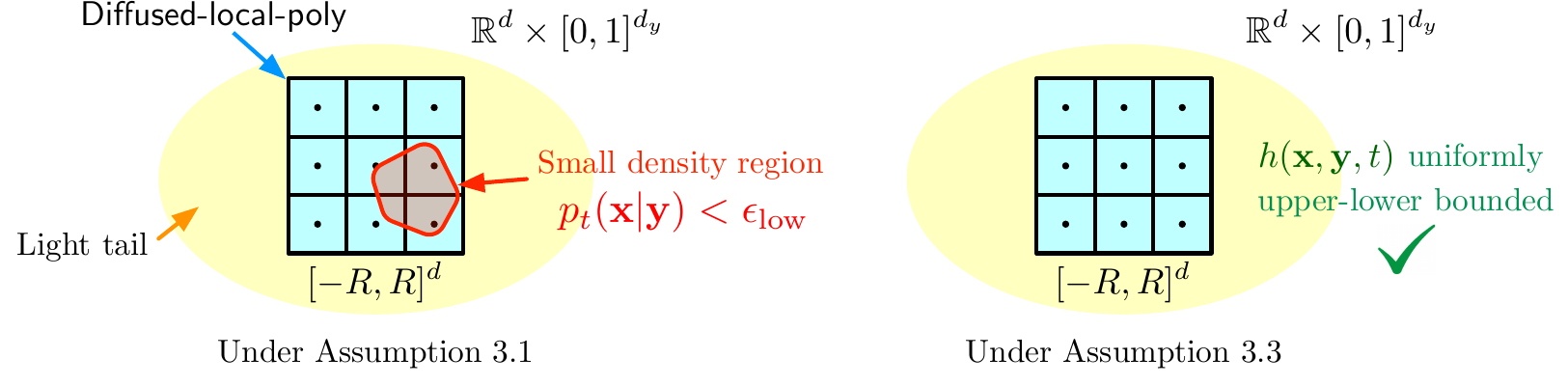}
    \caption{Comparison of approximation schemes in Theorems~\ref{thm::score approx} and \ref{thm::score approx exp}. On the left panel, we use diffused local polynomials to approximate the numerator and denominator on a truncated cube. However, the existence of of small density region necessitates a truncation at $\epsilon_{\rm low}$, which compromises the approximation efficiency. In contrast, under Assumption~\ref{assump::expdensity}, we eliminate small density regions within the cube, which leads to a fast approximation.}
    \label{fig:slow_fast_rate}
\end{figure}

%% file: Sections/Estimation.tex
\section{From Score Approximation to Distribution Estimation: Statistical Results}\label{sec:estimation}
Section~\ref{sec:cdm_approx} provides theoretical results of approximating conditional score functions using ReLU neural networks. 
In this section, we apply these theoretical results to statistical estimation problems and develop a few sample complexity results for methods that involve conditional score estimation. 
In particular, we first study the problem of estimating a conditional score function via the classifier-free guidance method introduced in Section \ref{sec:pre} and quantify the sample complexity of learning the conditional score from data with ReLU neural networks. 
We further apply this result to establish the sample complexity of learning a conditional distribution via the conditional diffusion model. 
Furthermore, we conclude this section with an application of our statistical theory to the problem of estimating the transition probability in model-based reinforcement learning.

\subsection{Conditional Score Estimation}\label{sec::esti}
Recall that classifier-free guidance method estimates the conditional score function by minimizing the empirical risk $\hat{\cL}$ defined in \eqref{eq:cdm_empirical_risk}. Given a score network $\cF$, we denote the corresponding empirical risk minimizer as
\begin{align*}
\hat{\sb} \in \argmin_{\sb \in \cF}~ \hat{\cL}(\sb).
\end{align*}

We measure the quality of the estimator $\hat{\sb}$ by its mean-squared deviation to the ground-truth conditional score function:
\begin{align*}
\cR(\shat)&=\int_{t_0}^T\frac{1}{T-t_0} \EE_{(\xb_t,\yb)}\norm{\shat(\xb_t,\yb,t)-\nabla \log p_t(\xb_t|\yb)}_2^2 \diff t.
\end{align*}
Here the expectation is taken over the joint distribution of $\xb_t$ and $\yb$. The following theorem presents upper bounds on $\cR(\hat{\sb})$ when the score network $\cF$ is chosen based on  Theorem~\ref{thm::score approx}.
\begin{theorem}\label{thm::Generalization}
Suppose Assumption \ref{assump:sub} holds and we choose the score network $\cF(M_t, W, \kappa, L, K)$ as in Theorem \ref{thm::score approx}. By taking the network size parameter $N=n^{\frac{d+d_y}{d+d_y+\beta}}$, the early-stopping time $t_0 < 1$ and the terminal time $T = \cO(\log n)$, it holds that
\begin{align*}
\EE_{\set{\xb_i,\yb_i}_{i=1}^{n}}\left[ \cR(\shat)\right] =\cO\left(\frac{1}{ t_0} \cdot n^{-\frac{\beta}{d+d_y+\beta}}(\log n)^{\max(17,d+\beta/2+1)} \right).
\end{align*}
Moroever,  when  Assumption \ref{assump::expdensity} holds, taking $N=n^{\frac{d+d_y}{d+d_y+2\beta}}$, we have  
\begin{align*}
\EE_{\set{\xb_i,\yb_i}_{i=1}^{n}}\left[ \cR(\shat)\right]=\cO\left(\log \frac{1}{t_0} \cdot n^{-\frac{2\beta}{d+d_y+2\beta}}(\log n)^{\max(17,\beta) }\right).
\end{align*}
\end{theorem}
The proof is provided in Appendix~\ref{sec::proof generation} and utilizes a sophisticated bias-variance trade-off with proper truncation. Several discussions are in turn.
\paragraph{Sample complexity bounds} Theorem~\ref{thm::Generalization} establishes  sample complexity results for conditional score estimation. 
We focus on the result under Assumption~\ref{assump::expdensity}. To obtain an $\epsilon$-error $L_2$ score estimator, the sample size scales in the order of $\tilde{\cO}(\epsilon^{-{(d + d_y + 2\beta)}/{(2\beta)}})$, where $\tilde{\cO}$ omits a polynomial in $\log (1/t_0)$. This sample complexity is reminiscent of the nonparametric regression rate for $\beta$-H\"{o}lder functions defined on the joint space of $(\xb, \yb)$. Yet we emphasize that the target conditional score function $\nabla \log p_t(\xb | \yb)$ here does not necessarily possess H\"{o}lder regularities, although the initial data distribution does. This indicates that the regularity of the initial data distribution dictates the complexity of score estimation.

\paragraph{Impact of early-stopping} Our risk bounds involve the early-stopping time $t_0$. As $t_0$ decreases, the estimation error grows, which implies the difficulty of potential score function blowup. Under Assumption~\ref{assump::expdensity}, however, the error bound only logarithmically depends on $t_0$, allowing flexible choice on the early-stopping. In the following section, we will optimally choose $t_0$ under both assumptions for distribution estimation.

\subsection{Distribution Estimation}
Given the trained conditional score network $\hat{\sb}(\xb,\yb,t)$ in the previous section, we study its distribution estimation power. To ease the presentation, we consider utilizing the continuous-time backward process \eqref{eq:backward approx} for distribution estimation. In practice, a proper discretization is applied to generate samples, whose deviation to the continuous-time backward process can be controlled by the step size of the discretization (see for example \cite[Theorem 2]{chen2022sampling}).

For a given guidance $\yb$, we denote the early-stopped generated data distribution as $\hat{P}_{t_0}(\cdot|\yb)$ using the estimated score $\hat{\sb}$. We bound the divergence between $\hat{P}_{t_0}(\cdot | \yb)$ to the ground-truth conditional data distribution $P(\cdot | \yb)$ in the following theorem.
\begin{theorem}\label{thm::TVbound}
Suppose Assumption \ref{assump:sub} holds. Assume in addition that there exists a constant $C$ such that $\text{KL}(P(\cdot|\yb)~|~{\sf N}(\boldsymbol{0}, I)) \leq C < \infty$ for all $\yb$. Taking the early-stopping time $t_0=n^{-\frac{\beta}{4(d+d_y+\beta)}}$ and the terminal time $T=\frac{2\beta}{d+d_y+2\beta}\log n$, it holds that
\begin{align*}
\EE_{\set{\xb_i,\yb_i}_{i=1}^{n}}\left[  \EE_{\yb} \left[\text{TV}\left(\hat{P}_{ t_0}(\cdot|\yb),P(\cdot|\yb)\right)\right]\right] = \cO\left(  n^{-\frac{\beta}{4(d+d_y+\beta)}}(\log n)^{\max(9,d/2+\beta/4+1)} \right).
\end{align*}
On the other hand, assume only Assumption \ref{assump::expdensity}. Taking $t_0=n^{-\frac{4\beta}{d+d_y+2\beta}-1}$, it holds that
\begin{align*}
\EE_{\set{\xb_i,\yb_i}_{i=1}^{n}}\left[  \EE_{\yb} \left[\text{TV}\left(\hat{P}_{t_0}(\cdot|\yb),P(\cdot|\yb)\right)\right]\right] = \cO\left( n^{-\frac{\beta}{d+d_y+2\beta}}(\log n)^{\max(19/2,(\beta+2)/2)} \right).
\end{align*}
\end{theorem}
The proof is provided in Appendix~\ref{sec::TV proof} and utilizes Girsanov's theorem to bridge the score estimation error to the distribution estimation error. We provide some interpretations of the results.

\paragraph{Bounded KL condition} Theorem~\ref{thm::TVbound} is the first conditional distribution estimation guarantee of diffusion models. We remark that under the weaker Assumption~\ref{assump:sub}, we need the additional bounded KL divergence condition on the initial distribution. The reason behind this is that the bounded KL divergence condition ensures the exponential mixing of the forward process \citep{chen2022sampling}. However, when Assumption~\ref{assump::expdensity} holds, such a bounded KL divergence condition is automatically verified and hence is lifted.

\paragraph{Minimax optimality} Theorem~\ref{thm::TVbound} also applies to unconditional distribution estimation by removing $\yb$ and setting $d_y = 0$. The obtained distribution estimation error rate is $n^{-\frac{\beta}{d + 2\beta}}$. We show that this matches the minimax optimal rate for estimating H\"{o}lder distributions.
\begin{proposition}\label{prop::lower bound TV}
Fix a constant $C_2>0$ and a H\"{o}lder index $\beta >0$. Consider estimating a distribution $P(\xb)$ with a density function belonging to the space $$\cP=\set{p(\xb)=f(\xb)\exp(-C_2\norm{\xb}_2^2):f(\xb)\in \mH^{\beta}(\R^d, B), f(\xb) \ge C>0}.$$ Given $n$ i.i.d. data $\{\xb_i\}_{i=1}^n$, we have
\begin{align*}
\inf\limits_{\hat{\mu}}~ \sup\limits_{p\in \cP}~ \EE_{\set{\xb_i}_{i=1}^n} \left[\text{TV}\left(\hat{\mu},P \right)\right]\gtrsim n^{-\frac{\beta}{d+2\beta}},
\end{align*}
where the infimum is taken over all possible estimators $\hat{\mu}$ based on the data.
\end{proposition}
The proof is provided in Appendix \ref{sec::proof of prop::lower bound TV}. We note that $\cP$ coincides with the condition in Assumption~\ref{assump::expdensity} by removing the conditional variable $\yb$. Proposition~\ref{prop::lower bound TV} implies that diffusion models are efficient distribution estimators. Our results corroborate the discovery in \cite{oko2023diffusion}, while substantially enlarging the distributions that can be optimally learned.

\subsection{Application to Transition Probability Estimation}
 In model-based reinforcement learning, estimating the transition kernel of the underlying dynamical system plays a vital role \citep{chen2023actions, chen2023efficient}. 
In the following, we study the sample complexity of using conditional diffusion models to estimate the transition kernel.  
We assume the dataset consists of  $n$ i.i.d. tuples of $\{(\sbb'_{i}, \sbb_i, \ab_i)\}_{i=1}^n$, where $\sbb$ and $\sbb'$ are the current and next state, respectively, and $\ab$ is the action. 
We denote the state space as $\cS$ and the action space as $\cA$. The state-action pair is sampled from some unknown visitation measure, 
and the next state $\sbb'$ is sampled according to a transition distribution $P(\sbb' | \sbb, \ab)$. Our goal is to estimate $P$ via the usage of conditional diffusion models. In practice, reinforcement learning and diffusion models have demonstrated promising synergies \citep{ajay2022conditional}. We study transition probability estimation to initiate the theoretical underpinnings of these successes.

To unify the notation, we denote $\xb = \sb'$ and $\yb = (\sb, \ab)$. We assume $\cS=\RR^{d_s}$ and $\cA=\RR^{d_a}$ for $d_s$ and $d_a$ being the dimension of the state and the action space, respectively. Therefore, we have $\xb \in \RR^{d_s + d_a}$ and $\yb \in \RR^{d_s}$. 
Note that $\yb$ is unbounded to be consistent over the state space. We state the following analogy of Assumption~\ref{assump::expdensity}.
\begin{assumption}\label{assump::transition kernel}
Let $C$, $C_2$ and $C_y$ be three positive constants and function $f \in \mH^{\beta}(\R^{d_s}\times \R^{d_s+d_a}, B)$ for a constant radius $B$. We assume $f(\xb, \yb) \geq C$ for all $(\xb, \yb)$ and the transition density function $p(\xb | \yb) = \exp(-C_2\norm{\xb}^2/2)f(\xb, \yb)$. 
Moreover, we assume that the visitation measure of $\yb$ has a sub-Gaussian tail, i.e., the marginal density satisfies $p(\yb)\le \exp(-C_y\norm{\yb}^2/2)$.
\end{assumption}
Compared to Assumption~\ref{assump::expdensity}, we extend to unbounded condition $\yb$ by imposing the light tail condition on $\yb$. Assumption~\ref{assump::transition kernel} also encompasses bounded $(\xb, \yb)$ as a special case.

Our conditional diffusion model will be trained using the classifier-free guidance method on the data $\{(\sb'_i, \sb_i, \ab_i)\}_{i=1}^n$. When evaluating the performance of the conditional diffusion model, we choose a state-action pair $\yb^\star = (\sb^\star, \ab^\star)$ and measure how well the transition probability $P(\cdot | \yb^\star)$ is estimated. 
Indeed, the performance heavily relies on how well data distribution covers the desired state-action pair $\yb^\star$. 
In the existing literature, this aspect is referred to as distribution shift, quantifying the knowledge transfer rate from the training data to $\yb^\star$ \citep{yuan2023reward}. We define the following class-restricted distribution shift coefficient
\begin{align}\label{equ::distribution shift}
\cT(\yb^\star)=\sup_{\sb \in \cF} \sqrt{\frac{\int_{t_0}^{T}\EE_{\xb \sim P(\xb | \yb^\star)}\EE_{\xb'\sim {\sf N}(\alpha_t \xb, \sigma_t^2 I)}[\norm{\sb(\xb',\yb^\star,t)-\nabla \log p_t(\xb'|\yb^\star)}^2] \diff t}{\int_{t_0}^{T}\EE_{(\xb, \yb)}\EE_{\xb' \sim {\sf N}(\alpha_t \xb, \sigma_t^2 I)}\left[\norm{\sb(\xb',\yb,t)-\nabla \log p_t(\xb'|\yb)}^2\right] \diff t}}.
\end{align}
Distribution coefficient $\cT(\yb^*)$ is related to the widely used concentrability coefficient -- $L_\infty$ density ratio -- in offline reinforcement learning \citep{munos2008finite, liu2018breaking, chen2019information, fan2020theoretical}. Since we use the score network $\cF$ as a smoothing factor, i.e., the network class $\cF$ may not be sensitive to certain differences between the queried $\yb^\star$ and the training data, $\cT(\yb^\star)$ is always smaller than the concentrability coefficient.

We denote the learned transition distribution as $\hat{P}(\cdot | \yb^\star)$ via a trained conditional diffusion model. Compared to Theorem~\ref{thm::Generalization}, we drop the early-stopping time $t_0$ for simplicity. The following proposition provides its performance guarantee.
\begin{proposition}\label{prop:transition_estimation}
Suppose Assumption~\ref{assump::transition kernel} holds. For any fixed $\yb^\star = (\sbb^{\star}, \ab^{\star})$, taking the early-stopping time $t_0=n^{-\frac{4\beta}{2d_s+d_a+2\beta}-1}$ and the terminal time $T=\frac{2\beta}{2d_s+d_a+2\beta}\log n$, the transition probability is estimated with
\begin{align*}
\EE_{\set{\sb_i^{\prime},\sb_i,\ab_i}_{i=1}^{n}}\left[\text{TV}\left(\hat{P}(\cdot | \yb^\star), P(\cdot|\yb^\star) \right) \right] &= \cT(\yb^\star)\cO\left(n^{-\frac{\beta}{2d_s+d_a+2\beta}} (\log 
  n)^{\max(19/2,(\beta+2)/2)} \right).
\end{align*}
\end{proposition}

The proof is provided in Appendix \ref{sec::proof transition estimation}. Proposition~\ref{prop:transition_estimation} shares the same rate of convergence with the fast rate in Theorem~\ref{thm::Generalization}. The convergence is adaptive to the smoothness of the transition probability. We remark that the dimension dependence may be improved in Theorem~\ref{prop:transition_estimation}, considering that practical state-action spaces, especially involving image-based states, often exhibit low-dimensional structures; see abundant examples in OpenAI Gym environments \citep{openaigym}. Nonetheless, exploitation of the data low-dimensional structures is beyond the scope of the paper.

%% file: Sections/Application.tex
\section{Further Applications: Reward-Directed Generation and Inverse Problem}\label{sec::applications}
We further present two applications of the conditional diffusion models and establish statistical guarantees leveraging the theory in previous sections. In particular, we study reward-directed sample generation and inverse problems. These applications demonstrate the versatility of our theory and provide new theoretical foundations of diffusion models in practice.

\subsection{Reward-Directed Conditional Generation}
In many use cases of diffusion models, we anticipate generating new samples of high quality. For example, in text-to-image synthesis, the generated image should align with the verbal description \citep{yuan2023reward}. In reinforcement learning, the state-action trajectory should achieve high reward \citep{janner2022planning}. In addition, in protein generation and drug discovery, the simulated protein or drug structure should satisfy biochemical properties \citep{watson2023novo}. In these applications, we associate an abstract scalar reward function $r$ to gauge each sample $\xb$. Consequently, conditional diffusion models are viewed as optimizing the reward function $r$ by generating new solutions. To facilitate the generation, conditional diffusion models take the reward as guidance. We formulate the aforementioned applications as the following offline reward maximization problem.

Suppose we are given an offline data set $\cD=\set{\xb_i,y_i}_{i=1}^{n}$, where scalar $y_i$ is a noisy measurement of the reward, i.e.,
\begin{align*}
y_i=r(\xb_i)+\xi_i \quad \text{with} ~\xi_i~\text{being an independent measurement noise}.
\end{align*}
We train a conditional diffusion model using the data set $\cD$. Afterward, we generate new samples under the guidance of $y^\star = a$ for some constant $a$. Here the value $a$ is the target reward value and we expect the generated samples to be faithful to the target reward value. We denote the generated distribution as $\hat{P}(\cdot | y^\star = a)$ and define a sub-optimality gap as
\begin{align}\label{equ::subopt}
\texttt{SubOpt} (\hat{P}, a)=a - \EE_{\xb \sim \hat{P}(\cdot | a)}\left[ r(\xb)\right].
\end{align}
Here sub-optimality gap is one-sided, as a negative $\subopt$ implies high reward samples beyond the target value. $\subopt$ also matches the definition of the off-policy sub-optimality gap in offline bandit problems. Before we state our main results, we impose the following assumptions.
\begin{assumption}\label{assump:reward lipschitz}
The ground-truth conditional distribution $P(\cdot|y^{\star}=a)$ satisfies Assumption \ref{assump::expdensity}, and the reward function $r(\xb)$ is bounded, i.e., there exists a constant $L$ such that $\abs{r(\xb)}\le L$ for any $\xb \in \RR^{d}$.
\end{assumption}
The bounded reward is a commonly adopted mild assumption in bandits and reinforcement learning \citep{bubeck2011pure,slivkins2019introduction}. The following theorem derives guarantees for the sub-optimality gap.
\begin{proposition}\label{thm::subopt}
Suppose Assumption~\ref{assump:reward lipschitz} holds. Taking the early-stopping time $t_0 =n^{-\frac{4\beta}{d+1+2\beta}-1} $ and the terminal time $T = \frac{2\beta}{d+1+2\beta}\log n$, conditional diffusion models yield new samples satisfying
\begin{align*}
\EE_{\set{\xb_i,y_i}_{i=1}^n} \left[\subopt (\hat{P},a)\right]&=\cT(a) \cdot \cO\left(L n^{-\frac{\beta}{d+1+2\beta}} (\log n)^{\max(19/2,(\beta+2)/2)}\right),
\end{align*}
where $\cT(a)$ is the distribution shift coefficient defined in \eqref{equ::distribution shift}.
\end{proposition}
The proof is provided in Appendix \ref{sec::subopt proof}, which utilizes the distribution estimation guarantee in Theorem~\ref{thm::TVbound} as an intermediate result to bound $\subopt$. We observe that   $\subopt$ is subject to a distribution shift. The reason is that the offline data is collected under some unknown sampling distribution, which is different from the target conditional distribution $P(\cdot | a)$. Therefore, the training data coverage interplays with the performance.

Proposition~\ref{thm::subopt} is closely related to the result in \cite{yuan2023reward}, yet the setup is different. \cite{yuan2023reward} consider the semi-parametric setting, where a large amount of unlabeled data is available. Thus, their analysis requires an estimation of the reward function and pseudo-labeling, with an additional assumption on the smoothness of the reward function. Our result circumvents the estimation of the reward function via classifier-free guidance. The obtained performance guarantee is adaptive to the regularity of the conditional distribution.

\subsection{Inverse Problems}
Diffusion models have shown remarkable performance in various types of inverse problems, spanning computer vision  \citep{chung2022diffusion,chung2022improving,song2023solving}, computational biology \citep{yi2023graph, wu2024protein}, and reinforcement learning \citep{ajay2022conditional}. 

We concentrate on a simple prototypical form of inverse problems: Retrieving an unknown $\xb$ from a linear measurement $\yb$, where $\xb$ and $\yb$ are related by
\begin{align}\label{equ::inverse problem}
\yb=\bH \xb + \epsilon, \quad \text{with} \quad \bH \in \R^{m\times d}.
\end{align}
Here $\xb \in \RR^d$ and $\yb \in \RR^m$ with $m <d$, representing common real-world scenarios such as $\yb$ being a low-dimensional sketching observation of $\xb$. Gaussian noise $\epsilon\sim {\sf N}(0, \sigma^2 I_m)$ is independent of $\xb$ for a positive variance $\sigma^2$. In general, solving for $\xb$ based on a measurement $\yb$ is underdetermined with infinitely many solutions. Hence, we primarily investigate whether it is possible to estimate the conditional distribution $P(\cdot | \yb)$ induced by a sampling distribution on $\xb$.

Suppose we are given a dataset $\cD=\set{\xb_i,\yb_i}_{i=1}^{n}$, where $\set{\xb_i}_{i=1}^{n}$ is sampled from an underlying distribution $P(\xb)$ and $\set{\yb_i}_{i=1}^{n}$ is obtained via \eqref{equ::inverse problem} with independent noise. We use classifier-free guidance to train a conditional diffusion model capable of generating samples $\xb \sim \hat{P}(\cdot | \yb^\star)$, where $\yb^*$ is a given measurement. Clearly, $\hat{P}(\cdot | \yb^*)$ is the estimated conditional distribution on $\xb$. We impose the following regularity assumption on the underlying distribution $P(\xb)$.
\begin{assumption}\label{assump:underlying_px}
The sampling distribution $P(\xb)$ has a density function $p(\xb)$. Moreover, there exist two positive constants $C$ and $C_2$, and a function $f \in \cH^{\beta}(\RR^d, B)$ for a H\"{o}lder index $\beta$ and a constant radius $B > 0$. The density function satisfies $p(\xb) = f(\xb) \exp(-C_2 \norm{\xb}^2 / 2)$ and $p(\xb) \geq C$ for all $\xb$.
\end{assumption}

Assumption~\ref{assump:underlying_px} is the same as Assumption~\ref{assump::expdensity} without the dependence on $\yb$. Indeed, $\yb$ is highly correlated to $\xb$ through the linear relation. The next result asserts the recovery of $\xb$ given a measurement $\yb^\star$.
\begin{proposition}\label{prop::inverse problem}
Suppose Assumption~\ref{assump:underlying_px} holds. We further assume $\log \sigma =\cO(\log n)$ and $\sigma^2 \le \lambda_i \lesssim \sigma^{4}$ for any $i \in [m]$, where  $\set{\lambda_i}_{i\in[m]}$ is the set of eigenvalues of $\bH\bH^{\top}$. Given an arbitrary measurement $\yb^\star$, taking $t_0=n^{-\frac{4\beta}{d+2\beta}-1}$ and $T=\frac{2\beta}{d+2\beta}\log n$, we have
\begin{align*}
\EE_{\set{\xb_i,\yb_i}_{i=1}^n} \left[ \tTV\left(\hat{P}(\cdot|\yb^\star),P(\cdot|\yb^\star)  \right)\right]=\cT(\yb^\star) \cdot \cO\left(n^{-\frac{\beta}{d+2\beta}} (\log n)^{\max(19/2,(\beta+2)/2)}\right).
\end{align*}
Moreover, the posterior mean of $\xb$ given $\yb^\star$ is estimated with
\begin{align*}
\EE_{\set{\xb_i,\yb_i}_{i=1}^n} \left[\left\|\EE_{P(\cdot|\yb^\star)}\left[\xb\right]-\EE_{\hat{P}(\cdot|\yb^\star)}\left[\xb\right]\right\|\right]=\cT(\yb^\star) \cdot \cO\left(n^{-\frac{\beta}{d+2\beta}} (\log n)^{\max(11,(\beta+5)/2)}\right).
\end{align*}
\end{proposition}
The proof is provided in Appendix \ref{sec::proof of prop::inverse problem}. This is the first statistical guarantee of diffusion models for linear inverse problems. The rate of convergence is dependent on the smoothness of $p(\xb)$ and the dimension of $\xb$, but independent of the measurement dimension $m$. Moreover, the statistical convergence rate is dependent on the distribution shift coefficient $\cT(\yb^*)$. This suggests that if $\yb^*$ significantly deviates from the training data distribution, the estimation of $\xb$ may suffer, advertising the importance of data coverage in inverse problems \citep{yu2023distribution}.

\section{Conclusion}
In this paper, we have developed a sharp statistical theory for conditional diffusion models trained with classifier-free guidance. By focusing on a broad class of conditional distributions characterized by H\"{o}lder smoothness and sub-Gaussian tails, we have demonstrated the existence of a suitably sized score neural network capable of approximating the score function with an arbitrarily small error. We have further established score estimation and distribution estimation guarantees using conditional diffusion models. The statistical rate of convergence matches the minimax optimal rate. Moreover, we have applied our established theories to explain the empirical success of diffusion models in reinforcement learning and inverse problems. These results showcase the practical relevance of our statistical analysis and provide the first theoretical underpinning of conditional diffusion models.

%% file: Sections/ReLU.tex
\section{Proof of Theorem \ref{thm::score approx}}\label{sec::proof score approx}
This section is organized as follows: Appendix~\ref{sec:key_steps_thm3.2} presents the key steps for proving Theorem~\ref{thm::score approx}.  Appendix \ref{sec::detail statements} lists the detailed statements for proving the key steps and Theorem~\ref{thm::score approx}. Appendix \ref{sec:: sub proof of score approx} shows the key steps for proving the most critical statement (Proposition \ref{prop::score approx bounded}) mentioned in Appendix \ref{sec::detail statements}. Appendices \ref{sec::diffuse approx} and \ref{sec::sub relu approx} elaborate on the proof of this statement. Appendix \ref{pf:score approx steps12} provides proofs for other statements mentioned in Appendix \ref{sec::detail statements}. Appendix \ref{sec::sub other lemmas} contains the proofs of further supporting lemmas. 
Moreover, to further simplify the notations and demonstrate the meaning of $N$ (see the detailed interpretation in Appendix \ref{sec::diffuse approx}), we replace $N$ by $N^{d+d_y}$ in the statements of Theorem~\ref{thm::score approx} without loss of generality. Correspondingly, we redefine $C_{\sigma}$ as $(d+d_y)C_{\sigma}$ and $C_{\alpha}$ as $(d+d_y)C_{\alpha}$ so that the time $t$ is still within $[N^{-C_\sigma},C_{\alpha}\log N]$. By adjusting the constants, our target becomes
\begin{align*}
    \int_{\R^{d}} \norm{\sb(\xb,\yb,t)-\nabla\log p_{t}(\xb|\yb)}^2_2 p_{t}(\xb|\yb) \diff \xb \lesssim \frac{B^2}{\sigma_t^4} N^{-\beta}(\log N)^{d+\beta/2+1},
\end{align*}
and the hyperparameters in the network class $\cF$ should satisfy
\begin{align*}
    & \hspace{0.4in} M_t = \cO\left(\sqrt{\log N}/\sigma^2_t\right),~
     W = {\cO}\left(N^{d+d_y}\log^7 N\right),\\
     & \kappa =\exp \left({\cO}(\log^4 N)\right),~
    L = {\cO}(\log^4 N),~
     K= {\cO}\left(N^{d+d_y}\log^9 N\right).
\end{align*}
\subsection{Key Steps for Proving Theorem \ref{thm::score approx}}\label{sec:key_steps_thm3.2}
To construct a ReLU network approximation, we rewrite the score function as $\nabla \log p_t(\xb | \yb) = 
 \frac{\nabla p_t(\xb | \yb)}{p_t(\xb | \yb)}$. The idea is to approximate $\nabla p_t(\xb | \yb)$ and $p_t(\xb | \yb)$ separately using similar techniques. However, even though the original data density function has H\"{o}lder regularity conditions, the diffused density function $p_t(\xb | \yb)$ gives rise to substantial caveats. The first challenge is $\xb$ being unbounded, which makes it difficult to derive a uniform approximation of $p_t(\xb|\yb)$. The second challenge is more intricate: $p_t(\xb | \yb)$ can be arbitrarily small so that $1 / p_t(\xb | \yb)$ quickly blows up. Consequently, our proof consists of three key steps, where the first two steps carefully address the caveats by proper truncation on domain $\xb_t$ and the value of $p(\xb | \yb)$.
\begin{description}
    \item[Step 1] ({\bf Truncate domain of $\xb$}). For any time $t$, we truncate the domain of input $\xb$ by an $\ell_\infty$-ball of radius $R$ (to be chosen later in {\bf Step 3}), that is, we denote $\cD_1 = \set{\xb:\norm{\xb}_{\infty}\le R}$. On the complement of $\cD_1$, we set our score approximation to be uniformly bounded by a constant depending on $R$ and $t$. We prove in Lemma \ref{lemma::truncation x} that this domain truncation induces a small approximation error when the radius $R$ is sufficiently large.
    \item[Step 2] ({\bf Truncate the value of $p_t$}). To prevent the explosion of $\nabla \log p_t(\xb|\yb)= \frac{\nabla p_t(\xb|\yb)}{p_t(\xb|\yb)}$ when $p_t(\xb|\yb)$ is small, we set a threshold $\epslow$ for $p_t$ and define $\cD_2 = \set{\xb : p_t(\xb | \yb) \geq \epsilon_{\rm low}}$. Analogous to {\bf Step 1}, we also set our approximation to be bounded by the constant we mention in Step 1 on the complement of $\cD_2$. We show in Lemma \ref{lemma::truncation p} that focusing on $\cD_2$ also induces controllable approximation error.
    \item[Step 3] ({\bf ReLU network approximation}). Let $\cD = \cD_1 \cap \cD_2$. We use a ReLU network to approximate $p_t$ and $\nabla p_t$ on $\cD$ and subsequently combine the network approximators to construct a score approximation $\sb(\xb,\yb,t)$. We establish an $L_\infty$ approximation error guarantee of $\sb$ to $\nabla \log p_t$ on $\cD$ in Proposition \ref{prop::score approx bounded}, building upon the approximation errors of $p_t$ and $\nabla p_t$.
\end{description}
In the sequel, we delve into each step by providing precise statements. We then use them to prove Theorem~\ref{thm::score approx}. All the supporting results are postponed to Appendices~\ref{sec:: sub proof of score approx} to \ref{sec::sub other lemmas}.
\subsection{Detailed Statements in Steps 1 - 3 and Proof of Theorem \ref{thm::score approx}}\label{sec::detail statements}
 
Now we present crucial results in {\bf Steps 1 - 3}  and use them to prove Theorem~\ref{thm::score approx}.
\subsubsection{Formal Statements in Steps 1 - 3}
This section contains the statements of Lemma~\ref{lemma::truncation x}, Lemma~\ref{lemma::truncation p}, and Proposition~\ref{prop::score approx bounded}.
\begin{lemma}[Truncate $\xb$]\label{lemma::truncation x}
Suppose Assumption \ref{assump:sub} holds.  For any $R>1$, $\yb$ and $t>0$, we have
\begin{align}
&\int_{\norm{\xb}_{\infty}\ge R}  p_{t}(\xb|\yb)\diff \xb \lesssim R\exp(-C_2'R^2),\label{equ::sub truncate x constant}\\
&\int_{\norm{\xb}_{\infty}\ge R} \norm{\nabla\log p_{t}(\xb|\yb)}_2^2 p_{t}(\xb|\yb)\diff \xb \lesssim \frac{1}{\sigma_t^4}R^3\exp(-C_2'R^2),\label{equ::sub truncate x l2score}
\end{align}
where $ C_2^{\prime}=\frac{C_2}{2\max(C_2,1)}.$
\end{lemma}
The proof of Lemma~\ref{lemma::truncation x} is provided in Appendix~\ref{pf:score approx steps12}. Lemma~\ref{lemma::truncation x} is a consequence of the light tail in the data distribution. To better interpret, we can set \eqref{equ::sub truncate x constant} to be smaller than $\epsilon > 0$. Then the truncation radius can be chosen as $R = \cO\left(\sqrt{\log 1/\epsilon}\right)$.

Moreover, since the score function can be written as $\nabla\log p_t(\xb|\yb)=\frac{\nabla p_t(\xb|\yb)}{p_t(\xb|\yb)}$, its magnitude will be difficult to control when the density function $p_t(\xb|\yb)$ is extremely small. Thus, we also truncate $p_t(\xb|\yb)$ as stated in the following result.
\begin{lemma}[Truncate $p(\xb|\yb)$]\label{lemma::truncation p}
Suppose Assumption \ref{assump:sub} holds. For any $R>0$, $\yb$ and $\epslow>0$, we have
    \begin{align} &\int_{\norm{\xb}_{\infty}\le R} \indic{\abs{p_{t}(\xb|\yb)}<\epslow} p_{t}(\xb|\yb)\diff \xb \lesssim R^{d}\epslow,\label{equ::sub truncate px constant}
    \\ &\int_{\norm{\xb}_{\infty}\le R} \indic{\abs{p_{t}(\xb|\yb)}<\epslow}\norm{\nabla\log p_{t}(\xb|\yb)}^2 p_{t}(\xb|\yb)\diff \xb \lesssim \frac{\epslow}{\sigma_t^4}R^{d+2},\label{equ::sub truncate px l2score}
    \end{align}
\end{lemma}
The proof of Lemmas \ref{lemma::truncation p} is provided in Appendix \ref{pf:score approx steps12}. Note that Lemma~\ref{lemma::truncation p} concerns the truncated domain $\cD_1$. Combining Lemmas~\ref{lemma::truncation x} and \ref{lemma::truncation p} controls the truncation error when restricting approximation to the domain $\cD$. Accordingly, we provide an approximation theory on $\cD$, with properly chosen $R$ and $\epsilon_{\rm low}$.
\begin{proposition}[Approximate the score]\label{prop::score approx bounded}
Suppose Assumption~\ref{assump:sub} holds. We consider time $t \in [N^{-C_{\sigma}},C_{\alpha}\log N]$ for constants $C_\sigma$ and $C_\alpha$. Given any integer $N > 0$, we constrain $(\xb, \yb) \in [-C_x \sqrt{\log N},C_x \sqrt{\log N}]^{d} \times [0,1]^{d_y}$, where $C_x$ is a constant depending on $d$, $\beta$, $B$, $C_1$ and $C_2$. 
Then there exists a ReLU neural network class $ \cF(M_t, W, \kappa, L, K) $ which contains a mapping $\sb(\xb,\yb,t) $ satisfying 
\begin{align}\label{equ::score linfty bound}
p_{t}(\xb|\yb)\norm{\nabla\log p_{t}(\xb|\yb)-\sb(\xb,\yb,t)}_{\infty} \lesssim  \frac{B}{\sigma_t^2}N^{-\beta}(\log N)^{\frac{d+s+1}{2}} \quad \text{for~any}~t \in [N^{-C_{\sigma}},C_{\alpha}\log N].
\end{align}
Furthermore, the neural network hyperparameters satisfy
\begin{align}
&\hspace{1.5cm} M_t=\cO\left(\sqrt{\log N}/\sigma^2_t \right),~ W = {\cO}\left(N^{d+d_y}\log^7 N\right),\\ &\kappa =\exp \left({\cO}(\log^4 N)\right),~L= {\cO}(\log^4 N),~K=  {\cO}\left(N^{d+d_y}\log^9 N\right). \label{equ::hyper}
\end{align}
\end{proposition}
The proof of Proposition \ref{prop::score approx bounded} is rather involved and is deferred to Appendix~\ref{sec:: sub proof of score approx}. Proposition~\ref{prop::score approx bounded} confirms that the score function can be approximated on $[-C_x \sqrt{\log N},C_x \sqrt{\log N}]^{d} \times [0, 1]^{d_y}$ in the $L_\infty$ sense, which is an essential ingredient in the proof of Theorem~\ref{thm::score approx}. Meanwhile, we observe that the approximation error depends on the H\"{o}lder index $\beta$ of the data distribution. Thus, when $\beta$ is large,  approximation is relatively easy. 
Moreover, the approximation error increases, as $t$ decreases as $\sigma_t$ shrinks to $0$.

\subsubsection{Proof of Theorem~\ref{thm::score approx}}
\begin{proof}
Given Proposition~\ref{prop::score approx bounded}, we claim that the resulting $\sb(\xb, \yb, t) \in \cF$ is an $L_2$ approximator of the score function. In this regard, we reduce the proof of Theorem~\ref{thm::score approx} to the verification of this claim. Indeed, choosing $R=C_x \sqrt{\log N}=\sqrt{\frac{2\beta}{C_2^{\prime}}\log N}$ and $\epsilon_{\rm low}=C_3N^{-\beta} (\log N)^{\frac{d+s}{2}}$, we decompose the $L_2$ score approximation error as
   \begin{align*}
       &\quad\int_{\R^{d}} \norm{\sb(\xb,\yb,t)-\nabla\log p_{t}(\xb|\yb)}_2^2 p_{t}(\xb|\yb)\diff \xb\\
        &= \underbrace{\int_{\norm{\xb}_{\infty}>\sqrt{\frac{2\beta}{C_2^{\prime}}\log N}} \norm{\sb(\xb,\yb,t)-\nabla\log p_{t}(\xb|\yb)}_2^2 p_t(\xb|\yb) \diff \xb}_{(\bA_1)}\\ 
        &\quad + \underbrace{\int_{\norm{\xb}_{\infty}\le\sqrt{\frac{2\beta}{C_2^{\prime}}\log N}} \indic{\abs{p_{t}(\xb|\yb)}<\epslow}\norm{\sb(\xb,\yb,t)-\nabla\log p_{t}(\xb|\yb)}_2^2 p_{t}(\xb|\yb)\diff \xb}_{(\bA_2)} \\
        &\quad + \underbrace{\int_{\norm{\xb}_{\infty}\le\sqrt{\frac{2\beta}{C_2^{\prime}}\log N}} \indic{\abs{p_{t}(\xb|\yb)}\ge\epslow}\norm{\sb(\xb,\yb,t)-\nabla\log p_{t}(\xb|\yb)}_2^2 p_{t}(\xb|\yb)\diff \xb }_{(\bA_3)}.
   \end{align*}
Here $(\bA_1)$ is the truncation error due to the unbounded range of $\xb$; $(\bA_2)$ is the truncation error due to small $p_t(\xb | \yb)$. The remaining $(\bA_3)$ is the approximation error of $\sb(\xb, \yb, t)$ on $\cD$. We will bound the three terms separately.
   
\hspace{-0.6cm}\textbf{Bounding $(\bA_1)$.} According to Proposition \ref{prop::score approx bounded}, we have $\norm{\sb(\xb,\yb,t) }_{\infty} \lesssim \frac{\sqrt{\log N}}{\sigma^2_t}$ and thus
   \begin{align*}
       (\bA_1) &\le 2\int_{\norm{\xb}_{\infty}>\sqrt{\frac{2\beta}{C_2^{\prime}}\log N}} \norm{\sb(\xb,\yb,t)}_2^2 p_t(\xb|\yb) \diff \xb +2\int_{\norm{\xb}_{\infty}>\sqrt{\frac{2\beta}{C_2^{\prime}}\log N}} \norm{\nabla\log p_{t}(\xb|\yb)}_2^2 p_t(\xb|\yb) \diff \xb\\
       &\overset{(i)}{\lesssim} 2d\left(\frac{1}{\sigma^2_t}\sqrt{\log N}\right)^2 \int_{\norm{\xb}_{\infty}>\sqrt{\frac{2\beta}{C_2^{\prime}}\log N}} p_t(\xb|\yb) \diff \xb +\frac{2}{\sigma^4_t}\left(\frac{2\beta}{C_2^{\prime}}\log N\right)^{3/2}N^{-2\beta} \\
       &\overset{(ii)}{\lesssim} 2d\left(\frac{1}{\sigma^2_t}\sqrt{\log N}\right)^2 \left(\frac{2\beta}{C_2^{\prime}}\log N\right)^{1/2}N^{-2\beta} +\frac{2}{\sigma^4_t}\left(\frac{2\beta}{C_2^{\prime}}\log N\right)^{3/2}N^{-2\beta}\\
       &\lesssim \frac{N^{-2\beta}(\log N)^{3/2}}{\sigma^4_t}.
   \end{align*}
Here in $(i)$, we invoke the upper bound $\norm{\sb(\xb, \yb, t)}_2^2 \leq d \norm{\sb(\xb,\yb,t) }_{\infty}^2$, \eqref{equ::sub truncate x l2score} in Lemma \ref{lemma::truncation x}, and inequality $(ii)$ follows from \eqref{equ::sub truncate x constant} in Lemma \ref{lemma::truncation x}. 

\hspace{-0.6cm}\textbf{Bounding $(\bA_2)$.}
Analogous to $(\bA_1)$, we have
\begin{align*}
    (\bA_2)&\le \int_{\norm{\xb}_\infty\le\sqrt{\frac{2\beta}{C_2^{\prime}}\log N}} \indic{\abs{p_{t}(\xb|\yb)}<\epslow}\left(d\left(\frac{1}{\sigma^2_t}\sqrt{\log N}\right)^2 +\norm{\nabla \log p_t(\xb|\yb)}_2^2 \right)p_t(\xb|\yb)\\
    &\overset{(i)}{\lesssim}d\left(\frac{1}{\sigma^2_t}\sqrt{\log N}\right)^2 \left( \frac{2\beta}{C_2^{\prime}}\log N\right)^{d/2} \epslow+\frac{\epslow}{\sigma_t^4}\left( \frac{2\beta}{C_2^{\prime}}\log N\right)^{1+d/2}\\
    &\lesssim \frac{\epslow (\log N)^{1+d/2}}{\sigma_t^4},
\end{align*}
where inequality $(i)$ invokes \eqref{equ::sub truncate px constant} and \eqref{equ::sub truncate px l2score} in Lemma \ref{lemma::truncation p}.

\hspace{-0.6cm}\textbf{Bounding $(\bA_3)$.}
By the approximation guarantee \eqref{equ::score linfty bound} in Proposition \ref{prop::score approx bounded}, we immediately have
\begin{align*}
(\bA_3)&\le\int_{\norm{\xb}_\infty\le\sqrt{\frac{2\beta}{C_2^{\prime}}\log N}} \indic{\abs{p_{t}(\xb|\yb)}\ge\epslow} d\norm{\nabla\log p_{t}(\xb,\yb,t)-\sb(\xb,\yb,t)}_{\infty}^{2} p_t(\xb|\yb)\diff \xb \\
&=\int_{\norm{\xb}_\infty\le\sqrt{\frac{2\beta}{C_2^{\prime}}\log N}} \indic{\abs{p_{t}(\xb|\yb)}\ge\epslow} \frac{d\norm{\nabla\log p_{t}(\xb,\yb,t)-\sb(\xb,\yb,t)}_{\infty}^{2} p_t^2(\xb|\yb)}{p_t(\xb | \yb)}\diff \xb \\
&\overset{(i)}{\lesssim}\int_{\norm{\xb}_\infty\le\sqrt{\frac{2\beta}{C_2^{\prime}}\log N}} \indic{\abs{p_{t}(\xb|\yb)}\ge\epslow}  \frac{B^2}{\sigma_t^4}N^{-2\beta}(\log N)^{d+s+1} \frac{d}{p_t(\xb|\yb)} \diff \xb\\
&=\frac{B^2 d}{\sigma_t^4\epslow}N^{-2\beta}(\log N)^{d+s+1}\int_{\norm{\xb}_\infty\le\sqrt{\frac{2\beta}{C_2^{\prime}}\log N}} \indic{\abs{p_{t}(\xb|\yb)}\ge\epslow}   \frac{\epslow}{p_t(\xb|\yb)}\diff \xb\\
&\le \frac{B^2 d}{\sigma_t^4\epslow}N^{-2\beta}(\log N)^{d+s+1}\left(\frac{2\beta}{C_2^{\prime}}\log N\right)^{d/2}\\
&\lesssim \frac{B^2 d}{\sigma_t^4\epslow}N^{-2\beta}(\log N)^{3d/2+s+1},
\end{align*}
where we invoke \eqref{equ::score linfty bound} in $(i)$.

Combining the bounds of $(\bA_1)$, $(\bA_2)$ and $(\bA_3)$ together, we have
\begin{align*}
    &\quad\int_{\R^{d}} \norm{\sb(\xb,\yb,t)-\nabla\log p_{t}(\xb|\yb)}_2^2 p_{t}(\xb|\yb)\diff \xb \\&\lesssim \frac{N^{-2\beta}(\log N)^{3/2} +\epslow(\log N)^{1+d/2} +B^2d\epslow^{-1}N^{-2\beta}(\log N)^{d+s+1}}{\sigma^4_t}.
\end{align*}
Substitute $\epslow=C_3N^{-\beta} (\log N)^{\frac{d+s}{2}}$ into the display above, the $L_2$ approximation error is bounded by $$\cO\left(\frac{B^2}{\sigma^4_t}N^{-\beta}(\log N)^{d+\frac{s}{2}+1}\right)=\cO\left(\frac{B^2}{\sigma^4_t}N^{-\beta}(\log N)^{d+\frac{\beta}{2}+1}\right).$$ Overloading $N$ by $N^{\frac{1}{d+d_{y}}}$, we complete the proof. 
    
\end{proof}

\subsection{Proof of Proposition~\ref{prop::score approx bounded}}\label{sec:: sub proof of score approx}
Proposition~\ref{prop::score approx bounded} is the crux in proving Theorem~\ref{thm::score approx}, which constructs the so-called ``{\bf diffused local monomials}'' for approximating the score function. Recall that we rewrite the score function as $\frac{\nabla p_t(\xb | \yb)}{p_t(\xb | \yb)}$. A na\"{i}ve approach is to approximate $\nabla p_t(\xb | \yb)$ and $p_t(\xb | \yb)$ using local polynomials. However, we observe that $p_t$ is indexed by time $t$, which creates extra difficulty in devising a proper local polynomial approximation for all $t$. Our diffused local monomials are proposed analogously to local Taylor polynomial bases, with the capability to approximate the target score function indexed by $t$. As a side product, the proof for Proposition~\ref{prop::score approx bounded} directly verifies Lemmas~\ref{lemma::truncation x} and \ref{lemma::truncation p}.

\subsubsection{Key Steps for Proving Proposition \ref{prop::score approx bounded}}
The crest of the proof is the use of a set of \textbf{diffused local polynomials} as the basis functions to approximate the integral form of $p_t(\xb|\yb)$ and $\nabla p_t(\xb|\yb)$. To motivate the diffused local polynomials, we repeat the integral form of $p_t(\xb | \yb)$ as follows,
\begin{align}
p_t(\xb|\yb)&=\int_{\RR^{d}}p(\zb|\yb)\frac{1}{\sigma_t^{d}(2\pi)^{d/2}} \exp\left(-\frac{\norm{\alpha_t\zb-\xb}^2}{2\sigma_t^2}\right)\diff \zb.\label{equ::integral form of p_t}
\end{align}
Now we first construct Taylor expansions of the density function $p(\zb | \yb)$ and the Gaussian kernel $\exp\left(-\frac{\norm{\alpha_t \zb - \xb}^2}{2 \sigma_t^2}\right)$, denoted as $h^{\rm density}_{\rm Taylor}(\zb, \yb)$ and $h^{\rm kernel}_{\rm Taylor}(\zb, \xb, t)$, respectively. We define diffused local polynomials as
\begin{align*}
\textsf{Diffused-local-poly}(\xb, \yb, t) = \int_{\RR^{d}} \frac{1}{\sigma_t^{d}(2\pi)^{d/2}} h^{\rm density}_{\rm Taylor}(\zb, \yb) h^{\rm kernel}_{\rm Taylor}(\zb, \xb, t) \diff \zb.
\end{align*}
Roughly speaking, diffused local polynomials can be viewed as evolving a Taylor approximation of the data distribution along the forward diffusion process. As Taylor polynomials can well approximate H\"{o}lder densities, we expect the marginal density $p_t(\xb | \yb)$ can also be approximated by the diffused local polynomials (formal statement is provided in Lemma~\ref{lemma::diffused localpoly approx}). This constitutes the key idea of proving Proposition~\ref{prop::score approx bounded}. For notation simplicity, we postpone the formal definition of diffused local polynomials to \eqref{equ::Diffused local monomial} when proving Lemma~\ref{lemma::diffused localpoly approx}. Here we summarize an overview for establishing Proposition \ref{prop::score approx bounded} based on the usage of diffused local polynomials.
\begin{description}
    \item[Step (i)] ({\bf Diffused local polynomial approximation of $p_t(\xb | \yb)$}). Given $N>0$, we approximate $p_t(\xb|\yb)$ by a 
    diffused local polynomial $f_1$ that consists of at most $\cO(N^{d+d_y})$ diffused local monomials. The approximation error is bounded by ${\cO}(N^{-\beta}\log ^{\frac{d+s}{2}}N)$ in Lemma \ref{lemma::diffused localpoly approx}.
    \item[Step (ii)] ({\bf ReLU network implementation of $f_1$}). We construct a ReLU network to implement the diffused local polynomial in Lemma \ref{lemma::relu approx diffused}. The constructed network implements $f_1^{\rm ReLU}$ for approximating $f_1$ with a small error.
    \item[Step (iii)] ({\bf Diffused local polynomial approximation of $\nabla p_t$ and ReLU network implementation}). Since $p_t$ and $\nabla p_t$ have the same structure, we replicate {\bf Steps (i)} and {\bf (ii)} above to approximate $\nabla p_t$ by a neural network in Lemmas \ref{lemma::diffused localpoly approx1} and \ref{lemma::relu approx diffused1}. The constructed diffused local polynomial is denoted as $\mathbf{f}_2$ and its network implementation is denoted as $\mathbf{f}^{\trelu}_2$.
    \item[Step (iv)] ({\bf ReLU network approximation of the score function}). We combine the approximations $f_1^{\rm ReLU}$ and $\fb_2^{\rm ReLU}$ to construct a ReLU score function approximator $\mathbf{f}^{\trelu}_3$, which approximates $\nabla \log p_t=\frac{\nabla p_t}{p_t}$ with a small error on domain $\cD$.
\end{description}
We note that {\bf Step (iii)} is reminiscent to {\bf Steps (i)} and {\bf (ii)} combined, as $\nabla p_t(\xb | \yb)$ takes a similar integral form as \eqref{equ::integral form of p_t}. In the next section, we present the main results in {\bf Steps (i)} - {\bf (iii)} and use them to show Proposition~\ref{prop::score approx bounded}.

\subsubsection{Detailed Statements in Steps (i) - (iii) and Proof of Proposition \ref{prop::score approx bounded}}\label{sec::proof propo A3}
We first introduce the results in {\bf Steps (i)} and {\bf (ii)}, while the statements in {\bf Step (iii)} are analogous.

\paragraph{Formal statements in Steps (i) - (iii)} We have the following lemma proving the approximation power of diffused local polynomials.
\begin{lemma}[Diffused local polynomial approximation]\label{lemma::diffused localpoly approx}
Suppose Assumption~\ref{assump:sub} holds. For sufficiently large integer $N>0$, there exists a diffused local polynomial $f_1(\xb, \yb, t)$, consisting of at most $N^{d+d_y}(d+d_y)^s$ diffused local monomials, such that 
\begin{align}
     \abs{f_1(\xb,\yb,t)-p_t(\xb|\yb)} \lesssim BN^{-\beta}\log^{\frac{d+s}{2}}N , ~~~\forall \xb \in \RR^d, \yb \in [0,1]^{d_y}, \text{and}~t>0.
\end{align}
\end{lemma}
The proof of Lemma \ref{lemma::diffused localpoly approx} is provided in Appendix~\ref{sec::diffuse approx}. We next show that the diffused local polynomial can be efficiently implemented by a ReLU network with controllable error.
\begin{lemma}[ReLU approximation]\label{lemma::relu approx diffused}
Suppose Assumption \ref{assump:sub} holds. Given the diffused local polynomial $f_1$ in Lemma \ref{lemma::diffused localpoly approx}, for any $\epsilon >0$, there exists a ReLU network $ \cF( W, \kappa, L, K) $ that gives rise to a function $f_1^{\trelu}(\xb,\yb,t) \in \cF$ satisfying 
\begin{align}
     \abs{f_1(\xb,\yb,t)-f_1^{\trelu}(\xb,\yb,t)}\le \epsilon,
\end{align}
for any $\xb \in  [-C_x\sqrt{\log N},C_x\sqrt{\log N}]^{d}$, $\yb\in[0,1]^{d_y}$ and $t\in[N^{-C_{\sigma}},C_{\alpha}\log N]$.
The network configuration is
\begin{align*}
& W = {\cO}\left(N^{d+d_y}(\log^7 N+\log N\log^3 \epsilon^{-1})\right), \quad \kappa =\exp \left({\cO}(\log^4 N+\log^2 \epsilon^{-1})\right), \\
& \qquad L = {\cO}(\log^4 N+\log^2 \epsilon^{-1}), \quad K = {\cO}\left(N^{d+d_y}(\log^9 N+\log N\log^3 \epsilon^{-1})\right).
\end{align*}
\end{lemma}
The proof of Lemma \ref{lemma::relu approx diffused} is provided in Appendix~\ref{sec::sub relu approx}. The approximation guarantee of $f_1^{\rm ReLU}$ holds on the truncated domain, although $f_1$ can approximate the score function in the whole space. Yet Lemma~\ref{lemma::relu approx diffused} is enough for establishing Proposition~\ref{prop::score approx bounded}, as the latter only concerns the truncated domain $\cD$. A direct consequence of Lemmas~\ref{lemma::diffused localpoly approx} and \ref{lemma::relu approx diffused} is the existence of a ReLU network capable of approximating the marginal density function $p_t(\xb | \yb)$. Turning towards $\nabla p_t(\xb | \yb)$, we have similar results.
\begin{lemma}[Counterpart of Lemma \ref{lemma::diffused localpoly approx}]\label{lemma::diffused localpoly approx1}
Suppose Assumption \ref{assump:sub} holds. For sufficiently large integer $N>0$, there exists a mapping $\fb_2 = [f_{2, 1}, \dots, f_{2, d}]^\top$ with $f_{2, i}$ a diffused local polynomial for $i = 1, \dots, d$. Each $f_{2, i}$ consists of at most $N^{d+d_y}(d+d_y)^s$ diffused local monomials and satisfies
\begin{align}
     \abs{f_{2,i}(\xb,\yb,t)-[\sigma_t\nabla p_t(\xb|\yb)]_{i}} \lesssim BN^{-\beta}\log^{\frac{d+s+1}{2}}N , ~~~\forall \xb \in \RR^d, \yb \in [0,1]^{d_y}, \text{and}~t>0.
\end{align}
\end{lemma}
\begin{lemma}[Counterpart of Lemma \ref{lemma::relu approx diffused}]\label{lemma::relu approx diffused1}
Suppose Assumption \ref{assump:sub} holds. Given the diffused local polynomial mapping $\mathbf{f}_2$ in Lemma \ref{lemma::diffused localpoly approx1}, for any $\epsilon >0$, there exists a ReLU network $ \cF( W, \kappa, L, K) $ that gives rise to a mapping $\mathbf{f}_2^{\trelu}(\xb,\yb,t) \in \cF$ satisfying 
\begin{align}
     \norm{\mathbf{f}_2(\xb,\yb,t)-\mathbf{f}_2^{\trelu}(\xb,\yb,t)}_{\infty}\le  \epsilon,
\end{align}
for any $\xb \in  [-C_x\sqrt{\log N},C_x\sqrt{\log N}]^{d}$, $\yb\in[0,1]^{d_y}$ and $t\in[N^{-C_{\sigma}},C_{\alpha}\log N]$.
The network configuration is the same as in Lemma \ref{lemma::relu approx diffused}.
\end{lemma}
The proofs of Lemmas~\ref{lemma::diffused localpoly approx1} and \ref{lemma::relu approx diffused1} are provided in Appendix~\ref{sec::diffuse approx} and \ref{sec::sub relu approx}, respectively. The only difference in $\fb_2^{\rm ReLU}$ is that it is a multi-dimensional input-output mapping. In the proof, we construct coordinate mappings of $\fb_2^{\rm ReLU}$, and therefore, it reduces to the construction of $f_1^{\rm ReLU}$.

\paragraph{Proof of Proposition~\ref{prop::score approx bounded}} We assemble the approximators $f_1^{\rm ReLU}$ and $\fb_2^{\rm ReLU}$ for approximating the score function.
\begin{proof}
Without loss of generality, we focus on the first coordinate of the score function, which we denote as $\nabla \log p_{t, 1} = [\nabla \log p_t]_1$. From {\bf Steps (i)} - {\bf (iii)}, we obtain $ f_1, f^{\trelu}_1 $ to approximate $p_t(\xb|\yb)$, and $\mathbf{f}_2,\mathbf{f}^{\trelu}_2$ to approximate $\nabla p_t(\xb|\yb)$. We denote the first coordinate of $\mathbf{f}_2$ as $f_2$. In the following, we first show a construction using $f_1$ and $f_2$ for approximating $\nabla \log p_{t, 1}$ and then use a ReLU neural network to implement it. The network implementation relies on $f_1^{\rm ReLU}$ and $\fb_2^{\rm ReLU}$.
    
According to Lemma \ref{lemma::diffused localpoly approx}, the deviation between $p_t(\xb|\yb)$ and $f_1(\xb,\yb,t)$ is upper bounded by $\cO\left(N^{-\beta}\log^{\frac{d+s}{2}}N\right)$. Thus, there exists a constant $C_3>0$ such that if $p_t(\xb|\yb)\ge C_3N^{-\beta}\log^{\frac{d+s}{2}}N =:\epslow$, we must have $f_1(\xb,\yb,t)>\frac{1}{2}p_t(\xb|\yb)$. 
    
Now for $\epslow>C_3N^{-\beta}\log^{\frac{d+s}{2}}N$, 
we denote the clipped version of $f_1$ by $f_{1,\text{clip}}=\max(f_1,\epslow)$, and define the score approximator 
\begin{align*}
    f_3(\xb,\yb,t)=\min\left(\frac{f_2}{\sigma_t f_{1,\text{clip}}}, \frac{C_5}{\sigma^2_t}\left(C_x\sqrt{d\log N}+1\right)\right)
\end{align*}
for a sufficiently large constant $C_5$. 

By its definition, we have $\left| f_3(\xb,\yb,t) \right| \le \frac{C_5}{\sigma^2_t}(C_x\sqrt{d\log N}+1)$. This upper bound coincides with the upper bound of $\norm{\nabla\log p_{t}(\xb|\yb)}_{\infty}$ when $\xb \in [-C_x\sqrt{\log N},C_x\sqrt{\log N}]^d$, as shown in Lemma~\ref{lemma::score bound}. It remains to bound the difference between $\nabla \log p_{t, 1}$ and $f_3$. We have
\begin{align*}
\abs{\nabla\log p_{t,1}-f_3} & \le\abs{\nabla\log p_{t,1}-\frac{f_2}{\sigma_tf_{1,\text{clip}}}} \\
& = \abs{\frac{[\nabla p_{t}]_1}{p_t} - \frac{[\nabla p_{t}]_1}{f_{1, \textrm{clip}}} + \frac{[\nabla p_{t}]_1}{f_{1, \textrm{clip}}} - \frac{f_2}{\sigma_t f_{1,\text{clip}}}} \\
& \le \abs{[\nabla p_{t}]_1}\abs{\frac{1}{p_t}-\frac{1}{ f_{1,\text{clip}}}}+\frac{\abs{\sigma_t[\nabla p_{t}]_1 -f_2}}{\sigma_t f_{1,\text{clip}}}.
\end{align*}
Since $\norm{\nabla \log p_t(\xb | \yb)}_\infty \leq \frac{C_5}{\sigma_t^2} (C_x\sqrt{d\log N}+1)$ implies $\abs{[\nabla p_{t}]_1} \le\frac{ C_5}{\sigma_t^2} (C_x\sqrt{d\log N}+1) p_t$, for $p_t \geq \epsilon_{\rm low}$, we have
\begin{align*}
\abs{\nabla\log p_{t,1}-f_3} &\le \frac{ C_5}{\sigma_t^2} (C_x\sqrt{d\log N}+1) p_t\abs{\frac{1}{p_t}-\frac{1}{ f_{1,\text{clip}}}}+\frac{\abs{\sigma_t[\nabla p_{t}]_1-f_2}}{\sigma_t f_{1,\text{clip}}}\\
&\lesssim \frac{1}{f_{1,\text{clip}}}\left(\frac{1}{\sigma_t^2} \sqrt{\log N}\abs{p_t-f_{1,\text{clip}}}+\abs{\frac{\sigma_t[\nabla p_{t}]_1 -f_2}{\sigma_t}}\right)\\
& \overset{(i)}{\lesssim} \frac{1}{p_t}\left(\frac{1}{\sigma_t^2} \sqrt{\log N}\abs{p_t-f_{1,\text{clip}}}+\abs{\frac{\sigma_t[\nabla p_{t}]_1 -f_2}{\sigma_t}}\right)\\
&\overset{(ii)}{\lesssim} \frac{B}{\sigma_t^2p_t}N^{-\beta}(\log N)^{\frac{d+s+1}{2}},
\end{align*}
where inequality $(i)$ follows from $f_{1, \textrm{clip}} \geq \frac{1}{2} p_t$ and inequality $(ii)$ invokes the approximation guarantees of $f_1$ and $f_2$ in Lemmas \ref{lemma::diffused localpoly approx} and \ref{lemma::diffused localpoly approx1}.
The other coordinates of $\nabla \log p_t$ can be approximated in the same manner. Stacking these coordinate approximations, we obtain a mapping $\textbf{f}_3$ for approximating $\nabla \log p_t$ with
\begin{align}\label{equ::poly error}
\norm{\nabla\log p_{t}-\textbf{f}_3}_{\infty} \lesssim  \frac{B}{\sigma_t^2p_t}N^{-\beta}(\log N)^{\frac{d+s+1}{2}}.
\end{align}
Here $\fb_3$ is defined as 
\begin{align}\label{equ:: definition fb3}
    \fb_3(\xb,\yb,t)=\min\left(\frac{\fb_2}{\sigma_t f_{1,\text{clip}}}, \frac{C_5}{\sigma^2_t}\left(C_x\sqrt{d\log N}+1\right)\right).
\end{align}
Now we construct a ReLU network $\fb_3^{\rm ReLU}$ to implement $\textbf{f}_3$. The majority of the network utilizes the network constructed in Lemmas \ref{lemma::relu approx diffused} and \ref{lemma::relu approx diffused1}. However, to facilitate the implementation, we also need to implement some basic operations using ReLU networks, namely, the inverse function, the product function, $\sigma_t$ as a function of $t$, and an entrywise minimization operator. With these ingredients, our constructed network architecture is depicted in Figure \ref{fig:ReLU2}. Details about how to determine the network size and the error propagation are deferred to Appendix \ref{sec::f3 relu}.
\begin{figure}
\hspace{2.4cm}
\includegraphics[width=0.57\linewidth]{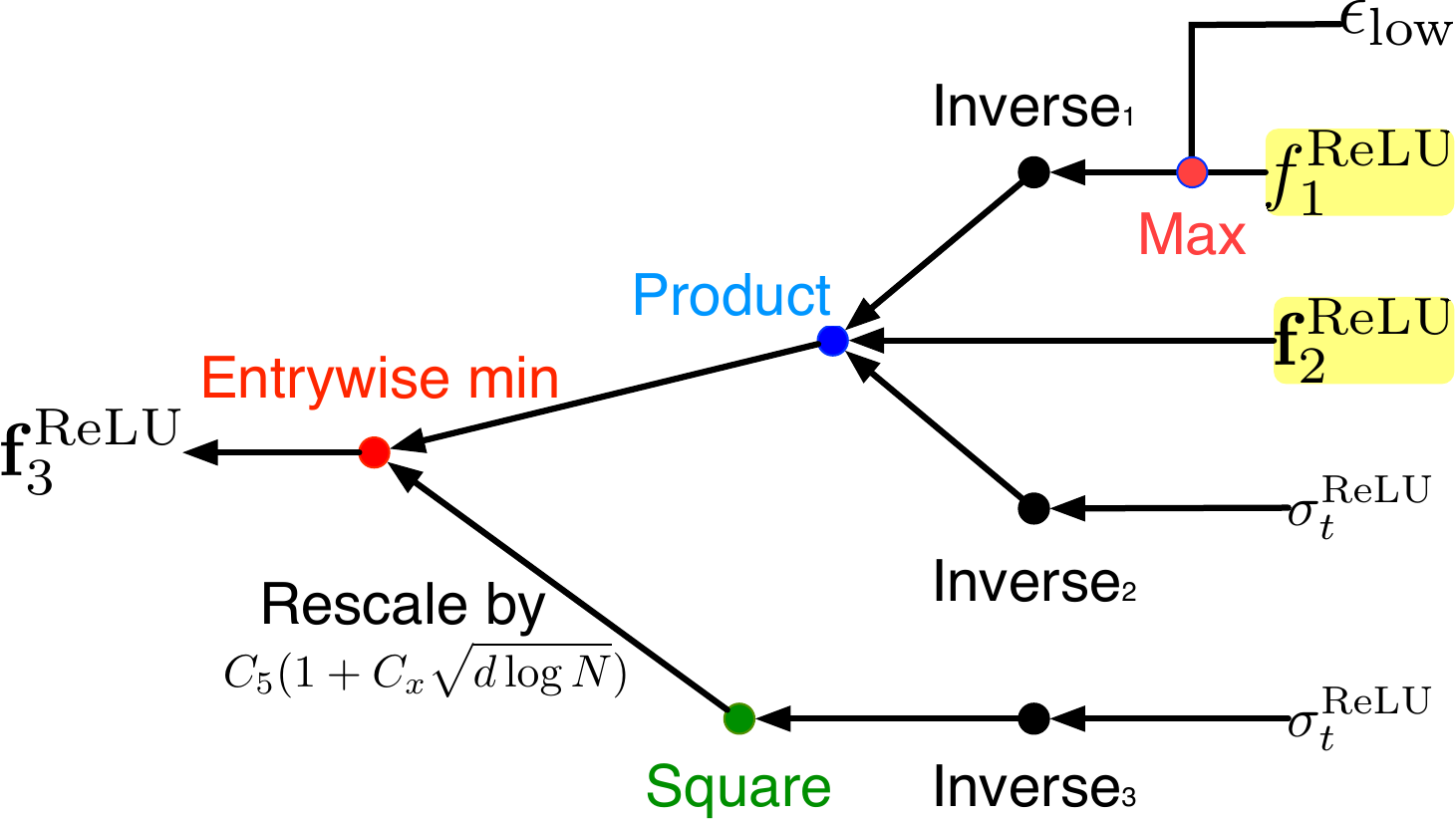}
\caption{The network architecture of $\fb_3^{\rm ReLU}$. We implement all the components of $\fb_3$ ($f_1$, $\fb_2$ and $\sigma_t$) through ReLU networks and combine them using the ReLU-approximated operators (\textit{product}, \textit{inverse} (\textit{reciprocal}) and \textit{entrywise-min}) to express $\fb_3$ according to its definition in \eqref{equ:: definition fb3}.}
\label{fig:ReLU2}
\end{figure}

 From the construction in the figure and the hyperparameter configuration in Lemma \ref{lemma::relu approx diffused} and \ref{lemma::relu approx diffused1}, we can obtain a ReLU network $ \cF(M_t, W, \kappa, L, K) $ with 
    \begin{align*}
    & \hspace{0.4in} M_t = \cO\left(\sqrt{\log N}/\sigma^2_t\right),~
     W = {\cO}\left(N^{d+d_y}\log^7 N\right),\\
     & \kappa =\exp \left({\cO}(\log^4 N)\right),~
    L = {\cO}(\log^4 N),~
     K= {\cO}\left(N^{d+d_y}\log^9 N\right).
\end{align*} This network contains $\fb^{\trelu}_{3}$ such that for any $\xb \in  [-C_x\sqrt{\log N},C_x\sqrt{\log N}]^{d}$, $\yb\in[0,1]^{d_y}$ and $t\in[N^{-C_{\sigma}},C_{\alpha}\log N]$,
\begin{align}\label{equ::f3 relu bound}
    \norm{\fb^{\trelu}_{3}(\xb,\yb,t)-\fb_{3}(\xb,\yb,t)}_{\infty}\le N^{-\beta}.
\end{align}
Thus, we have
\begin{align*}
    \norm{\nabla\log p_{t}(\xb|\yb)-\fb^{\trelu}_3(\xb,\yb,t)}_{\infty}\lesssim  \frac{B}{\sigma_t^2p_t(\xb|\yb)}N^{-\beta}(\log N)^{\frac{d+s+1}{2}}.
\end{align*}
We complete our proof.
\end{proof}

\subsection{Proofs of Lemmas \ref{lemma::diffused localpoly approx} and \ref{lemma::diffused localpoly approx1}}\label{sec::diffuse approx}
To prove the lemma, we first need some properties of the density function $p_t(\xb|\yb)$ and the score function $\nabla \log p_t(\xb|\yb)$.
\begin{lemma}[Clip the integral]\label{lemma::clipz}
    Under Assumption \ref{assump:sub}, for any $\mathbf{v}\in \Z_{+}^d$ with $\norm{\mathbf{v}}_1 \le n$. There exists a constant $C(n,d) \ge 1$ such that for any $\xb$ and $0 < \epsilon \le \frac{1}{e}$, it holds that
    \begin{align*}
        \int_{\R^d \backslash \textbf{B}_x} \left|\left(\frac{\alpha_t\zb-\xb}{\sigma_t}\right)^{\mathbf{v}}\right|&p(\zb|\yb)\frac{1}{\sigma_t^{d}(2\pi)^{d/2}}\exp\left(-\frac{\norm{\alpha_t\zb-\xb}^2}{2\sigma_t^2}\right)\diff \zb \le \epsilon,
    \end{align*}
    where 
    \begin{align}
        \textbf{B}_x&=\left[\frac{\xb-\sigma_tC(n,d)\sqrt{\log\epsilon^{-1}}}{\alpha_t},\frac{\xb+\sigma_tC(n,d)\sqrt{\log\epsilon^{-1}}}{\alpha_t}\right] \nonumber\\ &\quad\quad\quad\bigcap \left[-C(n,d)\sqrt{\log\epsilon^{-1}},C(n,d)\sqrt{\log\epsilon^{-1}}\right].\label{equ::definition clip}
    \end{align}
\end{lemma}
The proof of the lemma is provided in Appendix \ref{sec::Proof of lemma clipz}. Besides, we need to bound the density and the gradient of density.
\begin{lemma}\label{lemma::density bound}
Under Assumption \ref{assump:sub}, there exists a constant $C_4$ such that the diffused density function $p_t(\xb |\yb)$ can be bounded as:
\begin{align}
\frac{C_4}{\sigma_{t}^d} \exp\left(-\frac{\norm{\xb}^2+1}{\sigma_{t}^2}\right)    \le p_t(\xb |\yb)\le \frac{C_1}{(\alpha_t^2+C_2\sigma_t^2)^{d/2}}\exp\left(\frac{-C_2\norm{\xb}^2}{2(\alpha_t^2+C_2\sigma_t^2)}\right)
\end{align}
and the gradient can be bounded as
\begin{align}\label{equ::gradient bound}
    \norm{\nabla p_t(\xb,\yb)}_{\infty} \le \frac{C_1}{(\alpha_t^2+C_2\sigma_t^2)^{d/2}}\exp\left(\frac{-C_2\norm{\xb}^2}{2(\alpha_t^2+C_2\sigma_t^2)}\right) \left(\frac{\alpha_t}{\sigma_t\sqrt{\alpha^2_t+C_2\sigma^2_t}}+\frac{C_2\norm{\xb}_{\infty}}{\alpha^2_t+C_2\sigma^2_t} \right).
\end{align}
\end{lemma}
The proof is provided in Appendix \ref{sec::Proof of Lemma density bound}. Moreover, we can bound the score function.
\begin{lemma}\label{lemma::score bound}
Under Assumption \ref{assump:sub}, there exists a constant $C_5$ such that  the score function $\nabla\log p_t(\xb |\yb)$ can be bounded as:
\begin{align}
   \norm{\nabla \log p_t(\xb |\yb)}_\infty\le \frac{C_5}{\sigma_t^2} (\norm{\xb}+1).
\end{align}
\end{lemma}
The proof is provided in Appendix \ref{sec::Proof of Lemma score bound}. With all the previous lemmas, we begin to prove Lemma \ref{lemma::diffused localpoly approx}.

\begin{proof}[Proof of Lemma \ref{lemma::diffused localpoly approx}]
The main idea of the proof is to approximate the integral form of $p_t(\xb|\yb)$, which can be presented as
\begin{align}\label{equ::integral form of p_t restate}
p_t(\xb|\yb)=\underbrace{\int_{\RR^{d}}}_{\textbf{Step (i)}} \underbrace{p(\zb|\yb)}_{\textbf{Step (ii)}}\frac{1}{\sigma_t^{d}(2\pi)^{d/2}} \underbrace{\exp\left(-\frac{\norm{\alpha_t\zb-\xb}^2}{2\sigma_t^2}\right)}_{\textbf{Step (iii)}}\diff \zb.
\end{align}
We prove the lemma in the following steps:
\begin{description}
\item[Step (i)] (\textbf{Clip the domain}) We first truncate the integral of $p_t(\xb,\yb)$ in a bounded region using Lemma \ref{lemma::clipz}.
\item[Step (ii)] (\textbf{Approximate} $p(\cdot)$) We approximate the initial distribution function $p(\xb|\yb)$ using local polynomials in the bounded region, which fully utilizes the H\"older smoothness.
\item[Step (iii)] (\textbf{Approximate} $\exp(\cdot)$) We approximate the exponential function in the integrand by polynomials using Taylor expansion.
\end{description}
Combining {\bf Steps (ii)} and {\bf (iii)}, we can approximate the whole integrand by a polynomial, so $p_t$ can be approximated by a diffused polynomial. Now we begin our formal proof.

\vspace{5pt}
\noindent $\bullet$ \textbf{Step (i)} 
We approximate $p_t(\xb|\yb)$ by an integral on a bounded domain using Lemma \ref{lemma::clipz}. We denote the integral by 
\begin{align}\label{equ::f3}
f_2(\xb,\yb,t)=\int_{\textbf{B}_{\xb,N}}p(\zb|\yb)\frac{1}{\sigma_t^{d}(2\pi)^{d/2}}\exp\left(-\frac{\norm{\alpha_t\zb-\xb}^2}{2\sigma_t^2}\right)\diff \zb,
\end{align}
where we take 
\begin{align}\label{equ::BxN}
\textbf{B}_{\xb,N}&=\left[\frac{\xb-\sigma_tC(0,d)\sqrt{\beta\log N}}{\alpha_t},\frac{\xb+\sigma_tC(0,d)\sqrt{\beta\log N}}{\alpha_t}\right]\nonumber\\
&\quad\bigcap \left[-C(0,d)\sqrt{\beta\log N},C(0,d)\sqrt{\beta\log N}\right].
\end{align}
Thus, we have
\begin{align}\label{equ::p f2}
\abs{f_2(\xb,\yb,t)-p_t(\xb,\yb,t)} \le N^{-\beta}, \text{ for any } \xb \in \RR^{d}, \yb \in [0,1]^{d_y}.
\end{align}
\paragraph{Understanding $N$} The integer parameter $N$ represents the number of segments into which each axis of the bounded domain $\textbf{B}_{\xb, N} \times [0,1]^{d_y}$ is subdivided. Consequently, employing $N$ subdivisions along each of the $d+d_y$ dimensions results in $N^{d+d_y}$
  hypercubes covering the entire domain. Each of these hypercubes serves as a localized region where a Taylor polynomial is employed to approximate the function within that specific region. The choice of 
$N$ plays a crucial role in the accuracy of the following approximation scheme.

\vspace{5pt}
\noindent $\bullet$ \textbf{Step (ii)} 
Then we approximate $p(\zb|\yb)$ on this bounded region using local polynomials. We take $R=2C(0,d)\sqrt{\beta\log N}$ and denote
\begin{align}
    \label{equ::density with scaled support}
    f(\xb,\yb)=p(R(\xb-1/2)|\yb), ~~~\text{for } \xb\in [0,1]^d \text{ and } \yb\in [0,1]^{d_y}.
\end{align}
By assumption \ref{assump:sub}, we know that $\norm{f}_{\cH^{\beta}([0,1]^{d+d_y})}\le BR^s$. To implement the local polynomial approximation technique, we define $\phi$ as a trapezoid function:
\begin{align}\label{equ::trapezoid}
\phi(a) =\left\{\begin{matrix}
1 & \abs{a}<1, \\
2-\abs{a} & \abs{a}\in [1,2], \\
0 & \abs{a}>2. \\
\end{matrix} \right.
\end{align}
The trapezoid function is commonly used in the construction of continuous approximators of 
 target functions \citep{chen2022nonparametric}. Now we consider the following local polynomials
\begin{align}
q(\xb,\yb)= \sum_{\vb\in[N]^d,\wb\in[N]^{d_y}}\psi_{\vb,\wb}(\xb,\yb)P_{\vb,\wb}(\xb,\yb)
\end{align}
where $$ P_{\vb,\wb}(\xb,\yb)= \sum_{\norm{\nb}_1+\norm{\npb}_1\le s} \frac{1}{\nb!\npb!}\frac{\partial^{\nb+\npb} f}{\partial \xb^{\nb}\partial \yb^{\npb}} \bigg|_{\xb=\frac{\vb}{N},\yb=\frac{\wb}{N}} \left(\xb-\frac{\vb}{N}\right)^\nb\left(\yb-\frac{\wb}{N}\right)^\npb $$ 
is the $s$-order Taylor polynomial of $f(\xb,\yb)$ at the point $\left(\frac{\vb}{N},\frac{\wb}{N}\right)$, and
$$\psi_{\vb,\wb}(\xb,\yb)=\one\set{\xb\in \left(\frac{\vb-1}{N},\frac{\vb}{N}\right]}\prod \limits_{j=1}^{d_y}\phi\left(3N\left(y_j-\frac{\wb}{N}\right)\right)$$ can be seen as an indicator function supported on the neighbor of the point. To be specific, $\psi_{\vb,\wb}(\xb,\yb)\neq 0$ if and only if $\xb\in \left(\frac{\vb-1}{N},\frac{\vb}{N}\right]$ and $\yb \in \left[\frac{\wb-{2\cdot \mathbf{1}}/{3}}{N},\frac{\wb+{2\cdot \mathbf{1}}/{3}}{N}\right]$, so the $L_{\infty}$ distance between $[\xb,\yb]$ and $\frac{[\vb,\wb]}{N}$ is at most $\frac{1}{N}$. Moreover, by Taylor expansion, there exist $\btheta\in [0,1]^{d}$ and $ \btheta^{\prime}\in[0,1]^{d_y}$ such that
\begin{align*}
    f(\xb,\yb)&=\sum_{\norm{\nb}_1+\norm{\npb}_1< s} \frac{1}{\nb!\npb!}\frac{\partial^{\nb+\npb} f}{\partial \xb^{\nb}\partial \yb^{\npb}} \bigg|_{\xb=\frac{\vb}{N},\yb=\frac{\wb}{N}} \left(\xb-\frac{\vb}{N}\right)^\nb\left(\yb-\frac{\wb}{N}\right)^\npb \\
   & \quad+\sum_{\norm{\nb}_1+\norm{\npb}_1= s} \frac{1}{\nb!\npb!}\frac{\partial^{\nb+\npb} f}{\partial \xb^{\nb}\partial \yb^{\npb}} \bigg|_{\xb=(\mathbf{1}-\btheta)\frac{\vb}{N}+\btheta\xb,\yb=(\mathbf{1}-\btheta^{\prime})\frac{\wb}{N}+\btheta^{\prime}\yb}\left(\xb-\frac{\vb}{N}\right)^\nb\left(\yb-\frac{\wb}{N}\right)^\npb.
\end{align*}
Thus, we have
\begin{align*}
    &\quad \abs{P_{\vb,\wb}(\xb,\yb) -f(\xb,\yb)}\\
    &=\Bigg| \sum_{\norm{\nb}_1+\norm{\npb}_1= s} \frac{1}{\nb!\npb!}\frac{\partial^{\nb+\npb} f}{\partial \xb^{\nb}\partial \yb^{\npb}} \bigg|_{\xb=\frac{\vb}{N},\yb=\frac{\wb}{N}} \left(\xb-\frac{\vb}{N}\right)^\nb\left(\yb-\frac{\wb}{N}\right)^\npb\\
    &\quad -\sum_{\norm{\nb}_1+\norm{\npb}_1= s} \frac{1}{\nb!\npb!}\frac{\partial^{\nb+\npb} f}{\partial \xb^{\nb}\partial \yb^{\npb}} \bigg|_{\xb=(\mathbf{1}-\btheta)\frac{\vb}{N}+\btheta\xb,\yb=(\mathbf{1}-\btheta^{\prime})\frac{\wb}{N}+\btheta^{\prime}\yb}\left(\xb-\frac{\vb}{N}\right)^\nb\left(\yb-\frac{\wb}{N}\right)^\npb  \Bigg|\\
    &\le \sum_{\norm{\nb}_1+\norm{\npb}_1= s} \frac{1}{\nb!\npb!}\left(\xb-\frac{\vb}{N}\right)^\nb\left(\yb-\frac{\wb}{N}\right)^\npb BR^s\left\|[\btheta\xb,\btheta^{\prime}\yb]-\frac{[\btheta\vb,\btheta^{\prime}\wb]}{N}\right\|_{\infty}^{\gamma}\\
    &\le \sum_{\norm{\nb}_1+\norm{\npb}_1= s} \frac{BR^s}{\nb!\npb!N^{\norm{\nb}_1+\norm{\npb}_1+\gamma}}\\
    &=\frac{BR^s(d+d_y)^s}{s!N^{\beta}},
\end{align*}
Combining the result above with the fact that $\sum_{\vb\in[N]^d,\wb\in[N]^{d_y}}\psi_{\vb,\wb}(\xb,\yb)=1$ for any $\xb\in(0,1]^{d}$ and $\yb \in [0,1]^{d_y}$, we claim that $q(\xb,\yb)$ is an approximation to $f(\xb,\yb)$ which satisfies
\begin{align}\label{equ:: local error}
\abs{f(\xb,\yb)-q(\xb,\yb)}\le B\frac{R^s(d+d_y)^s}{s!N^{\beta}}, \forall \xb\in (0,1]^{d},~ \yb\in [0,1]^{d_y}.
\end{align}
Now we replace $p(\zb|\yb)$ by $q\left(\frac{\zb}{R}+1/2,\yb\right)$ in \eqref{equ::f3} and define
\begin{align*}
f_3(\xb,\yb,t) &=\frac{1}{\sigma_{t}^d(2\pi)^{d/2}}\int_{\textbf{B}_{\xb,N}}q\left(\frac{\zb}{R}+1/2,\yb\right)\exp\left(-\frac{\norm{\xb-\alpha_t\zb}^2}{2\sigma_{t}^2}\right)\diff \zb\\
&=\frac{1}{\sigma_{t}^d(2\pi)^{d/2}}\int_{\textbf{B}_{\xb,N}}\sum_{\vb\in[N]^d}\psi_{\vb,\wb}\left(\frac{\zb}{R}+1/2,\yb\right)P_{\vb,\wb}
\left(\frac{\zb}{R}+1/2,\yb\right)\exp\left(-\frac{\norm{\xb-\alpha_t\zb}^2}{2\sigma_{t}^2}\right)\diff \zb\\
&=\sum_{\vb\in[N]^d,\wb\in[N]^{d_y}}\sum_{\norm{\nb}_1+\norm{\npb}_1\le s}\frac{1}{\nb!\npb!}\frac{\partial^{\nb+\npb} f}{\partial \xb^{\nb}\partial \yb^{\npb}} \bigg|_{\xb=\frac{\vb}{N},\yb=\frac{\wb}{N}} \left(\yb-\frac{\wb}{N}\right)^\npb \prod \limits_{j=1}^{d_y}\phi\left(3N\left(y_j-\frac{\wb}{N}\right)\right) \\
&~~~~~~\cdot\prod \limits_{i=1}^d \frac{1}{\sigma_{t}(2\pi)^{1/2}} \int \left(\frac{z_i}{R}+1/2-\frac{v_i}{N}\right)^{n_i}\exp\left(-\frac{(x_i-\alpha_tz_i)^2}{2\sigma_t^2}\right)\diff z_i.
\end{align*} 
The domain of the integral
\begin{equation*}
\int \left(\frac{z_i}{R}+1/2-\frac{v_i}{N}\right)^{n_i}\exp\left(-\frac{(x_i-\alpha_tz_i)^2}{2\sigma_t^2}\right)\diff z_i
\end{equation*}
is 
\begin{align}
B_{v_i,n_i,x_i}:&=\left[\left(\frac{v_i-1}{N}-1/2\right)R,\left(\frac{v_i}{N}-1/2\right)R\right] \nonumber \\
&~~~~~~~\bigcap \left[\frac{x_i-\sigma_tC(0,d)\sqrt{\beta\log N}}{\alpha_t},\frac{x_i+\sigma_tC(0,d)\sqrt{\beta\log N}}{\alpha_t}\right]. \label{equ::domain inte}
\end{align}
Now we bound the difference between $f_2$ and $f_3$. By \eqref{equ:: local error}, we have
\begin{align*}
\abs{f_3(\xb,\yb,t)-f_2(\xb,\yb,t)} &=\left| \int_{\textbf{B}_{\xb,N}}\left(p(\zb|\yb)-q\left(\frac{\zb}{R}+1/2,\yb\right)\right)\frac{1}{\sigma_t^{d}(2\pi)^{d/2}}\exp\left(-\frac{\norm{\alpha_t\zb-\xb}^2}{2\sigma_t^2}\right)\diff \zb \right| \\
&\lesssim \int_{\textbf{B}_{\xb,N}}\frac{BR^s}{N^{\beta}}\frac{1}{\sigma_t^{d}(2\pi)^{d/2}}\exp\left(-\frac{\norm{\alpha_t\zb-\xb}^2}{2\sigma_t^2}\right)\diff \zb \\
&\le \frac{BR^s}{N^{\beta}}\int_{\R^{d}}\frac{1}{\sigma_t^{d}(2\pi)^{d/2}}\exp\left(-\frac{\norm{\alpha_t\zb-\xb}^2}{2\sigma_t^2}\right)\diff \zb\\
&\lesssim \frac{BN^{-\beta}\log^{\frac{s}{2}} N}{\alpha^d_t}.
\end{align*} 
At the same time, we have
\begin{align}
\abs{f_3(\xb,\yb,t)-f_2(\xb,\yb,t)}
&\lesssim \int_{\textbf{B}_{\xb,N}}\frac{BR^s}{N^{\beta}}\frac{1}{\sigma_t^{d}(2\pi)^{d/2}}\exp\left(-\frac{\norm{\alpha_t\zb-\xb}^2}{2\sigma_t^2}\right)\diff \zb \nonumber\\
&\le \frac{BR^s}{N^{\beta}}\frac{m(\textbf{B}_{\xb,N})}{\sigma_t^{d}(2\pi)^{d/2}} \nonumber \\
&\lesssim \frac{BN^{-\beta}\log^{\frac{s+d}{2}} N}{\sigma^d_t}, \label{equ::f2f3}
\end{align} 
where $ m(\textbf{B}_{\xb,N})$ is the Lebesgue measure of $\textbf{B}_{\xb,N}$ in $\RR^{d}$.
Taking the minimum gives rise to 
\begin{equation*}
\abs{f_3(\xb,\yb,t)-f_2(\xb,\yb,t)} \lesssim B\min\left(\frac{1}{\sigma_t^d}, \frac{1}{\alpha_t^d}\right) N^{-\beta}\log^{\frac{d+s}{2}}N\lesssim BN^{-\beta}\log^{\frac{d+s}{2}}N.
\end{equation*}

\vspace{5pt}

\noindent $\bullet$ \textbf{Step (iii)} 
Next, we approximate $\exp\left(-\frac{\abs{x-\alpha_tz}^2}{2\sigma_{t}^2}\right)$ using Taylor expansions. By the choice of $\textbf{B}_{\xb,N}$, we know that $\abs{\frac{x_i-\alpha_t z_i}{\sigma_t}} \le C(0,d)\sqrt{\beta\log N}$ for any $i \in [d]$ and $\zb \in \textbf{B}_{\xb,N}$. Thus, by Taylor expansions we have
\begin{align}
\abs{\exp\left(-\frac{\abs{x-\alpha_tz}^2}{2\sigma_{t}^2}\right)-\sum_{k<p} \frac{1}{k!}\left(-\frac{\abs{x-\alpha_tz}^2}{2\sigma_{t}^2}\right)^k}\le \frac{C(0,d)^{2p}\beta^{p}\log^pN}{p!2^p}, \forall z\in[\underline{C}(x),\overline{C}(x)],
\end{align}
where
$$\underline{C}(x)=\max\left(\frac{x-\sigma_tC(0,d)\sqrt{\beta\log N}}{\alpha_t},-C(0,d)\sqrt{\beta\log N}\right) ,$$ and
$$\overline{C}(x)
=\min\left(\frac{x+\sigma_tC(0,d)\sqrt{\beta\log N}}{\alpha_t} ,C(0,d)\sqrt{\beta\log N}\right).$$ 
By setting $p=\frac{2}{3}C^2(0,d)\beta^2u\log N$ and invoking the inequality $p!\ge (\frac{p}{3})^p$ when $p\ge 3$, we have
\begin{align}\label{equ::exp error}
    \abs{\exp\left(-\frac{\abs{x-\alpha_tz}^2}{2\sigma_{t}^2}\right)-\sum_{k<p} \frac{1}{k!}\left(-\frac{\abs{x-\alpha_tz}^2}{2\sigma_{t}^2}\right)^k} \le N^{-\frac{2}{3}C^2(0,d)\beta^2u\log u}.
\end{align}
 Thus, we can set $$u=\max\left(e,\frac{3}{2C^2(0,d)\beta}+\frac{3\log d}{2C^2(0,d)\beta^2 \log N}\right)$$ to bound \eqref{equ::exp error} by $N^{-\beta}/d$, where $p= \cO(\log N)$. By multiplying the $d$ terms along $i$, we have
\begin{equation}\label{equ::exp error d}
\left| \prod \limits_{i=1}^d\left(\sum_{k<p} \frac{1}{k!}\left(-\frac{\abs{x_i-\alpha_tz_i}^2}{2\sigma_{t}^2}\right)^k\right)-\exp\left(-\frac{\norm{\alpha_t\zb-\xb}^2}{2\sigma_t^2}\right) \right| \le d\left(1+\frac{N^{-\beta}}{d}\right)^{d-1}\frac{N^{-\beta}}{d}\lesssim N^{-\beta}.
\end{equation}

 \vspace{5pt}
\noindent $\bullet$ \textbf{Putting all together.} Now we define
\begin{align*}
    f_1(\xb,\yb,t)&=\frac{1}{\sigma_{t}^d(2\pi)^{d/2}}\int_{\textbf{B}_{\xb,N}}q\left(\frac{\zb}{R}+1/2,\yb\right)\sum_{k<p} \frac{1}{k!}\left(-\frac{\abs{x_i-\alpha_tz_i}^2}{2\sigma_{t}^2}\right)^k\diff \zb\\
    &=\sum_{\vb\in[N]^d,\wb\in[N]^{d_y}}\sum_{\norm{\nb}_1+\norm{\npb}_1<s}\frac{1}{\nb!\npb!}\frac{\partial^{\nb+\npb} f}{\partial \xb^{\nb}\partial \yb^{\npb}} \bigg|_{\xb=\frac{\vb}{N},\yb=\frac{\wb}{N}} \left(\yb-\frac{\wb}{N}\right)^\npb \prod \limits_{j=1}^{d_y}\phi\left(3N\left(y_j-\frac{\wb}{N}\right)\right) \\
        &\qquad \cdot\prod \limits_{i=1}^d \frac{1}{\sigma_{t}(2\pi)^{1/2}} \int \left(\frac{z_i}{R}+1/2-\frac{v_i}{N}\right)^{n_i}\sum_{k<p} \frac{1}{k!}\left(-\frac{\abs{x_i-\alpha_tz_i}^2}{2\sigma_{t}^2}\right)^k \diff z_i.
\end{align*}

By the definition of $f_1$ and \eqref{equ::exp error d}, we have 
\begin{align}
\left| f_3(\xb,\yb,t)-f_1(\xb,\yb,t)\right|&\lesssim\left| \int_{\textbf{B}_{\xb,N}}q\left(\frac{\zb}{R}+1/2,\yb\right)\frac{1}{\sigma_t^{d}(2\pi)^{d/2}}N^{-\beta} \diff \zb \right| \nonumber\\
&\overset{(i)}{\lesssim}  \left| \int_{\textbf{B}_{\xb,N}}\left(\abs{p(\zb|\yb)}+BN^{-\beta}\log^{\frac{s}{2}} N\right)\frac{1}{\sigma_t^{d}(2\pi)^{d/2}} N^{-\beta} \diff \zb \right|\nonumber \\
&\le \left| \int_{\textbf{B}_{\xb,N}}\left(B+BN^{-\beta}\log^{\frac{s}{2}} N\right)\frac{1}{\sigma_t^{d}(2\pi)^{d/2}}N^{-\beta} \diff \zb \right|\nonumber \\
&{\lesssim}~  m(\textbf{B}_{\xb,N})\frac{B}{\sigma_t^d}N^{-\beta}\nonumber\\
&\overset{(ii)}{\lesssim} \min\left( \frac{1}{\alpha_t^d}, \frac{1}{\sigma_t^d} \right)BN^{-\beta} \log ^{\frac{d}{2}} N \nonumber\\
&\lesssim BN^{-\beta} \log ^{\frac{d}{2}} N. \label{equ::f3 f1}
\end{align}
For inequality (i), we invoke \eqref{equ:: local error}. For inequality (ii), we invoke $m(\textbf{B}_{\xb,N}) \lesssim \min\left( \frac{\sigma_t^d}{\alpha_t^d}, 1 \right)\log ^{\frac{d}{2}} N$, which can be obtained by the definition $\textbf{B}_{\xb,N}$ in \eqref{equ::BxN}.
Thus, adding up all the errors \eqref{equ::p f2}, \eqref{equ::f2f3} and \eqref{equ::f3 f1} gives rise to 
\begin{equation}\label{equ::poly error2}
    \left| p_t(\xb|\yb)-f_1(\xb,\yb,t)\right| \lesssim BN^{-\beta}\log^{\frac{d+s}{2}}N.
\end{equation}
We note that $f_1$ can be written as a linear combination of the following form of functions
\begin{align}\label{equ::Diffused local monomial}
       \dmono(\xb,\yb,t)= \left(\yb-\frac{\wb}{N}\right)^\npb \prod \limits_{j=1}^{d_y}\phi\left(3N\left(y_j-\frac{\wb}{N}\right)\right)  \prod\limits_{i=1}^{d} \sum_{k<p} g(x_i,n_i,v_i,k),
\end{align}
where
\begin{align}\label{equ::h}
    g(x,n,v,k)=\frac{1}{\sigma_{t}(2\pi)^{1/2}} \int \left(\frac{z}{R}+1/2-\frac{v}{N}\right)^{n} \frac{1}{k!}\left(-\frac{\abs{x-\alpha_tz}^2}{2\sigma_{t}^2}\right)^k \diff z.
\end{align}
We call $\dmono$ a \textbf{diffused local monomial} and then $f_1$ is a \textbf{diffused local polynomial} with at most $N^{d+d_y}(d+d_y)^s$ diffused local monomials, which is presented as
\begin{align}
    f_1(\xb,\yb,t)&=\sum_{\vb\in[N]^d,\wb\in[N]^{d_y}}\sum_{\norm{\nb}_1+\norm{\npb}_1\le s}\frac{1}{\nb!\npb!}\frac{\partial^{\nb+\npb} f}{\partial \xb^{\nb}\partial \yb^{\npb}} \bigg|_{\xb=\frac{\vb}{N},\yb=\frac{\wb}{N}} \dmono(\xb,\yb,t) \nonumber\\
    &=\sum_{\vb\in[N]^d,\wb\in[N]^{d_y}}\sum_{\norm{\nb}_1+\norm{\npb}_1\le s}\frac{R^{\norm{\nb}_1}}{\nb!\npb!}\frac{\partial^{\nb+\npb} p}{\partial \xb^{\nb}\partial \yb^{\npb}} \bigg|_{\xb=R\left(\frac{\vb}{N}-\frac{1}{2}\right),\yb=\frac{\wb}{N}} \dmono(\xb,\yb,t). \label{equ::decomposition of f1}
\end{align}
Thus, we complete the proof.
\end{proof}

In the following, we prove Lemma \ref{lemma::diffused localpoly approx1}. 
\begin{proof}[Proof of Lemma \ref{lemma::diffused localpoly approx1}]
Since we have
    \begin{equation*}
    \sigma_t \nabla p_t(\xb|\yb)=-\int_{\R^d}{\frac{\xb-\alpha_t\zb}{\sigma_{t}}p(\zb|\yb)\exp\left(-\frac{\norm{\xb-\alpha_t\zb}^2}{2\sigma_{t}^2}\right)\diff \zb},
\end{equation*}
when approximating the $i-$th element of the vector $ \sigma_t \nabla p_t(\xb|\yb)$, we can apply Lemma \ref{lemma::clipz} with $\vb=\mathbf{e}_i$ to confine the integral in a similar $\textbf{B}_{\xb,N}$ and completely follow the proof of Lemma \ref{lemma::diffused localpoly approx} to obtain the polynomial approximation. The only difference is that the degree of local polynomials increases by one, so the deviation between $f_2$ and $f_3$ becomes $BN^{-\beta}\log^{\frac{d+s+1}{2}}N$ instead of $BN^{-\beta}\log^{\frac{d+s}{2}}N$ in \eqref{equ::f2f3}. Thus, the approximation error also increases to $BN^{-\beta}\log^{\frac{d+s+1}{2}}N$.
\end{proof}

\subsection{Proofs of Lemmas \ref{lemma::relu approx diffused} and \ref{lemma::relu approx diffused1}}\label{sec::sub relu approx}
We only elaborate on the proof of Lemma \ref{lemma::relu approx diffused}, since the proof of Lemma \ref{lemma::relu approx diffused1} is completely the same. The main idea is to use a ReLU network to approximate the single diffused local monomial with a small error. We recall that the diffused local monomial is defined as 
\begin{align}\label{equ::diffused monomial restate}
   \dmono(\xb,\yb,t)= \left(\yb-\frac{\wb}{N}\right)^\npb \prod \limits_{j=1}^{d_y}\phi\left(3N\left(y_j-\frac{\wb}{N}\right)\right)  \prod\limits_{i=1}^{d} \sum_{k<p} g(x_i,n_i,v_i,k).
\end{align}
We can approximate each part of $\dmono$ with ReLU networks and combine them using a multiplication operator constructed by a ReLU network in Lemma \ref{lemma::product}.
\begin{lemma}[Implement $\phi(\cdot)$]\label{lemma::approx trapezoid}\label{lemma::trapezoid}
    The trapezoid function $\phi:\R \rightarrow \R$ can be exactly represented by a constant-sized ReLU network $\cF(4,3,1,9)$.
\end{lemma}
\begin{proof}[Proof of Lemma \ref{lemma::approx trapezoid}]
    This can be verified by the fact that $\phi(x)=3-\sigma(x+1)-\sigma(-x+1)+\sigma(x-2)+\sigma(-x-2)$, where $\sigma$ is the ReLU activation function. Thus by setting $$A_1=[1,-1,1,-1]^{\top},~ \bb_1=[1,1,-2,-2]^{\top},~A_2=[-1,-1,1,1], \text{ and } \bb_2=[3],$$ we have $(A_2\sigma(\cdot)+\bb_2)\circ (A_1 x +\bb_1)=\phi(\cdot)$.
    The proof is complete.
\end{proof}

\begin{lemma}[Approximate $g(\cdot)$ in \eqref{equ::h}]\label{lemma::approx h}
    Given $N>0$, $C_x>0$, there exists a ReLU network $\cF(W,\kappa,L,K)$ such that for any $n\le s$, $v\le N$, $k\le p$ and $\epsilon >0$, this network gives rise to a function $g^{\trelu}(x,n,v,k)$ such that
    \begin{align*}
        \abs{g^{\trelu}(x,n,v,k)-g(x,n,v,k)}\le \epsilon ~~\text{ for any } x\in\left[-C_x\sqrt{\log N},C_x\sqrt{\log N}\right].
    \end{align*}
    The hyperparameter of the network satisfies
  \begin{align*}
    & W = {\cO}\left(\log^6 N +\log^3 \epsilon^{-1} \right),
      \kappa =\exp \left({\cO}(\log^4 N+\log^2 \epsilon^{-1})\right),\\
      ~&L = {\cO}(\log^4 N+\log^2 \epsilon^{-1}),~
     K= {\cO}\left(\log^8 N+\log^4 \epsilon^{-1}\right).
\end{align*}
\end{lemma}
The proof of Lemma \ref{lemma::approx h} is provided in Appendix \ref{sec::sub other lemmas}. With all the lemmas above, we begin our proof of Lemma \ref{lemma::relu approx diffused}.
\begin{proof}[Proof of Lemma \ref{lemma::relu approx diffused}
]
    From Lemma \ref{lemma::approx h}, for any $i \in[d]$, $n \in [s]$, $v \in [N]$ and $k \le p$, we can substitute $g(x_i,n_i,v_i,k)$ with the corresponding ReLU approximator $g^{\trelu}(x_i,n_i,v_i,k)$ in the expression of diffused local monomial \eqref{equ::diffused monomial restate}. To be specific, we aggregate $g^{\trelu}(x_i,n_i,v_i,k)$ along $k$ to approximate $\sum_{k<p}g(x_i,n_i,v_i,k)$. Then we multiply $$\left(\yb-\frac{\wb}{N}\right)^\npb, 
 \set{\phi\left(3N\left(y_j-\frac{\wb}{N}\right)\right)}_{j=1}^{d_y}~ \text{ and }~ \set{\sum_{k<p}g^{\trelu}(x_i,n_i,v_i,k)}_{i=1}^d$$ together using $f_{\text{mult}}$ in Lemma \ref{lemma::product} to get a series of ReLU network functionss $$\cD^{\trelu}=\set{\dmono^{\trelu}:\norm{\nb} +\norm{\npb} \le s,\vb \in [N]^{d},\wb \in [N]^{d_y}} $$ such that 
 \begin{align*}
     \abs{\dmono(\xb,\yb,t)-\dmono^{\trelu}(\xb,\yb,t)}\le \frac{s!\epsilon}{(d+d_y)^{s}R^{s}N^{d+d_y}},
 \end{align*}
for any $\xb \in  [-C_x\sqrt{\log N},C_x\sqrt{\log N}]^{d}$, $\yb\in[0,1]^{d_y}$ and $t\in[N^{-C_{\sigma}},C_{\alpha}\log N]$ to approximate $\dmono$. Details about how to determine the network size and the error propagation are deferred to Appendix \ref{sec::f1 relu}.

 At last, we derive a linear combination of these ReLU network functions in $\cD^{\trelu}$ to get an approximation for $f_1$, which is presented as 
 \begin{align*}
f^{\trelu}_1(\xb,\yb,t)=\sum_{\vb\in[N]^d,\wb\in[N]^{d_y}}\sum_{\norm{\nb}_1+\norm{\npb}_1\le s}\frac{R^{\norm{\nb}_1}}{\nb!\npb!}\frac{\partial^{\nb+\npb} p}{\partial \xb^{\nb}\partial \yb^{\npb}} \bigg|_{\xb=R\left(\frac{\vb}{N}-\frac{1}{2}\right),\yb=\frac{\wb}{N}} \dmono^{\trelu}(\xb,\yb,t).
 \end{align*}
 
 From the choice of hyperparameters in Lemma \ref{lemma::approx h} and our process of constructing these ReLU network functions 
 above, we know that $f_1^{\trelu}(\xb,\yb,t) \in \cF( W, \kappa, L, K) $ with 
\begin{align*}
& W = {\cO}\left(N^{d+d_y}(\log^7 N+\log N\log^3 \epsilon^{-1})\right), \quad \kappa =\exp \left({\cO}(\log^4 N+\log^2 \epsilon^{-1})\right), \\
& \qquad L = {\cO}(\log^4 N+\log^2 \epsilon^{-1}), \quad K = {\cO}\left(N^{d+d_y}(\log^9 N+\log N\log^3 \epsilon^{-1})\right).
\end{align*}
and satisfies that for any $x \in  [-C_x\sqrt{\log N},C_x\sqrt{\log N}]$, $\yb\in[0,1]^{d_y}$ and $t\in[N^{-C_{\sigma}},C_{\alpha}\log N]$,
\begin{align*}
     \abs{f_1(\xb,\yb,t)-f_1^{\trelu}(\xb,\yb,t)}&\le \sum_{\vb\in[N]^d,\wb\in[N]^{d_y}}\sum_{\norm{\nb}_1+\norm{\npb}_1\le s}\frac{R^{\norm{\nb}_1}}{\nb!\npb!}\cdot \frac{s!\epsilon}{(d+d_y)^{s}R^{s}N^{d+d_y}}\\
     &\le \frac{(d+d_y)^{s}R^{s}N^{d+d_y}}{s!}\cdot \frac{s!\epsilon}{(d+d_y)^{s}R^{s}N^{d+d_y}} =\epsilon.
\end{align*}    
The proof is complete.
\end{proof}

\subsection{Proofs in Steps 1 and 2 for Theorem~\ref{thm::score approx}}\label{pf:score approx steps12}
\subsubsection{Proof of Lemma \ref{lemma::truncation x}}\label{sec::lemma::truncation x}

\begin{proof}
By applying Lemmas \ref{lemma::density bound} and \ref{lemma::score bound}, we have
\begin{align*}
& \quad \int_{\norm{\xb}\ge R} \norm{\nabla\log p_{t}(\xb|\yb)}^2 p_{t}(\xb|\yb)\diff \xb \\
&\lesssim \frac{C_5^2}{\sigma_t^4}\frac{C_1}{(\alpha_t^2+C_2\sigma_t^2)^{d/2}}  \int_{\norm{\xb}\ge R}  (\norm{\xb}^2_2+1)\exp\left(\frac{-C_2\norm{\xb}^2_2}{2(\alpha_t^2+C_2\sigma_t^2)}\right) \diff \xb\\
&\lesssim \frac{1}{\sigma_t^4}\left(\frac{R^3}{3(\alpha_t^2+C_2\sigma_t^2)^{3/2}}+\frac{R}{(\alpha_t^2+C_2\sigma_t^2)^{1/2}}\right)\exp\left(\frac{-C_2R^2}{2(\alpha_t^2+C_2\sigma_t^2)}\right)\\
&\lesssim \frac{1}{\sigma_t^4}R^3\exp(-C_2^{\prime}R^2),
\end{align*}
where $C^{\prime}_2=\min_{t>0}\frac{C_2}{2(\alpha_t^2+C_2\sigma_t^2)}=\frac{C_2}{2\max(C_2,1)}$. Similarly, we have
\begin{align*}
    \int_{\norm{\xb}\ge R}  p_{t}(\xb|\yb)\diff \xb 
    &\lesssim \frac{C_1}{(\alpha_t^2+C_2\sigma_t^2)^{d/2}}  \int_{\norm{\xb}\ge R}  \exp\left(\frac{-C_2\norm{\xb}^2_2}{2(\alpha_t^2+C_2\sigma_t^2)}\right) \diff \xb\\
    &\lesssim \frac{R}{(\alpha_t^2+C_2\sigma_t^2)^{1/2}}\exp\left(\frac{-C_2R^2}{2(\alpha_t^2+C_2\sigma_t^2)}\right)\\
    &\lesssim R\exp(-C_2^{\prime}R^2).
\end{align*}
The proof is complete.
\end{proof}
\subsubsection{Proof of Lemma \ref{lemma::truncation p}.} \label{sec::proof lemma truncation p}
\begin{proof}
By Lemma \ref{lemma::score bound}, for any $\epslow>0$ we have
\begin{align*}
    &\quad\int_{\norm{\xb}_{\infty}\le R} \indic{\abs{p_{t}(\xb|\yb)}<\epslow}\norm{\nabla\log p_{t}(\xb|\yb)}^2 p_{t}(\xb|\yb)\diff \xb\\
    &\le \int_{\norm{\xb}_{\infty}\le R} \epslow\norm{\nabla\log p_{t}(\xb|\yb)}^2 \diff \xb
    \\ &\le \int_{\norm{\xb}_{\infty}\le R} \epslow\left(\frac{C_5}{\sigma_t^2} (\norm{\xb}+1)\right)^2 \diff \xb
    \\&\lesssim \frac{\epslow}{\sigma_t^4}R^{d+2}.
\end{align*}
and 
\begin{align*}
    \int_{\norm{\xb}_{\infty}\le R} \indic{\abs{p_{t}(\xb|\yb)}<\epslow} p_{t}(\xb|\yb)\diff \xb
    &\le \int_{\norm{\xb}_{\infty}\le R} \epslow \diff \xb\lesssim R^{d} \epslow.
\end{align*}
The proof is complete.
\end{proof}

\subsection{Proofs of Further Supporting Lemmas}\label{sec::sub other lemmas}
In this section, we will prove Lemmas \ref{lemma::clipz}, \ref{lemma::density bound}, \ref{lemma::score bound} and
\ref{lemma::approx h}. 

\subsubsection{Proof of Lemma \ref{lemma::clipz}}\label{sec::Proof of lemma clipz}
\begin{proof}
We first decompose the domain of integration $\R^d \backslash \textbf{B}_x$ into a cartesian product of $d$ univariate domains.
We define 
    \begin{align*}
        B_x^i&= \left[\frac{x_i-\sigma_tC(n,d)\sqrt{\log\epsilon^{-1}}}{\alpha_t},\frac{x_i+\sigma_tC(n,d)\sqrt{\log\epsilon^{-1}}}{\alpha_t}\right] \cap \left[-C(n,d)\sqrt{\log\epsilon^{-1}},C(n,d)\sqrt{\log\epsilon^{-1}}\right]\\
       & =:B^{i}_{x,1} \cap B^{i}_{x,2}.
    \end{align*}
By the definition of $\textbf{B}_x$ and $ B_x^i$, it holds that $\R^d \backslash \textbf{B}_x \subseteq \bigcup\limits_{i=1}^n \left(\R \times ... \times (\R \backslash B_x^i) \times ...\times \R\right).$ Thus, we have
    \begin{align*}
       &\quad \int_{\R^d \backslash \textbf{B}_x} \left|\left(\frac{\alpha_t\zb-\xb}{\sigma_t}\right)\right|^{\mathbf{v}}p(\zb|\yb)\frac{1}{\sigma_t^{d}(2\pi)^{d/2}}\exp\left(-\frac{\norm{\alpha_t\zb-\xb}^2}{2\sigma_t^2}\right)\diff \zb \\
        &\le \sum_{i=1}^d \int_{\R \times ... \times (\R \backslash B_x^i) \times ...\times \R} \left|\left(\frac{\alpha_t\zb-\xb}{\sigma_t}\right)^{\mathbf{v}}\right|p(\zb|\yb)\frac{1}{\sigma_t^{d}(2\pi)^{d/2}}\exp\left(-\frac{\norm{\alpha_t\zb-\xb}^2}{2\sigma_t^2}\right)\diff \zb\\
        &\le C_1\sum_{i=1}^d \prod_{j\neq i} \underbrace{ \int_{\R} \left|\frac{\alpha_t z_j-x_j}{\sigma_t}\right|^{v_j}\exp(-C_2z_j^2/2)\frac{1}{\sigma_t(2\pi)^{1/2}}\exp\left(-\frac{\abs{\alpha_tz_j-x_j}^2}{2\sigma_t^2}\right)\diff z_j}_{A_{i,j}}\\
        &\qquad \times \underbrace{ \int_{\R \backslash B_x^i} \left|\frac{\alpha_t z_i-x_i}{\sigma_t}\right|^{v_i}\exp(-C_2z_i^2/2)\frac{1}{\sigma_t(2\pi)^{1/2}}\exp\left(-\frac{\abs{\alpha_tz_i-x_i}^2}{2\sigma_t^2}\right)\diff z_i}_{A_i},
    \end{align*}
    where in the last inequality, we invoke the sub-Gaussian tail condition in Assumption \ref{assump:sub} to bound $p(\zb|\yb)$ and decompose the integral according to coordinates. It remains to bound $A_{i,j}$ and $A_i$, respectively. For the term $A_{i, j}$, we have
    \begin{align}
    A_{i,j} & \le \frac{1}{\sigma_t(2\pi)^{1/2}} \int_{\R} (\frac{v_j}{e})^{v_j/2} \exp(-C_2z_j^2/2) \diff z_j  \lesssim \frac{1}{\sigma_t}, \label{equ::Aij bound1}
    \end{align}
    where the first inequality follows from
    \begin{align*}
        \left|\frac{\alpha_t z-x}{\sigma_t}\right|^{v}\exp\left(-\frac{\abs{\alpha_tz-x}^2}{2\sigma_t^2}\right)\le \left(\frac{v}{e}\right)^{v/2}, ~~~\forall z,x \in \RR^{d}.
    \end{align*}
    At the same time, we have
    \begin{align}
        A_{i,j}  &\le  \int_{\R} \left|\frac{\alpha_t z_j-x_j}{\sigma_t}\right|^{v_j}\frac{1}{\sigma_t(2\pi)^{1/2}}\exp\left(-\frac{\abs{\alpha_tz_j-x_j}^2}{2\sigma_t^2}\right)\diff z_j \nonumber\\
        &\lesssim  \int_{\R} \frac{w^{v_j}}{\alpha_t (2\pi)^{1/2}} \exp(-w^2/2) \diff w  \lesssim \frac{1}{\alpha_t}. \label{equ::Aij bound2}
    \end{align} 
    Combining \eqref{equ::Aij bound1} and \eqref{equ::bound A1 1} yields
    \begin{align*}
    A_{i, j} \lesssim \min\left(\frac{1}{\sigma_t},\frac{1}{\alpha_t}\right) = \cO(1).
    \end{align*}
    For term $A_{i}$, we bound by
    \begin{align}
        A_i&=\int_{\R \backslash B^{i}_x} \left|\frac{\alpha_t z_i-x_i}{\sigma_t}\right|^{v_i}\exp(-C_2z_i^2/2)\frac{1}{\sigma_t(2\pi)^{1/2}}\exp\left(-\frac{\abs{\alpha_tz_i-x_i}^2}{2\sigma_t^2}\right)\diff z_i\nonumber \\ 
   &\le \underbrace{\int_{\R \backslash B^{i}_{x,1}} \left|\frac{\alpha_t z_i-x_i}{\sigma_t}\right|^{v_i}\exp(-C_2 z_{i}^2/2)\frac{1}{\sigma_t(2\pi)^{1/2}}\exp\left(-\frac{\abs{\alpha_t z_{i}-x_{i}}^2}{2\sigma_t^2}\right) \diff z_{i} }_{(\spadesuit)}\nonumber\\
             &\quad+ \underbrace{\int_{\R \backslash B^{i}_{x,2} } \left|\frac{\alpha_t z_i-x_i}{\sigma_t}\right|^{v_i}\exp(-C_2 z_i^2/2)\frac{1}{\sigma_t(2\pi)^{1/2}}\exp\left(-\frac{\abs{\alpha_tz_i-x_i}^2}{2\sigma_t^2}\right)\diff z_i}_{(\clubsuit)} \nonumber.
    \end{align}
    Now we deal with $(\spadesuit)$ and $(\clubsuit)$ using the same technique. Note that when $z_i \in \R \backslash B^{i}_{1,x}$, we have
   \begin{equation}\label{equ::bound region}
       \left|\frac{\alpha_t z_i-x_i}{\sigma_t}\right|>C(n,d)\sqrt{\log \epsilon^{-1}} .
   \end{equation}
   By setting $C(n,d) \ge \norm{\vb}_1$ and $\epsilon<\frac{1}{e}$, we obtain that when $\left|\frac{\alpha_t z_i-x_i}{\sigma_t}\right|\ge v_i$, 
    $$ \left|\frac{\alpha_t z_i-x_i}{\sigma_t}\right|^{v_i}\exp\left(-\frac{\abs{\alpha_tz_i-x_i}^2}{2\sigma_t^2}\right)$$
    decreases as $\left|\frac{\alpha_t z_i-x_i}{\sigma_t}\right|$ increases. Therefore, inequality \eqref{equ::bound region} leads to
\begin{align}
    (\spadesuit) &\le \frac{1}{\sigma_t (2\pi)^{1/2}} \int_{\R \backslash B^{i}_{1,x}} C(n,d)^{v_i}\left(\log \epsilon^{-1}\right)^{v_i/2}\epsilon^{\frac{C(n,d)^2}{2}}\exp(-C_2z_i^2/2) \diff z_i \nonumber\\
    &\lesssim \frac{C(n,d)^{v_i}\left(\log \epsilon^{-1}\right)^{v_i/2}\epsilon^{\frac{C(n,d)^2}{2}}}{\sigma_t}. \label{equ::bound A1 1}
\end{align}
Meanwhile, we also have
\begin{align}
    (\spadesuit) &\le \int_{\R \backslash B^{i}_{x,1}} \left|\frac{\alpha_t z_i-x_i}{\sigma_t}\right|^{v_i}\frac{1}{\sigma_t(2\pi)^{1/2}}\exp\left(-\frac{\abs{\alpha_tz_i-x_i}^2}{2\sigma_t^2}\right)\diff z_i \nonumber \\
    &=\frac{1}{\alpha_t (2\pi)^{1/2}}\int_{|w|>C(n,d)\sqrt{\log \epsilon^{-1}}} w^{v_i}\exp(-w^2/2) \diff w \nonumber \\
    &\lesssim \frac{1}{\alpha_t (2\pi)^{1/2}} C(n,d)^{v_i+2}\left(\log \epsilon^{-1}\right)^{(v_i+2)/2}\epsilon^{\frac{C(n,d)^2}{2}}. \label{equ::bound A1 2}
\end{align}
Combining \eqref{equ::bound A1 1} and \eqref{equ::bound A1 2}, we deduce
\begin{equation*}
    (\spadesuit) \lesssim C(n,d)^{v_i+2}\left(\log \epsilon^{-1}\right)^{(v_i+2)/2}\epsilon^{\frac{C(n,d)^2}{2}}.
\end{equation*}
The term $(\clubsuit)$ assumes analogous upper bounds in the following:
\begin{align*}
    (\clubsuit) &\le \epsilon^{\frac{C_2C(n,d)^2}{2}} \int_{\R }\frac{1}{\alpha_t(2\pi)^{1/2}} w^{v_i}\exp(-w^2/2) \diff w \lesssim \frac{1}{\alpha_t}\epsilon^{\frac{C_2C(n,d)^2}{2}}
\end{align*}
and 
\begin{align*}
    (\clubsuit) &\le \int_{\R \backslash B^{i}_{x,2}}\frac{1}{\sigma_t(2\pi)^{1/2}} \left(\frac{v_i}{e}\right)^{v_i} \exp(-C_2z_i^2/2) \diff z_i \lesssim \frac{1}{\sigma_t} \left(\log \epsilon^{-1}\right)^{1/2}\epsilon^{\frac{C_2C(n,d)^2}{2}},
\end{align*}
Taking minimum over the upper bounds above, we derive
\begin{align*}
    (\clubsuit) \lesssim C(n,d)\left(\log \epsilon^{-1}\right)^{1/2}\epsilon^{\frac{C_2C(n,d)^2}{2}}.
\end{align*}
Adding up $(\spadesuit)$ and $(\clubsuit)$, we obtain
    \begin{align*}
        &A_i \leq (\spadesuit)+(\clubsuit)\lesssim C(n,d)^{v_i+2}\left(\log \epsilon^{-1}\right)^{(v_i+2)/2}\epsilon^{\frac{C(n,d)^2}{2}} +C(n,d)\left(\log \epsilon^{-1}\right)^{1/2}\epsilon^{\frac{C_2C(n,d)^2}{2}}.
    \end{align*}
Setting the constant $C(n,d)$ sufficiently large, we ensure that $A_i \le \epsilon$. The proof is complete by taking the product of $A_{i, j}$ and $A_{i}$.
\end{proof}

\subsubsection{Proof of Lemma \ref{lemma::density bound}}\label{sec::Proof of Lemma density bound}
\begin{proof}
For the upper bound of the diffused density function, we have
    \begin{align*}
        p_t(\xb |\yb)&=\frac{1}{\sigma_{t}^d(2\pi)^{d/2}}\int{p(\zb|\yb)\exp\left(-\frac{\norm{\xb-\alpha_t\zb}^2}{2\sigma_{t}^2}\right)\diff \zb}\\
        &\le \frac{C_1}{\sigma_{t}^d(2\pi)^{d/2}}\int{\exp(-C_2\norm{\zb}_2^2/2)\exp\left(-\frac{\norm{\xb-\alpha_t\zb}^2}{2\sigma_{t}^2}\right)\diff \zb}\\
        &=\frac{C_1}{\sigma_{t}^d(2\pi)^{d/2}}\exp\left(\frac{-C_2\norm{\xb}^2_2}{2(\alpha_t^2+C_2\sigma_t^2)}\right)\int{\exp\left(-\frac{\norm{\zb-\alpha_t\xb/(\alpha_t^2+C_2\sigma_t^2)}^2}{2\sigma_{t}^2/(\alpha_t^2+C_2\sigma_t^2)}\right)\diff \zb}\\
        &=\frac{C_1}{(\alpha_t^2+C_2\sigma_t^2)^{d/2}}\exp\left(\frac{-C_2\norm{\xb}^2_2}{2(\alpha_t^2+C_2\sigma_t^2)}\right),
    \end{align*}
where we invoke assumption \ref{assump:sub} to bound $p(\zb|\yb)$ in the first inequality. For the lower bound, we have
    \begin{align*}
        p_t(\xb |\yb)&=\frac{1}{\sigma_{t}^d(2\pi)^{d/2}}\int{p(\zb|\yb)\exp\left(-\frac{\norm{\xb-\alpha_t\zb}^2}{2\sigma_{t}^2}\right)\diff \zb}\\
        &\ge \frac{1}{\sigma_{t}^d(2\pi)^{d/2}}\int_{{\norm{\zb}_2\le R}}{p(\zb|\yb)\exp\left(-\frac{\norm{\xb-\alpha_t\zb}^2}{2\sigma_{t}^2}\right)\diff \zb}\\
        &\ge \frac{1}{\sigma_{t}^d(2\pi)^{d/2}}\int_{\norm{\zb}_2\le R}{p(\zb|\yb)\exp\left(-\frac{2(\norm{\xb}_2^2+\alpha_t^2R^2)}{2\sigma_{t}^2}\right)\diff \zb}\\
        &\ge \frac{1}{\sigma_{t}^d(2\pi)^{d/2}}\exp\left(-\frac{\norm{\xb}_2^2+\alpha_t^2R^2}{\sigma_{t}^2}\right) \Pr_{\zb \sim p}[\norm{\zb}_2\le R]\\
        &\ge \frac{C_4}{\sigma_{t}^d} \exp\left(-\frac{\norm{\xb}_2^2+1}{\sigma_{t}^2}\right) ,
    \end{align*}
where we take $R=1$ and $C_4=\Pr_{\zb \sim p}[\norm{\zb}\le 1]/(2\pi)^{d/2}$ in the last inequality. 

Now we consider bounding the gradient. By symmetry, we only need to bound the first element of $\nabla p_t(\xb |\yb)$, i.e. 
\[
\abs{\nabla p_t(\xb |\yb)_{1}}=\frac{1}{\sigma_{t}^d(2\pi)^{d/2}}\left|\int{\frac{x_1-\alpha_tz_1}{\sigma_{t}^2}p(\zb|\yb)\exp\left(-\frac{\norm{\xb-\alpha_t\zb}^2}{2\sigma_{t}^2}\right)\diff \zb}\right|.
\]
We have 
\begin{align}
      & \quad \abs{\nabla p_t(\xb |\yb)_{1}} \nonumber \\
      &\le \frac{C_1}{\sigma_{t}^d(2\pi)^{d/2}}\int{\left|\frac{x_1-\alpha_tz_1}{\sigma_{t}^2}\right|\exp(-C_2\norm{\zb}_2^2/2)\exp\left(-\frac{\norm{\xb-\alpha_t\zb}^2}{2\sigma_{t}^2}\right)\diff \zb}\nonumber \\
        &=\frac{C_1}{\sigma_{t}^d(2\pi)^{d/2}}\exp\left(\frac{-C_2\norm{\xb}^2_2}{2(\alpha_t^2+C_2\sigma_t^2)}\right)\underbrace{\int{\left|\frac{x_1-\alpha_tz_1}{\sigma_{t}^2}\right|\exp\left(-\frac{\norm{\zb-\alpha_t\xb/(\alpha_t^2+C_2\sigma_t^2)}^2}{2\sigma_{t}^2/(\alpha_t^2+C_2\sigma_t^2)}\right)\diff \zb}}_{D}. \label{equ::D bound}
\end{align}
For term $D$, we have
\begin{align*}
    D&\le\int{ \frac{\alpha_t}{\sigma_t\sqrt{\alpha^2_t+C_2\sigma^2_t}} \abs{\frac{z_1-\alpha_t x_1/(\alpha_t^2+C_2\sigma_t^2)}{\sigma_t /\sqrt{\alpha^2_t+C_2\sigma^2_t}}}     \exp\left(-\frac{\norm{\zb-\alpha_t\xb/(\alpha_t^2+C_2\sigma_t^2)}^2}{2\sigma_{t}^2/(\alpha_t^2+C_2\sigma_t^2)}\right)\diff \zb} \\
    &\quad +\int{ \frac{C_2\sigma^2_t \left|x_1\right|}{\sigma^2_t(\alpha^2_t+C_2\sigma^2_t)}    \exp\left(-\frac{\norm{\zb-\alpha_t\xb/(\alpha_t^2+C_2\sigma_t^2)}^2}{2\sigma_{t}^2/(\alpha_t^2+C_2\sigma_t^2)}\right)\diff \zb}\\
    &= \frac{(2\pi)^{d/2}\sigma^d_t}{(\alpha_t^2+C_2\sigma_t^2)^{d/2}} \left(\frac{\alpha_t}{\sigma_t\sqrt{\alpha^2_t+C_2\sigma^2_t}}\int \frac{1}{\sqrt{2\pi}}\abs{w}\exp(-w^2/2)\diff w+\frac{C_2\abs{x_1}}{\alpha^2_t+C_2\sigma^2_t} \right)\\
    &\le \frac{(2\pi)^{d/2}\sigma^d_t}{(\alpha_t^2+C_2\sigma_t^2)^{d/2}} \left(\frac{\alpha_t}{\sigma_t\sqrt{\alpha^2_t+C_2\sigma^2_t}}+\frac{C_2\abs{x_1}}{\alpha^2_t+C_2\sigma^2_t} \right),
\end{align*}
where the second inequality follows from the fact that $\int \abs{w}\exp(-w^2/2)\diff w\le \sqrt{2\pi}$ . Plugging this result into \eqref{equ::D bound} gives rise to
\begin{equation}
    \abs{\nabla p_t(\xb |\yb)_{1}}\le \frac{C_1}{(\alpha_t^2+C_2\sigma_t^2)^{d/2}}\exp\left(\frac{-C_2\norm{\xb}^2_2}{2(\alpha_t^2+C_2\sigma_t^2)}\right) \left(\frac{\alpha_t}{\sigma_t\sqrt{\alpha^2_t+C_2\sigma^2_t}}+\frac{C_2\abs{x_1}}{\alpha^2_t+C_2\sigma^2_t} \right).
\end{equation}
Thus, by repeating this proof to each element of $\nabla p_t(\xb,\yb)$, we have
\begin{align*}
    \norm{\nabla p_t(\xb,\yb)}_{\infty} \le \frac{C_1}{(\alpha_t^2+C_2\sigma_t^2)^{d/2}}\exp\left(\frac{-C_2\norm{\xb}^2_2}{2(\alpha_t^2+C_2\sigma_t^2)}\right) \left(\frac{\alpha_t}{\sigma_t\sqrt{\alpha^2_t+C_2\sigma^2_t}}+\frac{C_2\norm{\xb}_{\infty}}{\alpha^2_t+C_2\sigma^2_t} \right).
\end{align*}
We complete our proof.
\end{proof}
\subsubsection{Proof of Lemma \ref{lemma::score bound}}\label{sec::Proof of Lemma score bound}
\begin{proof}
By symmetry, we only consider the first element of $\nabla \log p_t(\xb |\yb)$, i.e.,
\[
\nabla \log p_t(\xb |\yb)_{1}=\dfrac{-\int_{\R^d}{\frac{x_1-\alpha_t z_1}{\sigma_{t}^2}p(\zb|\yb)\exp\left(-\frac{\norm{\xb-\alpha_t\zb}^2}{2\sigma_{t}^2}\right)\diff z}}{\int_{\R^d}{p(\zb|\yb)\exp(-\frac{\norm{\xb-\alpha_t \zb}^2}{2\sigma_{t}^2})d z}}.
\] 
Invoking Lemma \ref{lemma::clipz} with $n=1$, we know for any $0<\epsilon<1/e$ to be chosen later, we can clip the integral in the denominator to a bounded region so that
\begin{align}
    &\int_{\R^d \backslash \textbf{B}_{x}}{\left|\frac{x_1-\alpha_t z_1}{\sigma_{t}^2}\right|p(\zb|\yb)\exp\left(-\frac{\norm{\xb-\alpha_t\zb}^2}{2\sigma_{t}^2}\right)\diff z} \le \epsilon, \label{equ::clip score}\\
   & \int_{\R^d \backslash \textbf{B}_{x}}{p(\zb|\yb)\exp\left(-\frac{\norm{\xb-\alpha_t\zb}^2}{2\sigma_{t}^2}\right)\diff z} \le \epsilon. \label{equ::clip score1}
\end{align}
where $\textbf{B}_{x}$ is defined in \eqref{equ::definition clip}.
Suppose that $2\epsilon<p_t(\xb|\yb)$, then by \eqref{equ::clip score1} we know that 
\begin{align}\label{equ:: pt upper trunc}
    \int_{\textbf{B}_x}{p(\zb|\yb)\exp\left(-\frac{\norm{\xb-\alpha_t\zb}^2}{2\sigma_{t}^2}\right)\diff z} > \epsilon.
\end{align}
Combining \eqref{equ::clip score} and \eqref{equ:: pt upper trunc}, we have
\begin{align}\label{equ:: trunc B density}
    \frac{1}{p_t(\xb|\yb)}<\frac{2}{\int_{\textbf{B}_x}{p(\zb|\yb)\exp\left(-\frac{\norm{\xb-\alpha_t\zb}^2}{2\sigma_{t}^2}\right)\diff z}}.
\end{align}
Thus, we have
\begin{align*}
    \abs{\nabla \log p_t(\xb |\yb)_{1}} &\le  \abs{\frac{2\int_{\R^d}{\frac{x_1-\alpha_tz_1}{\sigma_{t}^2}p(\zb|\yb)\exp\left(-\frac{\norm{\xb-\alpha_t\zb}^2}{2\sigma_{t}^2}\right)\diff z}}{\int_{\textbf{B}_x}{p(\zb|\yb)\exp\left(-\frac{\norm{\xb-\alpha_t\zb}^2}{2\sigma_{t}^2}\right)\diff z}}}\\
    &\le 2\abs{\frac{\int_{\textbf{B}_x}{\frac{x_1-\alpha_tz_1}{\sigma_{t}^2}p(\zb|\yb)\exp\left(-\frac{\norm{\xb-\alpha_t\zb}^2}{2\sigma_{t}^2}\right)\diff z}}{\int_{\textbf{B}_x}{p(\zb|\yb)\exp\left(-\frac{\norm{\xb-\alpha_t\zb}^2}{2\sigma_{t}^2}\right)\diff z}}}+\frac{\epsilon}{\epsilon}\\
&\le\frac{2C(0,d)\sqrt{\log\epsilon^{-1}}}{\sigma_t}+1.
\end{align*}
 According to Lemma \ref{lemma::density bound}, we can set $\epsilon$ to be $\epsilon=\min \left(\frac{C_4}{2\sigma_{t}^d} \exp\left(-\frac{\norm{\xb}_2^2+1}{\sigma_{t}^2}\right),\frac{1}{e}\right)$. Thus, there exists a constant $C_5$ dependent on $d$ and $C_4$ such that $\abs{\nabla \log p_t(\xb|\yb)_i}\le \frac{C_5}{\sigma_t^2}(\norm{\xb}_2+1)$. By repeating the proof to each element of $\nabla \log p_t(\xb|\yb)$, we complete our proof.
\end{proof}

\subsubsection{Proof of Lemma \ref{lemma::approx h}}\label{sec::Proof of lemma::approx h}
\begin{proof}
    By the definition of $g(x,n,v,k)$ (see \ref{equ::h}), we have
    \begin{align*}
       g(x,n,v,k) &=\frac{1}{\sigma_{t}(2\pi)^{1/2}}\int \left(\frac{z}{R}+1/2-\frac{v}{N}\right)^{n}\frac{1}{k!}\left(-\frac{\abs{x-\alpha_tz}^2}{2\sigma_{t}^2}\right)^k\diff z\\
        &=\frac{1}{\alpha_t(2\pi)^{1/2}}\int\left(\frac{x-\sigma_tw}{\alpha_tR}+1/2-\frac{v}{N}\right)^n\frac{1}{k!}\left(-\frac{w^2}{2}\right)^k\diff w\\
        &=\frac{1}{\alpha_t(2\pi)^{1/2}}\sum_{j=0}^{n}C_n^j\left(\frac{x}{\alpha_tR}+1/2-\frac{v}{N}\right)^{n-j}\int\left(\frac{-\sigma_tw}{\alpha_tR}\right)^j\frac{1}{k!}\left(-\frac{w^2}{2}\right)^k\diff w\\
        &=\frac{1}{\alpha_t(2\pi)^{1/2}}\sum_{j=0}^{n}C_n^j\frac{(x+\alpha_tR/2-\alpha_tRv/N)^{n-j}\sigma_t^{j}}{\alpha_t^nR^n}\int (-w)^j\frac{1}{k!}\left(-\frac{w^2}{2}\right)^k\diff w\\
        &=\frac{1}{\alpha_t(2\pi)^{1/2}}\sum_{j=0}^{n}C_n^j\frac{(x+\alpha_tR/2-\alpha_tRv/N)^{n-j}\sigma_t^{j}}{\alpha_t^nR^n}\frac{(-1)^{j}}{(-2)^kk!}\frac{f_{\overline{D}}^{j+2k+1}(x)-f_{\underline{D}}^{j+2k+1}(x)}{j+2k+1}
    \end{align*}
    We point out that integral is taken in $\left[f_{\underline{D}}(x),f_{\overline{D}}(x)\right]$, 
 where 
 \begin{equation*}
     f_{\underline{D}}(x)=\clip{\frac{x-\alpha_tC(0,d)\left(\frac{2v}{N}-1\right)\sqrt{\beta\log N}}{\sigma_t},C(0,d)\sqrt{\beta\log N}}
 \end{equation*}
and
 \begin{equation*}
     f_{\overline{D}}(x)=\clip{\frac{x-\alpha_tC(0,d)\left(\frac{2(v-1)}{N}-1\right)\sqrt{\beta\log N}}{\sigma_t},C(0,d)\sqrt{\beta\log N}}.
 \end{equation*}
 Here the clip function represents $\text{clip}(x,B)=\max(\min(x,B),-B)$ for $B\ge 0$. 
Thus, we only need to approximate the function in the form of:
\begin{align}\label{equ::single component}
    f_{v,k,j}=\frac{(x+\alpha_tR/2-\alpha_tRv/N)^{n-j}(f_{\overline{D}}^{j+2k+1}(x)-f_{\underline{D}}^{j+2k+1}(x))\sigma_t^{j}}{\alpha_t^{n+1}}.
\end{align}
We construct a ReLU network to approximate the functions above step by step (see Figure \ref{fig:ReLU1}). By appropriately setting the parameters for each ReLU approximation function (details about how to determine the network size and the error propagation are deferred to Appendix \ref{sec::fvkj relu}), we derive a ReLU network $f^{\trelu}_{v,k,j}\in \cF( W, \kappa, L, K) $ with 
\begin{align*}
    & W = {\cO}\left(\log^6 N +\log^3 \epsilon^{-1} \right),
      \kappa =\exp \left({\cO}(\log^4 N+\log^2 \epsilon^{-1})\right),\\
      ~&L = {\cO}(\log^4 N+\log^2 \epsilon^{-1}),~
     K= {\cO}\left(\log^8 N+\log^4 \epsilon^{-1}\right).
\end{align*}
such that for any $x \in  [-C_x\sqrt{\log N},C_x\sqrt{\log N}]$ and $t\in[N^{-C_{\sigma}},C_{\alpha}\log N]$,
\begin{align}\label{equ::fvkj relu}
     \abs{f_{v,k,j}(x,t)-f_{v,k,j}^{\trelu}(x,t)}\le \epsilon .
\end{align}
\begin{figure}
\centering
\includegraphics[width=0.9\linewidth]{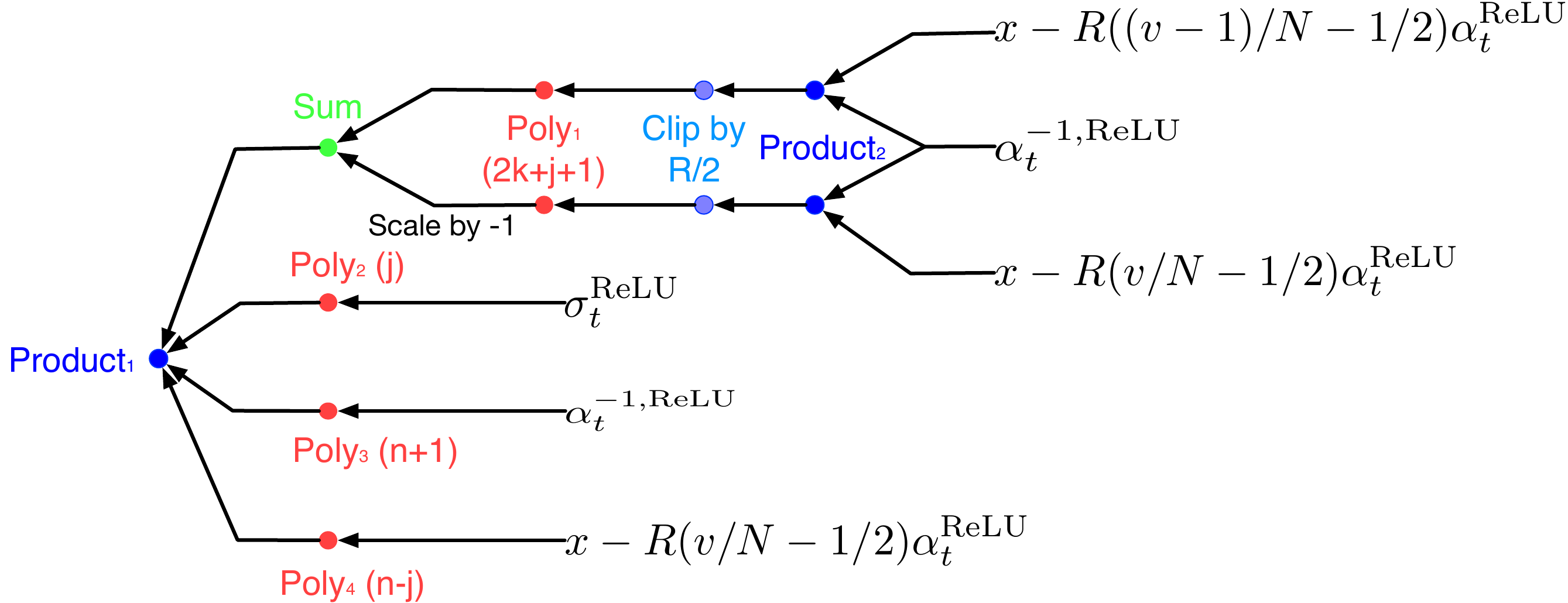}
\caption{Network architecture of $f^{\trelu}_{v,k,j}$. We implement all the basic functions (e.g., $x$, $\alpha_t$ and $\sigma_t$) through ReLU networks and combine them using the ReLU-expressed operators (\textit{product}, \textit{inverse}, \textit{clip} and \textit{poly}) to express $f_{v,k,j}$ according to its definition in \eqref{equ::single component}.}
\label{fig:ReLU1}
\end{figure}
Now we consider the following ReLU network
\begin{align}\label{equ::h relu}
    h^{\trelu}(x,n,v,k)=\frac{(-1)^{j}}{(2\pi)^{1/2}R^n(-2)^kk!}\sum_{j=0}^n \frac{C_n^j}{j+2k+1}f^{\trelu}_{v,k,j}.
\end{align}
By \eqref{equ::fvkj relu}, we have 
\begin{equation*}
    \left|h^{\trelu}(x,n,v,k)-g(x,n,v,k)\right| \le \frac{2^n\epsilon}{(2\pi)^{1/2}R^n 2^kk!} \le \epsilon.
\end{equation*}
We complete our ReLU approximation for $g(x,n,v,k)$.
\end{proof}

\section{Proof of Theorem \ref{thm::score approx exp}}\label{sec::all proof score sub exp}

\subsection{Key Steps for Proving Theorem \ref{thm::score approx exp}}
To prove Theorem \ref{thm::score approx exp} we mainly follow the proof of Theorem \ref{thm::score approx}. We also replace $N$ by $N^{d+d_y}$ for simplicity. The difference between the proofs is that the stronger assumption allows us to extract a Gaussian distribution from $p_t$. We provide an overview of our proof.
\begin{description}
    \item[Step 1] Under Assumption \ref{assump::expdensity}, we can decompose the score function into a linear function of $\xb$ and a diffused score function $\nabla \log h=\frac{\nabla h}{h}$, where $h$ is a convolution of $f$ and a Gaussian kernel (Lemma \ref{lemma::score decompose exp}).
    \item [Step 2] We truncate the domain of input $\xb$ as we do in Theorem \ref{thm::score approx} (Lemma \ref{lemma::truncation x exp}).
    \item[Step 3] We use the ReLU network to approximate $h(\xb,\yb,t)$ and $\nabla h(\xb,\yb,t)$ in the truncated domain and combine the results to construct a score approximator $ \sb(\xb,\yb,t)$ with small approximation error (Proposition \ref{prop::logh approx bounded exp}).
\end{description}
\subsection{Statements of Steps 1 - 3 and Using Them to Prove Theorem \ref{thm::score approx exp}}
\subsubsection{Formal Statements in Steps 1 - 3}
Firstly, under Assumption \ref{assump::expdensity}, we can decompose the score function as shown in the following lemma:
\begin{lemma}[Decomposing the score]\label{lemma::score decompose exp}
    Under Assumption \ref{assump::expdensity}, the score function can be written as
    \begin{align*}
        \nabla \log p_t(\xb |\yb)&
    =\frac{-C_2\xb}{\alpha_t^2+C_2\sigma_t^2} +\frac{\hat{\alpha}_t}{\hat{\sigma}_t}\cdot\frac{\int f(\zb,\yb)\left( \frac{\zb-\hat{\alpha}_{t}\xb}{\hat{\sigma}_t}\right) \exp\left(-\frac{\norm{\zb-\hat{\alpha}_t\xb}^2}{2\hat{\sigma}_{t}^2}\right)\diff \zb}{\int f(\zb,\yb) \exp\left(-\frac{\norm{\zb-\hat{\alpha}_t\xb}^2}{2\hat{\sigma}_{t}^2}\right)\diff \zb}\\
        &=\frac{-C_2\xb}{\alpha_t^2+C_2\sigma_t^2}+\frac{\nabla h(\xb,\yb,t)}{h(\xb,\yb,t)},
    \end{align*}
    where $\hat{\sigma}_t=\frac{\sigma_t}{(\alpha_t^2+C_2\sigma_t^2)^{1/2}}$, $\hat{\alpha}_t=\frac{\alpha_t}{\alpha_t^2+C_2\sigma_t^2}$ and $h(\xb,\yb,t)=\int f(\zb,\yb) \frac{1}{(2\pi)^{d/2}\hat{\sigma}_t^d}\exp\left(-\frac{\norm{\zb-\hat{\alpha}_t\xb}^2}{2\hat{\sigma}_{t}^2}\right)\diff \zb$.
\end{lemma}
Due to the smoothness and boundedness of $f(z,y)$, we can verify that $h(\xb,\yb,t)$ is both lower and upper bounded by some constants, which is a helpful property in approximating $\frac{\nabla h(\xb,\yb,t)}{h(\xb,\yb,t)}$. Following the proof of Theorem \ref{thm::score approx}, we also truncate the domain of $\xb$ on a bounded space $\set{\xb:\norm{\xb}_2 \le R}$. The proof of  Lemma \ref{lemma::score decompose exp} is provided in Appendix~\ref{sec::Proof of Lemma score decompose exp}.

\begin{lemma}[Truncation on $\xb$]\label{lemma::truncation x exp}
    Under Assumption \ref{assump::expdensity}, for any $R>1$, we have
    \begin{align}
             &\int_{\norm{\xb}_{\infty}\ge R}  p_{t}(\xb|\yb)\diff \xb \lesssim R\exp(-C_2'R^2),\\
            &\int_{\norm{\xb}_{\infty}\ge R} \norm{\nabla\log p_{t}(\xb|\yb)}_2^2 p_{t}(\xb|\yb)\diff \xb \lesssim \frac{1}{\sigma_t^2}R^3\exp(-C_2'R^2).
    \end{align}
    where $C_2^{\prime}=\frac{C_2}{2\max(1,C_2)}$.
\end{lemma}
This lemma is a counterpart of Lemma \ref{lemma::truncation x}, and the proof is provided in Appendix \ref{sec::proof of lemma truncation x exp}.
Note that the truncation error scales with $\frac{1}{\sigma_t^2}$ instead of $\frac{1}{\sigma_t^4}$ as we obtained in Lemma \ref{lemma::truncation x}, which results from a tighter bound of the score function $\nabla \log p_t(\xb|\yb)$.

\begin{proposition}[Approximate the score]\label{prop::logh approx bounded exp}
    For sufficiently large integer $N$, there exists a ReLU network $ \cF( W, \kappa, L, K) $ that gives rise to a mapping $\sb(\xb,\yb,t) \in \cF$ satisfying 
    \begin{align}\label{equ::logh linfty bound exp}
        \norm{\sb(\xb,\yb,t)-\nabla\log h(\xb,\yb,t)+\frac{C_2\xb}{\alpha_t^2+C_2\sigma_t^2}}_{\infty} \lesssim  \frac{B}{\sigma_t}N^{-\beta}(\log N)^{\frac{s+1}{2}},
    \end{align}
    for any $\xb \in  [-C_x\sqrt{\log N},C_x\sqrt{\log N}]^{d}$, $\yb\in[0,1]^{d_y}$ and $t\in[N^{-C_{\sigma}},C_{\alpha}\log N]$. The network hyperparameter configuration satisfies
\begin{align}
    & \hspace{0.4in} M_t = \cO\left(\sqrt{\log N}/\sigma_t\right),~
     W = {\cO}\left(N^{d+d_y}\log^7 N\right),\label{equ::hyper1 exp}\\
     & \kappa =\exp \left({\cO}(\log^4 N)\right),~
    L = {\cO}(\log^4 N),~
     K= {\cO}\left(N^{d+d_y}\log^9 N\right). \label{equ::hyper2 exp}
\end{align}
\end{proposition} The proof of Proposition \ref{prop::logh approx bounded exp} is provided in \ref{sec::prop logh proof}. Now we are ready to prove Theorem~\ref{thm::score approx exp}.
\subsubsection{Proof of Theorem \ref{thm::score approx exp}}
\begin{proof}
    With the lemmas and the proposition above, the proof is quite straightforward. We take $C_x=\sqrt{\frac{2\beta}{C_2^{\prime}}}$ in Proposition \ref{prop::logh approx bounded exp} to obtain a ReLU score estimator $\sb$. According to the hyperparameter configuration \eqref{equ::hyper1 exp}, we have $\norm{\sb(\xb,\yb,t)}_2\lesssim\frac{\sqrt{\log N}}{\sigma_t}$ for any $\xb \in \RR^{d}$, $\yb\in \RR^{d_y}$ and $t>0$.
    Besides, we set the truncation radius $R=C_x\sqrt{\log N}$. By Lemma \ref{lemma::truncation x exp}, 
    \begin{align*}
        &\int_{\R^{d}} \norm{\sb-\nabla\log p_{t}(\xb|\yb)}_2^2 p_{t}(\xb|\yb)\diff \xb\\
        &\lesssim \int_{\norm{\xb}_{\infty}>\sqrt{\frac{2\beta}{C_2^{\prime}}\log N}} \left(2\left(\frac{1}{\sigma_t}\sqrt{\log N}\right)^2 +2\norm{\nabla\log p_{t}(\xb|\yb)}_2^2 \right)p_t(\xb|\yb) \diff \xb\\ 
        &+\int_{\norm{\xb}_{\infty}\le \sqrt{\frac{2\beta}{C_2^{\prime}}\log N}}\norm{\sb(\xb,\yb,t)-\nabla\log p_{t}(\xb|\yb)}_2^2 p_{t}(\xb|\yb)\diff \xb\\
        & \overset{(i)}{\lesssim} \frac{2d \log N}{\sigma^2_t}\left(\frac{2\beta}{C_2^{\prime}}\log N\right)^{1/2}N^{-2\beta} + \frac{2}{\sigma_t^2}\left(\frac{2\beta}{C_2^{\prime}}\log N\right)^{3/2}N^{-2\beta}+  \frac{B^2}{\sigma^2_t}N^{-2\beta}\log ^{s+1} N\\
        &\lesssim \frac{B^2}{\sigma^2_t}N^{-2\beta}\log ^{s+1} N.
    \end{align*}
    In (i) we invoke the truncation error bound in Lemma \ref{lemma::truncation x exp} and the approximation error bound in Proposition \ref{prop::logh approx bounded exp}. By turning $N$ back to $N^{\frac{1}{d+d_y}}$, the proof is complete.
\end{proof}

\subsection{Proofs in Step 3 for Theorem~\ref{thm::score approx exp}}\label{sec::prop logh proof}

Similarly to the proof of Proposition~\ref{prop::score approx bounded}, the approximation process is also divided into two stages. In the first stage, we approximate $h$ and $\nabla h$ up to a small error separately in the same way. In approximating $h(\xb,\yb,t)$, we first use another set of diffused local monomials (see \ref{equ::diffused monomial restate exp}) to approximate $p(\xb|\yb)$, which is presented in the following lemma.

\begin{lemma}[Diffused local polynomial approximation]\label{lemma::diffused localpoly approx exp}
    Under Assumption \ref{assump::expdensity}, for sufficiently large integer $N>0$ and constant $C_x>0$, there exists a diffused local polynomial with at most $N^{d+d_y}(d+d_y)^s$ diffused local monomials $f_1(\xb,\yb,t)$ such that 
\begin{align}
     \abs{f_1(\xb,\yb,t)-h(\xb,\yb,t)} \lesssim BN^{-\beta}\log^{\frac{s}{2}}N,
\end{align}
for any $\xb \in [-C_x\sqrt{\log N},C_x\sqrt{\log N}]^{d}$, $\yb \in [0,1]^{d_y}$, and $t>0$.
\end{lemma}
The proof of Lemma \ref{lemma::diffused localpoly approx exp} is provided in Appendix  \ref{sec::proof of lemma::diffused localpoly approx exp}. We remark that the polynomial dependence on $\log N$ is smaller than the one we obtain in Lemma \ref{lemma::diffused localpoly approx}. In the following lemma,  we present the second stage of our approximation process in which we construct a ReLU network to approximate this diffused local polynomial with a sufficiently small error.
\begin{lemma}[ReLU approximation]\label{lemma::relu approx diffused exp}
    Under Assumption \ref{assump::expdensity}, given the diffused local polynomial $f_1$ in Lemma \ref{lemma::diffused localpoly approx exp}, for any $\epsilon >0$, there exists a ReLU network $ \cF( W, \kappa, L, K) $ that gives rise to a function $f_1^{\trelu}(\xb,\yb,t) \in \cF$ satisfying 
\begin{align}
     \abs{f_1(\xb,\yb,t)-f_1^{\trelu}(\xb,\yb,t)}\lesssim \epsilon,
\end{align}
for any $\xb \in  [-C_x\sqrt{\log N},C_x\sqrt{\log N}]^{d}$, $\yb\in[0,1]^{d_y}$ and $t\in[N^{-C_{\sigma}},C_{\alpha}\log N]$.
The network configuration is
\begin{align*}
& W = {\cO}\left(N^{d+d_y}(\log^7 N+\log N\log^3 \epsilon^{-1})\right), \quad \kappa =\exp \left({\cO}(\log^4 N+\log^2 \epsilon^{-1})\right), \\
& \qquad L = {\cO}(\log^4 N+\log^2 \epsilon^{-1}), \quad K = {\cO}\left(N^{d+d_y}(\log^9 N+\log N\log^3 \epsilon^{-1})\right).
\end{align*}

\end{lemma}
The proof of Lemma \ref{lemma::relu approx diffused exp} is provided in Appendix \ref{sec::sub relu approx exp}. Moreover, we have similar results for approximating $\nabla h(\xb,\yb,t)$:
\begin{lemma}[Counterpart of Lemma \ref{lemma::diffused localpoly approx exp}]\label{lemma::diffused localpoly approx1 exp}
    Under Assumption \ref{assump::expdensity}, for sufficiently large integer $N>0$ and $1\le i\le d$, there exists a diffused local polynomial with at most $N^{d+d_y}(d+d_y)^s$ diffused local monomials $f_{2,i}(\xb,\yb,t)$ such that 
\begin{align}
     \abs{f_{2,i}(\xb,\yb,t)-\left[\frac{\hat{\sigma}_t}{\hat{\alpha}_t}\nabla h(\xb,\yb,t)\right]_{i}} \lesssim BN^{-\beta}\log^{\frac{s+1}{2}}N , ~~~\forall \xb \in \RR^d, \yb \in [0,1]^{d_y}, t>0.
\end{align}
\end{lemma}
\begin{lemma}[Counterpart of Lemma \ref{lemma::relu approx diffused exp}]\label{lemma::relu approx diffused1 exp}
    Under Assumption \ref{assump::expdensity}, given the diffused local polynomial mapping $\mathbf{f}_2=[f_{2,1},f_{2,2},...,f_{2,d}]^{\top}$ in Lemma \ref{lemma::diffused localpoly approx1 exp}, for any $\epsilon >0$, there exists a ReLU network $ \cF( W, \kappa, L, K) $ that gives rise to a mapping $\mathbf{f}_2^{\trelu}(\xb,\yb,t) \in \cF$ satisfying 
\begin{align}
     \norm{\mathbf{f}_2(\xb,\yb,t)-\mathbf{f}_2^{\trelu}(\xb,\yb,t)}_{\infty}\lesssim \epsilon,
\end{align}
for any $\xb \in  [-C_x\sqrt{\log N},C_x\sqrt{\log N}]^{d}$, $\yb\in[0,1]^{d_y}$ and $t\in[N^{-C_{\sigma}},C_{\alpha}\log N]$.
The network configuration is the same as in Lemma \ref{lemma::relu approx diffused exp}.
\end{lemma}

With all the lemmas above, we can begin our proof of Proposition \ref{prop::logh approx bounded exp}.

\begin{proof}[Proof of Proposition \ref{prop::logh approx bounded exp}]
From the lemmas above, we obtain $ f_1(\xb,\yb,t), f^{\trelu}_1(\xb,\yb,t) $ to approximate $h(\xb,\yb,t)$, and $\mathbf{f}_2,\mathbf{f}^{\trelu}_2$ to approximate $\frac{\hat{\sigma}_t}{\hat{\alpha}_t} \nabla h(\xb,\yb)$.
    By symmetry, we only consider approximating the first element of $ \nabla h(\xb,\yb,t)$, which we denote by $\nabla h_1(\xb,\yb)$.
    For simplicity, we denote the first element of $\mathbf{f}_2$ by $f_2$. 
    
By the definition of $h$, for any $\xb \in [-C_x\sqrt{\log N}]$, $\yb \in[0,1]^{d_y}$ and $N^{-C_\sigma}t \le t \le C_\alpha \log N$, we have $ C_1\le h(\xb,\yb,t) \le B$ and $\norm{\frac{\hat{\sigma}_t}{\hat{\alpha}_t} \nabla h(\xb,\yb,t)}_\infty \le \sqrt{\frac{2}{\pi}}B$ (see Lemma \ref{lemma::bound h and nabla h}). Accordingly, we make $N$ sufficiently large so that $\frac{C_1}{2} \le f_1(\xb,\yb,t) \le 2B$ and $f_2 \le B$. Then we have 
\begin{align*}
    & \quad \abs{ \frac{\nabla h_1(\xb,\yb,t)}{h(\xb,\yb,t)}- \frac{\hat{\alpha}_t}{\hat{\sigma}_t} \frac{f_2(\xb,\yb,t)}{f_1(\xb,\yb,t)} } \\
    &\le \abs{ \frac{\nabla h_1(\xb,\yb,t)}{h(\xb,\yb,t)}-\frac{\nabla h_1(\xb,\yb,t)}{f_1(\xb,\yb,t)} }+ \abs{ \frac{\nabla h_1(\xb,\yb,t)}{f_1(\xb,\yb,t)} -\frac{\hat{\alpha}_t}{\hat{\sigma}_t} \frac{f_2(\xb,\yb,t)}{f_1(\xb,\yb,t)} }
    \\
    &\le \abs{\nabla h_1(\xb,\yb,t)}\abs{\frac{h(\xb,\yb,t)-f_1(\xb,\yb,t)}{h(\xb,\yb,t)f_1(\xb,\yb,t)}}+
    \frac{\hat{\alpha}_t}{\hat{\sigma}_t}\abs{\frac{f_2(\xb,\yb,t)-\frac{\hat{\sigma}_t}{\hat{\alpha}_t}\nabla h_1(\xb,\yb,t)}{f_1(\xb,\yb,t)}}  \\ &\overset{(i)}{\le}
    \frac{\hat{\alpha}_t}{\hat{\sigma}_t} \left(\frac{2}{C^2_1} BN^{-\beta}\log^{\frac{s}{2}}N +\frac{2}{C_1} BN^{-\beta}\log^{\frac{s+1}{2}}N.\right)\\
    &\lesssim \frac{B}{\sigma_t}N^{-\beta}\log ^{\frac{s+1}{2}} N.
\end{align*}
In (i), we invoke the diffused polynomial approximation error bound in Lemmas \ref{lemma::diffused localpoly approx exp} and \ref{lemma::diffused localpoly approx1 exp} and the lower bound of $f_1$. Applying to other elements give rise to the bounded $L_{\infty}$ error:
\begin{align*}
    \left\|\frac{\nabla h(\xb,\yb,t)}{h(\xb,\yb,t)}- \frac{\hat{\alpha}_t}{\hat{\sigma}_t} \frac{\fb_2(\xb,\yb,t)}{f_1(\xb,\yb,t)}\right\|_{\infty} \lesssim \frac{B}{\sigma_t}N^{-\beta}\log ^{\frac{s+1}{2}} N.
\end{align*}

For the ReLU approximation, 
we use Lemmas \ref{lemma::relu approx diffused exp} and \ref{lemma::relu approx diffused1 exp} to construct a ReLU network to approximate 
\begin{align}\label{equ::fb3 new definition}
    \fb_3(\xb,\yb,t):=\frac{\hat{\alpha}_t}{\hat{\sigma}_t} \frac{\fb_2}{f_1}-\frac{C_2\xb}{\alpha_t^2+C_2\sigma_t^2}.
\end{align}
Our constructed network architecture is depicted in Figure \ref{fig:ReLU4}, and the details about how to determine the network size and the error propagation are presented in Appendix \ref{sec::f3 new relu}.
\begin{figure}
\centering\includegraphics[width=0.6\linewidth]{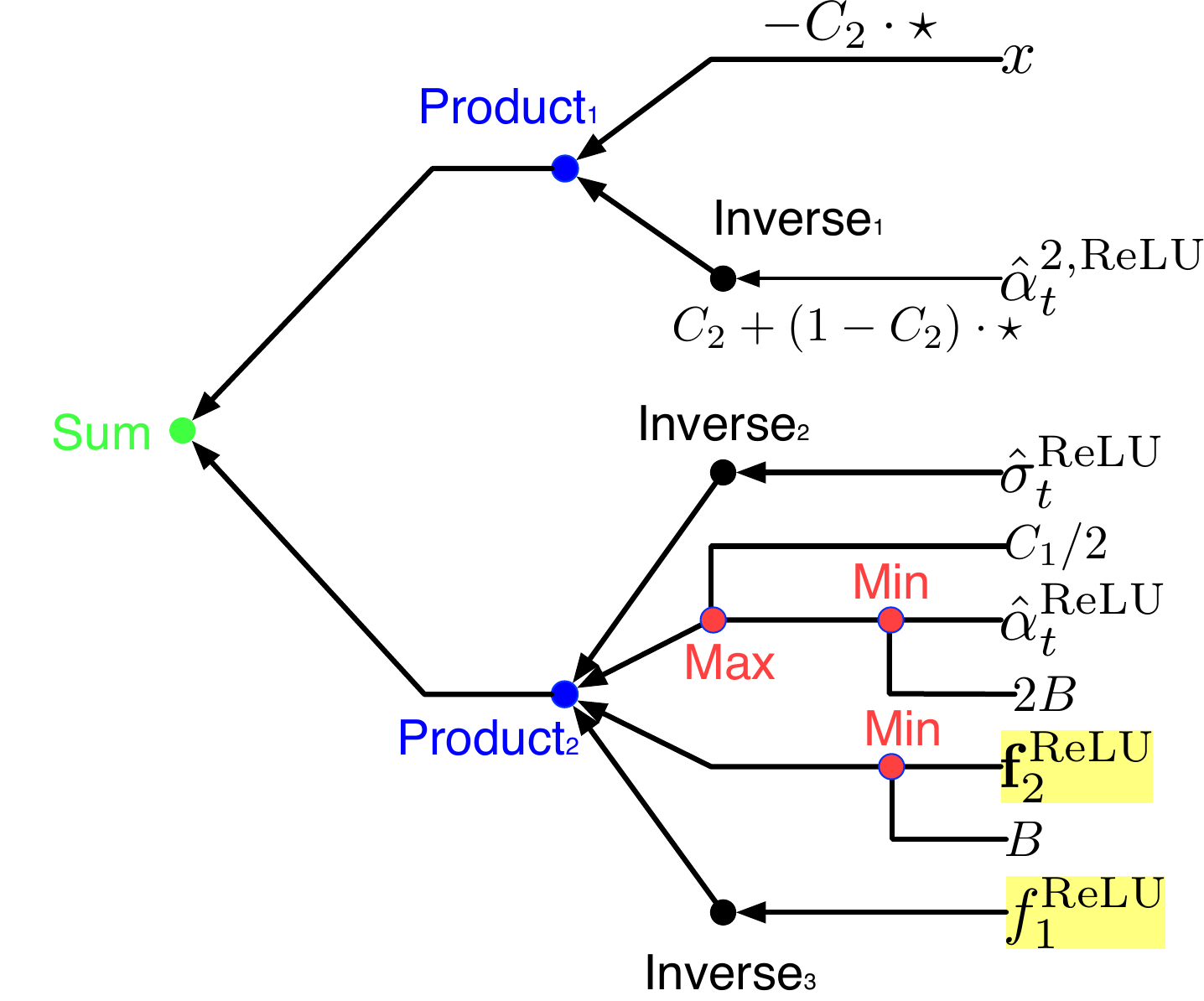}
\caption{Network architecture of $\fb^{\trelu}_3$. We implement all the components of $\fb_3$ ($f_1$, $\fb_2$, $\hat{\sigma}_t$ and $\hat{\alpha}_t$) through ReLU networks and combine them using the ReLU-expressed operators (\textit{product}, \textit{inverse} and \textit{entrywise-min/max}) to express $\fb_3$ according to its definition in \eqref{equ::fb3 new definition}.}
\label{fig:ReLU4}
\end{figure}
From the construction above we obtain a ReLU network $\fb^{\trelu}_{3} \in \cF(M_t, W, \kappa, L, K) $ with $M_t\lesssim \frac{\sqrt{\log N}}{\sigma_t} $, $ L = {\cO}(\log^4 N) , W= {\cO}\left(N^{d+d_y}(\log^7 N)\right)$, $ K= {\cO}\left(N^{d+d_y}(\log^9 N)\right)$ and $\kappa =\exp \left({\cO}(\log^4 N)\right)$ such that 
$$
\left\|\fb^{\trelu}_{3}(\xb,\yb,t)-\frac{\hat{\alpha}_t}{\hat{\sigma}_t} \frac{\fb_2(\xb,\yb,t)}{f_1(\xb,\yb,t)}+\frac{C_2\xb}{\alpha_t^2+C_2\sigma_t^2}\right\|\le N^{-\beta}.
$$
Thus, we have
$$\left\|\fb^{\trelu}_{3}-\nabla\log p_{t}(\xb|
\yb)\right\| \lesssim \frac{B}{\sigma_t}N^{-\beta}\log ^{\frac{s+1}{2}} N+N^{-\beta}\lesssim \frac{B}{\sigma_t}N^{-\beta}\log ^{\frac{s+1}{2}} N$$
for any $\xb \in [-C_x\sqrt{\log N},C_x\sqrt{\log N}]^d $, $\yb \in [0,1]^{d_y}$ and $ N^{-C_\sigma} \le t \le C_{\alpha}\log N$. By taking $\sb=\fb^{\trelu}_{3}$, the proof is complete.
\end{proof}
\subsection{Proofs of Lemmas \ref{lemma::diffused localpoly approx exp} and \ref{lemma::diffused localpoly approx1 exp}}\label{sec::proof of lemma::diffused localpoly approx exp}
To prove the lemma, we first show some properties of the $h(\xb,\yb,t)$ and $\nabla \log h(\xb,\yb,t)$.
\begin{lemma}\label{lemma::bound h and nabla h}
Under Assumption \ref{assump::expdensity}, $h(\xb,\yb,t)$ and $\frac{\hat{\sigma}_t}{\hat{\alpha}_t} \nabla h(\xb,\yb,t)$ can be bounded as:
\begin{align}
 C_1\le h(\xb,\yb,t) \le B, ~~~ \norm{\frac{\hat{\sigma}_t}{\hat{\alpha}_t} \nabla h(\xb,\yb,t)}_\infty \le \sqrt{\frac{2}{\pi}}B
\end{align}
\end{lemma}

\begin{lemma}\label{lemma::scorebound2}
Under Assumption \ref{assump::expdensity}, the diffused density function $p_t(\xb |\yb)$ can be bounded as:
\begin{align}
\norm{\nabla \log p_t(\xb |\yb)}_{\infty} &\le \frac{C_2\norm{\xb}_{\infty}}{\alpha_t^2+C_2\sigma_t^2}+\frac{B}{C_1}\frac{\hat{\alpha}_t}{\hat{\sigma}_t}
\lesssim \norm{\xb}_{\infty}+\frac{1}{\sigma_t}
\end{align}
\end{lemma}
\begin{lemma}[Clip the integral]\label{lemma::clipz2}
    Under Assumption \ref{assump::expdensity}, for any $\xb \in \R^d$ and $\mathbf{v}\in \Z_{+}^d$ with $\norm{\mathbf{v}}_1 \le n$. There exists a constant $C(n,d)$ such that for any $\xb$ and  $\epsilon <0.99$ :
    \begin{align*}
        \bigg|\int_{\R^d}\left(\frac{\zb-\hat{\alpha}_t\xb}{\hat{\sigma}_t}\right)^\vb& f(\zb,\yb) \frac{1}{(2\pi)^{d/2}\hat{\sigma}_t^d}\exp\left(-\frac{\norm{\zb-\hat{\alpha}_t\xb}^2}{2\hat{\sigma}_{t}^2}\right)\diff \zb\\
        &-\int_{\textbf{B}_x}\left(\frac{\zb-\hat{\alpha}_t\xb}{\hat{\sigma}_t}\right)^\vb f(\zb,\yb) \frac{1}{(2\pi)^{d/2}\hat{\sigma}_t^d}\exp\left(-\frac{\norm{\zb-\hat{\alpha}_t\xb}^2}{2\hat{\sigma}_{t}^2}\right)\diff \zb\bigg| \le \epsilon,
    \end{align*}
    where $\textbf{B}_x=\left[\hat{\alpha}_t\xb-C(n,d)\hat{\sigma}_t\sqrt{\log \epsilon^{-1}},\hat{\alpha}_t\xb+C(n,d)\hat{\sigma}_t\sqrt{\log \epsilon^{-1}} \right] $.
    
\end{lemma}
The proofs of the lemmas above are provided in Appendix \ref{sec::proof other lemmas B}. With all the previous lemmas, we begin to prove Lemma \ref{lemma::diffused localpoly approx exp}.
\begin{proof}[Proof of Lemma \ref{lemma::diffused localpoly approx exp}]
    We prove the lemma following the proof of Lemma \ref{lemma::diffused localpoly approx}. Recall that the integral form of  $h(\xb,\yb,t)$ is
    \begin{align}\label{equ::restate h integral}
h(\xb,\yb,t)=\underbrace{\int_{\RR^{d}}}_{\textbf{Step (i)}} \underbrace{f(\zb,\yb)}_{\textbf{Step (ii)}} \frac{1}{(2\pi)^{d/2}\hat{\sigma}_t^d}\underbrace{\exp\left(-\frac{\norm{\zb-\hat{\alpha}_t\xb}^2}{2\hat{\sigma}_{t}^2}\right)}_{\textbf{Step (iii)}}\diff \zb.
    \end{align}
    We will prove the lemma in the following steps.
\begin{description}
\item[Step (i)] (\textbf{Clip the domain}) We first truncate the integral of $h(\xb,\yb,t)$ in a bounded region using Lemma \ref{lemma::clipz2}.
\item[Step (ii)] (\textbf{Approximate} $f(\cdot)$) We approximate  $f(\zb,\yb)$ in the integrand using local polynomials.
\item[Step (iii)] (\textbf{Approximate} $\exp(\cdot)$) We approximate the exponential function in the integrand by polynomials using Taylor expansion.
\end{description}
Now we begin our formal proof.

\noindent \textbf{Step (i)}
For a sufficiently large positive integer $N>0$ and a constant $C_x>0$, we suppose $\xb \in [-C_x\sqrt{\log N},C_x\sqrt{\log N}]^{d}$. Similarly to the proof of Lemma \ref{lemma::diffused localpoly approx}, we first invoke Lemma \ref{lemma::clipz2} to approximate $h(\xb,\yb,t)$ by an integral on a bounded region, which we denote by
    \begin{align}\label{equ::f2 exp}
        f_2(\xb,\yb,t)=\int_{\textbf{B}_{\xb,N}}f(\zb,\yb)\frac{1}{(2\pi)^{d/2}\hat{\sigma}_t^d}\exp\left(-\frac{\norm{\zb-\hat{\alpha}_t\xb}^2}{2\hat{\sigma}_{t}^2}\right)\diff \zb.
    \end{align}
    We take $\textbf{B}_{\xb,N}=\left[\hat{\alpha}_t\xb-C^{\prime}(0,d)\hat{\sigma}_t\sqrt{\beta\log N},\hat{\alpha}_t\xb+C^{\prime}(0,d)\hat{\sigma}_t\sqrt{\beta\log N} \right]$ so that 
    \begin{align}\label{equ:: f2h}
        \left|f_2(\xb,\yb,t)-h(\xb,\yb,t) \right|\le N^{-\beta},~~~\forall \xb\in[-C_x\sqrt{\log N},C_x\sqrt{\log N}]^{d}, \yb\in[0,1]^{d_y}, t \ge 0.
    \end{align}

By the definition of $\hat{\alpha}_t$ and $\hat{\sigma}_t$, we have $\hat{\alpha}_t \le \max(1/C_2,1)$ and $\hat{\sigma}_t \le \max\left(\sqrt{1/C_2},1\right)$. Thus, if we take $L=\max(1/C_2,1)C_x+\max\left(\sqrt{1/C_2},1\right)C^{\prime}(0,d)\sqrt{\beta}$, $\textbf{B}_{\xb,N}$ is always contained in the domain $[-L\sqrt{\log N},L\sqrt{\log N}]^{d}$ for any $\xb \in [-C_x\sqrt{\log N},C_x\sqrt{\log N}]^{d}$.

\noindent \textbf{Step (ii)}
We consider a local polynomial approximation of $f(\zb,\yb)$ on $[-L\sqrt{\log N},L\sqrt{\log N}]^{d} \times [0,1]^d$. Denote $R=2L\sqrt{\log N}$. We compress the domain of $f$ on $[0,1]^{d+d_y}$ and define 
\begin{align}\label{equ::expdensity on compressed domain}
    r(\xb,\yb)=f(R(\xb-1/2),\yb) ~~\text{for }  \xb\in [0,1]^d,\yb\in [0,1]^{d_y}.
\end{align}
Then the h\"older norm of $g$ is bounded by $BR^s$. We consider using local polynomials to approximate $r$ as:
    \begin{align}
       q(\xb,\yb)= \sum_{\vb\in[N]^d,\wb\in[N]^{d_y}}\psi_{\vb,\wb}(\xb,\yb)P_{\vb,\wb}(\xb,\yb)
    \end{align}
    where 
    \begin{align*}
        \psi_{\vb,\wb}(\xb,\yb)&=\one\set{\xb\in \left[\frac{\vb-1}{N},\frac{\vb}{N}\right]}\prod \limits_{j=1}^{d_y}\phi\left(3N\left(y_j-\frac{\wb}{N}\right)\right)~~~\text{ and } \\P_{\vb,\wb}(\xb,\yb)
        &= \sum_{\norm{\nb}_1+\norm{\npb}_1<s} \frac{1}{\nb!\npb!}\frac{\partial^{\nb+\npb} r}{\partial \xb^{\nb}\partial \yb^{\npb}} \big|_{\xb=\frac{\vb}{N},\yb=\frac{\wb}{N}} \left(\xb-\frac{\vb}{N}\right)^\nb\left(\yb-\frac{\wb}{N}\right)^\npb ,
    \end{align*}
and $\phi$ is the trapezoid function defined in \eqref{equ::trapezoid}.   
    Using Taylor expansion as we do in \eqref{equ:: local error}, we directly have
    \begin{align}
        \abs{r(\xb,\yb)-q(\xb,\yb)}\lesssim B\frac{R^s(d+d_y)^s}{s!N^{\beta}}, \forall \xb\in [0,1]^{d},~ \yb\in [0,1]^{d_y}.
    \end{align}
    Thus, by transforming $r$ back to $f$, we obtain that
    \begin{align}\label{equ:: qf error}
        \abs{q\left(\frac{\zb}{R}+1/2,\yb\right)-f(\zb,\yb)}\lesssim B\frac{(\log N)^{s/2}(d+d_y)^s}{s!N^{\beta}}, \forall \xb\in [-R/2,R/2]^{d},~ \yb\in [0,1]^{d_y}.
    \end{align}
Now we replace $f(\xb,\yb)$ by $q\left(\frac{\zb}{R}+1/2,\yb\right)$ in \eqref{equ::f2 exp} and define
    \begin{align*}
& \quad f_3(\xb,\yb,t) \\
&=\frac{1}{\hat{\sigma}_{t}^d(2\pi)^{d/2}}\int_{\textbf{B}_{\xb,N}}q\left(\frac{\zb}{R}+1/2,\yb\right)\exp\left(-\frac{\norm{\zb-\hat{\alpha}_t\xb}^2}{2\hat{\sigma}_{t}^2}\right)\diff \zb\\
        &=\frac{1}{\hat{\sigma}_{t}^d(2\pi)^{d/2}}\int_{\textbf{B}_{\xb,N}}\sum_{\vb\in[N]^d}\psi_{\vb,\wb}\left(\frac{\zb}{R}+1/2,\yb\right)P_{\vb,\wb}
        \left(\frac{\zb}{R}+1/2,\yb\right)\exp\left(-\frac{\norm{\zb-\hat{\alpha}_t\xb}^2}{2\hat{\sigma}_{t}^2}\right)\diff \zb\\
        &=\sum_{\vb\in[N]^d,\wb\in[N]^{d_y}}\sum_{\norm{\nb}_1+\norm{\npb}_1<s}\frac{1}{\nb!\npb!}\frac{\partial^{\nb+\npb} r}{\partial \xb^{\nb}\partial \yb^{\npb}} \bigg|_{\xb=\frac{\vb}{N},\yb=\frac{\wb}{N}} \left(\yb-\frac{\wb}{N}\right)^\npb \prod \limits_{j=1}^{d_y}\phi\left(3N\left(y_j-\frac{\wb}{N}\right)\right) \\
        &~~~~~~\cdot\prod \limits_{i=1}^d \frac{1}{\hat{\sigma}_{t}(2\pi)^{1/2}} \int \left(\frac{z_i}{R}+1/2-\frac{v_i}{N}\right)^{n_i}\exp\left(-\frac{\abs{z_i-\hat{\alpha}_tx_i}^2}{2\hat{\sigma}_{t}^2}\right)\diff z_i.
    \end{align*} 
    The domain of the integral
    \begin{equation*}
         \int \left(\frac{z_i}{R}+1/2-\frac{v_i}{N}\right)^{n_i}\exp\left(-\frac{\abs{z_i-\hat{\alpha}_tx_i}^2}{2\hat{\sigma}_{t}^2}\right)\diff z_i
    \end{equation*}
    is 
    $$\left[\left(\frac{v_i-1}{N}-1/2\right)R,\left(\frac{v_i}{N}-1/2\right)R\right] \cap \left[\hat{\alpha}_t\xb-C^{\prime}(n,d)\hat{\sigma}_t\sqrt{\beta\log N},\hat{\alpha}_t\xb+C^{\prime}(n,d)\hat{\sigma}_t\sqrt{\beta\log N} \right].$$
    Thus, by \eqref{equ:: qf error}, we have
    \begin{align}
        \abs{f_2(\xb,\yb,t)-f_3(\xb,\yb,t)} &\lesssim  \int_{\textbf{B}_{\xb,N}}\abs{q\left(\frac{\zb}{R}+1/2,\yb\right)-f(\zb,\yb)}\frac{1}{\hat{\sigma}_{t}^d(2\pi)^{d/2}}\exp\left(-\frac{\norm{\zb-\hat{\alpha}_t\xb}^2}{2\hat{\sigma}_{t}^2}\right)\diff \zb \nonumber \\
        &\lesssim BN^{-\beta}  \log^{\frac{s}{2}}N \label{equ::f2f3 exp}.
    \end{align}

\noindent \textbf{Step (iii)}  Next, we approximate $\exp\left(-\frac{\abs{z-\hat{\alpha}_{t}x}^2}{2\hat{\sigma}_{t}^2}\right)$ with polynomials. We use again the inequality
 \begin{align}\label{equ:: exp error2}
     \abs{\exp\left(-\frac{\abs{z-\hat{\alpha}_{t}x}^2}{2\hat{\sigma}_{t}^2}\right)-\sum_{k<p} \frac{1}{k!}\left(-\frac{\abs{z-\hat{\alpha}_{t}x}^2}{2\hat{\sigma}_{t}^2}\right)^k}\le \frac{C(0,d)^{2p}\beta^{p}\log^pN}{p!2^p}, \forall z\in[\underline{C}(x),\overline{C}(x)],
 \end{align}
 where 
 $\underline{C}(x)=\hat{\alpha}_tx-C(n,d)\hat{\sigma}_t\sqrt{\beta\log N} $ and 
 $\overline{C}(x)
=\hat{\alpha}_tx+C(n,d)\hat{\sigma}_t\sqrt{\beta\log N}.$
By setting $$p=\frac{2}{3}C^2(0,d)\beta^2u\log N~~  \text{and}~~ u=\max\left(e,\frac{1+\log d}{\frac{2}{3}C^2(0,d)\beta }+\frac{\log \log N}{\log N}\frac{d}{\frac{4}{3}C^2(0,d)\beta^2 }\right),$$
we ensure the error \eqref{equ:: exp error2} is bounded by $N^{-\beta}\log ^{-\frac{d}{2}} N/d$. We remark that $p$ is still bounded by $\cO(\log N)$. Now we replace the exponential function with its polynomial approximation in $f_3$ and define
\begin{align*}
    f_1(\xb,\yb,t)&=\sum_{\vb\in[N]^d,\wb\in[N]^{d_y}}\sum_{\norm{\nb}_1+\norm{\npb}_1<s}\frac{1}{\nb!\npb!}\frac{\partial^{\nb+\npb} r}{\partial \xb^{\nb}\partial \yb^{\npb}} \bigg|_{\xb=\frac{\vb}{N},\yb=\frac{\wb}{N}} \left(\yb-\frac{\wb}{N}\right)^\npb \prod \limits_{j=1}^{d_y}\phi\left(3N\left(y_j-\frac{\wb}{N}\right)\right) \\
        &~~~~~~\cdot\prod \limits_{i=1}^d \frac{1}{\hat{\sigma}_{t}(2\pi)^{1/2}} \int \left(\frac{z_i}{R}+1/2-\frac{v_i}{N}\right)^{n_i}\sum_{k<p} \frac{1}{k!}\left(-\frac{\abs{z_i-\hat{\alpha}_tx_i}^2}{2\hat{
    \sigma}_{t}^2}\right)^k \diff z_i.
\end{align*}
We obtain that 
\begin{align}
    \left| f_3(\xb,\yb,t)-f_1(\xb,\yb,t)\right| &\lesssim \frac{1}{\hat{\sigma}_{t}^d(2\pi)^{d/2}}\int_{\textbf{B}_{\xb,N}} q\left(\frac{\zb}{R}+1/2,\yb\right)N^{-\beta}\log ^{-\frac{d}{2}} N \diff \zb  \nonumber\\
    &\lesssim \frac{1}{\hat{\sigma}_{t}^d(2\pi)^{d/2}}\int_{\textbf{B}_{\xb,N}}\left(B+BN^{-\beta}\log^{\frac{s}{2}}N\right)N^{-\beta}\log ^{-\frac{d}{2}} N\diff \zb \nonumber\\
    &\lesssim \frac{m(\textbf{B}_{\xb,N})}{\hat{\sigma}^d_t} BN^{-\beta}\log ^{-\frac{d}{2}} N \nonumber\\&\lesssim BN^{-\beta},\label{equ::f3f1 exp}
\end{align}
where we invoke \eqref{equ::f2f3 exp} for the first inequality.
Thus, adding up all the errors \eqref{equ:: f2h}, \eqref{equ::f2f3 exp} and \eqref{equ::f3f1 exp} gives rise to 
\begin{equation}
    \left| h(\xb,\yb,t)-f_1(\xb,\yb,t)\right| \lesssim BN^{-\beta}\log^{\frac{s}{2}}N.
\end{equation}
Moreover, if we redefine $g$ in \eqref{equ::h} as
\begin{align}
    g(x,n,v,k)=\frac{1}{\hat{\sigma}_{t}(2\pi)^{1/2}} \int \left(\frac{z}{R}+1/2-\frac{v}{N}\right)^{n}\sum_{k<p} \frac{1}{k!}\left(-\frac{\abs{z-\hat{\alpha}_tx}^2}{2\hat{
    \sigma}_{t}^2}\right)^k \diff z. \label{equ:: new h defi}
\end{align}
and the corresponding diffused local monomial as
\begin{align}\label{equ:: new g}
    \dmono(\xb,\yb,t)= \left(\yb-\frac{\wb}{N}\right)^\npb \prod \limits_{j=1}^{d_y}\phi\left(3N\left(y_j-\frac{\wb}{N}\right)\right)  \prod\limits_{i=1}^{d} \sum_{k<p} g(x_i,n_i,v_i,k),
\end{align}
we can write $f_1$ as a diffused local polynomial with at most $N^{d+d_y}(d+d_y)^s$ diffused local monomials, which is presented as
\begin{align} \label{equ::f1 exp}
f_1(\xb,\yb)=\sum_{\vb\in[N]^d,\wb\in[N]^{d_y}}\sum_{\norm{\nb}_1+\norm{\npb}_1\le s}\frac{R^{\norm{\nb}_1}}{\nb!\npb!}\frac{\partial^{\nb+\npb} f}{\partial \xb^{\nb}\partial \yb^{\npb}} \bigg|_{\xb=R\left(\frac{\vb}{N}-\frac{1}{2}\right),\yb=\frac{\wb}{N}} \dmono(\xb,\yb,t)
\end{align}
We complete our proof.    
\end{proof}
\begin{proof}[Proof of Lemma \ref{lemma::diffused localpoly approx1 exp}]
Since we have
    \begin{equation*}
    \frac{\hat{\sigma_t}}{\hat{\alpha_t}}\nabla h(\xb,\yb,t)=\int f(\zb,\yb)\left( \frac{\zb-\hat{\alpha}_{t}\xb}{\hat{\sigma}_t}\right) \exp\left(-\frac{\norm{\zb-\hat{\alpha}_t\xb}^2}{2\hat{\sigma}_{t}^2}\right)\diff \zb,
\end{equation*}
when approximating the $i-$th element of the vector $ \frac{\hat{\sigma_t}}{\hat{\alpha_t}}\nabla h(\xb,\yb,t)$, we can apply Lemma \ref{lemma::clipz2} with $\vb=\mathbf{e}_i$ to confine the integral in a similar $\textbf{B}_{\xb,N}$ and completely follow the proof of Lemma \ref{lemma::diffused localpoly approx exp} to obtain the polynomial approximation. The only difference is that the degree of local polynomials increases by one, so the deviation between $f_2$ and $f_3$ in \ref{equ::f2f3 exp} becomes $BN^{-\beta}\log^{\frac{s+1}{2}}N$ instead of $BN^{-\beta}\log^{\frac{s}{2}}N$. Thus, the approximation error also scales with $BN^{-\beta}\log^{\frac{s+1}{2}}N$.
\end{proof}

\subsection{Proofs of Lemma \ref{lemma::relu approx diffused exp}  and \ref{lemma::relu approx diffused1 exp}} \label{sec::sub relu approx exp}
We only elaborate on the proof of Lemma \ref{lemma::relu approx diffused exp}, since the proof of Lemma \ref{lemma::relu approx diffused1 exp} is completely the same. We follow the proof of Lemma \ref{lemma::relu approx diffused}, in which we use a ReLU network to approximate the single diffused local monomial with a small error. We recall that the diffused local monomial is defined as 
\begin{align}\label{equ::diffused monomial restate exp}
   \dmono(\xb,\yb,t)= \left(\yb-\frac{\wb}{N}\right)^\npb \prod \limits_{j=1}^{d_y}\phi\left(3N\left(y_j-\frac{\wb}{N}\right)\right)  \prod\limits_{i=1}^{d} \sum_{k<p} g(x_i,n_i,v_i,k),
\end{align}
where we have redefined $g$ in \eqref{equ:: new g}. Since the first two parts remain the same as in the proof of Lemma \ref{lemma::relu approx diffused exp}, we only focus on the ReLU approximation of $g$.
\begin{lemma}[Approximate $g$ in \eqref{equ:: new g}]\label{lemma::approx g new}
Given $N$, there exists a ReLU network $\cF(W,\kappa,L,K)$ such that for any $n\le s$, $v\le N$, $k\le p$ and $\epsilon >0$, this network gives rise to a function $g^{\trelu}(x,n,v,k)$ such that
\begin{align*}
\abs{g^{\trelu}(x,n,v,k)-g(x,n,v,k)}\le \epsilon,~~~  \forall x\in[-C_x\sqrt{\log N},C_x\sqrt{\log N}].
\end{align*}
    The hyperparameter of the network satisfies
  \begin{align*}
    & W = {\cO}\left(\log^6 N +\log^3 \epsilon^{-1} \right),
      \kappa =\exp \left({\cO}(\log^4 N+\log^2 \epsilon^{-1})\right),\\
      ~&L = {\cO}(\log^4 N+\log^2 \epsilon^{-1}),~
     K= {\cO}\left(\log^8 N+\log^4 \epsilon^{-1}\right).
\end{align*}
\end{lemma}
The proof of Lemma \ref{lemma::approx g new} is provided in Appendix \ref{sec::relu approx g}. With all the lemmas above, we completely follow the proof of Lemma \ref{lemma::relu approx diffused} to construct the ReLU network for Lemma \ref{lemma::relu approx diffused exp}. We do not elaborate on the proof for conciseness.

\subsection{Proofs of Other Lemmas}\label{sec::proof other lemmas B}
\subsubsection{Proof of Lemma \ref{lemma::score decompose exp}}\label{sec::Proof of Lemma score decompose exp}
\begin{proof}
    Under Assumption  \ref{assump::expdensity}, we have 
 \begin{align}
        p_t(\xb |\yb)&=\frac{1}{\sigma_{t}^d(2\pi)^{d/2}}\int{p(\zb|\yb)\exp\left(-\frac{\norm{\xb-\alpha_t\zb}^2}{2\sigma_{t}^2}\right)\diff \zb} \nonumber\\
        &= \frac{1}{\sigma_{t}^d(2\pi)^{d/2}}\int{f(\zb,\yb)\exp(-C_2\norm{\zb}_2^2/2)\exp\left(-\frac{\norm{\xb-\alpha_t\zb}^2}{2\sigma_{t}^2}\right)\diff \zb} \nonumber\\
        &=\frac{1}{\sigma_{t}^d(2\pi)^{d/2}}\exp\left(\frac{-C_2\norm{\xb}^2_2}{2(\alpha_t^2+C_2\sigma_t^2)}\right)\int{f(\zb,\yb)\exp\left(-\frac{\norm{\zb-\alpha_t\xb/(\alpha_t^2+C_2\sigma_t^2)}^2}{2\sigma_{t}^2/(\alpha_t^2+C_2\sigma_t^2)}\right)\diff \zb} \nonumber\\
        &=\frac{1}{(\alpha_t^2+C_2\sigma_t^2)^{d/2}}\exp\left(\frac{-C_2\norm{\xb}^2_2}{2(\alpha_t^2+C_2\sigma_t^2)}\right) \nonumber \\
        & \quad \cdot \int f(\zb,\yb) \frac{(\alpha_t^2+C_2\sigma_t^2)^{d/2}}{(2\pi)^{d/2}\sigma_t^d}\exp\left(-\frac{\norm{\zb-\alpha_t\xb/(\alpha_t^2+C_2\sigma_t^2)}^2}{2\sigma_{t}^2/(\alpha_t^2+C_2\sigma_t^2)}\right)\diff \zb \nonumber\\
        &=\frac{1}{(\alpha_t^2+C_2\sigma_t^2)^{d/2}} \exp\left(\frac{-C_2\norm{\xb}^2_2}{2(\alpha_t^2+C_2\sigma_t^2)}\right)\int f(\zb,\yb) \frac{1}{(2\pi)^{d/2}\hat{\sigma}_t^d}\exp\left(-\frac{\norm{\zb-\hat{\alpha}_t\xb}^2}{2\hat{\sigma}_{t}^2}\right)\diff \zb  \nonumber\\
        &=\frac{1}{(\alpha_t^2+C_2\sigma_t^2)^{d/2}} \exp\left(\frac{-C_2\norm{\xb}^2_2}{2(\alpha_t^2+C_2\sigma_t^2)}\right)h(\xb,\yb,t), \label{equ::decompose pt exp}
    \end{align}
    where $\hat{\sigma}_t=\frac{\sigma_t}{(\alpha_t^2+C_2\sigma_t^2)^{1/2}}$, $\hat{\alpha}_t=\frac{\alpha_t}{\alpha_t^2+C_2\sigma_t^2}$ and $h(\xb,\yb,t)=\int f(\zb,\yb) \frac{1}{(2\pi)^{d/2}\hat{\sigma}_t^d}\exp\left(-\frac{\norm{\zb-\hat{\alpha}_t\xb}^2}{2\hat{\sigma}_{t}^2}\right)\diff \zb$.
Thus, we can compute the score function as:
\begin{align*}
    \nabla \log p_t(\xb |\yb)&=\frac{\nabla p_t(\xb|\yb)}{ p_t(\xb|\yb)}
    =\frac{-C_2\xb}{\alpha_t^2+C_2\sigma_t^2} +\frac{\hat{\alpha}_t}{\hat{\sigma}_t}\cdot\frac{\int f(\zb,\yb)\left( \frac{\zb-\hat{\alpha}_{t}\xb}{\hat{\sigma}_t}\right) \exp\left(-\frac{\norm{\zb-\hat{\alpha}_t\xb}^2}{2\hat{\sigma}_{t}^2}\right)\diff \zb}{\int f(\zb,\yb) \exp\left(-\frac{\norm{\zb-\hat{\alpha}_t\xb}^2}{2\hat{\sigma}_{t}^2}\right)\diff \zb}\\
        &=\frac{-C_2\xb}{\alpha_t^2+C_2\sigma_t^2}+\frac{\nabla h(\xb,\yb,t)}{h(\xb,\yb,t)}.
\end{align*}
We complete our proof.
\end{proof}
\subsubsection{Proof of Lemma \ref{lemma::truncation x exp}}\label{sec::proof of lemma truncation x exp}
According to Lemma \ref{lemma::bound h and nabla h} and \eqref{equ::decompose pt exp}, we have
\begin{align}
    p_t(\xb|\yb)\le \frac{B}{(\alpha_t^2+C_2\sigma_t^2)^{d/2}} \exp\left(\frac{-C_2\norm{\xb}^2_2}{2(\alpha_t^2+C_2\sigma_t^2)}\right). \label{equ::pt upper bound exp}
\end{align}
Combining \eqref{equ::pt upper bound exp} with Lemma \ref{lemma::scorebound2} gives rise to
\begin{align*}
& \quad \int_{\norm{\xb}_{\infty}\ge R} \norm{\nabla\log p_{t}(\xb|\yb)}^2 p_{t}(\xb|\yb)\diff \xb \\
&\lesssim   \int_{\norm{\xb}_{\infty}\ge R}  \left(\norm{\xb}_{\infty}^2+\frac{1}{\sigma_t^2}\right)\frac{B}{(\alpha_t^2+C_2\sigma_t^2)^{d/2}} \exp\left(\frac{-C_2\norm{\xb}^2_2}{2(\alpha_t^2+C_2\sigma_t^2)}\right) \diff \xb\\
&\lesssim \left(\frac{R^3}{3(\alpha_t^2+C_2\sigma_t^2)^{3/2}}+\frac{R}{\sigma_t^2(\alpha_t^2+C_2\sigma_t^2)^{1/2}}\right)\exp\left(\frac{-C_2R^2}{2(\alpha_t^2+C_2\sigma_t^2)}\right)\\
&\lesssim \frac{1}{\sigma_t^2}R^3\exp(-C_2^{\prime}R^2),
\end{align*}
where $C^{\prime}_2=\min_{t>0}\frac{C_2}{2(\alpha_t^2+C_2\sigma_t^2)}=\frac{C_2}{2\max(C_2,1)}$. Similarly, we have
\begin{align*}
    \int_{\norm{\xb}\ge R}  p_{t}(\xb|\yb)\diff \xb 
    &\lesssim \frac{B}{(\alpha_t^2+C_2\sigma_t^2)^{d/2}}  \int_{\norm{\xb}_{\infty}\ge R}  \exp\left(\frac{-C_2\norm{\xb}^2_2}{2(\alpha_t^2+C_2\sigma_t^2)}\right) \diff \xb\\
    &\lesssim \frac{R}{(\alpha_t^2+C_2\sigma_t^2)^{1/2}}\exp\left(\frac{-C_2R^2}{2(\alpha_t^2+C_2\sigma_t^2)}\right)\\
    &\lesssim R\exp(-C_2^{\prime}R^2).
\end{align*}
The proof is complete.

\subsubsection{Proof of Lemma \ref{lemma::bound h and nabla h}}\label{sec::proof of lemma::bound h and nabla h}
\begin{proof}
Under assumption \ref{assump::expdensity}, we have $C_1\le f(\xb,\yb,t)\le B$. By plugging the bound into the integral form of $h$ and $\nabla h$ and invoking the fact that
\begin{align*}
    &\int \frac{1}{(2\pi)^{d/2}\hat{\sigma}_t^d} \exp\left(-\frac{\norm{\zb-\hat{\alpha}_t\xb}^2}{2\hat{\sigma}_{t}^2}\right)\diff \zb =1,\\
    &\int \frac{1}{(2\pi)^{d/2}\hat{\sigma}_t^d}\left| \frac{\zb-\hat{\alpha}_{t}\xb}{\hat{\sigma}_t}\right| \exp\left(-\frac{\norm{\zb-\hat{\alpha}_t\xb}^2}{2\hat{\sigma}_{t}^2}\right)\diff \zb =\sqrt{\frac{2}{\pi}},
\end{align*}
we directly obtain
\begin{align*}
 C_1\le h(\xb,\yb,t) \le B, ~~~ \norm{\frac{\hat{\sigma}_t}{\hat{\alpha}_t} \nabla h(\xb,\yb,t)}_\infty \le \sqrt{\frac{2}{\pi}}B.
\end{align*}
We complete our proof.
\end{proof}
\subsubsection{Proof of Lemma \ref{lemma::clipz2}}\label{sec::proof of lemma::clipz2}
\begin{proof}
    We denote $\wb=\frac{\zb-\hat{\alpha}_t\xb}{\hat{\sigma}_t}$. Suppose the truncated domain is $$\textbf{B}_x=\left[\hat{\alpha}_t\xb-C^{\prime}\hat{\sigma}_t\sqrt{\log \epsilon^{-1}},\hat{\alpha}_t\xb+C^{\prime}\hat{\sigma}_t\sqrt{\log \epsilon^{-1}} \right]. $$
    We note that $\zb \in \textbf{B}_x$ is equivalent to $\wb \in [-C^{\prime}\sqrt{\log \epsilon^{-1}},C^{\prime}\sqrt{\log \epsilon^{-1}}]^{d}$, the truncation error can be presented as
    \begin{align}
&\quad\abs{\int_{\norm{\wb}_{\infty}\ge C^{\prime}\sqrt{\log \epsilon^{-1}}}\wb^{\vb}f(\hat{\sigma_t}\wb+\hat{\alpha_t}\xb,\yb)\frac{1}{(2\pi)^{d/2}} \exp \left(-\frac{\norm{\wb}^2_2}{2}\right)\diff \wb } \nonumber\\
        &\le \frac{B}{(2\pi)^{1/2}}\sum_{i=1}^{n}\abs{\int_{w_i\ge C^{\prime}\sqrt{\log \epsilon^{-1}}}w_i^{v_i} \exp \left(-\frac{w_i^{2}}{2} \right)\diff w_i \cdot \prod_{j\neq i}^{d} \int_{\RR}\frac{1}{(2\pi)^{1/2}}w_j^{v_j} \exp \left(-\frac{w_j^{2}}{2} \right)\diff w_j } \nonumber\\
        &\le  \frac{B}{(2\pi)^{1/2} }\sum_{i=1}^{d}\frac{2}{v_i+1}(C^{\prime}\sqrt{\log \epsilon^{-1}})^{v_i+1}\epsilon^{\frac{C^{\prime}}{2}}\prod_{j \neq i} \EE_{Z_j \sim N(0,1)}\left[\abs{Z_j^{v_j}}\right] \nonumber\\
        &\le B\prod_{j=1}^{d} \EE_{Z_j \sim N(0,1)}\left[\abs{Z_j^{v_j}}\right] \sum_{i=1}^{d}\frac{1}{v_i+1}(C^{\prime}\sqrt{\log \epsilon^{-1}})^{v_i+1}\epsilon^{\frac{C^{\prime}}{2}} \nonumber\\
        &\le B \max \left(1,\left(\EE_{Z \sim N(0,1)}\left[\abs{Z^{n}}\right]\right)^{d} \right)\sum_{i=1}^{d}(C^{\prime}\sqrt{\log \epsilon^{-1}})^{v_i+1}\epsilon^{\frac{C^{\prime}}{2}}.\label{equ::trunc z2}
\end{align}
In the third inequality we invoke $\EE_{Z_i \sim N(0,1)}\left[\abs{Z_i^{v_j}}\right] \ge \sqrt{\frac{2}{\pi}} $ for any nonnegative integer $v_i$. In the last inequality, we invoke $\EE_{Z \sim N(0,1)}\left[\abs{Z^{m}}\right] \le \EE_{Z \sim N(0,1)}\left[\abs{Z^{n}}\right] $ for $1\le m\le n$. For $\epsilon <0.99$, by setting $C^{\prime}=C^{\prime}(n,d)$ sufficiently large (depending on $n$, $d$ and $B$), \eqref{equ::trunc z2} can be bounded by $\epsilon$. The proof is complete.
\end{proof}
\subsubsection{Proof of Lemma \ref{lemma::approx g new}}\label{sec::relu approx g}
\begin{proof}
   We denote $w=\frac{z-\hat{\alpha_t}x}{\hat{\sigma_t}}$. By the definition of $g(x,n,v,k)$, we have
\begin{align*}
    g(x,n,v,k) &=  \frac{1}{\hat{\sigma}_{t}(2\pi)^{1/2}} \int \left(\frac{z}{R}+1/2-\frac{v}{N}\right)^{n} \frac{1}{k!}\left(-\frac{\abs{z-\hat{\alpha}_tx}^2}{2\hat{\sigma}_{t}^2}\right)^k \diff z\\
    &=\frac{1}{(2\pi)^{1/2}k!(-2)^k} \int \left(\frac{\hat{\alpha}_t x+\hat{\sigma}_t w}{R}+1/2-\frac{v}{N}\right)^{n} w^{2k} \diff w\\
    &=\frac{1}{(2\pi)^{1/2}k!(-2)^k} \sum_{j=0}^{n} C_n^j \int \left(\frac{\hat{\alpha}_t x}{R}+1/2-\frac{v}{N}\right)^{n-j} \left(\frac{\hat{\sigma}_t w}{R}\right)^j w^{2k} \diff w\\
    &=\frac{1}{(2\pi)^{1/2}k!(-2)^kR^n} \sum_{j=0}^{n} C_n^j\hat{\sigma}_t^{j} \left(\hat{\alpha}_t x+\frac{R}{2}-\frac{vR}{N}\right)^{n-j}  \int w^{2k+j} \diff w.
\end{align*}
Remember that the domain of the integral is
\begin{align*}
    z\in\left[\left(\frac{v_i-1}{N}-1/2\right)R,\left(\frac{v_i}{N}-1/2\right)R\right] \cap \left[\hat{\alpha}_t\xb-C^{\prime}(n,d)\hat{\sigma}_t\sqrt{\beta\log N},\hat{\alpha}_t\xb+C^{\prime}(n,d)\hat{\sigma}_t\sqrt{\beta\log N} \right],
\end{align*}
which means that 
\begin{align*}
    w\in \left[\frac{(v-1-N/2)R-\hat{\alpha}_{t}Nx}{N\hat{\sigma}_t},\frac{(v-N/2)R-\hat{\alpha}_{t}Nx}{N\hat{\sigma}_t}\right]\cap \left[-C^{\prime}(0,d)\sqrt{\beta\log N},C^{\prime}(0,d)\sqrt{\beta\log N}\right].
\end{align*}
Thus, we have
\begin{align} \label{equ::new g}
    g(x,n,v,k)&=\frac{1}{(2\pi)^{1/2}k!(-2)^kR^n} \sum_{j=0}^{n} C_n^j \hat{\sigma}_t^{j}\left(\hat{\alpha}_t x+\frac{R}{2}-\frac{vR}{N}\right)^{n-j} \frac{f_{\overline{D}}^{j+2k+1}(x)-f_{\underline{D}}^{j+2k+1}(x)}{j+2k+1},
\end{align}
 where 
 \begin{equation*}
     f_{\underline{D}}(x)=\clip{\frac{(v-1-N/2)R-\hat{\alpha}_{t}Nx}{N\hat{\sigma}_t},C(0,d)\sqrt{\beta\log N}}
 \end{equation*}
and
 \begin{equation*}
     f_{\overline{D}}(x)=\clip{\frac{(v-N/2)R-\hat{\alpha}_{t}Nx}{N\hat{\sigma}_t},C(0,d)\sqrt{\beta\log N}}.
 \end{equation*}
 Therefore, we only need to approximate the following form of function
 \begin{align}\label{equ::single component new}
     f_{v,k,j}=\hat{\sigma}_t^{j}\left(\hat{\alpha}_t x+\frac{R}{2}-\frac{vR}{N}\right)^{n-j} \left(f_{\overline{D}}^{j+2k+1}(x)-f_{\underline{D}}^{j+2k+1}(x)\right).
 \end{align}
 We construct a ReLU network to approximate $f_{v,k,j}$ (see Figure \ref{fig::ReLU3}).
 \begin{figure}
 \centering
\includegraphics[width=0.8\linewidth]{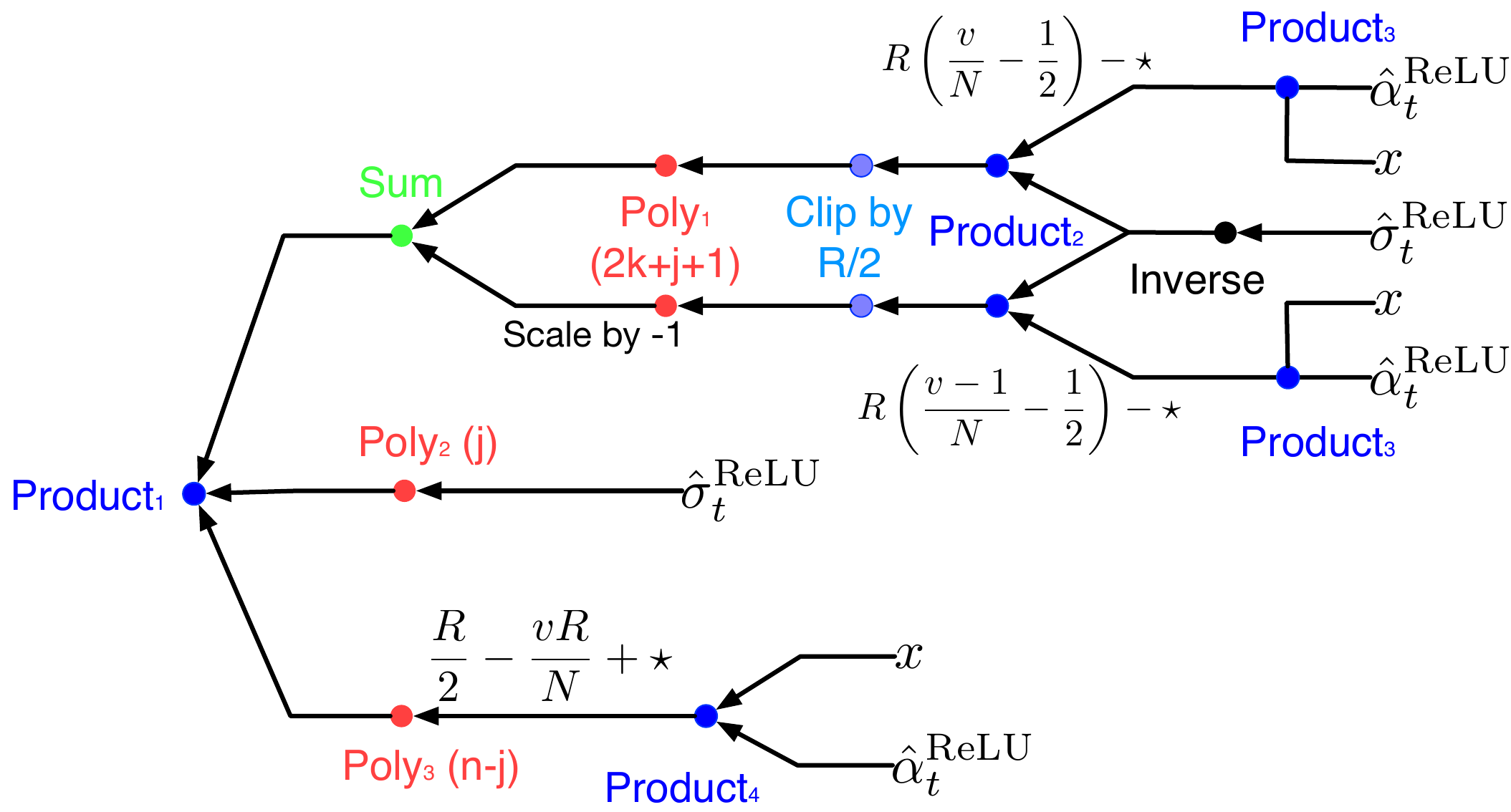}
\caption{Network architecture of $f^{\trelu}_{v,k,j}$. We implement all the basic functions (e.g., $x$, $\hat{\alpha}_t$ and $\hat{\sigma}_t$) through ReLU networks and combine them using the ReLU-expressed operators (\textit{product}, \textit{inverse}, \textit{clip} and \textit{poly}) to express $f_{v,k,j}$ according to its definition in \eqref{equ::single component new}.}
\label{fig::ReLU3}
\end{figure}
By appropriately setting the parameters for each ReLU approximation function (details about how to determine the network size and the error propagation are deferred to Appendix~\ref{sec::fvkj new} for construction details), we derive a ReLU network $f^{\trelu}_{v,k,j}\in \cF( W, \kappa, L, K) $ with 
\begin{align*}
    & W = {\cO}\left(\log^6 N +\log^3 \epsilon^{-1} \right),
      \kappa =\exp \left({\cO}(\log^4 N+\log^2 \epsilon^{-1})\right),\\
      ~&L = {\cO}(\log^4 N+\log^2 \epsilon^{-1}),~
     K= {\cO}\left(\log^8 N+\log^4 \epsilon^{-1}\right).
\end{align*}
such that for any $x \in  [-C_x\sqrt{\log N},C_x\sqrt{\log N}]$ and $t\in[N^{-C_{\sigma}},C_{\alpha}\log N]$,
\begin{align}\label{equ::fvkj relu2}
     \abs{f_{v,k,j}(x,t)-f_{v,k,j}^{\trelu}(x,t)}\le \epsilon.
\end{align}
 
At last, we add up $f^{\trelu}_{v,k,j}$ along $j$ to get an approximation for $g(x,n,v,k)$, adds one additional layer to the ReLU network. According to \eqref{equ::new g}, the ReLU approximation of $g(x,n,v,k)$ is presented as
\begin{align}\label{equ::g relu}
    g^{\trelu}(x,n,v,k)=\frac{1}{(2\pi)^{1/2}R^n(-2)^kk!}\sum_{j=0}^n \frac{C_n^j}{j+2k+1}f^{\trelu}_{v,k,j}.
\end{align}
By \eqref{equ::fvkj relu2}, we have 
\begin{equation*}
    \left|g^{\trelu}(x,n,v,k)-g(x,n,v,k)\right| \le \frac{2^n\epsilon}{(2\pi)^{1/2}R^n 2^kk!} \le \epsilon.
\end{equation*}
The proof is complete.
\end{proof}

\section{Variants of Score Approximation}\label{sec::extend score approx}

\subsection{Extension to Unconditional Score Approximation}\label{sec::unconditional score}
Building upon the foundation established in the proof of our main theorems, we now extend our analysis to unconditional score approximation. Denote the marginal initial distribution of $\xb$ by $p(\xb)$. Furthermore, we denote the marginal distribution of $X_t$ by $p_t(\xb)$. We point out that our results also apply to the marginal distribution of $\xb$. We present the counterpart of Theorem \ref{thm::score approx} and \ref{thm::score approx exp}.

\begin{proposition}[Counterpart of Theorem \ref{thm::score approx}]\label{prop:: score approx marginal}
    Suppose Assumption \ref{assump:sub} holds. For sufficiently large integer $N>0$ and constants $C_{\sigma}, C_{\alpha}>0$, by taking $t_0=N^{-C_\sigma}$ and $T=C_{\alpha}\log N$, there exists $\sb^{\star} \in \cF(M_t, W, \kappa, L, K) $ such that for any and $t \in [t_0,T]$,
\begin{align}
    \int_{\R^{d}} \norm{\sb^{\star}(\xb,t)-\nabla\log p_{t}(\xb)}^2 p_{t}(\xb)d\xb \lesssim \frac{1}{\sigma_t^4}B^2N^{-\frac{\beta}{d}}(\log N)^{d+s/2+1}.
\end{align}
The hyperparameters in the network class $\cF$ satisfy
\begin{align*}
    & \hspace{0.4in} M_t = \cO\left(\sqrt{\log N}/\sigma^2_t\right),~
     W = {\cO}\left(N\log^7 N\right),\\
     & \kappa =\exp \left({\cO}(\log^4 N)\right),~
    L = {\cO}(\log^4 N),~
     K= {\cO}\left(N\log^9 N\right).
\end{align*}
\end{proposition}

\begin{proposition}[Counterpart of Theorem \ref{thm::score approx exp}]\label{prop:: score approx marginal exp}
    Suppose Assumption \ref{assump::expdensity} holds. For sufficiently large integer $N>0$ and constants $C_{\sigma}, C_{\alpha}>0$, by taking $t_0=N^{-C_\sigma}$ and $T=C_{\alpha}\log N$, there exists $\sb^{\star} \in \cF(M_t, W, \kappa, L, K) $ such that for any and $t \in [t_0,T]$,
\begin{align}
    \int_{\R^{d}} \norm{\sb^{\star}(\xb,t)-\nabla\log p_{t}(\xb)}^2 p_{t}(\xb)d\xb \le \frac{1}{\sigma_t^2}B^2N^{-\frac{2\beta}{d}}(\log N)^{s+1}.
\end{align}
The hyperparameters in the network class $\cF$ satisfy
\begin{align*}
    & \hspace{0.4in} M_t = \cO\left(\sqrt{\log N}/\sigma_t\right),~
     W = {\cO}\left(N\log^7 N\right),\\
     & \kappa =\exp \left({\cO}(\log^4 N)\right),~
    L = {\cO}(\log^4 N),~
     K= {\cO}\left(N\log^9 N\right).
\end{align*}
\end{proposition}
Note that the marginal density function $p(\xb)$ fully inherits the regularity assumptions 
and the subGaussian assumption on the conditional distribution function. Thus, we can derive these results by simply removing the step of approximating the part related to $\yb$ in the proof of Lemmas \ref{lemma::diffused localpoly approx} to \ref{lemma::relu approx diffused1} and  \ref{lemma::diffused localpoly approx exp} to \ref{lemma::relu approx diffused1 exp} while keeping other parts of the proof completely the same. To be specific, we perform the same truncation to $\xb$ and construct diffused local polynomials without the components of $\yb$ to approximate $p_t(\xb)$ and $\nabla p_t(\xb)$, which is in the following form:
\begin{align}\label{equ::Diffused local monomial uncondtional}
       \Phi_{\nb,\vb}(\xb,t)=   \prod\limits_{i=1}^{d} \sum_{k<p} g(x_i,n_i,v_i,k),
\end{align}
Here we invoke the definition of $g(x,n,v,k)$ in \eqref{equ::h} under Assumption \ref{assump:sub} or \eqref{equ:: new h defi} under Assumption \ref{assump::expdensity}. Moreover, we redefine
\begin{align*}
f_1(\xb,\yb)=\sum_{\vb\in[N]^d}\sum_{\norm{\nb}_1\le s}\frac{R^{\norm{\nb}_1}}{\nb!}\frac{\partial^{\nb} f}{\partial \xb^{\nb} } \bigg|_{\xb=R\left(\frac{\vb}{N}-\frac{1}{2}\right)} \Phi_{\nb,\vb}(\xb,t)
\end{align*}
as an approximation of $p_t(\xb)$.
Thus, by following a similar process of ReLU network construction in Appendix \ref{sec::sub relu approx}, the approximation error removes the dependence on $d_y$.
\subsection{Conditional Score Approximation with Unbounded Label}
When considering the case that the label $\yb$ is unbounded, we need additional assumptions on the distribution of $\yb$.
\begin{assumption}\label{assump::expdensity unbound y }
    Let $C_y$ be a positive constant. We assume that the density function of $\yb$, i.e., $p(\yb)$ has subGaussian tails $p(\yb)\le \exp(-C_y\norm{\yb}^2/2)$.
\end{assumption}
Now we present a new version of Theorem \ref{thm::score approx exp} with unbounded $\yb$.
\begin{proposition}[Theorem \ref{thm::score approx exp} with unbounded $\yb$]\label{prop:: score approx exp unbounded y}
    Suppose Assumption \ref{assump::expdensity} and \ref{assump::expdensity unbound y } hold. For sufficiently large integer $N>0$ and constants $C_{\sigma}, C_{\alpha}>0$, by taking $t_0=N^{-C_\sigma}$ and $T=C_{\alpha}\log N$, there exists $\sb^{\star} \in \cF(M_t, W, \kappa, L, K) $ such that for any and $t \in [t_0,T]$,
\begin{align}
\EE_{\yb \sim P_{y}}\left[\EE_{\xb \sim P_t(\cdot|\yb)} \left[\norm{\sb^{\star}(\xb,\yb,t)-\nabla\log p_{t}(\xb|\yb)}^2 \right]\right]\lesssim \frac{1}{\sigma_t^2}B^2N^{-\frac{2\beta}{d+d_y}}(\log N)^{s+1}.
\end{align}
The hyperparameters in the network class $\cF$ satisfy
\begin{align*}
    & \hspace{0.4in} M_t = \cO\left(\sqrt{\log N}/\sigma_t\right),~
     W = {\cO}\left(N\log^7 N\right),\\
     & \kappa =\exp \left({\cO}(\log^4 N)\right),~
    L = {\cO}(\log^4 N),~
     K= {\cO}\left(N\log^9 N\right).
\end{align*}
\end{proposition}
\subsubsection{Proof of Proposition \ref{prop:: score approx exp unbounded y}}
Besides confining $\xb$ on a bounded area $\cD$ as we've shown in Section \ref{sec:key_steps_thm3.2}, we also consider constraining $\yb$ on a bounded region with small truncation error, which is presented as the following lemma.
\begin{lemma}\label{lemma::clip y}
    Suppose Assumption \ref{assump::expdensity} and \ref{assump::expdensity unbound y } hold. Then for any truncation radius $R_y>0$ and time $t>0$, 
\begin{align}
\EE_{\yb \sim P_{y}}\left[\EE_{\xb \sim P_t(\cdot|\yb)} \left[\indic{\norm{\yb}_{\infty}\ge R_y}\norm{\nabla\log p_{t}(\xb|\yb)}^2 \right]\right]\lesssim \frac{\exp(-C_yR_y^2/2)}{\sigma^2_t}.
\end{align}
\end{lemma}
The proof of the lemma is provided in Appendix \ref{sec::proof of lemma::clip y}. To be specific, for a fixed integer $N>0$ we can take $R_y\ge\sqrt{\frac{4\beta\log N}{C_y}}$ to ensure that the truncation error is upper bounded by $\cO\left(\frac{N^{-2\beta}}{\sigma^2_t}\right)$.
With the lemma above, we present a brief proof of Proposition \ref{prop:: score approx exp unbounded y}.
\begin{proof}[Proof of Proposition \ref{prop:: score approx exp unbounded y}]
    Now we confine the domain in $\RR^{d}\times [-R_y,R_y]^{d_y}$. We can approximate $p_{t}(\xb|\yb)$ and $\nabla p_{t}(\xb|\yb)$ within this bounded area using diffused polynomials and ReLU networks by repeating the proof of Lemma \ref{lemma::diffused localpoly approx exp} to Lemma \ref{lemma::relu approx diffused1 exp}. For conciseness, we only present the content that is different from the original proof of Theorem \ref{thm::score approx exp}. In the proof of Lemma \ref{lemma::diffused localpoly approx exp}, we replace the definition of $r(\xb,\yb)$ in \eqref{equ::expdensity on compressed domain} and the radius $R$ defined above that equation by
\begin{align}\label{equ::unbounded y}
     r(\xb,\yb)=f(R_{\star}(\xb-1/2),R_{\star}(\yb-1/2)) ~~\text{for }  \xb\in [0,1]^d,\yb\in [0,1]^{d_y},
\end{align}
where $R_{\star}=\max(2L\sqrt{\log N},2R_y)$. Since $R_{\star}$
is also bounded by $\cO(\sqrt{\log N})$, and the h\"older norm of $g$ is bounded by $\cO \left(B(\log N)^{s/2}\right)$, the inequality \ref{equ:: qf error} still holds. Thus, we can similarly construct a series of diffused local polynomials like \eqref{equ:: new g} to approximate $p_t(\xb|\yb) $ or $\nabla p_t(\xb|\yb)$ on the bounded region. The diffused local polynomial is in the form of
\begin{align}\label{equ::Diffused local monomial unbound y}
       \dmono(\xb,\yb,t)= \left(\yb-\frac{\wb}{N}\right)^\npb \prod \limits_{j=1}^{d_y}\phi\left(3N\left(\frac{y_j}{R_{\star}}+\frac{1}{2}-\frac{\wb}{N}\right)\right)  \prod\limits_{i=1}^{d} \sum_{k<p} g(x_i,n_i,v_i,k),
\end{align}
and we can similarly define 
\begin{align}
& f_1(\xb,\yb,t) = \nonumber \\
& ~\sum_{\vb\in[N]^d,\wb\in[N]^{d_y}}\sum_{\norm{\nb}_1+\norm{\npb}_1\le s}\frac{R_{\star}^{\norm{\nb}_1+\norm{\npb}_1}}{\nb!\npb!}\frac{\partial^{\nb+\npb} f}{\partial \xb^{\nb}\partial \yb^{\npb}} \bigg|_{\xb=R_{\star}\left(\frac{\vb}{N}-\frac{1}{2}\right),\yb=R_{\star}\left(\frac{\wb}{N}-\frac{1}{2}\right)} \dmono(\xb,\yb,t). \label{equ::decomposition of f1 unbounded y}
\end{align}
as the approximation of $p_t(\xb|\yb)$.
The process of constructing the ReLU network is completely the same as we do in Appendix \ref{sec::sub relu approx exp}. Compared with the definition of $f_1$ in \eqref{equ::f1 exp}, the weights of these diffused local polynomials scale up by $R_{\star}^{\norm{\npb}_1}$. Thus, we only need to set the accuracy of the ReLU approximator to be $R_{\star}^{s}$ times more precise than before. This change increases in the network parameters $(W, \kappa, L, K)$ only by $\cO(\log N)$. Thus, the hyperparameters of the network that contains $f^{\trelu}_1$ still satisfy
\begin{align*}
& W = {\cO}\left(N^{d+d_y}(\log^7 N+\log N\log^3 \epsilon^{-1})\right), \quad \kappa =\exp \left({\cO}(\log^4 N+\log^2 \epsilon^{-1})\right), \\
& \qquad L = {\cO}(\log^4 N+\log^2 \epsilon^{-1}), \quad K = {\cO}\left(N^{d+d_y}(\log^9 N+\log N\log^3 \epsilon^{-1})\right).
\end{align*}
Therefore, by repeating the proof in Lemma \ref{lemma::relu approx diffused exp} and Proposition \ref{prop::logh approx bounded exp}, we can construct a ReLU network $\cF(M_t,W,\kappa,L,K)$ that gives rise to a function $\sb\in \cF$ such that for any $ t \in [t_0,T]$, $ \xb \in [-C_x\sqrt{\log N},C_x\sqrt{\log N}]^d$ and $\yb \in [-R_y,R_y]^{d_y}$, we have
    \begin{align}\label{equ::infty bound unbounded y}
\norm{\textbf{f}^{\trelu}_{3}(\xb,\yb,t)-\nabla \log p_t(\xb|
        \yb)}_{\infty}\lesssim \frac{B}{\sigma_t}N^{-\beta}\log ^{\frac{s+1}{2}} N,   \end{align}
The hyperparameters $(M_t,W,\kappa,L,K)$ of the network satisfy
\begin{align*}
    & \hspace{0.4in} M_t = \cO\left(\sqrt{\log N}/\sigma_t\right),~
     W = {\cO}\left(N^{d+d_y}\log^7 N\right),\\
     & \kappa =\exp \left({\cO}(\log^4 N)\right),~
    L = {\cO}(\log^4 N),~
     K= {\cO}\left(N^{d+d_y}\log^9 N\right).
\end{align*}
From the construction of the network (see Figure \ref{fig:ReLU4}), we know that $\norm{\textbf{f}^{\trelu}_{3}(\xb,\yb,t)}_{2} \lesssim \frac{\sqrt{\log N}}{\sigma_t}$. Thus, by taking $R_y=\sqrt{\frac{4\beta\log N}{C_y}}$ and $C_x=\sqrt{\frac{2\beta\log N}{C^{\prime}_2}}$, we have
\begin{align*}
    &~~\EE_{\yb \sim P_{y}}\left[\EE_{\xb \sim P_t(\cdot|\yb)} \left[\norm{\sb^{\star}(\xb,\yb,t)-\nabla\log p_{t}(\xb|\yb)}^2 \right]\right]\\
    &\le \EE_{\yb \sim P_{y}}\left[\EE_{\xb \sim P_t(\cdot|\yb)} \left[\indic{\norm{\yb}_{\infty}\ge R_y}\norm{\nabla\log p_{t}(\xb|\yb)}^2 \right]\right]\\
    &\quad +\EE_{\yb \sim P_{y}}\left[\EE_{\xb \sim P_t(\cdot|\yb)} \left[\indic{\norm{\yb}_{\infty}\ge
    R_y}\norm{\textbf{f}^{\trelu}_{3}(\xb,\yb,t)}^2 \right]\right]\\
    &\quad +\EE_{\yb \sim P_{y}}\left[\EE_{\xb \sim P_t(\cdot|\yb)} \left[\indic{\norm{\yb}_{\infty}< R_y,\norm{\xb}_{\infty} \ge C_x}\norm{\sb^{\star}(\xb,\yb,t)-\nabla\log p_{t}(\xb|\yb)}^2 \right]\right]\\
    &\quad +\EE_{\yb \sim P_{y}}\left[\EE_{\xb \sim P_t(\cdot|\yb)} \left[\indic{\norm{\yb}_{\infty}< R_y,\norm{\xb}_{\infty} < C_x}\norm{\sb^{\star}(\xb,\yb,t)-\nabla\log p_{t}(\xb|\yb)}^2 \right]\right]\\
    &\overset{(i)}{\lesssim }\frac{\exp(-C_yR_y^2)}{\sigma^2_t} + \frac{\log N \exp(-C_yR_y^2)}{\sigma^2_t} + \frac{\log^{3/2} N \exp(-C^{\prime}_{2}C_x^2) C_x^{3}}{\sigma^2_t} + \frac{1}{\sigma_t^2}B^2N^{-2\beta}(\log N)^{s+1}\\
    & \lesssim \frac{1}{\sigma_t^2}B^2N^{-2\beta}(\log N)^{\beta+1}.
\end{align*}
Here in (i), we invoke Lemmas \ref{lemma::clip y}, \ref{lemma::truncation x exp}  and the approximation error bound \eqref{equ::infty bound unbounded y}.
Replacing $N$ by $N^{\frac{1}{d+d_y}}$ completes our proof. 
\end{proof}
\subsubsection{Proof of Lemma \ref{lemma::clip y}}\label{sec::proof of lemma::clip y}
According to Lemma \ref{lemma::scorebound2}, we have for any $\yb \in \R^{d_y}$,
\begin{align*}
    \EE_{\xb \sim P_t(\cdot|\yb)}\left[\norm{\nabla \log p_t(\xb|\yb)}_2^2\right] \lesssim \EE_{\xb \sim P_t(\cdot|\yb)}\left[\norm{\xb}_2^2+\frac{1}{\sigma_t^2}\right] \lesssim \frac{1}{\sigma_t^2}.
\end{align*}
Thus, we obtain that
\begin{align*}
    \quad\EE_{\yb \sim P_{y}}\left[\EE_{\xb \sim P_t(\cdot|\yb)} \left[\indic{\norm{\yb}_{\infty}\ge R_y}\norm{\nabla\log p_{t}(\xb|\yb)}^2 \right]\right]
    & \lesssim \frac{1}{\sigma^2_t}\EE_{\yb \sim P_{y}}\left[\indic{\norm{\yb}_{\infty}\ge R_y} \right]\\
    &\lesssim \frac{\exp(-C_yR_y^2/2)}{\sigma^2_t}.
\end{align*}
The proof is complete.

%% file: Sections/AppendixEstimation.tex
\section{Proofs for Section \ref{sec::esti}}\label{sec::proof gener}
\subsection{Notation Recap}

Given a score approximator $\sb$, we aim to bound the following conditional score
\begin{align*}
\cR(\sb)&=\int_{t_0}^T\frac{1}{T-t_0} \EE_{\xb_t,\yb}\norm{\sb(\xb_t,\yb,t)-\nabla \log p_t(\xb_t|\yb)}_2^2 \diff t.
\end{align*}
Due to the structure of classifier-free guidance we define in \eqref{eq:classifierfree_guidance}, we first consider the following mixed score error
\begin{align}
\cR_{\star}(\sb)&=\int_{t_0}^T\frac{1}{T-t_0}  \EE_{\xb_t,\yb,\tau}\norm{\sb(\xb_t,\tau\yb,t)-\nabla \log p_t(\xb_t|\tau\yb)}_2^2 \diff t \label{equ::popu loss score}\\
&=\underbrace{\frac{1}{2}\int_{t_0}^T\frac{1}{T-t_0} \EE_{\xb_t,\yb}\norm{\sb(\xb_t,\yb,t)-\nabla \log p_t(\xb_t|\yb)}_2^2 \diff t}_{\cR: \text{ conditional score error}} \nonumber \\
&\quad+\underbrace{\frac{1}{2}\int_{t_0}^T\frac{1}{T-t_0} \EE_{\xb_t}\norm{\sb(\xb_t,\void,t)-\nabla \log p_t(\xb_t)}_2^2 \diff t}_{\cR_0: \text{unconditional score error}}, \nonumber
\end{align}
which naturally gives rise to the inequality $\cR(\sb) \le 2\cR_{\star}(\sb)$. Thus, we only need to analyze the bound of $\cR_{\star}(\sb)$. In practice, we consider minimizing an equivalent loss of $\cR_{\star}$, which is written as
\begin{equation}\label{equ::popu loss}
\cL(\sb):=\int_{t_0}^{T}\frac{1}{T-t_0}\EE_{\xb_0,\yb}\bigg[\EE_{\tau,\xb_t|\xb_0}\bigg[\norm{\sb(\xb_t,\tau\yb,t)-\nabla \log \phi_t(\xb_t|\xb_0)}_2^2\bigg]\bigg]\diff t
\end{equation}
According to Lemma C.3 in \cite{vincent2011connection}, \eqref{equ::popu loss score} differs \eqref{equ::popu loss} by a constant independent of $\sb$. Now we consider training the model with $n$ samples $\set{\xb_i,\yb_i}_{i=1}^d$ by minimizing the corresponding empirical loss
\begin{align}\label{equ::empirical loss}
    \hat{\cL}(\sb)=\frac{1}{n}\sum_{i=1}^{n}\ell(\xb_i,\yb_i,\sb),
\end{align}
where 
\begin{equation}\label{equ::single empirical loss recap}
    \ell(\xb,\yb,\sb):=\int_{t_0}^{T}\frac{1}{T-t_0}\EE_{\tau,\xb_t|\xb_0=\xb}\big[\norm{\sb(\xb_t,\tau\yb,t)-\nabla \log \phi_t(\xb_t|\xb_0)}_2^2\big]\diff t.
\end{equation}
Moreover, in order to derive a bounded covering number of our ReLU network function class, we use a truncated loss $\lt(\sb,\xb,\yb)$ defined as:
\begin{equation*} \lt(\xb,\yb,\sb):=\ell(\xb,\yb,\sb)\indic{\norm{\xb}_{\infty}\le R}.
\end{equation*}
Accordingly, we denote the truncated domain of the score function by $\cD=[-R, R]^{d}\times [0,1]^{d_y}\cup\set{\void}$. We consider the truncated loss function class defined as
\begin{equation}
    \cS(R)=\set{\ell(\cdot,\cdot,\sb):\cD \rightarrow \RR\bigg|\sb \in \cF}.
\end{equation}
\subsection{Proof of Theorem \ref{thm::Generalization}}\label{sec::proof generation}
Firstly we give a uniform $L_\infty$ bound on $\cS(R)$. 
\begin{lemma}\label{lemma::bound loss class}
Suppose that we configure the network parameters $M_t, W,\kappa, L, K$ according to Theorem \ref{thm::score approx} or Theorem \ref{thm::score approx exp} and we denote $m_t=M_t/\sqrt{\log N}$. Then for any $\sb\in \cF(M_t, W,\kappa,L,K)$ and $(\xb,\yb)\in \cD$, we have $\abs{\ell(\sb,\xb,\yb)}\lesssim \int_{t_0}^{T} m_{t}^2 \diff t\triangleq M$. In particular, if we take $t_0=n^{-\cO(1)}$ and $T=\cO(\log n)$, we have $M=\cO(\log t_0)$ for $m_t=\frac{1}{\sigma_t}$, and $M=\cO\left(\frac{1}{t_0}\right)$ for $m_t=\frac{1}{\sigma^2_t} $, respectively.
\end{lemma}
The proof of the lemma is provided in Appendix \ref{sec::proof of lemma::bound loss class}. Moreover, to convert our approximation guarantee to statistical theory, we need to calculate the covering number of the loss function class $\cS(R)$, which is defined as follows.
\begin{definition}
We denote $\cN(\delta, \cF, \norm{\cdot})$ to be the $\delta-$covering number of any function class $\cF$ w.r.t. the norm $\norm{\cdot}$, i.e., $$\cN(\delta, \cF,\norm{\cdot})=\min \set{N: \exists \set{f_{i}}_{i=1}^N \subseteq \cF, \text{s.t.}~ \forall f \in \cF, \exists i \in [N],~ \norm{f_{i}-f}\le \delta}$$
\end{definition}
The following lemma presents the covering number of $\cS(R)$:
\begin{lemma}\label{lemma::covering number S}
Given $\delta>0$, when $\norm{\xb}_{\infty} \le R$, the $\delta-$covering number of the loss function class $\cS(R)$ w.r.t. $\norm{\cdot}_{L_{\infty}\cD}$ satisfies
    \begin{equation}\label{equ::covering number S}
     \cN\left(\delta, \cS(R), \norm{\cdot}_{L_{\infty}\cD}\right) \lesssim \left(\frac{2L^2(W\max(R,T)+2)\kappa^LW^{L+1}\log N}{\delta}\right)^{2K}.
\end{equation}
Here the norm $\norm{\cdot}_{L_{\infty}\cD}$ is defined as
\begin{align*}
    \norm{f(\cdot,\cdot)}_{L_{\infty}\cD}=\max_{\xb \in [-R,R]^{d},\yb \in [0,1]^{d_y}\cup\set{\void}}\abs{f(\xb,\yb)}.
\end{align*}
The proof is provided in Appendix \ref{sec::proof of lemma::covering number S}. Particularly, under the network configuration in Theorem \ref{thm::score approx} or Theorem \ref{thm::score approx exp}, we know that log covering number is bounded by
\begin{align} 
    \log \cN &\lesssim N\log^9 N \left( \text{Poly}(\log \log N)+\text{Poly}(\log \log N)\log N\log 
 R+\log^8 N+\log \frac{1}{\delta}  \right)  \nonumber\\&\lesssim N\log^{9} N \left( \log^8 N+\log^2 N\log R +\log \frac{1}{\delta}  \right). \label{equ::logcover}
\end{align}
\end{lemma}
With the lemmas above, we begin our proof of Theorem \ref{thm::Generalization}.
\begin{proof}[Proof of Theorem \ref{thm::Generalization}]

    We denote the truth score by $\sstar(\xb,\yb,t)=\nabla \log p_t(\xb|\yb)$ if $\yb \neq \void$ and $\sstar(\xb,\void,t)=\nabla \log p_t(\xb)$. We create $n$ i.i.d ghost samples 
$$(\xb'_1,\yb'_1),(\xb'_2,\yb'_2),...,(\xb'_n,\yb'_n) \sim \cP_{\xb,\yb}.$$
Since $\cR_{\star}(\sstar)=0$ and $\cR_{\star}(\sb)$ differs $\cL(\sb)$ by a constant for any $\sb$, it suffices to bound

\begin{align}
    \cR_{\star}(\shat)&=\cR_{\star}(\shat)-\cR_{\star}(\sstar)=\cL(\shat)-\cL(\sstar) =\EE_{\set{\xb'_i,\yb'_i}_{i=1}^n}\left[\frac{1}{n}\sum_{i=1}^n\left(\ell(\xb'_i,\yb'_i,\shat)-\ell(\xb'_i,\yb'_i,\sstar)  \right) \right].
\end{align}
Define $$\cL_1=\frac{1}{n}\sum_{i=1}^n\left(\ell(\xb_i,\yb_i,\shat)-\ell(\xb_i,\yb_i,\sstar)  \right),~~~\Lt_1=\frac{1}{n}\sum_{i=1}^n\left(\lt(\xb_i,\yb_i,\shat)-\lt(\xb_i,\yb_i,\sstar)  \right)$$ and $$\cL_2=\frac{1}{n}\sum_{i=1}^n\left(\ell(\xb'_i,\yb'_i,\shat)-\ell(\xb'_i,\yb'_i,\sstar)  \right),~~~\Lt_2=\frac{1}{n}\sum_{i=1}^n\left(\lt(\xb'_i,\yb'_i,\shat)-\lt(\xb'_i,\yb'_i,\sstar)  \right).$$
We consider decomposing  $\EE_{\set{\xb_i,\yb_i}_{i=1}^n}\left[  \cR_{\star}(\shat)\right]$ as
\begin{align}
    \EE_{\set{\xb_i,\yb_i}_{i=1}^n}\left[  \cR_{\star}(\shat)\right]&=\underbrace{\EE_{\set{\xb_i,\yb_i}_{i=1}^n}\left[\EE_{\set{\xb'_i,\yb'_i}_{i=1}^n}\left[ \cL_2-\Lt_2\right] \right]}_{A_2}
     \label{equ::score esti A}\\&\underbrace{\quad+\EE_{\set{\xb_i,\yb_i}_{i=1}^n}\left[\EE_{\set{\xb'_i,\yb'_i}_{i=1}^n}\left[ \Lt_2\right] -\Lt_1\right]}_{B}\label{equ::score esti B} \\
    &\quad+ \underbrace{\EE_{\set{\xb_i,\yb_i}_{i=1}^n}\left[\Lt_1 -\cL_1 \right]}_{A_1}+\underbrace{\EE_{\set{\xb_i,\yb_i}_{i=1}^n}\left[\cL_1\right]}_{C}. \label{equ::score esti C}
\end{align}
\paragraph{Bounding Term $A_1$ and $A_2$.}

Since we have for any $\sb \in \cF $, ($\sb$ can depend on $\xb, \yb$)
\begin{align}
    &\quad\EE_{\xb,\yb}\left[\left| \ell(\xb,\yb,\sb)-\lt(\xb,\yb,\sb) \right|\right]\nonumber \\&=\int_{t_0}^{T}\int_{\yb}\int_{\norm{\xb}>R}\EE_{\tau,\xb_t|\xb_0=\xb}\big[\norm{\sb(\xb_t,\tau\yb,t)-\nabla \log \phi_t(\xb_t|\xb_0)}_2^2\big]p(\xb|\yb)p(\yb) \diff \xb \diff \yb \diff t \nonumber\\
&\le 2\int_{t_0}^{T}\frac{1}{T-t_0}\int_{\yb}\int_{\norm{\xb}>R}\EE_{\tau,\xb_t|\xb_0=\xb}\big[\norm{\sb(\xb_t,\tau\yb,t)}_2^2+\norm{\nabla \log \phi_t(\xb_t|\xb_0)}_2^2\big] p(\xb|\yb)p(\yb) \diff \xb \diff \yb \diff t\nonumber\\
&\lesssim \int_{t_0}^{T}\frac{1}{\log N}\int_{\norm{\xb}>R}\EE_{\tau,\yb,\xb_t|\xb_0=\xb}\big[m_t^2\log N+\norm{\nabla \log \phi_t(\xb_t|\xb_0)}_2^2\big]\exp(-C_2 \norm{\xb}_2^2/2) \diff \xb \diff t\nonumber\\
&\lesssim  \exp\left(-C_2R^2\right)R\int_{t_0}^{T} m_{t}^2 \diff t +\exp\left(-C_2R^2\right)\int_{t_0}^{T}\frac{1}{\sigma^2_t}\diff t\nonumber\\
&\lesssim  \exp\left(-C_2R^2\right)RM, \label{equ::trunc popu loss}
\end{align}
where the second inequality follows from the subGaussian property of $p(\xb|\yb)$ under either assumption \ref{assump:sub} or \ref{assump::expdensity}, and the third inequality invokes the fact $\EE_{\xb_t|\xb_0=\xb}\left[\norm{\nabla \log \phi_t(\xb_t|\xb_0)}_2^2 \right]=1/\sigma_t^2$. Thus, both terms $A_1$ and $A_2$ are bounded by $\cO\left( \exp\left(-C_2R^2\right)RM\right)$.

\paragraph{Bounding Term $B$.}

For simplicity, we take $\zb=(\xb,\yb)$. We denote $\lt(\xb,\yb,\shat)$ by $\lhat(\zb)$ and $\lt(\xb,\yb,\sstar)$ by $\lstar(\zb)$.
For $\delta>0$ to be chosen later, let $\cJ=\set{\ell_1,\ell_2,...,\ell_{\cN}}$ be a $\delta$-covering of the loss function class $\cS(R)$ with the minimum cardinality in the $L^\infty$ metric in the bounded space $\cD$, and $J$ be a random variable such that $\norm{\lhat-\ell_J }_{\infty} \le\delta$. Moreover, we define $u_j=\max \left\{A,\sqrt{\EE_{\zb}\left[ \ell_j(\zb)-\lstar(\zb)\right]}\right\}$, where $\zb \sim P_{\xb,\yb}$ is independent of ${\set{\zb_i,\zb'_i}_{i=1}^n}$. Besides, we define
\begin{align*}
    D=\max_{1\le j\le \cN}\left| \sum_{i=1}^n \frac{\left(\ell_j(\zb_i)-\lstar(\zb_i)\right) -\left(\ell_j(\zb'_i)-\lstar(\zb'_i)\right)}{u_j}\right|.
\end{align*}
Then we can further bound term $B$ as follows:
\begin{align}
    \abs{B}&=\left|\EE_{\set{\zb_{i}}_{i=1}^n} \left[\frac{1}{n}\sum_{i=1}^n \left(\lhat(\zb_i)-\lstar(\zb_i)\right)  -\EE_{\set{\zb'_{i}}_{i=1}^n} \left[\sum_{i=1}^n \left(\lhat(\zb'_i)-\lstar(\zb'_i)\right) \right]\right]\right|\nonumber\\
    &=\left|\frac{1}{n} \EE_{\set{\zb_i,\zb'_i}_{i=1}^n} \left[ \sum_{i=1}^{n} \left(\left(\lhat(\zb_i)-\lstar(\zb_i)\right) -\left(\lhat(\zb'_i)-\lstar(\zb'_i)\right)\right)\right]\right|\nonumber\\ 
    &\le \left| \frac{1}{n} \EE_{\set{\zb_i,\zb'_i}_{i=1}^n} \left[ \sum_{i=1}^{n} \left(\left(\ell_J(\zb_i)-\lstar(\zb_i)\right) -\left(\ell_J(\zb'_i)-\lstar(\zb'_i)\right)\right)\right]\right| +2\delta \nonumber\\
    &\le 
   \frac{1}{n} \EE_{\set{\zb_i,\zb'_i}_{i=1}^n} \left[ u_JD \right] +2\delta \nonumber\\
    &\le \frac{1}{2} \EE_{\set{\zb_i,\zb'_i}_{i=1}^n} \left[ u_J^2 \right] +\frac{1}{2n^2}  \EE_{\set{\zb_i,\zb'_i}_{i=1}^n} \left[ D^2 \right] +2\delta. \label{equ:: excess bound}
\end{align}
Denote $h_j(\zb)=\ell_j(\zb)-\lstar(\zb)$ and $\hat{h}(\zb)=\lhat(\zb)-\lstar(\zb)$. Moreover, we define the truncated population loss as $\Rt(\shat)=\EE_{\zb}\left[ \hat{h}\right]$, and define the truncated empirical loss as $\hat{\cR_{\star}}^{\text{trunc}}(\shat)=\frac{1}{n}\sum_{i=1}^n \hat{h}(\zb_i) $. By
\eqref{equ::trunc popu loss} we know that $\abs{\Rt(\shat)-\cR_{\star}(\shat)}\lesssim \exp\left(-C_2R^2\right)RM $. Now we bound $ \EE_{\set{\zb_i,\zb'_i}_{i=1}^n} \left[ u_J^2 \right]$ and $\EE_{\set{\zb_i,\zb'_i}_{i=1}^n} \left[ D^2 \right]$ separately.

\hspace{-0.6cm}\textbf{Bounding term $ \EE_{\set{\zb_i,\zb'_i}_{i=1}^n} \left[ u_J^2 \right]$.}
By the definition of $u_J$, we have
\begin{align}
    \EE_{\set{\zb_i,\zb'_i}_{i=1}^n} \left[ u_J^2 \right]&\le A^2+ \EE_{\set{\zb_i,\zb'_i}_{i=1}^n}\left[\EE_{\zb}\left[h_J(\zb) \right] \right] \nonumber\\
    &\le A^2+ \EE_{\set{\zb_i,\zb'_i}_{i=1}^n}\left[\EE_{\zb}\left[ \hat{h}(\zb)\right]\right] +2\delta  \nonumber\\
    &=A^2+\EE_{\set{\zb_i,\zb'_i}_{i=1}^n}\left[\Rt(\shat)\right] +2\delta. \label{equ:: bound uj}
\end{align}

\hspace{-0.6cm}\textbf{Bounding term $ \EE_{\set{\zb_i,\zb'_i}_{i=1}^n} \left[ D^2 \right]$.} Denote $g_j = \sum_{i=1}^n \frac{h_j(\zb_i)-h_j(\zb'_i)}{u_j}$. It is easy to observe that $\EE_{\zb_i,\zb'_i} \left[ \frac{h_j(\zb_i)-h_j(\zb'_i)}{u_j} \right]=0$ for any $i,j$. By independence of $g_j$, we have
\begin{align*}
\EE_{\set{\zb_i,\zb'_i}_{i=1}^n} \left[\sum_{i=1}^n \left(\frac{h_j(\zb_i)-h_j(\zb'_i)}{u_j}\right)^2\right]
    &\le  \sum_{i=1}^n \EE_{\zb_i,\zb'_i} \left[\left(\frac{h_j(\zb_i)}{u_j}\right)^2+\left(\frac{h_j(\zb'_i)}{u_j}\right)^2\right]\\
    &\le M \sum_{i=1}^n \EE_{\zb_i,\zb'_i} \left[\frac{h_j(\zb_i)}{u^2_j}+\frac{h_j(\zb'_i)}{u^2_j}\right]\\
    &\le 2nM.
\end{align*}
Since $\left|\frac{h_j(\zb_i)-h_j(\zb'_i)}{u_j}\right| \le \frac{M}{A}$ and $g_j$ is centered, by Bernstein's Inequality, we have for any $j$,
\begin{align*}
    \Pr\left[g_j^2\ge h\right]=2\Pr\left[\sum_{i=1}^n\frac{h_j(\zb_i)-h_j(\zb'_i)}{u_j} \ge \sqrt{h}\right]\le 2\exp \left( -\frac{h/2}{M(2n+\frac{\sqrt{h}}{3A})} \right).
\end{align*}
Thus, we have
\begin{align*}
    \Pr\left[D^2\ge h\right]\le\sum_{j=1}^{\cN} \Pr\left[g_j^2\ge h\right]\le  2\cN\exp \left( -\frac{h/2}{M(2n+\frac{\sqrt{h}}{3A})} \right).
\end{align*}
Thus, for any $h_0>0$,
\begin{align*}
   \EE_{\set{\zb_i,\zb'_i}_{i=1}^n} \left[ D^2 \right]&=\int_{0}^{h_0} \Pr\left[D^2\ge h\right]\diff h 
  +\int_{h_0}^\infty \Pr\left[D^2\ge h\right]\diff h\\ &\le h_0+\int_{h_0}^\infty 2\cN\exp \left( -\frac{h/2}{M(2n+\frac{\sqrt{h}}{3A})} \right) \diff h \\
   &\le h_0 +2\cN \int_{h_0}^\infty \left[\exp\left(-\frac{h}{8Mn}\right)+\exp\left(-\frac{3A\sqrt{h}}{4M}\right)\right]\diff h\\
   &\le h_0 +2\cN\left[8Mn\exp\left(-\frac{h_0}{8Mn}\right)+\left(\frac{8M\sqrt{h_0}}{3A}+\frac{32M}{9A^2}\right)\exp\left(-\frac{3A\sqrt{h_0}}{4M}\right)\right]
\end{align*}
Taking $A=\sqrt{h_0}/6n$ and $h_0=8Mn\log\cN$, we have 
\begin{align}
    \EE_{\set{\zb_i,\zb'_i}_{i=1}^n} \left[ D^2 \right] &\le 8Mn\log \cN +2\left(8Mn+16Mn+\frac{16}{\log\cN}\right) \nonumber\\
    &\lesssim Mn\log \cN.\label{equ:: bound G}
\end{align}
By applying the bounds \eqref{equ:: bound uj}, \eqref{equ:: bound G} to \eqref{equ:: excess bound}, we obtain that
\begin{align*}
    &\left|\EE_{\set{\zb_i}_{i=1}^n}\left[ \hat{\cR_{\star}}^{\text{trunc}}(\shat) - \Rt(\shat)  \right] \right| \\
     &\lesssim \frac{1}{2}\left(A^2+\EE_{\set{\zb_i,\zb'_i}_{i=1}^n}\left[\Rt(\shat)\right]+2\delta
\right)+ \frac{M}{n}\log \cN +2\delta\\
&=\frac{1}{2}\EE_{\set{\zb_i}_{i=1}^n}\left[\Rt(\shat)\right]+\frac{M}{n}\log \cN+\frac{7}{2}\delta
\end{align*}
Thus, we have
\begin{equation}\label{equ:: bound trunc popu}
    \EE_{\set{\zb_i}_{i=1}^n}\left[ \Rt(\shat)  \right] \lesssim 2\EE_{\set{\zb_i}_{i=1}^n}\left[ \hat{\cR_{\star}}^{\text{trunc}}(\shat) \right] +\frac{M}{n}\log \cN+7\delta,
\end{equation}
which means that 
\begin{align*}
      B&\lesssim \EE_{\set{\xb_i,\yb_i}_{i=1}^n}\left[\Lt_1 \right]+\frac{M}{n}\log \cN+7\delta\\
      &\le \EE_{\set{\xb_i,\yb_i}_{i=1}^n}\left[\cL_1 \right]+\abs{A_1}+\frac{M}{n}\log \cN+7\delta\\
      &\lesssim C+ \exp\left(-C_2R^2\right)RM+\frac{M}{n}\log \cN+7\delta.
\end{align*}

\paragraph{Bounding Term $C$}
For any $\sb$, define $\hat{\cR_{\star}}(\sb)=\hat{\cL}(\sb)-\hat{\cL}(\sstar)$. Then we have $\cL_1=\hat{\cR_{\star}}(\shat)$. Since $\shat$ minimizes $\hat{\cL}$, we obtain that $$\hat{\cR_{\star}}(\shat)=\hat{\cL}(\shat)-\hat{\cL}(\sstar)\le \hat{\cL}(\sb)-\hat{\cL}(\sstar)=\hat{\cR_{\star}}(\sb). $$ Thus, we have
\begin{align*}
    C=\EE_{\set{\zb_i}_{i=1}^n}\left[ \hat{\cR_{\star}}(\shat)  \right]\le \EE_{\set{\zb_i}_{i=1}^n}\left[ \hat{\cR_{\star}}(\sb)\right]=\cR_{\star}(\sb). 
\end{align*}
By taking minimum w.r.t. $\sb \in \cF$, we have $C\le \min_{\sb \in \cF} \cR_{\star}(\sb)$. 
\paragraph{Balancing the error}
Now, combining the bounds for term $A_1$, $A_2$, $B$ and $C$ and plugging the log covering number \eqref{equ::logcover}, we have
\begin{align}
    \EE_{\set{\zb_i}_{i=1}^n}\left[ \cR_{\star}(\shat)  \right] & \le 2\min_{\sb \in \cF}\int_{t_0}^{T}\frac{1}{T-t_0}\EE_{\tau,\xb_t,\yb}\norm{\sb(\xb_t,\tau\yb,t)-\nabla \log p_t(\xb_t|\tau \yb)}_2^2\diff t\nonumber\\
    &\quad +\cO\left(\frac{M}{n}N^{d+d_y}\log^{9} N \left( \log^8 N+\log^2 N\log R +\log \frac{1}{\delta}  \right)\right)\nonumber \\&+\cO\left( \exp\left(-C_2R^2\right)RM\right) +7\delta. \label{equ::balance estimation error}
\end{align}
Thus, by taking $R=\sqrt{\frac{(C_\sigma+2\beta)\log N }{C_2(d+d_y)}}$ and $\delta=N^{-2\beta/(d+d_y)}$, we ensure that under either Assumption \ref{assump:sub} or \ref{assump::expdensity}, we have
\begin{align*}
    \EE_{\set{\zb_i}_{i=1}^n}\left[ \cR_{\star}(\shat)  \right] & \le 2\min_{\sb \in \cF}\int_{t_0}^{T}\frac{1}{T-t_0}\EE_{\tau,\xb_t,\yb}\norm{\sb(\xb_t,\tau\yb,t)-\nabla \log p_t(\xb_t|\tau\yb)}_2^2\diff t \\
    &\qquad +\cO\left(\frac{M}{n}N\log^{17} N \right)+\cO\left(MN^{-2\beta-C_{\sigma}}\right) \\
    & \le 2\min_{\sb \in \cF}\int_{t_0}^{T}\frac{1}{T-t_0}\EE_{\tau,\yb}\left[\EE_{\xb_t}\norm{\sb(\xb_t,\tau\yb,t)-\nabla \log p_t(\xb_t|\tau\yb)}_2^2\right]\diff t \\
    &\qquad +\cO\left(\frac{M}{n}N\log^{17} N \right)+\cO\left(N^{-\frac{2\beta}{d+d_y}}\right).
\end{align*}
We invoke the inequality $M\lesssim{\frac{1}{t_0}}=N^{C_\sigma}$ for the second inequality. Recall that for any time $t>0$ and score approximator  $\sb(\cdot,\cdot,t)$, we have
\begin{align*}
\EE_{\tau,\xb_t,\yb}\norm{\sb(\xb_t,\tau\yb,t)-\nabla \log p_t(\xb_t|\tau\yb)}_2^2&=\frac{1}{2}\int_{\R^{d}} \norm{\sb(\xb,\void,t)-\nabla\log p_{t}(\xb)}^2 p_{t}(\xb)d\xb\\
&\quad+\frac{1}{2}\EE_{\yb}\left[\int_{\R^{d}} \norm{\sb(\xb,\yb,t)-\nabla\log p_{t}(\xb|\yb)}^2 p_{t}(\xb|\yb)d\xb\right].
\end{align*}
Therefore, we can invoke the score approximation error guarantee in Section \ref{sec::approx} and Appendix \ref{sec::extend score approx} to bound the score estimation error.
Particularly, under Assumption \ref{assump:sub}, we have $M=\cO({1}/{t_0})$. By taking $N=n^{(d+d_y)/(d+d_y+\beta)}$ and invoking Theorem \ref{thm::score approx} and Proposition \ref{prop:: score approx marginal}, the error is bounded by
\begin{equation}
    \EE_{\set{\zb_i}_{i=1}^n}\left[ \cR(\shat)  \right]\le 2\EE_{\set{\zb_i}_{i=1}^n}\left[ \cR_{\star}(\shat)  \right]\lesssim \frac{1}{t_0}n^{-\frac{\beta}{d+d_y+\beta}}\log^{\max(17,d+\beta/2+1)} n.
\end{equation}
Similarly, under Assumption \ref{assump::expdensity}, we have $M=\cO(\log \frac{1}{t_0})$. By taking $N=n^{(d+d_y)/(d+d_y+2\beta)}$ and invoking Theorem \ref{thm::score approx exp} and Proposition \ref{prop:: score approx marginal exp}, the conditional score error is bounded by
\begin{equation}
    \EE_{\set{\zb_i}_{i=1}^n}\left[ \cR(\shat)  \right]\lesssim\log \frac{1}{t_0} n^{-\frac{2\beta}{d+d_y+2\beta}}\log^{\max(17,(\beta+1)/2)} n.
\end{equation}
We complete our proof.

\end{proof}

%% file: Sections/appendixTV.tex
\subsection{Proof for Theorem \ref{thm::TVbound}}\label{sec::TV proof}
Although neither Assumption \ref{assump:sub} nor \ref{assump::expdensity} ensures the Novikov's condition to hold, according to \cite{chen2022sampling}, as long as we have bounded second moment for the score estimation error and finite KL divergence w.r.t. the standard Gaussian, we could still adopt Girsanov’s Theorem and bound the KL divergence between the two distribution. 
We restate the lemma as follows:
\begin{lemma}[Proposition D.1 in \cite{oko2023diffusion}, see also 
 Theorem 2 in \cite{chen2022sampling}]\label{lemma:Girsanov}
    Let $p_0$ be a probability distribution, and let $Y=\set{Y_t}_{t\in[0,T]}$ and $Y'=\set{Y'_t}_{t\in[0,T]}$ be two stochastic processes that satisfy the following SDEs:
    \begin{align*}
        \diff Y_t&=s(Y_t,t)\diff t+\diff W_t,~~~Y_0\sim p_0 \\
        \diff Y'_t&=s'(Y'_t,t)\diff t+\diff W_t,~~~Y'_0\sim p_0
    \end{align*}
    We further define the distributions of $Y_t$ and $Y'_t$ by $p_t$ and $p_t$. Suppose that
    \begin{align}\label{equ::weak condition}
        \int_{\xb}p_t(\xb)\norm{(s-s')(\xb,t)}^2\diff \xb \le C
    \end{align}
    for any $t\in[0,T]$. Then we have
    \begin{align*}
        \text{KL} \left(p_T|p'_T\right)\le \int_{0}^{T}\frac{1}{2}\int_{\xb}p_t(\xb)\norm{(s-s')(\xb,t)}^2\diff \xb \diff t.
    \end{align*}
\end{lemma}
To prove Theorem \ref{thm::TVbound}, we also need to bound the total variation between the initial distribution and the diffused distribution at the early stopping time $t_0$, which is presented in the following lemma.
\begin{lemma}\label{lemma::t0 TV bound}
    Under either Assumption \ref{assump:sub} or \ref{assump::expdensity}, we have for any $\yb \in [0,1]^{d_y}$,
\begin{align}\label{equ::t t0 TV bound}
    \tTV(P(\cdot|\yb),P_{t_0}(\cdot|\yb))=\cO\left(\sqrt{t_0}\log ^{(d+1)/2} \frac{1}{t_0}\right).
\end{align}
\end{lemma}
The proof of Lemma \ref{lemma::t0 TV bound} is provided in Appendix \ref{sec::proof of lemma::t0 TV bound}.
With the lemmas above, we begin our proof of Theorem \ref{thm::TVbound}.
\begin{proof}[Proof of Theorem \ref{thm::TVbound}]
    Note that under either of Assumption \ref{assump:sub} or \ref{assump::expdensity} and for any $\sb \in \cF, \yb \in [0,1]^{d_y}$, we have
\begin{align*}
     \int_{\xb}p_t(\xb|\yb)\norm{\sb(\xb,\yb,t)-\nabla \log p_t(\xb|\yb)}^2\diff \xb &\lesssim  \int_{\xb}p_t(\xb|\yb) \frac{\norm{\xb}^2+C}{\sigma^4_t} \diff \xb \lesssim \frac{1}{\sigma^4_t}.
\end{align*}
Here we invoke the bound on the score function (Lemma \ref{lemma::score bound}, \ref{lemma::scorebound2}) and the bound on ReLU network $\norm{\sb}_{\infty}\le M_t\lesssim \frac{\log N}{\sigma^2_t}$ for the first inequality, and we use the subGaussian property of $p_t(\xb|\yb)$ (Lemma \ref{lemma::density bound})\footnote{The subGaussian property also holds under the stronger Assumption \ref{assump::expdensity}. We refer to \eqref{equ::pt upper bound exp} for the proof.} for the second inequality. Thus, the condition \eqref{equ::weak condition} holds for any $t\in[t_0,T]$.

Now we use another backward process as a transition term between $X^{\leftarrow}_t$ and $\tilde{X}^{\leftarrow}_t$, which is defined as
\begin{equation}
    \diff X_t^{\prime\leftarrow} = \left[ \frac{1}{2} X_t^{\prime\leftarrow} + \nabla \log p_{T-t}(X_t^{\prime\leftarrow} | \yb)\right] \diff t + \diff \bar{W}_t \quad \text{with} \quad X_0^{\prime\leftarrow} \sim  {\sf N}(0, I).
\end{equation}
We denote the distribution of $X_t^{\prime\leftarrow}$ conditional on $\yb$ by $P
^{\prime}_{T-t}(\cdot|\yb)$.

Since $X^{\prime\leftarrow}$ and $X^{\leftarrow}$ are obtained through the same backward SDE but with different initial distributions, by Data Processing Inequality and Pinsker’s Inequality (see e.g., Lemma 2 in \cite{canonne2023short}), we have
\begin{align*}
    \tTV(P_{t_0}(\cdot|\yb),P^{\prime}_{t_0}(\cdot|\yb)) &\lesssim \sqrt{\tKL (P_{t_0}(\cdot|\yb)||P^{\prime}_{t_0}(\cdot|\yb) )}\\& \lesssim \sqrt{\tKL (P_{T}(\cdot|\yb)||{\sf N}(0, I) )}\\& \lesssim \sqrt{\tKL (P(\cdot|\yb)||{\sf N}(0, I) )}\exp(-T).
\end{align*}
Thus, we could decompose the TV bound into
\begin{align}\label{equ:: TV bound conditional y}
    \tTV(P(\cdot|\yb),\tilde{P}_{t_0}(\cdot|\yb) ) &\lesssim \tTV(P(\cdot|\yb),P_{t_0}(\cdot|\yb))+ \tTV(P_{t_0}(\cdot|\yb),P^{\prime}_{t_0}(\cdot|\yb)  ) +\tTV (P^{\prime}_{t_0}(\cdot|\yb),\tilde{P}_{t_0}(\cdot|\yb)  )\nonumber \\
    &\lesssim \sqrt{t_0}\log ^{(d+1)/2} \frac{1}{t_0}+ \exp(-T) \\ &\quad+ \sqrt{\int_{t_0}^{T}\frac{1}{2}\int_{\xb}p_t(\xb|\yb)\norm{\shat(\xb,\yb,t)-\nabla \log p_t(\xb|\yb)}^2\diff \xb \diff t}.
\end{align}
By taking expectation w.r.t. $\yb$, we have
\begin{align*}
    & \quad \EE_{\yb}\left[\tTV(P_{t_0}(\cdot|\yb),\tilde{P}_{t_0} )\right] \\
    &\lesssim  \sqrt{t_0}\log ^{(d+1)/2} \frac{1}{t_0}+ \exp(-T)
+\EE_{\yb}\left[\sqrt{\int_{t_0}^{T}\frac{1}{2}\int_{\xb}p_t(\xb|\yb)\norm{\shat-\nabla \log p_t}^2\diff \xb \diff t} \right]\\
    &\lesssim   \sqrt{t_0}\log ^{(d+1)/2} \frac{1}{t_0}+\exp(-T) +\sqrt{\frac{T}{2}\cR(\shat)},
\end{align*}
where we invoke Jensen's inequality for the second inequality. Now we set $T=C_{\alpha}\log n$ for the constant $C_{\alpha}=\frac{2\beta}{d
+d_y+2\beta}$ and take expectation w.r.t. $\set{\xb_i,\yb_i}_{i=1}^n$. Again by Jensen's Inequality, we have
\begin{align}
    \EE_{\set{\xb_i,\yb_i}_{i=1}^n}\left[ \EE_{\yb}\left[\tTV(P_{t_0},\tilde{P}_{t_0} )\right] \right] &\lesssim \sqrt{t_0}\log ^{(d+1)/2} \frac{1}{t_0}+n^{-\frac{2\beta}{d
+d_y+2\beta}}+ \sqrt{\log n}  \EE_{\set{\xb_i,\yb_i}_{i=1}^n}\left[ \sqrt{\cR(\shat)} \right]\nonumber\\
    &\le \sqrt{t_0}\log ^{(d+1)/2} \frac{1}{t_0}+n^{-\frac{2\beta}{d
+d_y+2\beta}}+ \sqrt{\log n}  \sqrt{\EE_{\set{\xb_i,\yb_i}_{i=1}^n}\left[\cR(\shat)\right]}. \nonumber
\end{align}
Now we plug the bound of $\EE_{\set{\xb_i,\yb_i}_{i=1}^n}\left[\cR(\shat)\right]$ in Theorem \ref{thm::Generalization} into the inequality above. Under Assumption \ref{assump:sub}, we have
\begin{align*}
    \EE_{\set{\xb_i,\yb_i}_{i=1}^n}\left[ \EE_{\yb}\left[\tTV(P_{t_0},\tilde{P}_{t_0} )\right] \right] &\lesssim \sqrt{t_0}\log ^{(d+1)/2} \frac{1}{t_0}+n^{-\frac{2\beta}{d
+d_y+2\beta}}+   \sqrt{\frac{1}{t_0} }n^{-\frac{\beta}{2(d+d_y+\beta)}}(\log n)^{c(\beta)},
\end{align*}
where $c(\beta)=\max\left(9,d/2+\beta/4+1\right)$. We take $t_0=n^{-\frac{\beta}{4(d+d_y+\beta)}}$ to bound the expected total variation by 
$$\EE_{\set{\xb_i,\yb_i}_{i=1}^n}\left[ \EE_{\yb}\left[\tTV(P_{t_0},\tilde{P}_{t_0} )\right] \right] =\cO\left(  n^{-\frac{\beta}{4(d+d_y+\beta)}}(\log 
  n)^{c(\beta)}\right).$$
  On the other hand, under Assumption \ref{assump::expdensity}, we have
  \begin{align*}    \EE_{\set{\xb_i,\yb_i}_{i=1}^n}\left[ \EE_{\yb}\left[\tTV(P_{t_0},\tilde{P}_{t_0} )\right] \right] &\lesssim \sqrt{t_0}\log ^{(d+1)/2} \frac{1}{t_0}+n^{-\frac{2\beta}{d
+d_y+2\beta}}+   \sqrt{\log \frac{1}{t_0} }n^{-\frac{\beta}{2(d+d_y+\beta)}}(\log n)^{c'(\beta)},
  \end{align*}
  where $c'(\beta)=\max(9,(\beta+1)/2)$. We can take $t_0=n^{-\frac{4\beta}{d+d_y+2\beta}-1}$ so that $$ \sqrt{t_0}\log ^{(d+1)/2} \frac{1}{t_0} \lesssim n^{-\frac{2\beta}{d+d_y+2\beta}},~~~\text{for sufficiently large}~n.   $$ Thus, we can bound the expected total variation by $$\EE_{\set{\xb_i,\yb_i}_{i=1}^n}\left[ \EE_{\yb}\left[\tTV(P_{t_0},\tilde{P}_{t_0} )\right] \right] =\cO\left(  n^{-\frac{2\beta}{d+d_y+2\beta}}(\log 
  n)^{c'(\beta)+1/2}\right).$$ The proof is complete.
\end{proof} 
\subsection{Proof of Proposition \ref{prop::lower bound TV}}\label{sec::proof of prop::lower bound TV}

First, we derive a lower bound for the entropy number of our proposed density function class.
\begin{lemma}\label{lemma:entropy number}
    For any fixed nonnegative constants $C$, $C_2$ and $B$ such that
    \begin{align*}
        \int_{\RR^{d}} C\exp(-C_2\norm{x}_2^2) \diff \xb < 1 < \int_{\RR^{d}} B\exp(-C_2\norm{x}_2^2) \diff \xb,
    \end{align*}
    the $\epsilon-$ entropy number density function space $$\cP=\set{p(\xb)=f(\xb)\exp(-C_2\norm{x}_2^2):f(\xb)\in \mH^{\beta}(\R^d, B), f(\xb) \ge C>0}$$ with respect to $L^1$ norm in the $d-$dimensional ball $\cB=\set{\xb:\norm{\xb}_2 \le 2}$ has a lower bound
    \begin{align*}
        \log \cN(\epsilon,\cP,\norm{\cdot}_1^{\cB})  \gtrsim \left( \frac{1}{\epsilon}\right)^{\frac{d}{\beta}}.
    \end{align*}
\end{lemma}
 
    The proof of the lemma is provided in Appendix \ref{sec::proof of lemma:entropy number}. We remark that by replacing $\norm{\cdot}_1$ by $\norm{\cdot}_2$ in the proof, we can obtain the same lower bound for the entropy number of $\cP$ w.r.t. $L^2$ norm, which means that
    \begin{align*}
        \log \cN(\epsilon,\cP,\norm{\cdot}^{\cB}_1) \simeq  \log \cN(\epsilon,\cP,\norm{\cdot}^{\cB}_2) \ge \epsilon^{-d/\beta}.
    \end{align*}
    With the lemma above, we begin our proof of Proposition \ref{prop::lower bound TV}.
    \begin{proof}[Proof of Proposition \ref{prop::lower bound TV}]
        By Lemma \ref{lemma:entropy number} and the remark above, we have verified the conditions required in Theorem 4 of \cite{yang1999information} (Condition 3 of the theorem directly holds when we confine the domain of the density function on $\cB$). Applying their results gives rise to
    \begin{align*}
        \inf\limits_{\hat{\mu}} \sup\limits_{p\in \cP} \EE_{\set{\xb_i}_{i=1}^n} \left[\norm{\hat{\mu}-p}^{\cB}_1\right]\gtrsim n^{-\frac{\beta}{d+2\beta}},
    \end{align*}
    so we have
    \begin{align*}
        \inf\limits_{\hat{\mu}} \sup\limits_{p\in \cP} \EE_{\set{\xb_i}_{i=1}^n} \left[\tTV\left(\hat{\mu},p\right)\right]\ge  \inf\limits_{\hat{\mu}} \sup\limits_{p\in \cP} \EE_{\set{\xb_i}_{i=1}^n} \left[\norm{\hat{\mu}-p}^{\cB}_1\right]\gtrsim n^{-\frac{\beta}{d+2\beta}}.
    \end{align*}
    The proof is complete.
    \end{proof}

\subsection{Proof of Proposition \ref{prop:transition_estimation}}\label{sec::proof transition estimation}
When $\yb=(\sb,\ab)$ is unbounded, we can invoke the corresponding score approximation guarantee in Proposition \ref{prop:: score approx exp unbounded y} and establish the same score estimation theory by following the proof of Theorem \ref{thm::Generalization}. We present the theory as the following lemma.
\begin{lemma}[Counterpart of Theorem \ref{thm::Generalization}]\label{lemma::score estimation unbounded y}
    Suppose Assumption \ref{assump::transition kernel} holds. Given the ReLU neural network $\cF(M_t, W, \kappa, L, K)$ in Proposition \ref{prop:: score approx exp unbounded y}, by taking the network size parameter $N=n^{\frac{1}{d+d_y+2\beta}}$, the early-stopping time $t_0=n^{-\cO(1)}$ and terminal time $T=\cO(\log n)$, the empirical loss minimizer $\shat$ satisfies
    \begin{align}
    \label{equ::exp generalization unbounded y}
\EE_{\set{\sb'_i, \sb_i,\ab_i}_{i=1}^{n}}\left[ \cR(\shat)\right] =\cO\left(\log \frac{1}{t_0} n^{-\frac{2\beta}{2d_s+d_a+2\beta}}(\log n)^{\max(17,\beta) }\right).
    \end{align}
\end{lemma}
The proof of Lemma \ref{lemma::score estimation unbounded y} is provided in Appendix \ref{sec::proof of lemma::score estimation unbounded y}. Now we begin to prove Proposition \ref{prop:transition_estimation}.
\begin{proof}[Proof of Proposition \ref{prop:transition_estimation}]
By Lemma \ref{lemma::score estimation unbounded y}, we obtain a score estimator $\shat$ satisfying 
\begin{align}
    \label{equ::exp generalization unbounded y restate}
\EE_{\set{\sb'_i, \sb_i,\ab_i}_{i=1}^{n}}\left[ \cR(\shat)\right] =\cO\left(\log \frac{1}{t_0} n^{-\frac{2\beta}{2d_s+d_a+2\beta}}(\log n)^{\max(17,\beta) }\right).
    \end{align}
Given the state and action $\yb^{\star}=(\sb^{\star}, \ab^{\star})$, we can generate an estimated conditional distribution $\tilde{\cP}_{t_0}(\cdot|\sb^{\star}, \ab^{\star})$ using backward diffusion process \eqref{eq:backward approx}.
We repeat the proof of Theorem \ref{thm::TVbound} until Equation \eqref{equ:: TV bound conditional y}, obtaining that 
\begin{align*}
    \tTV\left(P(\cdot|\sb^{\star}, \ab^{\star}),\tilde{P}_{t_0}(\cdot|\sb^{\star}, \ab^{\star}) \right) 
    &\lesssim \sqrt{t_0}\log ^{(d+1)/2} \frac{1}{t_0}+ \exp(-T)\\ &\quad+ \sqrt{\int_{t_0}^{T}\frac{1}{2}\int_{\xb}p_t(\xb|\sb^{\star}, \ab^{\star})\norm{\shat(\xb,\sb^{\star}, \ab^{\star},t)-\nabla \log p_t(\xb|\sb^{\star}, \ab^{\star})}^2\diff \xb \diff t}\\
    &= \sqrt{t_0}\log ^{(d+1)/2} \frac{1}{t_0}+ \exp(-T)\\
&\quad+\sqrt{\frac{\int_{t_0}^{T}\EE_{\xb_t}\left[\norm{\shat(\xb_t,\sbb^{\star}, \ab^{\star},t)-\nabla \log p_t(\xb_t|\sbb^{\star}, \ab^{\star})}^2\right] \diff t}{\int_{t_0}^{T}\EE_{\xb_t,\sb,\ab}\left[\norm{\shat(\xb_t,\sb,\ab,t)-\nabla \log p_t(\xb_t|\sb,\ab)}^2\right] \diff t}} \cdot \sqrt{\frac{T}{2}\cR(\shat)}\\
&\le\sqrt{t_0}\log ^{(d+1)/2} \frac{1}{t_0}+ \exp(-T)+\cT(\sb^{\star}, \ab^{\star})\sqrt{\frac{T}{2}\cR(\shat)},
\end{align*}
where we invoke the definition of $\cT(\sb^{\star}, \ab^{\star})$ in the last inequality.
Taking expectations w.r.t. the samples $\set{\sb_i',\sb_i,\ab_i}_{i=1}^{n}$ and applying \eqref{equ::exp generalization unbounded y restate}, we have 
\begin{align*}
\EE_{\set{\sb_i',\sb_i,\ab_i}_{i=1}^{n}}\left[\tTV\left(P(\cdot|\sb^{\star}, \ab^{\star}),\tilde{P}_{t_0}(\cdot|\sb^{\star}, \ab^{\star}) \right) \right]
    &\lesssim \sqrt{t_0}\log ^{(d+1)/2} \frac{1}{t_0}+ \exp(-T)\\ &\quad+\cT(\sb^{\star}, \ab^{\star})\sqrt{\frac{T}{2}\log \frac{1}{t_0}n^{-\frac{2\beta}{2d_s+d_a+2\beta}}(\log n)^{\max(17,\beta)}}.
\end{align*}
 We can take $t_0=n^{-\frac{4\beta}{2d_s+d_a+2\beta}-1}$ and $T=\frac{2\beta}{2d_s+d_a+2\beta}\log n$ to bound the expected total variation by $$\EE_{\set{\sb_i',\sb_i,\ab_i}_{i=1}^{n}}\left[\tTV\left(P(\cdot|\sb^{\star}, \ab^{\star}),\tilde{P}_{t_0}(\cdot|\sb^{\star}, \ab^{\star}) \right) \right]=\cT(a)\cO\left(n^{-\frac{2\beta}{2d_s+d_a+2\beta}}(\log n)^{\max(19/2,(\beta+2)/2)}\right).$$
We complete our proof.

\end{proof}

\subsection{Proof for Other Lemmas}
\subsubsection{Proof of Lemma \ref{lemma::bound loss class}}\label{sec::proof of lemma::bound loss class}
\begin{proof}
    By the definition of $\ell(\xb,\yb,\sb)$, we have for any $\xb,\yb$ and $\sb \in \cF$
\begin{align*}
\ell(\xb,\yb,\sb)&\le 2\int_{t_0}^{T}\frac{1}{T-t_0}\EE_{\tau,\xb_t|\xb_0=\xb}\big[\norm{\sb(\xb_t,\tau\yb,t)}_2^2+\norm{\nabla \log \phi_t(\xb_t|\xb_0)}_2^2\big]  \diff t\\
&\lesssim \int_{t_0}^{T}\frac{1}{T-t_0}\EE_{\tau,\xb_t|\xb_0=\xb}\big[m_t^2\log N+\norm{\nabla \log \phi_t(\xb_t|\xb_0)}_2^2\big]  \diff t\\
&\lesssim  \int_{t_0}^{T} M_{t}^2 \diff t +\int_{t_0}^{T}\frac{1}{T-t_0}\frac{1}{\sigma^2_t}\diff t\lesssim \int_{t_0}^{T} M_{t}^2 \diff t=M,
\end{align*}
where we invoke $\abs{\sb}\lesssim m_t\sqrt{\log N}$  for the second inequality and $1/\sigma_t \lesssim m_t$ for the last inequality.
\end{proof}
\subsubsection{Proof of Lemma \ref{lemma::covering number S}}\label{sec::proof of lemma::covering number S}
We first introduce a standard result of bounding the covering number of a ReLU neural network.
\begin{lemma}[\cite{chen2022nonparametric},  Lemma.7]\label{lemma::covering number}
Suppose $\delta>0$ and the input $\zb$ satisfies $\norm{\zb}_\infty \le R$,  the $\delta-$covering number of the neural network class $\cF(W,\kappa,L,K)$ w.r.t. $\norm{\cdot}_{L_{\infty}}$ satisfies
\begin{align}
    \cN\left(\delta, \cF(W,\kappa,L,K), \norm{\cdot}_{L_{\infty}}\right) \le \left(\frac{2L^2(WR+2)\kappa^LW^{L+1}}{\delta}\right)^K.
\end{align}
\end{lemma}
We remark that our input $(\xb,\yb,t)$ is uniformly bounded by $\cO(\log N)$. 
Now we begin our proof of Lemma \ref{lemma::covering number S}.
    \begin{proof}[Proof of Lemma \ref{lemma::covering number S}]
        For any two ReLU network $\sb_1,\sb_2$ such that $\norm{\sb_1-\sb_2}_{L_{\infty}\cD}\le \epsilon$, we can bound the $L_{\infty}$ error between $\ell(\cdot,\cdot,\sb_1)$ and $\ell(\cdot,\cdot,\sb_2)$. For any $(\xb,\yb)\in \cD$, we have
\begin{align}
\abs{\ell(\xb,\yb,\sb_1)-\ell(\xb,\yb,\sb_2)}
&\le \int_{t_0}^{T}\frac{1}{T-t_0}\EE_{\tau,\xb_t|\xb_0=\xb}\big[\left(\sb_1(\xb_t,\tau\yb,t)-\sb_2(\xb_t,\tau\yb,t)\right)^{\top} \nonumber \\
& \hspace{1.2in} \cdot \left(\sb_1(\xb_t,\tau\yb,t)+\sb_2(\xb_t,\tau\yb,t)-2\phi_t(\xb_t|\xb_0)\right) \big]\diff t. \nonumber\\
&\lesssim \epsilon\int_{t_0}^{T}\frac{1}{T-t_0}\EE_{\tau,\xb_t|\xb_0=\xb}\big[\norm{\sb_1(\xb_t,\tau\yb,t)+\sb_2(\xb_t,\tau\yb,t)-2\phi_t(\xb_t|\xb_0)} \big]\diff t \nonumber\\
&\lesssim \epsilon\int_{t_0}^{T}\frac{1}{T-t_0}\EE_{\tau,\xb_t|\xb_0=\xb}\big[m_t\sqrt{\log N}+\norm{\phi_t(\xb_t|\xb_0)} \big]\diff t \nonumber\\
&\lesssim \frac{\epsilon}{T-t_0}\left(\sqrt{\log N} \int_{t_0}^{T} m_t \diff t +\int_{t_0}^{T} \frac{1}{\sigma_t} \diff t \right) \lesssim \epsilon  \log  N . \label{equ:: lip loss}
\end{align}
For the second inequality, we invoke $\abs{\sb(\xb_t,\tau\yb,t)}\le m_t \sqrt{\log N}$. 
In the last inequality, we invoke 
\begin{equation*}
    m_t\le\frac{1}{\sigma^2_t} \le \cO\left(\frac{1}{t}\right) \text{ when } t=o(1) \text{ and } m_t=\cO(1) \text{ when } t\gg 1.
\end{equation*}
and the inequality
\begin{equation*}
    \frac{1}{T-t_0} \lesssim \frac{1}{\log N}.
\end{equation*}
Since $\cF$ is a concatenation of two ReLU neural networks of the same size and the domain of the input $\zb=(\xb,\yb,t)$ (or $\zb=(\xb,t)$ for the unconditional score approximator) satisfies $\norm{(\xb,\yb,t)}_{\infty}\le \max (R,T)$, by Lemma \ref{lemma::covering number} we have the covering number of $\cF$ bounded as
\begin{equation}\label{equ::covering number ReLU}
     \cN\left(\delta, \cF, \norm{\cdot}_{L_{\infty}\cD}\right) \lesssim \left(\frac{2L^2(W\max(R,T)+2)\kappa^LW^{L+1}}{\delta}\right)^{2K}.
\end{equation}
Combining this result with \eqref{equ:: lip loss} , we can bound the covering number of $\cS(R)$ as
\begin{equation}\label{equ::covering number}
     \cN\left(\delta, \cS(R), \norm{\cdot}_{L_{\infty}\cD}\right) \lesssim \left(\frac{2L^2(W\max(R,T)+2)\kappa^LW^{L+1}\log N}{\delta}\right)^{2K}.
\end{equation}
The proof is complete.
    \end{proof}

\subsubsection{Proof of Lemma \ref{lemma::t0 TV bound}}\label{sec::proof of lemma::t0 TV bound}
\begin{proof}
    For any $\yb \in [0,1]^{d_y}$,
we have
\begin{align*}
& \quad p(\xb|\yb)-p_t(\xb|\yb) \\
&=p(\xb|\yb)-\int_{\RR^d}p(\zb|\yb)\frac{1}{\sigma_t^{d}(2\pi)^{d/2}}\exp\left(-\frac{\norm{\alpha_t\zb-\xb}^2}{2\sigma_t^2}\right)\diff \zb\\
&=\int_{\RR^d}\left(p(\xb|\yb)-p(\zb|\yb)\right)\frac{\alpha_t^d}{\sigma_t^{d}(2\pi)^{d/2}}\exp\left(-\frac{\norm{\alpha_t\zb-\xb}^2}{2\sigma_t^2}\right)\diff \zb +(\alpha_t^d-1)p_t(\xb|\yb)\\
&=\int_{A_x}\left(p(\xb|\yb)-p(\zb|\yb)\right)\frac{\alpha_t^d}{\sigma_t^{d}(2\pi)^{d/2}}\exp\left(-\frac{\norm{\alpha_t\zb-\xb}^2}{2\sigma_t^2}\right)\diff \zb \\
&\quad+\int_{\RR^d \backslash A_x}\left(p(\xb|\yb)-p(\zb|\yb)\right)\frac{\alpha_t^d}{\sigma_t^{d}(2\pi)^{d/2}}\exp\left(-\frac{\norm{\alpha_t\zb-\xb}^2}{2\sigma_t^2}\right)\diff \zb +(\alpha_t^d-1)p_t(\xb|\yb),
\end{align*}
where we take $A_x=\left[\frac{\xb-\sigma_tC\sqrt{\log\epsilon_1^{-1}}}{\alpha_t},\frac{\xb+\sigma_tC\sqrt{\log\epsilon_1^{-1}}}{\alpha_t}\right]$ for some constant $C$ such that
\begin{align*}
    \abs{\int_{\RR^d \backslash A_x}\left(p(\xb|\yb)-p(\zb|\yb)\right)\frac{\alpha_t^d}{\sigma_t^{d}(2\pi)^{d/2}}\exp\left(-\frac{\norm{\alpha_t\zb-\xb}^2}{2\sigma_t^2}\right)\diff \zb } \le \epsilon_1.
\end{align*}
By Lemma \ref{lemma::density bound} or \eqref{equ::pt upper bound exp}, we know $$p_t(\xb|\yb)\le\frac{C_1}{(\alpha_t^2+C_2\sigma_t^2)^{d/2}}\exp\left(\frac{-C_2\norm{\xb}^2_2}{2(\alpha_t^2+C_2\sigma_t^2)}\right)$$ under Assumption \ref{assump:sub} and $$p_t(\xb|\yb)\le\frac{B}{(\alpha_t^2+C_2\sigma_t^2)^{d/2}}\exp\left(\frac{-C_2\norm{\xb}^2_2}{2(\alpha_t^2+C_2\sigma_t^2)}\right)$$ under Assumption \ref{assump::expdensity}. Since $\alpha_t^2+C_2\sigma_t^2\le \max(1,C_2)$,  $p_t$ is bounded by a constant only dependent on $C_1$ (or $B$) and $C_2$. Moreover, since both $p(\xb|\yb)$ and $p_t(\xb|\yb)$ have subGaussian tails, we know that there exists another constant $C'$ such that for any $\epsilon_2 <1,$
\begin{align*}
    \int_{\RR^{d}\backslash B_x} \abs{p(\xb|\yb)-p_t(\xb|\yb)} \diff \xb \le \epsilon_2,~~\text{where}~~B_x=\left[-C'\sqrt{\log \epsilon_2^{-1}},C'\sqrt{\log \epsilon_2^{-1}}\right]^{d}.
\end{align*}
Thus, the total variation between $P(\cdot|\yb)$ and $P_{t}(\cdot|\yb)$ can be bounded as
\begin{align*}
    \tTV(P(\cdot|\yb),P_{t}(\cdot|\yb))&=\int_{B_x} \abs{p(\xb|\yb)-p_t(\xb|\yb)} \diff \xb +\int_{\RR^{d}\backslash B_x} \abs{p(\xb|\yb)-p_t(\xb|\yb)} \diff \xb \\
    &\le \int_{B_x} \int_{A_x}\left|p(\xb|\yb)-p(\zb|\yb)\right|\frac{\alpha_t^d}{\sigma_t^{d}(2\pi)^{d/2}}\exp\left(-\frac{\norm{\alpha_t\zb-\xb}^2}{2\sigma_t^2}\right)\diff \zb \diff \xb \\
    &\qquad+ \int_{B_x} \int_{\RR_d \backslash A_x}\left|p(\xb|\yb)-p(\zb|\yb)\right|\frac{\alpha_t^d}{\sigma_t^{d}(2\pi)^{d/2}}\exp\left(-\frac{\norm{\alpha_t\zb-\xb}^2}{2\sigma_t^2}\right)\diff \zb \diff \xb \\
    &\qquad+ \int_{B_x} \abs{\alpha_t^d-1}p_t(\xb|\yb) \diff \xb+\epsilon_2 \\
    &\le  \int_{B_x} \int_{A_x}\frac{2\sigma_t C\sqrt{d\log \epsilon_1^{-1}}B}{\alpha_t}\frac{\alpha_t^d}{\sigma_t^{d}(2\pi)^{d/2}}\exp\left(-\frac{\norm{\alpha_t\zb-\xb}^2}{2\sigma_t^2}\right)\diff \zb \diff \xb\\
    &\qquad+ \int_{B_x} \epsilon_1 \diff \xb +\abs{\alpha_t^d-1}+\epsilon_2\\
    &\le \left(\frac{2\sigma_t C\sqrt{d\log \epsilon_1^{-1}}B}{\alpha_t}+\epsilon_1\right)\left(2C'\sqrt{\log \epsilon_2^{-1}}\right)^{d}  +\abs{1-\exp(-dt/2)}+\epsilon_2.
\end{align*}
When $t=t_0=n^{-\cO(1)}=o(1)$, we take $\epsilon_1=\epsilon_2=t_0$. Since $\frac{\sigma_{t}}{\alpha_{t}}=\cO\left(\sqrt{t}\right)$ when $t=o(1)$, we have 
\begin{align}\label{equ::t t0 TV bound restate}
    \tTV(P(\cdot|\yb),P_{t_0}(\cdot|\yb))=\cO\left(\sqrt{t_0}\log ^{(d+1)/2} \frac{1}{t_0}\right).
\end{align}
The proof is complete.
\end{proof}

\subsubsection{Proof of Lemma \ref{lemma:entropy number}}\label{sec::proof of lemma:entropy number}
To prove Lemma \ref{lemma:entropy number}, we first introduce a standard result for the entropy number of binary variables.
\begin{lemma}[Varshamov-Gilbert bound, see, e.g., Lemma 1 in \cite{azizyan2013minimax}]\label{lemma::cover set}
    Suppose that $N\ge 8$. Let $\cI=\set{\gamma=(\gamma_1,\gamma_2,\dots,\gamma_N):\gamma_i\in\set{0,1},1\le i \le N}$. There exists $\gamma^{(1)},\gamma^{(2)}, \dots, \gamma^{(M)}\in \cI$ such that $M \ge 2^{N/8}$ and $\norm{\gamma^{(i)}-\gamma^{(j)}}_1 \ge N/8$ for $1\le i < j \le N$.
\end{lemma}
Now we begin our proof of Lemma \ref{lemma:entropy number}.
   \begin{proof}[Proof of Lemma \ref{lemma:entropy number}]
   Let $C'=\frac{1}{\int_{\RR^{d}} \exp(-C_2\norm{x}_2^2) \diff \xb}$, then we have $C < C' < B$. Denote $B'=\min(B-C',C'-C,1)$. We use the following basis function to construct a large set of functions in $\cP$ that is $\epsilon-$distinguishable. Let
        \begin{align*}
            \varphi(\xb)=\left\{ \begin{matrix}&a\prod_{i=1}^{d}(1+x_i)^{\beta}(1-x_i)^{\beta},~~~ \text{if} \norm{\xb}_{\infty} \le 1,\\
            &0,~~~ \text{otherwise},\end{matrix} \right.
        \end{align*}
        where we choose $a$ such that $\varphi(\xb) \in \cH_{\beta}(B')$. Let $c=\norm{\varphi}_1$. In the hyper ball $\cB=\set{\xb:\norm{\xb}_2 \le 1}$, we choose a $2\Delta-$ distinguishable set of points (in $L_{\infty}$ norm)
        \begin{align*}
            \xb_1,~\xb_2,...,~\xb_m,
        \end{align*}
        and we take
        \begin{align*}
            \Delta = \left(\frac{\epsilon}{a}\right)^{\frac{1}{\beta}}.
        \end{align*}
        Then we know that $m$ can be taken of order $\Delta^{-d}=\epsilon^{-d/\beta}$.
        Now we consider a set of functions in the form
        \begin{align*}
            f_{\gamma}(\xb)=\sum_{j=1}^{m} \gamma_j \Delta^{\beta} \varphi\left(\frac{\xb-\xb_j}{\Delta}\right), \gamma_j \in \set{0,1}.
        \end{align*}
        Since the support of the m basis functions $\set{\varphi\left(\frac{\xb-\xb_j}{\Delta}\right)}_{j=1}^{m}$ do not intersect, we have for any $\gamma,\gamma^{\prime} \in \set{0,1}^m$,
        \begin{align*}
            \norm{f_{\gamma}-f_{\gamma'}}^{\cB}_1=\sum_{j=1}^{m} \abs{\gamma_j-\gamma^{\prime}_j}\Delta^{\beta+d}\norm{\varphi}^{\cB}_1= c\Delta^{\beta+d}\norm{\gamma-\gamma'}_1,
        \end{align*}
        where the norm $\norm{\cdot}_1^{\cB}$ is defined as $\norm{f}_1^{\cB}=\int_{\cB}\abs{f(\xb)}\diff \xb$.
        By Lemma \ref{lemma::cover set}, there exists a subset $\cG \subseteq \set{0,1}^m$ with cardinality $\| \cG \| \ge 2^{\frac{m}{8}}$ such that for any $\gamma,\gamma^{\prime} \in \cG$ and $\gamma\neq \gamma^{\prime}$, we have 
        \begin{align*}
            \norm{\gamma-\gamma'}_1\ge \frac{m}{8}.
        \end{align*}
        Thus, we can construct a set of functions $\cU=\set{f_{\gamma}:\gamma \in \cG}$ that is $c\Delta^{\beta+d}m/8 =\Omega(\epsilon)$- distinguishable with respect to $L^1$ norm. 
        Now we consider constructing the density function as follows
        \begin{align*}
            g_{\gamma}(\xb)=\exp(-C_2\norm{\xb}_2^2)(C'+f_{\gamma}(\xb)+s_{\gamma}h(\xb,\epsilon_{\gamma}))
        \end{align*}
        where we take $s_{\gamma} \in \set{-1,1}$ and
        \begin{align*}
            h(\xb,\epsilon_{\gamma})=\epsilon_{\gamma}^{\beta} \varphi\left(\frac{\xb-2*\mathbf{1}}{\epsilon_{\gamma}}\right)
        \end{align*}
        for some parameter $\epsilon_{\gamma}\le B'$ so that
        \begin{align*}
            \int_{\RR^d}\exp(-C_2\norm{x}_2^2)(f_{\gamma}(\xb)+s_{\gamma}h(\xb,\epsilon_\gamma)) \diff \xb=0.
        \end{align*}
        Then, it is easy to check that $\int_{\RR^d} g_{\gamma}(\xb) \diff \xb =1$, so $g_{\gamma}$ is indeed a probability density function.
        Note that
        \begin{align*}
            \left|\int_{\RR^d}\exp(-C_2\norm{x}_2^2)f_{\gamma}(\xb)  \diff \xb \right|\le  \left|\int_{\RR^d}f_{\gamma}(\xb)  \diff \xb \right| \lesssim \epsilon.
        \end{align*}
        and $ h(\xb,\epsilon_{\gamma})$ is continuous w.r.t. $\epsilon_{\gamma}$ with 
        \begin{align*}
            \int_{\RR^d}\exp(-C_2\norm{x}_2^2)h(\xb,0) \diff \xb=0 ~~~\text{and}~~~ \int_{\RR^d}\exp(-C_2\norm{x}_2^2)h(\xb,B') \diff \xb =\Omega(1) \gg \epsilon.
        \end{align*}
        Thus, by the Intermediate Value Theorem, we can always find such $\epsilon_{\gamma}$ and $s_{\gamma}$.
       Moreover, since $h(\xb,\epsilon_{\gamma})=0$ if $\norm{\xb}_2 \le 1$, its support does not intersect the support of $f_\gamma$. Therefore, we ensure that $f_{\gamma}^{\star}(\xb):=C'+f_{\gamma}(\xb)+s_{\gamma}h(\xb,\epsilon_{\gamma})\ge C$ and $f_{\gamma}^{\star} \in \cH^{\beta}(B)$, which means $\exp(-C_2 \norm{\xb}^2_2)f_{\gamma}^{\star}(\xb) \in \cP$. Now for any $\gamma \neq \gamma^{\prime}$, we have
        \begin{align*}
            \norm{ g_{\gamma}- g_{\gamma^{\prime}}}^{\cB}_1&\ge\norm{ \exp(-C_2\norm{x}_2^2)(f_{\gamma}- f_{\gamma^{\prime}})}^{\cB}_1
            \ge \exp(-C_2) \norm{ f_{\gamma}- f_{\gamma^{\prime}}}^{\cB}_1
            \gtrsim \epsilon.
        \end{align*}
        Since $\log \|\cG\| \ge \log (2^{m/8}) \gtrsim m \gtrsim \epsilon^{-d/\beta}$, we have $ \log \cN(\epsilon,\cP,\norm{\cdot}^{\cB}_1)  \gtrsim \epsilon^{-d/\beta}$. We complete our proof of Lemma \ref{lemma:entropy number}.
    \end{proof}
\subsubsection{Proof of  Lemma \ref{lemma::score estimation unbounded y}} \label{sec::proof of lemma::score estimation unbounded y}
\begin{proof}
    We prove this lemma mainly by following the proof of Theorem \ref{thm::Generalization} in Appendix \ref{sec::proof generation}. For conciseness, we only present the part of proof that is different from before. The only difference is that besides truncating $\xb$, we also impose a truncation on $\yb$ so that the domain of $(\xb,\yb)$ is bounded, which is necessary for the covering number calculation. To be specific, we redefine the truncated loss function as
    \begin{equation*} \lt(\xb,\yb,\sb):=\ell(\xb,\yb,\sb)\indic{\norm{\xb}\le R,\norm{\yb}\le R}.
\end{equation*}
Moreover, denoting the truncated domain of score as $\cD=[-R,R]^{d+d_y}$ with $R=\cO(\sqrt{\log n})$, we consider the truncated loss function class defined as
\begin{equation}
\cS(R)=\set{\ell(\cdot,\cdot,\sb):\cD \rightarrow \RR\bigg|\sb \in \cF}.
\end{equation}
Then by Lemma \ref{lemma::covering number S}, we know the covering number of $\cS(R)$ can be also bounded by \eqref{equ::covering number S}. Following the proof of Theorem \ref{thm::Generalization}, we also decompose the score error $\cR(\shat)$ as  \eqref{equ::score esti A}, \eqref{equ::score esti B} and \eqref{equ::score esti C}. We use the same way to bound terms $B$ and $C$, and we add the error of truncating $\yb$ to terms $A_1$ and $A_2$. Note that we have for any $\sb \in \cF $, ($\sb$ can depend on $\xb, \yb$)
\begin{align}
&\quad\EE_{\xb,\yb}\left[\left| \ell(\xb,\yb,\sb)-\lt(\xb,\yb,\sb) \right|\right]\nonumber \\&=\int_{t_0}^{T}\int_{\yb}\int_{\norm{\xb}>R}\EE_{\tau,\xb_t|\xb_0=\xb}\big[\norm{\sb(\xb_t,\tau\yb,t)-\nabla \log \phi_t(\xb_t|\xb_0)}_2^2\big]p(\xb|\yb)p(\yb) \diff \xb \diff \yb \diff t \nonumber\\
&\quad+\int_{t_0}^{T}\int_{\norm{\yb}>R}\int_{\norm{\xb}\le R}\EE_{\tau,\xb_t|\xb_0=\xb}\big[\norm{\sb(\xb_t,\tau\yb,t)-\nabla \log \phi_t(\xb_t|\xb_0)}_2^2\big]p(\xb|\yb)p(\yb) \diff \xb \diff \yb \diff t \nonumber\\
&\lesssim \exp\left(-C_2R^2\right)RM+ \exp\left(-C_yR^2\right) RM.
\end{align}
where we repeat the derivation of \eqref{equ::trunc popu loss} using the subGaussian tails of both $p(\xb|\yb)$ and $p(\yb)$ to obtain the inequality. Thus, both terms $A_1$ and $A_2$ are bounded by $\cO\left( \exp\left(-C_3R^2\right)RM\right)$, where $C_3=\min(C_2,C_y)$. Therefore, when balancing the error terms $A_1$, $A_2$, $B$ and $C$, we can take $R=\sqrt{\frac{(C_\sigma+2\beta)\log N }{C_3}}$ instead of $R=\sqrt{\frac{(C_\sigma+2\beta)\log N }{C_2}}$ in \eqref{equ::balance estimation error} while keeping other parameter choices the same as in the proof of Theorem \ref{thm::Generalization}, so the error is still bounded by
\begin{equation}
    \EE_{\set{\sb'_i,\sb_i,\ab_i}_{i=1}^n}\left[ \cR(\shat)  \right]\lesssim\log \frac{1}{t_0} n^{-\frac{2\beta}{d+d_y+2\beta}}\log^{\max(17,(\beta+1)/2)} n.
\end{equation}
We complete our proof.
\end{proof}

%% file: Sections/appendixclassfree.tex
\section{Proof of Section \ref{sec::applications}}
\subsection{Proof of Proposition \ref{thm::subopt}}\label{sec::subopt proof}
\begin{proof}

By the definition of $\text{SubOpt} (P,y^{\star})$, for any target reward $a$, we have
\begin{align}\label{equ:: subopt ineq}
   \text{SubOpt} (\tilde{P},y^{\star}=a) &=\EE_{\xb \sim \tilde{P}(\cdot|a)}\left[ r(\xb)\right]- \EE_{\xb \sim P(\cdot|a)}\left[ r(\xb)\right]\le \text{TV}(P(\cdot|a),\tilde{P}(\cdot|a))L.
\end{align}    
According to \eqref{equ:: TV bound conditional y} in the proof of Theorem \ref{thm::TVbound}, we can obtain a score estimator $\shat$ and the corresponding generated distribution $\tilde{\cP}_t$ such that 
\begin{align*}
\tTV(P(\cdot|a),\tilde{P}_{t_0}(\cdot|a) ) 
&\lesssim \sqrt{t_0}\log ^{(d+1)/2} \frac{1}{t_0}+ \exp(-T) \\ 
& \quad + \sqrt{\int_{t_0}^{T}\frac{1}{2}\int_{\xb}p_t(\xb|a)\norm{\shat(\xb,a,t)-\nabla \log p_t(\xb|a)}^2\diff \xb \diff t}\\
&= \sqrt{t_0}\log ^{(d+1)/2} \frac{1}{t_0}+ \exp(-T)\\
&\quad+\sqrt{\frac{\int_{t_0}^{T}\EE_{\xb\sim \xb_t|a}\norm{\shat(\xb,a,t)-\nabla \log p_t(\xb|a)}^2 \diff t}{\int_{t_0}^{T}\EE_{x\sim\xb_t,a'\sim \cP_{a}}\left[\norm{\shat(\xb,a',t)-\nabla \log p_t(\xb|a')}^2\right] \diff t}} \cdot \sqrt{\frac{T}{2}\cR(\shat)}\\
&\le\sqrt{t_0}\log ^{(d+1)/2} \frac{1}{t_0}+ \exp(-T)+\cT(a)\sqrt{\frac{T}{2}\cR(\shat)},
\end{align*}
where we invoke the definition of $\cT(a)$ in the last inequality.
Taking expectations w.r.t. the samples $\set{\xb_i,y_i}_{i=1}^{n}$ and applying Theorem \ref{thm::Generalization}, we have 
\begin{align*}
\EE_{\set{\xb_i,y_i}}\left[ \tTV(P(\cdot|a),\tilde{P}_{t_0}(\cdot|a) ) \right] & \lesssim \sqrt{t_0}\log ^{(d+1)/2} \frac{1}{t_0}+ \exp(-T) \\
& \quad +\cT(a)\sqrt{T\log \frac{1}{t_0}n^{-\frac{2\beta}{d+1+2\beta}}(\log n)^{\max(17,\beta)}}.
\end{align*}
 We can take $t_0=n^{-\frac{4\beta}{d+1+2\beta}-1}$ and $T=\frac{2\beta}{d+1+2\beta}\log n$ to bound the expected total variation by $\cT(a)\cO\left(n^{-\frac{2\beta}{d+1+2\beta}}(\log n)^{\max(19/2,(\beta+2)/2)}\right)$ for sufficiently large $n$. Plugging the bound into \eqref{equ:: subopt ineq} gives rise to
\begin{align}
    \EE_{\set{\xb_i,y_i}}\left[  \text{SubOpt} (\tilde{P}_{t_0},y^{\star}=a)\right] &\lesssim L\cT(a)n^{-\frac{2\beta}{d+1+2\beta}}(\log n)^{\max(19/2,(\beta+2)/2)}
\end{align}
We complete our proof.
\end{proof}

\subsection{Proof of Proposition \ref{prop::inverse problem}}\label{sec::proof of prop::inverse problem}
\begin{proof}
    First, we derive an explicit form of the conditional score function $\log p_t(\xb|\yb)$. By the definition of forward diffusion process \eqref{eq:forward}, we have $\xb_t=\alpha_t \xb+\sigma_t\xi$, where $\xi\sim {\sf N}(0, I_d)$. By writing the equation as $\xb={\xb_t}/{\alpha_t}-{\sigma_t}\xi/{\alpha_t}$ and plugging it into \eqref{equ::inverse problem}, we obtain that
    \begin{align*}
        \yb=\frac{1}{\alpha_t}\bH\xb_t-\frac{\sigma_t}{\alpha_t}\bH\xi +\epsilon,~~~\xi\sim {\sf N}(0,  I_d),~~~\epsilon \sim {\sf N}(0, \sigma^2 I_m).
    \end{align*}
    Since $\xi$ and $\epsilon$ are independent, we obtain that the posterior distribution $p_t(\yb|\xb_t)$ satisfies
    \begin{align*}
        p_t(\yb|\xb_t) \sim {\sf N}\left(\frac{1}{\alpha_t}\bH\xb_t, \sigma^2I_m+\frac{\sigma_t^2}{\alpha_t^2}\bH\bH^{\top}\right).
    \end{align*}
    Thus, by Bayes rule, the conditional score function $\nabla_\xb \log  p_t(\xb|\yb)$ can be written as
    \begin{align}\label{equ::score decompose inverse problem}
        \nabla_\xb \log  p_t(\xb|\yb)&=\nabla_\xb \log  p_t(\yb|\xb)+\nabla_\xb \log  p_t(\xb)\nonumber \\
        &=-\left(\sigma^2I_m+\frac{\sigma_t^2}{\alpha_t^2}\bH\bH^{\top}\right)^{-1}\left(\yb-\frac{1}{\alpha_t}\bH\xb\right)+\nabla \log  p_t(\xb).
    \end{align}
    We note that the first part of the score function can be seen as a linear mapping of $\xb$ and $\yb$, i.e.,
    \begin{align*}
        -\left(\sigma^2I_m+\frac{\sigma_t^2}{\alpha_t^2}\bH\bH^{\top}\right)^{-1}\left(\yb-\frac{1}{\alpha_t}\bH\xb\right)=\left[\bA(t),\bB(t)\right]\left[\yb^{\top},\xb^{\top}\right]^{\top}.
    \end{align*}
To be specific, suppose the singular value decomposition of $\bH$ is $\bH=\bP^{\top}\bH_0\bU$, where $\bH_0=\left[\text{diag}(\mu_1,\mu_2,\dots,\mu_m),\mathbf{0},\dots,\mathbf{0}\right]$ satisfies $\abs{\mu_1}\ge\abs{\mu_2}\ge\dots\ge\abs{\mu_m}$, and $\bP\in\R^{m\times m}$, $\bU\in \R^{d\times d}$ are two orthogonal matrices. We denote $\lambda_i=\mu_i^2$ and diagonalize $\sigma^2I_m+\frac{\sigma_t^2}{\alpha_t^2}\bH\bH^{\top}$ as $\bP^{\top}\left(\sigma^2I_m+\frac{\sigma_t^2}{\alpha_t^2}\bD \right)\bP$, where $\bD=\text{diag}(\lambda_1,\lambda_2,\dots,\lambda_m)$ is a diagnal matrix. Since $\frac{\sigma_t^2}{\alpha_t^2}=e^{t}-1$, we have
    \begin{align*}
        \left(\sigma^2I_m+\frac{\sigma_t^2}{\alpha_t^2}\bH\bH^{\top}\right)^{-1}=\bP^{\top}\text{diag}
        \left( \frac{1}{\sigma^2+(e^t-1)\lambda_1},\frac{1}{\sigma^2+(e^t-1)\lambda_2},\dots,\frac{1}{\sigma^2+(e^t-1)\lambda_m} \right)\bP.
    \end{align*}
    Thus, we can express the linear mappings $\bA(t)$ and $\bB(t)$ as
    \begin{align}
        \bA(t)&=-\bP^{\top}\text{diag}
        \left( \frac{1}{\sigma^2+(e^t-1)\lambda_1},\frac{1}{\sigma^2+(e^t-1)\lambda_2},\dots,\frac{1}{\sigma^2+(e^t-1)\lambda_m} \right)\bP \label{equ::At},~~~\text{and}\\
        \bB(t)&=-\bP^{\top}\text{diag}
        \left( \frac{e^{\frac{t}{2}}}{\sigma^2+(e^t-1)\lambda_1},\frac{e^{\frac{t}{2}}}{\sigma^2+(e^t-1)\lambda_2},\dots,\frac{e^{\frac{t}{2}}}{\sigma^2+(e^t-1)\lambda_m} \right)\bP\bH\nonumber\\
        &=-\bP^{\top}\left[\text{diag}
        \left( \frac{e^{\frac{t}{2}}\mu_1}{\sigma^2+(e^t-1)\lambda_1},\frac{e^{\frac{t}{2}}\mu_2}{\sigma^2+(e^t-1)\lambda_2},\dots,\frac{e^{\frac{t}{2}}\mu_m}{\sigma^2+(e^t-1)\lambda_m} \right),\mathbf{0},\dots,\mathbf{0}\right]\bU. \label{equ::Bt}
    \end{align}

    For any $N>0$, using Lemmas \ref{lemma::product}, \ref{lemma::inv} and \ref{lemma::relu alpha}, we can construct a ReLU neural network $\cF_1(M_{t,1}, W_1,\kappa_1,L_1,K_1) $ that gives rise to a mapping $\sb_1^{\trelu}(\xb,\yb,t) $ such that
    \begin{align*}
        \norm{\sb_1^{\trelu}(\xb,\yb,t)-\left[\bA(t),\bB(t)\right]\left[\yb^{\top},\xb^{\top}\right]^{\top}}_{\infty} \le N^{-\frac{2\beta}{d}}
    \end{align*}
when $t\in[t_0,T]$ and $\norm{[\xb,\yb]}_{\infty}\le R\sqrt{\log N}$ for some constant $R>0$ to be chosen later. Moreover, we can clip the function value of $\sb_1^{\trelu}$ so that
\begin{align*}
    \norm{\sb_1^{\trelu}(\xb,\yb,t)}_{\infty}\le \max_{t\in [t_0,T],\norm{[\xb,\yb]}_{\infty}\le R\sqrt{\log N}} \norm{\left[\bA(t),\bB(t)\right]\left[\yb^{\top},\xb^{\top}\right]}_{\infty}\le \frac{R\sqrt{(d+d_y)\log N}}{\lambda_{\star}},
\end{align*}
where $\lambda_{\star}=\min_{t\ge t_0,i \in [m]}\frac{\sigma^2+(e^{t}-1)\lambda_i}{\sqrt{e^{t}\lambda_i}}$ satisfies $\lambda_{\star} \ge {\sigma^2}/{\sqrt{\lambda_i}}= \Omega(1)$ according to our assumption on $\sigma$ and $\lambda_i$. Details about how to determine the network size and the error propagation are deferred to Appendix \ref{sec::construct sb1}, where we verify that the network parameters $(M_{t,1}, W_1,\kappa_1,L_1,K_1)$ satisfy 
    \begin{align*}
    & \hspace{0.4in} M_{t,1} = \cO\left(\sqrt{\log N}\right),~
     W_1 = {\cO}\left(\log^3 N\right),\\
     & \kappa_1 =\exp \left({\cO}(\log^2 N)\right),~
    L_1 = {\cO}(\log^2 N),~
     K_1= {\cO}\left(\log^4 N\right).
\end{align*}

Furthermore, since $p(\xb)$ has subGaussian tails, we know that the distribution of $\yb$ also has subGaussian tails. Therefore, we can choose an appropriate constant $R$ and follow the proof of the score approximation theory with unbounded $\yb$ (Proposition \ref{prop:: score approx exp unbounded y}) to establish approximation guarantees with the following $L_2$ error:
    \begin{align}
\EE_{\yb \sim P_{y}}\left[\EE_{\xb \sim P_t(\cdot|\yb)} \left[\norm{\sb_1^{\star}(\xb,\yb,t)-\left[\bA(t),\bB(t)\right]\left[\yb^{\top},\xb^{\top}\right]^{\top}}^2 \right]\right]\lesssim N^{-\frac{2\beta}{d}}\log^2 N.
\end{align}
    The dependence on $\log N$ results from the truncation of $\xb$ and $\yb$.

    For the second part of the score function in \eqref{equ::score decompose inverse problem}, i.e., $\nabla \log p_t(\xb)$, we can apply our approximation theory for the unconditional distribution $p_t(\xb)$ in Proposition \ref{prop:: score approx marginal exp}. So there exists $\sb_2^{\star} \in \cF_2(M_{t,2}, W_2, \kappa_2, L_2, K_2) $ such that for any and $t \in [t_0,T]$,
\begin{align}
    \int_{\R^{d}} \norm{\sb_2^{\star}(\xb,t)-\nabla\log p_{t}(\xb)}^2 p_{t}(\xb)d\xb \lesssim \frac{1}{\sigma_t^2}B^2N^{-\frac{2\beta}{d}}(\log N)^{s+1}.
\end{align}
The hyperparameters in the network class $\cF$ satisfy
\begin{align}
    & \hspace{0.4in} M_{t,2} = \cO\left(\sqrt{\log N}/\sigma_t\right),~
     W_2 = {\cO}\left(N\log^7 N\right) \label{equ::inv:hyper1}\\
     & \kappa_2 =\exp \left({\cO}(\log^4 N)\right),~
    L_2 = {\cO}(\log^4 N),~
     K_2= {\cO}\left(N\log^9 N\right). \label{equ::inv:hyper2}
\end{align}

    By aggregating these two networks $\cF_1$ and $\cF_2$ together, we derive a ReLU network that contains a score approximator with small $L_{2}$ error of $\cO\left(\frac{1}{\sigma_t^2}B^2N^{-\frac{2\beta}{d}}(\log N)^{s+1}\right)$.  That is to say, there exists $\sb^{\star} \in \cF(M_t, W, \kappa, L, K) $ such that for any and $t \in [t_0,T]$,
\begin{align}\label{equ::score approx error inverse problem}
   \EE_{\xb_t,\yb}  \norm{\sb^{\star}(\xb_t,\yb,t)-\nabla\log p_{t}(\xb_t|\yb)}^2 \lesssim \frac{1}{\sigma_t^2}B^2N^{-\frac{2\beta}{d}}(\log N)^{s+1}.
\end{align}
Here the hyperparameters in the network class $\cF$ also satisfy \eqref{equ::inv:hyper1} and \eqref{equ::inv:hyper2}.

Now, we can plug in the score approximation error bound \eqref{equ::score approx error inverse problem} in the proof of Theorem \ref{thm::Generalization} and take $N=n^{\frac{d}{d+2\beta}}$, obtaining that 
\begin{equation}\label{}
    \EE_{\set{\zb_i}_{i=1}^n}\left[ \cR(\shat)  \right]\lesssim\log \frac{1}{t_0} n^{-\frac{2\beta}{d+2\beta}}\log^{\max(17,(\beta+1)/2)} n.
\end{equation}
After that, to convert our score estimation theory to the distribution estimation theory, we repeat the proof of Proposition \ref{prop:transition_estimation} with a similarly defined distribution shift $\cT(\yb)$. By taking $t_0=n^{-\frac{4\beta}{d+2\beta}-1}$ and $T=\frac{2\beta}{d+2\beta}\log n$, we have
 \begin{align*}
         \EE_{\set{\xb_i,\yb_i}_{i=1}^n} \left[ \tTV\left(\tilde{P}_{t_0}(\cdot|\yb),P(\cdot|\yb)  \right)\right]=\cT(\yb)\cO\left(n^{-\frac{\beta}{d+2\beta}} (\log n)^{\max(19/2,(\beta+2)/2)}\right).
     \end{align*}

To derive the estimation error of the posterior mean, we first prove that the generated distribution $\tilde{P}_{t_0}(\cdot|\yb)$ has subGaussian tails. Recall that $\tilde{P}_{t_0}(\cdot|\yb)$ is generated by the backward diffusion process
\begin{align*}
    \diff \tilde{X}_t^{\leftarrow} = \left[ \frac{1}{2} \tilde{X}_t^{\leftarrow} + \shat(\tilde{X}_t,\yb,T-t)\right] \diff t + \diff \bar{W}_t, \quad \tilde{X}_0^{\leftarrow} \sim {\sf N}(0, I),~~0\le t \le T-t_0,
\end{align*}
and $\tilde{P}_{t_0}(\cdot|\yb)$ is the distribution of $X_{T-t_0}$. By the choice of score network $\cF$, there exists a constant $C$ such that $\norm{\shat(\xb,\yb,t)}_{\infty}\le \frac{C\sqrt{\log n}}{\sigma_{t}}$ for all $\xb$, $\yb$ and $t_0\le t \le T$. Therefore, we can construct two auxiliary random variables $\tilde{Y}_t^{\leftarrow}$ and $\tilde{Z}_t^{\leftarrow}$ as the lower bound and upper bound of $X_{t}$, which satisfy the following stochastic process:
\begin{align*}
    \diff \tilde{Y}_t^{\leftarrow} &= \left[ \frac{1}{2} \tilde{X}_t^{\leftarrow} + \frac{C\sqrt{\log n}}{\sigma_{T-t}} \right] \diff t + \diff \bar{W}_t, \quad \tilde{Y}_0^{\leftarrow} \sim {\sf N}(0, I),\\
    \diff \tilde{Z}_t^{\leftarrow} &= \left[ \frac{1}{2} \tilde{X}_t^{\leftarrow} - \frac{C\sqrt{\log n}}{\sigma_{T-t}} \right] \diff t + \diff \bar{W}_t, \quad \tilde{Z}_0^{\leftarrow} \sim {\sf N}(0, I).
\end{align*}
Suppose the three processes share the same random noise $\bar{W}_t$. Then we have $\tilde{Z}_t^{\leftarrow} \le \tilde{X}_t^{\leftarrow}  \le \tilde{Y}_t^{\leftarrow}$. Let $M=\int_{t_0}^{T}\frac{C\sqrt{\log n}}{\sigma_t}\diff t = \cO(\log^{3/2} n)$. Then we have $\tilde{Y}_{T-t_0}^{\leftarrow} \sim  {\sf N}(M\cdot\mathbf{1}, I)$ and $\tilde{z}_{T-t_0}^{\leftarrow} \sim {\sf N}(-M\cdot\mathbf{1}, I)$.
Thus, by the subGaussian tail of $ \tilde{Y}_{T-t_0}^{\leftarrow}$ and $\tilde{Z}_{T-t_0}^{\leftarrow}$ we know that
\begin{align*}
    \Pr\left[\norm{\tilde{X}_{T-t_0}^{\leftarrow}}_{\infty} \ge M + u\right]\le 2 \exp(-u^2/2).
\end{align*}
Let $u=\sqrt{\frac{2\beta}{d+2\beta}\max(1/C_2,1)\log n}$. We have both
\begin{align*}
    \norm{\EE_{\xb \sim \tilde{P}_{t_0}(\cdot|\yb)}\left[\indic{\norm{\xb}_{\infty} \ge M + u}\xb\right]} \lesssim n^{-\frac{\beta}{d+2\beta}} 
\end{align*}
and
\begin{align*}
    \norm{\EE_{\xb \sim P(\cdot|\yb)}\left[\indic{\norm{\xb}_{\infty} \ge M + u}\xb\right]} \lesssim n^{-\frac{\beta}{d+2\beta}} .
\end{align*}
Therefore, we have
\begin{align*}
     &\quad\EE_{\set{\xb_i,\yb_i}_{i=1}^n} \left[\norm{\EE_{P(\cdot|\yb)}\left[\xb\right]-\EE_{\tilde{P}_{t_0}(\cdot|\yb)}\left[\xb\right]}\right]\\
    &\lesssim  n^{-\frac{\beta}{d+2\beta}} +\EE_{\set{\xb_i,\yb_i}_{i=1}^n}\left[\norm{\EE_{x\sim P(\cdot|\yb)}\left[\indic{\norm{\xb}_{\infty} < M + u}\xb\right]-\EE_{\xb \sim\tilde{P}_{t_0}(\cdot|\yb)}\left[\indic{\norm{\xb}_{\infty} < M + u}\xb\right]}\right]\\
    &\le n^{-\frac{\beta}{d+2\beta}} + \EE_{\set{\xb_i,\yb_i}_{i=1}^n} \left[ \tTV\left(\tilde{P}_{t_0}(\cdot|\yb),P(\cdot|\yb)  \right)\right](M + u)\\
    &\lesssim \cT(\yb)\cO\left(n^{-\frac{\beta}{d+2\beta}} (\log n)^{\max(11,(\beta+5)/2)}\right).
\end{align*}
We complete our proof.
\end{proof}

%% file: Sections/basicsReLU.tex
\section{Basics on ReLU Approximation} \label{sec::relu construction}

\subsection{Construction of a Large ReLU Network}\label{sec:: relu basic}
In the construction of ReLU neural networks, we often need to concatenate sub-networks that approximate
some basic functions to express more complicated functions. We provide the following lemmas for the concatenation and further operations among sub-networks.
\begin{lemma}[Concatenation,  Remark 13 of \cite{nakada2020adaptive}]\label{lemma::Concat}
For a series of ReLU networks $f_1\colon \R^{d_1}\to \R^{d_2},f_2\colon \R^{d_2}\to \R^{d_3},\cdots,f_k\colon \R^{d_k}\to \R^{d_{k+1}}$ with $f_{i}\in \cF( W_i, \kappa_i, L_i, K_i)\ (i=1,2,\cdots,d)$, there exists a neural network $f \in \cF( W, \kappa, L, K)$ satisfying $f(x) = f_{k} \circ f_{k-1} \cdots \circ f_{1}(x)$ for all $x \in\R^{d_1}$, with
    \begin{align}
        L = \sum_{i=1}^{k} L_i,\quad
     W  \leq 2\sum_{i=1}^{k} W_i,\quad
         K  \leq 2\sum_{i=1}^{k} K_{i} ,\quad \text{and }
         \kappa  \leq \max_{1\leq i \leq k} \kappa_{i}.
    \end{align} 
\end{lemma}

\begin{lemma}[Identity function]\label{lemma::identity}
    Given $d\in \N$ and $L\ge 2$, there exists  $f^{L}_{\text{id}}\in \cF(2d,1,L,2dL)$ that realizes an $L-$layer $d$-dimensional identity map $f^{L}_{\text{id}}(x)=x$, $x\in\R^d$.
\end{lemma}

\begin{proof}
The identity function can be exactly expressed by an $L-$layer ReLU network with $A_1=I_d$, $A_2=A_3=\dots=A_{L}=[I_d,-I_d]^{\top}$ and $\bb_1=\bb_2=\dots=\bb_L=\mathbf{0}_{d}$. The proof is complete.
\end{proof}
Thus, when we need to conduct operations among sub-networks with different numbers of layers $L$, we could fill in the identity networks with an appropriate number of layers before the shallow sub-networks so that all these sub-networks have the same number of layers, which brings convenience to their concatenation and further interaction. 

\begin{lemma}[Parallelization and Summation, Lemma F.3 of \cite{oko2023diffusion}]\label{lemma::Parallel}
    For any neural networks $f_{1},f_{2}, \cdots, f_{k}$ with  $f_{i}\colon\R^{d_i}\to \R^{d_i'}$ and $f_{i}\in \cF( W_i, \kappa_i, L_i, K_i)\ (i=1,2,\cdots,d)$, there exists a neural network $f\in \cF( W, \kappa, L, K)$ satisfying $f(x) = [f_{1}(x_{1})^\top\ f_{2}(x_{2})^\top\ \cdots\ f_{k}(x_{k})^\top]^\top\colon$ $ \R^{d_1+d_2+\cdots+d_k}\to\R^{d_1'+d_2'+\cdots+d_k'}$ for all $x = (x_1^\top\ x_2^\top\ \cdots \ x_k^\top)^\top \in\R^{d_1+d_2+\cdots+d_k}$ (here $x_i$ can be shared), with
    \begin{align}
    L = \max_{1\leq i \leq k}L_{i} , \quad
      W \leq 2\sum_{i=1}^{k} W_{i},\quad
         K  \leq 2\sum_{i=1}^{k} (K_{i} + Ld'_i),\quad \text{and }
         \kappa  \leq \max\{\max_{1\leq i \leq k} \kappa_{i},1\}         .
    \end{align}
    Moreover, for $x_1=x_2=\dots=x_k=x\in \R^{d}$ and $d'_1=d'_2=\dots=d'_k=d'$, there exists $f_{\rm sum}(x)\in \cF( W, \kappa, L, K)$ that expresses $f_{\text{sum}}(x) = \sum_{i=1}^k f_{i}(x)$, with
       \begin{align}
    L = \max_{1\leq i \leq k}L_{i} + 1 , \quad
      W \leq 4\sum_{i=1}^k W_{i},\quad
         K  \leq 4\sum_{i=1}^k (K_{i} +L d'_i) + 2W,\quad \text{and }
         \kappa  \leq \max\{\max_{1\leq i \leq k} \kappa_{i},1\}.         
    \end{align}
\end{lemma}

\begin{lemma}[Entry-wise Minimum and Maximum]\label{lemma::min}
    For any two neural networks $f_{1},f_{2}$ with  $f_{i}\colon\R^{d}\to \R^{d'}$, $f_{i}\in \cF( W_i, \kappa_i, L_i, K_i)\ (i=1,2)$ and $L_1\ge L_2$, there exists a neural network $f\in \cF( W, \kappa, L, K)$ satisfying $f(x) =  \min (f_1(x),f_2(x)) $ (or $ \max (f_1(x),f_2(x))$) for all $x \in\R^{d}$ , with
    \begin{align*}
    L = L_1+1 , \quad
      W \leq 2(W_1+W_2),\quad
         K  \leq 2(K_1+K_2) +2(L_1-L_2)d',\quad \text{and }
         \kappa  \leq \max\{\max_{1\leq i \leq 2} \kappa_{i},1\}         .
    \end{align*}
\end{lemma}
\begin{proof}
    First we use Lemma \ref{lemma::identity} to add $(L_1-L_2)$ layers to $f_2$ without changing its output, i.e., $f'_2=f_{\text{id}}^{L_1-L_2}\circ f_2$. Then we concatenate $f_1$ and $f'_2$ and add a new layer to realize $\max(f_1,f_2)=\sigma(f_1-f_2)+f_2$ or $\min(f_1,f_2)=f_1-\sigma(f_1-f_2)$. According to the lemmas above, the network hyperparameters $(W,\kappa,L,K) $ satisfy 
    \begin{align*}
    L = L_1+1 , \quad
      W \leq 2(W_1+W_2),\quad
         K  \leq 2(K_1+K_2) +2(L_1-L_2)d',\quad \text{and }
         \kappa  \leq \max\{\max_{1\leq i \leq 2} \kappa_{i},1\}         .
    \end{align*}
    The proof is complete.
\end{proof}
This lemma helps us to resolve problems caused by unboundedness in the sample complexity analysis of the conditional diffusion model. To be specific, we can easily apply Lemma \ref{lemma::min} to implement the clipping operation: $$f_{\text{clip},R}(\cdot):=\min\left(\max\left(\cdot,-R\right),R\right)$$ to bound the value our network within any radius $R>0$.

\subsection{Use ReLU Network to Approximate Basic Operators and Functions}\label{sec::relu operator}
In this section, we introduce how to construct ReLU networks to realize basic operations such as product, inverse (reciprocal), and square root. The lemmas below are adapted from \cite{oko2023diffusion}.

\begin{lemma}[Approximating the product, Lemma F.6 of~\cite{oko2023diffusion}]
\label{lemma::product}
    Let $d\geq 2$, $C \geq 1$.
    For any $\epsilon_{\text{product}}>0$, there exists $f_{\text{mult}}(x_1,x_2,\cdots,x_d)\in \cF( W, \kappa, L, K)$ with
   $L = \cO(\log d(\log \epsilon_{\text{product}}^{-1}+ d \log C)), W = 48d, K = \cO(d \log \epsilon_{\text{product}}^{-1} + d\log C)), \kappa=C^d$
    such that
    \begin{align}
     &   \left|f_{\text{mult}}(x_1',x_2',\cdots,x_d') - \prod_{i=1}^d x_{i}\right| \leq \epsilon_{\text{product}} + d C^{d-1} \epsilon_1.
    \end{align}
    for all $x\in [-C,C]^d\text{ and } x'\in \R$ with $\|x-x'\|_\infty \leq \epsilon_1$. 
    $|f_{\text{mult}}(x)|\leq C^d$ for all $x\in \R^d$, and $f_{\text{mult}}(x_1',x_2',\cdots,x_d')=0$ if at least one of $x_i'$ is $0$.
    
    We note that if $d=2$ and $x_1=x_2=x$, it approximates the square of $x$. We denote the network by $f_{\text{square}}(x)$ and the corresponding $\epsilon_{\text{product}}$ by $\epsilon_{\text{square}}$. Moreover, for any $\xb \in \R^d$ and $\nb \in \N^{d}$, we denote the approximation of $\xb^{\nb}=\prod_{i=1}^{d}x_i^{n_i}$ by $f_{\text{poly},\nb}(\xb)$ and the corresponding error by $\epsilon_{\text{poly}}$.
\end{lemma}

\begin{lemma}[Approximating the reciprocal function, Lemma F.7 of \cite{oko2023diffusion}]\label{lemma::inv}
    For any $0<\epsilon_{\text{inv}} <1$, there exists $f_{-1} \in \cF( W, \kappa, L, K)$ with $L= \cO(\log^2 \epsilon_{\text{inv}}^{-1}), W = \cO(\log^3 \epsilon_{\text{inv}}^{-1}), K = \cO(\log^4 \epsilon_{\text{inv}}^{-1})$, and $\kappa= \cO(\epsilon_{\text{inv}}^{-2})$ such that
    \begin{align}
        \left|f_{-1}(x') - \frac{1}{x}\right| \leq \epsilon_{\text{inv}} + \frac{|x'-x|}{\epsilon_{\text{inv}}^2}, \quad \text{for all }x\in [\epsilon_{\text{inv}},\epsilon_{\text{inv}}^{-1}] \text{ and }x'\in \R.
    \end{align}
\end{lemma}

\begin{lemma}[Approximating the square root, Lemma F.9 of \cite{oko2023diffusion}]\label{lemma::square root}
    For any $0<\epsilon_{\text{root}}<1$, there exists $f_{\text{root}} \in \cF( W, \kappa, L, K)$ with $L= \cO(\log^2 \epsilon_{\text{root}}^{-1}), W= \cO(\log^3 \epsilon_{\text{root}}^{-1}), K = \cO(\log^4 \epsilon_{\text{root}}^{-1})$, and $\kappa= \cO(\epsilon_{\text{root}}^{-1})$ such that
    \begin{align}
        \left|f_{\text{root}}(x') - \sqrt{x}\right| \leq \epsilon_{\text{root}} + \frac{|x'-x|}{\sqrt{\epsilon_{\text{root}}}}, \quad \text{for all }x\in [\epsilon_{\text{root}},\epsilon_{\text{root}}^{-1}] \text{ and }x'\in \R.
    \end{align}
\end{lemma}

\subsection{Use ReLU Network to Approximate Functions Related to $t$}\label{sec::t function}
\begin{lemma}[Approximating $\alpha_t=e^{-t/2}$]\label{lemma::relu alpha}
    For any $0<\epsilon_{\alpha}<1$, there exists $f_{\alpha} \in \cF( W, \kappa, L, K)$ with $L= \cO(\log^2 \epsilon_{\alpha}^{-1}), W= \cO(\log \epsilon_{\alpha}^{-1}), K = \cO(\log^2 \epsilon_{\alpha}^{-1})$, and $\kappa= \exp\left(\cO(\log ^2\epsilon_{\alpha}^{-1})\right)$ such that
    \begin{align}
        \left|f_{\alpha}(t) - \alpha_t\right| \leq \epsilon_{\alpha}, \quad \text{for all } t\ge 0
    \end{align}
    holds.
\end{lemma}
\begin{proof}
For a fixed \(T > 0\), to be chosen later, we utilize the Taylor expansion to establish the following inequality for \(0 \leq t \leq T\) and $k \in \N_{+}$:

\[
\left| e^{-\frac{t}{2}} - \sum_{i=0}^{k-1} \frac{(-1)^i}{i!} \left(\frac{t}{2}\right)^i \right| \leq \frac{(T/2)^k}{k!}.
\]

Since \(\frac{T^k}{k!} \leq \left(\frac{eT}{k}\right)^k\), we set \(T = 2\log 3\epsilon^{-1}_{\alpha}\) and \(k = \max(eT, \log_{2} 3\epsilon^{-1}_{\alpha})\) to bound the right-hand side by \(\frac{\epsilon_{\alpha}}{3}\). By approximating \(x^i\) using \(f_{\text{poly},i}\) in Lemma \ref{lemma::product} with \(\epsilon_{\text{poly}} = \frac{\epsilon_{\alpha}}{3k}\) and summing them up using Lemma \ref{lemma::Parallel}, we construct a ReLU neural network \(g_{\alpha} \in \mathcal{F}(W, \kappa, L, K)\) with \(L \leq \mathcal{O}(T^2 + \log^2 \epsilon_{\alpha}^{-1})\), \(W = \mathcal{O}(T + \log \epsilon_{\alpha}^{-1})\), \(K = \mathcal{O}(T^2 + \log^2 \epsilon_{\alpha}^{-1})\), and \(\kappa = \exp\left(\log T \cdot \mathcal{O}(T + \epsilon_{\alpha}^{-1})\right)\) such that

\[
\left| g_{\alpha}(t) - \sum_{i=0}^{k-1} \frac{(-1)^i}{i!} \left(\frac{t}{2}\right)^i \right| \leq \frac{\epsilon_{\alpha}}{3}.
\]

This implies \(\left| g_{\alpha}(t) - e^{-t/2} \right| \leq \frac{2\epsilon_{\alpha}}{3} < \epsilon_{\alpha}\). Finally, by adding a layer of the minimum operator and a layer of the maximum operator before this network to constrain the input \(t\) within \([0, T]\), we denote the entire network by \(f_{\alpha}\). Thus, we have \(f_{\alpha}(t) = g_{\alpha}(t)\) for \(0 \leq t \leq T\) and \(f_{\alpha}(t) = g_{\alpha}(T)\) for \(t > T\). Thus, we ensure that for any \(t > T\),

\[
\left| f_{\alpha}(t) - e^{-t/2} \right| \leq \left| e^{-t/2} - e^{-T/2} \right| + \left| f_{\alpha}(t) - e^{-T/2} \right| \leq \frac{\epsilon_{\alpha}}{3} + \frac{2\epsilon_{\alpha}}{3} = \epsilon_{\alpha}.
\]
Moreover, by the choice of $T$, we verify that the network parameters $ \cF( W, \kappa, L, K)$ satisfy 
    \begin{align*}
L &= \mathcal{O}(\log^2 \epsilon_{\alpha}^{-1}),~ W = \mathcal{O}(\log  \epsilon_{\alpha}^{-1}),~
K = \mathcal{O}(\log^2 \epsilon_{\alpha}^{-1})~ \text{ and } \kappa = \exp\left(\mathcal{O}(\log^2\epsilon_{\alpha}^{-1})\right).
\end{align*} The proof is complete.
\end{proof}

Similarly, we can readily extend the approximation of \(\alpha_t\) to \(\alpha_t^{2} = e^{t}\) by doubling the coefficients in the first linear layer. We denote the corresponding network and error as \(f_{\alpha^2}\) and \(\epsilon_{\alpha^2}\), respectively. Furthermore, the Taylor expansion technique applies to the approximation of \(1/\alpha_t = e^{t/2}\), yielding the following direct corollary.

\begin{lemma}[Approximating $1/\alpha_t=e^{t/2}$]\label{lemma::relu alpha-1}
    For any $\epsilon_{\alpha^{-1}}\in (0, 1)$ and terminal time $T>0$, there exists $f_{\alpha^{-1}} \in \cF( W, \kappa, L, K)$ with $L= \cO(T^2+\log^2 \epsilon_{\alpha^{-1}}^{-1}), W= \cO(T+\log \epsilon_{\alpha^{-1}}^{-1}), K = \cO(T^2+\log^2 \epsilon_{\alpha^{-1}}^{-1})$, and $\kappa= \exp\left(\log T\cdot \cO(T+\epsilon_{\alpha^{-1}}^{-1})\right)$ such that
    \begin{align}
        \left|f_{\alpha^{-1}}(t) -1/ \alpha_t\right| \leq \epsilon_{\alpha^{-1}} , \quad \text{for all } 0\le t \le T
    \end{align}
    holds, and $\abs{f_{\alpha^{-1}}(t)}\le \exp(T/2)$ for $t>T$.
\end{lemma}

\begin{lemma}[Approximating $\sigma_t=\sqrt{1-e^{-t}}$]\label{lemma::relu sigma}
For $\epsilon_{\sigma}\in (0, 1)$, there exists $f_{\sigma} \in \cF( W, \kappa, L, K)$ with $L= \cO(\log^2 \epsilon_{\sigma}^{-1}), W= \cO(\log^3 \epsilon_{\sigma}^{-1}), K = \cO(\log^4 \epsilon_{\sigma}^{-1})$, and $\kappa= \exp\left(\cO(\log ^2\epsilon_{\sigma}^{-1})\right)$ such that
    \begin{align}
        \left|f_{\sigma}(t) - \sigma_t\right| \leq \epsilon_{\sigma}, \quad \text{for all } t\ge \epsilon_{\sigma}
    \end{align}
    holds.
\end{lemma}
\begin{proof}
We define the network as $f_{\sigma}=f_{\text{root}}\left(1-f_{\alpha^2} \right)$. According to Lemmas \ref{lemma::square root} and \ref{lemma::relu alpha}, the approximation error gives rise to $\epsilon_{\text{root}}+\frac{\epsilon_{\alpha^2}}{\sqrt{\epsilon_{\text{root}}}}$. Thus, by setting $ \epsilon_{\text{root}}=\min(\epsilon_{\sigma}/2, \sqrt{1-e^{-\epsilon_{\sigma}}}  )=\cO(\epsilon_{\sigma}) $ and $\epsilon_{\alpha^2}=\sqrt{\epsilon_{\text{root}}}\epsilon_{\sigma}/2   $, we ensure that the total error is bounded by $\epsilon_{\sigma}$. Moreover, according to Lemmas \ref{lemma::Concat}, \ref{lemma::square root} and \ref{lemma::relu alpha}, we can verify that the network parameters $ \cF( W, \kappa, L, K)$ satisfy 
    \begin{align*}
L &= \mathcal{O}(\log^2 \epsilon_{{\sigma}}^{-1}),~ W = \mathcal{O}(\log^3 \epsilon_{{\sigma}}^{-1}),~
K = \mathcal{O}(\log^4 \epsilon_{{\sigma}}^{-1})~ \text{ and } \kappa = \exp\left(\mathcal{O}(\log^2\epsilon_{{\sigma}}^{-1})\right).
\end{align*} The proof is complete.
\end{proof}

\begin{lemma}[Approximating $\hat{\alpha}_t={e^{-t/2}}/{(C_2+(1-C_2)e^{-t}})$]\label{lemma::relu alpha hat}
    For any $0<\epsilon_{\hat{\alpha}}<1$, there exists $f_{\hat{\alpha}} \in \cF( W, \kappa, L, K)$ with $L= \cO(\log^2 \epsilon_{\hat{\alpha}}^{-1}), W= \cO(\log^3 \epsilon_{\hat{\alpha}}^{-1}), K = \cO(\log^4 \epsilon_{\hat{\alpha}}^{-1})$, and $\kappa= \exp\left(\cO(\log ^2\epsilon_{\hat{\alpha}}^{-1})\right)$ such that
    \begin{align}
        \left|f_{\hat{\alpha}}(t) - \hat{\alpha}_t\right| \leq \epsilon_{\hat{\alpha}}, \quad \text{for all } t\ge 0
    \end{align}
    holds.
\end{lemma}
\begin{proof}
We express the function with the network
\[f_{\hat{\alpha}} = f_{\text{mult}}\left(f_{-1}(C_1 + (1 - C_1) \cdot f_{\alpha^2}), f_{\alpha}\right).\]
According to Lemmas \ref{lemma::min}, \ref{lemma::product}, \ref{lemma::inv} and \ref{lemma::relu alpha}, the approximation error is bounded by
\[ \epsilon_{\text{product}} + 2 \max\left(1, \frac{1}{C_2}\right) \max\left(\epsilon_{\text{inv}} + \frac{\epsilon_{\alpha^2}}{\epsilon_{\text{inv}}^2}, \epsilon_{\alpha}\right).\]
By taking 
\begin{align*}
    \epsilon_{\text{product}} = \frac{\epsilon_{\hat{\alpha}}}{2}, \epsilon_{\text{inv}} = \frac{\epsilon_{\hat{\alpha}}}{8\max(1,1/C_2)}, \epsilon_{\alpha^2} = \frac{\epsilon_{\hat{\alpha}}\epsilon_{\text{inv}}^2}{8\max(1,1/C_2)}, \text{ and } \epsilon_{\alpha} = \frac{\epsilon_{\hat{\alpha}}}{8\max(1,1/C_2)},
\end{align*}
we ensure that the total error is bounded by \(\epsilon_{\hat{\alpha}}\). Since the reciprocals of all the error terms ($\epsilon_{\alpha}$, $\epsilon_{\text{inv}}$, e.t.c.) are polynomials of \(\epsilon_{\hat{\alpha}}\), according to Lemmas \ref{lemma::Concat}, \ref{lemma::min}, \ref{lemma::product}, \ref{lemma::inv} and \ref{lemma::relu alpha},   the parameters $(W, \kappa, L, K)$ of the entire network satisfy
\begin{align*}
L &= \mathcal{O}(\log^2 \epsilon_{\hat{\sigma}}^{-1}),~ W = \mathcal{O}(\log^3 \epsilon_{\hat{\sigma}}^{-1}),~
K = \mathcal{O}(\log^4 \epsilon_{\hat{\sigma}}^{-1})~ \text{ and } \kappa = \exp\left(\mathcal{O}(\log^2\epsilon_{\hat{\sigma}}^{-1})\right).
\end{align*}
The proof is complete.
\end{proof}

\begin{lemma}[Approximating $\hat{\sigma}_t$]\label{lemma::relu sigma hat}
   For any $0<\epsilon_{\hat{\sigma}}<1$, there exists $f_{\hat{\sigma}} \in \cF( W, \kappa, L, K)$ with $L= \cO(\log^2 \epsilon_{\hat{\sigma}}^{-1}), W= \cO(\log^3 \epsilon_{\hat{\sigma}}^{-1}), K = \cO(\log^4 \epsilon_{\hat{\sigma}}^{-1})$, and $\kappa= \exp\left(\cO(\log ^2\epsilon_{\hat{\sigma}}^{-1})\right)$ such that
    \begin{align}
        \left|f_{\hat{\sigma}}(t) - \hat{\sigma}_t\right| \leq \epsilon_{\hat{\sigma}}, \quad \text{for all } t\ge \epsilon_{\hat{\sigma}}
    \end{align}
    holds.
\end{lemma}
\begin{proof}

    Recall that 
\[
\hat{\sigma}_t =\frac{\sigma_t}{\sqrt{C_2\sigma_t^2+\alpha_t^2}}= \sqrt{\frac{1 - e^{-t}}{C_2 + (1 - C_2)e^{-t}}}.
\]

Therefore, we can express the function using the network 
\[
f_{\hat{\sigma}} = f_{\text{root}}\left(f_{\text{mult}}\left(1 - f_{\alpha}, f_{-1}(C_1 + (1 - C_1) \cdot f_{\alpha^2})\right) \right).
\]

Then, according to Lemmas \ref{lemma::min}, \ref{lemma::product}, \ref{lemma::inv}, \ref{lemma::square root} and \ref{lemma::relu alpha}, the network approximation error is bounded by
\begin{align*}
    &\epsilon_{\text{root}} + \frac{1}{\sqrt{\epsilon_{\text{root}}}}\left(\epsilon_{\text{product}} + 2\max\left(1, \frac{1}{C_2}\right)\max\left(\epsilon_{\text{inv}} + \frac{\epsilon_{\alpha^2}}{\epsilon_{\text{inv}}^2}, \epsilon_{\alpha^2}\right)\right).
\end{align*}

By setting 
\begin{align*}
\epsilon_{\text{root}} &= \min\left(\frac{\epsilon_{\hat{\sigma}}}{2}, \sqrt{\frac{1 - e^{-\epsilon_{\hat{\sigma}}}}{C_2 + (1 - C_2)e^{-\epsilon_{\hat{\sigma}}}}}\right),
\epsilon_{\text{product}} = \frac{\sqrt{\epsilon_{\text{root}}}\epsilon_{\hat{\sigma}}}{4},
\epsilon_{\text{inv}} = \frac{\epsilon_{\text{product}}}{4\max(1, 1/C_2)}
\end{align*}
and $\epsilon_{\alpha^2} = \epsilon_{\text{inv}}^3$, the total error is bounded by $\epsilon_{\hat{\sigma}}$. Moreover, according to Lemmas \ref{lemma::Concat}, \ref{lemma::min}, \ref{lemma::product}, \ref{lemma::inv}, \ref{lemma::square root} and \ref{lemma::relu alpha}, the parameters $(W, \kappa, L, K)$ of the entire network satisfy 
\begin{align*}
L &= \mathcal{O}(\log^2 \epsilon_{\hat{\sigma}}^{-1}),~ W = \mathcal{O}(\log^3 \epsilon_{\hat{\sigma}}^{-1}),~
K = \mathcal{O}(\log^4 \epsilon_{\hat{\sigma}}^{-1})~ \text{ and } \kappa = \exp\left(\mathcal{O}(\log^2\epsilon_{\hat{\sigma}}^{-1})\right).
\end{align*}
The proof is complete.
\end{proof}

We remark that we can extend the input domain from $t$ to $(x,t)$ by adding additional $d$ columns of zeros in the first linear layer of the corresponding ReLU neural network, where $x\in \R^{d}$, so we can obtain $f_{\alpha}(x,t)$ (or $f_{\alpha^{-1}}(x,t)$, $f_{\sigma}(x,t)$) to approximate $\alpha_t$ (or $1/\alpha_{t}$, $\sigma_t$) with the same error $\epsilon$. The network width parameters $W$ and the measure of sparsity $K$ only increase by a constant linearly dependent on $d$.
\subsection{Omitted Construction Details in the Proof}
\subsubsection{Construction of $\fb_3^{\rm ReLU}$ in Figure \ref{fig:ReLU2} for the Proof of Proposition~\ref{prop::score approx bounded}} \label{sec::f3 relu}
We elaborate on the choice of accuracy $\epsilon$ in the implementation of the basic operations (product, inverse, etc.) and the components of the network ($f^{\trelu}_1$, $\fb^{\trelu}_2$, $\sigma_t^{\trelu}$). According to the lemmas in the last three sections, we know that the approximation error of the entire network can be bounded as
\begin{align*}
    \epsilon_{\text{score}}&\le \max \left(\epsilon_{\text{square}}+2C_7\left(\epsilon_{\text{inv},3}+\frac{\epsilon_{\sigma,2}}{\epsilon_{\text{inv},3}^2}\right), \epsilon_{\text{product}}+3C_6^2 \epsilon_1\right),~~~\text{where} \\
    \epsilon_1&=\max \left(\epsilon_{\text{inv},1}+\frac{\epsilon_{f_1}}{\epsilon^2_{\text{inv},1}},\epsilon_{\fb_2}, \epsilon_{\text{inv},2}+\frac{\epsilon_{\sigma,1}}{\epsilon^2_{\text{inv},2}} \right).
\end{align*}
For the two upper bounds $C_6$ and $C_7$ which behave as the parameter $C$ in Lemma \ref{lemma::product}, $$C_6=\max\left(\frac{1}{\epsilon_{\text{low}}},\frac{ C_5}{\sigma_{{t_0}}} (C_x\sqrt{d\log N}+1),\frac{1}{\sigma_{{t_0}}} \right)$$ is the maximum of ${1}/{f_1^{\text{clip}}}$, $\fb_2$ and $\sigma^{-1}_t$ with $t\in [t_0,T]$ and $\xb \in [-C_x\sqrt{\log N},C_x\sqrt{\log N}]$, and 
$$C_7=\max_{t_0\le t\le T}\frac{C_5(C_x\sqrt{d\log N}+1)}{\sigma^2_t}=\frac{C_5(C_x\sqrt{d\log N}+1)}{\sigma^2_{{t_0}}}.$$ 
Now we choose a set of parameters to ensure that $\epsilon_{\text{score}}\le N^{-\beta}$. To be specific, for the three inverse operators, we set
$$\epsilon_{\text{inv},1}=\epsilon_{\text{inv},2}=\frac{N^{-\beta}C_6^{-2}}{8}~~~\text{and}~~~\epsilon_{\text{inv},3}=\frac{N^{-\beta}C_7^{-2}}{6}$$ in Lemma \ref{lemma::inv}.
Moreover, to approximate $f_1$ and $\fb_2$, we choose 
$$\epsilon_{f_1}=\frac{N^{-\beta}\epsilon_{\text{inv},1}^{2}C_6^{-2}}{8}~~~~\text{and}~~~~ \epsilon_{\fb_2}=\frac{N^{-\beta}C_6^{-2}}{4}$$ in Lemmas \ref{lemma::relu approx diffused} and \ref{lemma::relu approx diffused1}, respectively. 
To approximate $\sigma_t$,  we choose $$\epsilon_{\sigma,1}=\frac{N^{-\beta}\epsilon_{\text{inv},2}^2C_6^{-2}}{8}~~\text{and}~~\epsilon_{\sigma,2}=\frac{N^{-\beta}\epsilon_{\text{inv},2}^2C_7^{-1}}{6}~~ $$ in Lemma \ref{lemma::relu sigma}.  Now by the definition of $\epsilon_1$, we have
$$\epsilon_1=\max \left(\epsilon_{\text{inv},1}+\frac{\epsilon_{f_1}}{\epsilon^2_{\text{inv},1}},\epsilon_{\fb_2}, \epsilon_{\text{inv},2}+\frac{\epsilon_{\sigma,1}}{\epsilon^2_{\text{inv},2}} \right)=\frac{N^{-\beta}}{4C_6^2}. $$
Thus, by taking $\epsilon_{\text{product}}=\frac{1}{4}N^{-\beta}$ we construct a network that approximates ${\fb_2}/{(\sigma_t f_1^{\text{clip}})}$ with error bounded by $\epsilon_{\text{product}}+3C^2_{6}\epsilon_1= N^{-\beta}$. 
Then for the square operator, we take $\epsilon_{\text{square}}=\frac{1}{3}N^{-\beta}$ so that $$\epsilon_{\text{square}}+2C_7\left(\epsilon_{\text{inv},3}+{\epsilon_{\sigma,2}}/{\epsilon_{\text{inv},3}^2}\right) \le N^{-\beta}.$$ Last, since the entry-wise minimum operator does not induce additional approximation error, the approximation error of the entire network is exactly bounded by $N^{-\beta}$. We remark that the reciprocals of all the error terms ($\epsilon_{\Phi},\epsilon_{g}$, e.t.c.) and the upper bound parameters ($C_6$, $C_7$) are in the order of $\exp (\cO(\log N))$. Also, the entry-wise minimum operator indicates that the output value of the network is bounded by $\cO\left(\sqrt{\log N}/\sigma^2_t\right)$. Therefore, by Lemma \ref{lemma::Concat} and the lemmas we mention above, the hyperparameters $(M_t, W, \kappa, L, K)$ of the entire network satisfy
\begin{align*}
    & \hspace{0.4in} M_t = \cO\left(\sqrt{\log N}/\sigma^2_t\right),~
     W = {\cO}\left(N^{d+d_y}\log^7 N\right),\\
     & \kappa =\exp \left({\cO}(\log^4 N)\right),~
    L = {\cO}(\log^4 N),~
     K= {\cO}\left(N^{d+d_y}\log^9 N\right).
\end{align*}
We complete our proof.
\subsubsection{Construction of $f_{v,k,j}$ in Figure \ref{fig:ReLU1} for the proof of Lemma \ref{lemma::approx h}}\label{sec::fvkj relu}
Similarly, the total error can be written as
\begin{align*}
    \epsilon_{f}&=\epsilon_{\text{product},1}+4C_8^{6(2k+j+1)}\max\left(\epsilon_1,\epsilon_2,\epsilon_3,\epsilon_4 \right) ,~~\text{where}\\
C_8&=\max\left(\frac{1}{\sigma_{t_0}},C_x\sqrt{\log N} \frac{R}{\alpha_T} \right)=\text{Poly}(N).
\end{align*}
For $\epsilon_1$, $\epsilon_2$, $\epsilon_3$ and $\epsilon_4$, we define
$$\epsilon_1=\epsilon_{\text{poly},1}+(2k+j+1)C_8^{2(2k+j)}\left(\epsilon_{\text{product},2}+2\max \left(\epsilon_{\text{inv}},2C_8R\epsilon_{\alpha,1} \right) \right)$$ as the error of approximating $f_{\overline{D}}^{j+2k+1}(x)-f_{\underline{D}}^{j+2k+1}(x)$, and $$\epsilon_2=\epsilon_{\text{poly},2}+j\sigma^{-(j-1)}_{t_0}\epsilon_{\sigma},~~\epsilon_3=\epsilon_{\text{poly},3}+(n+1)\alpha_t^{-n}\epsilon_{\alpha,2},~~\epsilon_4=\epsilon_{\text{poly},4}+jC_8^{j-1}R\epsilon_{\alpha,3}$$ are the errors of approximating $\sigma_t^{j}$, $\alpha_t^{-(n+1)}$ and $(x+\alpha_tR/2-\alpha_tRv/N)^{n-j}$, respectively. To ensure that $\epsilon_{f}\le \epsilon$, we take $\epsilon_{\text{product},1}=\frac{\epsilon}{2}$ and choose a set of error terms so that $$\max\left(\epsilon_1,\epsilon_2,\epsilon_3,\epsilon_4 \right)\le \frac{\epsilon}{8C_8^{6(2k+j+1)}}:=\epsilon_{\star}.$$ To be specific, to approximate the power operators ($f_{\text{poly}}$ in the remark of Lemma \ref{lemma::product}), we take $$\epsilon_{\text{poly},1}=\epsilon_{\text{poly},2}=\epsilon_{\text{poly},3}=\epsilon_{\text{poly},4}=\frac{\epsilon_{\star}}{2}.$$ 
Moreover, in the approximation of the inverse operator (Lemma \ref{lemma::inv}) and the second product operator (Lemma \ref{lemma::product}), we set $$\epsilon_{\text{product},2}=\frac{\epsilon_{\star}}{4(2k+j+1)C_8^{2(2k+j)}} ~~~\text{and}~~~\epsilon_{\text{inv}}=\frac{\epsilon_{\text{product},2}}{2},$$ respectively. Last, to approximate $\sigma_t$ and $\alpha_t$, we take $$\epsilon_{\sigma}=\frac{\epsilon_{\star}}{2j\sigma_{t_0}^{j-1}}~~~\text{and}~~~\epsilon_{\alpha,1}=\epsilon_{\alpha,2}=\epsilon_{\alpha,3}=\min\left(\frac{\epsilon_{\text{product},2}}{4C_8R}, \frac{\epsilon_{\star}}{2jC_8^{j-1}R}\right).$$ Then by the definition of $\epsilon_1$, $\epsilon_2$, $\epsilon_3$ and $\epsilon_4$, it is easy to verify that $\max\left(\epsilon_1,\epsilon_2,\epsilon_3,\epsilon_4 \right)\le \epsilon_{\star}$. Note that $j\le n \le s$ and $k \le p=\cO(\log N)$, so the reciprocals all the error terms ($\epsilon_{\text{inv}},\epsilon_{\sigma}$, e.t.c.) are in the order of $\exp (\cO(\log^2 N +\log \epsilon^{-1}))$. Thus, the network parameters $(W, \kappa, L, K)$ of the entire network satisfy
\begin{align*}
    & W = {\cO}\left(\log^6 N +\log^3 \epsilon^{-1} \right),
      \kappa =\exp \left({\cO}(\log^4 N+\log^2 \epsilon^{-1})\right),\\
      ~&L = {\cO}(\log^4 N+\log^2 \epsilon^{-1}),~
     K= {\cO}\left(\log^8 N+\log^4 \epsilon^{-1}\right).
\end{align*}
The proof is complete.
 \subsubsection{Construction of $\fb_3^{\trelu}$ in Figure \ref{fig:ReLU4} for the Proof of Proposition \ref{prop::logh approx bounded exp}} \label{sec::f3 new relu}
 According to the figure, the total error can be written as
 \begin{align*}
     \epsilon_{\text{score}}&=\epsilon_{\text{product},1}+\epsilon_{\text{product},2}+2C_9\left(\epsilon_{\text{inv},1}+\frac{\epsilon_{\alpha^2}}{\epsilon^2_{\text{inv},1}}\right)+4C_{10}^3\max\left(\epsilon_{\hat{\sigma}^{-1}},\epsilon_{f_1^{-1}},\epsilon_{\fb_2},\epsilon_{\hat{\alpha}} \right), \text{where}\\
     C_9&=\max\left(C_2C_x\sqrt{\log N},\frac{1}{C_2+(1-C_2)\alpha_{T}}\right)~\text{ and }~C_{10}=\max_{t_0\le t \le T}\left(\hat{\sigma}_t^{-1},\frac{2}{C_1},B,\hat{\alpha}_t\right).
 \end{align*}
Here $$\epsilon_{\hat{\sigma}^{-1}}=\epsilon_{\text{inv},2}+\frac{\epsilon_{\hat{\sigma}}}{\epsilon_{\text{inv},2}^2}~~~\text{and}~~~\epsilon_{f_1^{-1}}=\epsilon_{\text{inv},3}+\frac{\epsilon_{f_1}}{\epsilon_{\text{inv},3}^3}$$ are the errors of approximating $\hat{\sigma}^{-1}_t$ and $f_1^{-1}$, respectively. Now we choose a set of error terms to ensure that $\epsilon_{\text{score}}\le N^{-\beta}$. Specifically, to approximate $f_1$ and $\fb_2$, we take
$$\epsilon_{f_1}=\frac{N^{-\beta}\epsilon_{\text{inv},2}^2}{32C_{10}^3} ~~~\text{and}~~~\epsilon_{\fb_2}= \frac{N^{-\beta}}{16C_{10}^3} $$
in Lemmas \ref{lemma::relu approx diffused exp} and \ref{lemma::relu approx diffused1 exp}, respectively. Moreover, in the approximation of $\alpha^2_t$, $\hat{\sigma}_t$ and $\hat{\alpha}_t$, we set
$$ \epsilon_{\alpha^2}=\frac{N^{-\beta}\epsilon_{\text{inv},1}^2}{16C_9},~~~ \epsilon_{\hat{\sigma}}=\frac{N^{-\beta}\epsilon_{\text{inv},2}^2}{32C_{10}^3}~~~\text{and}~~~\epsilon_{\hat{\alpha}}= \frac{N^{-\beta}}{16C_{10}^3}.$$
Last, to approximate the two product operators (Lemma \ref{lemma::product}) and the three inverse operators (Lemma \ref{lemma::inv}), we take
$$\epsilon_{\text{product},1}=\epsilon_{\text{product},2}=\frac{N^{-\beta}}{4},~~~\epsilon_{\text{inv},1}=\frac{N^{-\beta}}{16C_9}~~~\text{and}~~~\epsilon_{\text{inv},2}=\epsilon_{\text{inv},3}=\frac{N^{-\beta}}{32C_{10}^3}.$$ 
Since the reciprocals of all the error terms ($\epsilon_{\hat{\sigma}},\epsilon_{\hat{\alpha}}$, e.t.c.) and the upper bound parameters ($C_9$ and  $C_{10}$) are in the order of $\exp (\cO(\log N))$, the network hyperparameters $(M_t, W, \kappa, L, K)$ of the entire network satisfy
\begin{align*}
    & \hspace{0.4in} M_t = \cO\left(\sqrt{\log N}/\sigma_t\right),~
     W = {\cO}\left(N^{d+d_y}\log^7 N\right),\\
     & \kappa =\exp \left({\cO}(\log^4 N)\right),~
    L = {\cO}(\log^4 N),~
     K= {\cO}\left(N^{d+d_y}\log^9 N\right).
\end{align*}
The proof is complete.
\subsubsection{Construction of $f_{v,k,j}$ in Figure \ref{fig::ReLU3} for the Proof of Lemma \ref{lemma::approx g new}}\label{sec::fvkj new}
Similarly, the total approximation error of the network is bounded by
\begin{align*}
\epsilon_{f}&=\epsilon_{\text{product},1}+3C_{11}^{4(2k+j+1)}\max\left(\epsilon_1,\epsilon_2,\epsilon_3\right), \text{where}\\
C_{11}&=2\max_{t_0\le t \le T}\left(\frac{1}{\hat{\sigma}_{t}},\hat{\sigma}_{t},C_x\sqrt{\log N}\hat{\alpha}_t+R \right)=\text{Poly}(N).
\end{align*}
Here $$\epsilon_1=2\epsilon_{\text{poly},1}+4(2k+j+1)R^{2k+j}\left(\epsilon_{\text{product,2}}+2C_{11}\max\left(\epsilon_{\text{product},3}+2C_{11}\epsilon_{\hat{\alpha},1}, \epsilon_{\text{inv}}+\frac{\epsilon_{\hat{\sigma}},1}{\epsilon_{\text{inv}}^2}\right)\right)$$ is the error of approximating $f_{\overline{D}}^{j+2k+1}(x)-f_{\underline{D}}^{j+2k+1}(x)$, and $$
\epsilon_2=\epsilon_{\text{poly},2}+jC_{11}^{j-1}\epsilon_{\hat{\sigma},2
},~~~\epsilon_{3}=\epsilon_{\text{poly},3}+(n-j)C_{11}^{n-j-1}\left(\epsilon_{\text{product},4}+2C_{11}\epsilon_{\hat{\alpha},2}\right) $$ are the errors of approximating $\hat{\sigma}_t^j$ and $\left(\hat{\alpha}_t x+R/2-vR/N\right)^{n-j}$, respectively. To ensure $\epsilon_{f}\le \epsilon,$ we choose $\epsilon_{\text{product},1}={\epsilon}/{2}$ and set other error terms so that
$$\max\left(\epsilon_1,\epsilon_2,\epsilon_3\right)\le \frac{\epsilon}{6C_{11}^{4(2k+j+1)}}=:\epsilon_{\star}.$$ To be specific, we set 
$$\epsilon_{\text{product},2}=\frac{\epsilon_{\star}}{4(2k+j+1)R^{2k+j}},~~\epsilon_{\text{product},3}=\epsilon_{\text{inv}}=\frac{\epsilon_{\text{product},2}}{4C_{11}},~~\text{and}~~\epsilon_{\text{product},4}=\frac{\epsilon_{\star}}{4(n-j)C_{11}^{n-j-1}}  $$
for the remaining three product operators and the inverse operator. Moreover, in the approximation of the power operators (remark of Lemma \ref{lemma::product}), we take
$$\epsilon_{\text{poly},1}=\frac{\epsilon_{\star}}{4},~~~\text{and}~~~\epsilon_{\text{poly},2}=\epsilon_{\text{poly},3}=\frac{\epsilon}{2} .$$ Last, to approximate $\hat{\sigma}_t$ and $\hat{\alpha}_t$, we take 
$$\epsilon_{\hat{\sigma},1}=\epsilon_{\text{inv}}^3,~~\epsilon_{\hat{\sigma},2}=\frac{\epsilon_{\star}}{2jC_{11}^{j-1}},~~\epsilon_{\hat{\alpha},1}=\frac{\epsilon
_{\text{product},2}}{8C_{11}^2}~~\text{and}~~\epsilon_{\hat{\alpha},2}=\frac{\epsilon_{\text{product},4}}{2C_{11}}.$$ Thus, we have $\epsilon_{f}\le \epsilon$. Since $j\le n \le s$ and $k \le p=\cO(\log N)$, the reciprocals of all the error terms ($\epsilon_{\hat{\alpha}},\epsilon_{\hat{\sigma}}$, e.t.c.) are in the order of $\exp (\cO(\log^2 N +\log \epsilon^{-1}))$. Thus, the network parameters $(W, \kappa, L, K)$  satisfy
\begin{align*}
    & W = {\cO}\left(\log^6 N +\log^3 \epsilon^{-1} \right),
      \kappa =\exp \left({\cO}(\log^4 N+\log^2 \epsilon^{-1})\right),\\
      ~&L = {\cO}(\log^4 N+\log^2 \epsilon^{-1}),~
     K= {\cO}\left(\log^8 N+\log^4 \epsilon^{-1}\right).
\end{align*}
The proof is complete.
\subsubsection{Construction of $\dmono$ for the Proof of Lemma \ref{lemma::relu approx diffused}}\label{sec::f1 relu}
To construct $\dmono$, we use the following ReLU network:
\begin{align*}
    \dmono^{\trelu}=f_{\text{mult}}\left(f_{\text{poly},\nb^{\prime}}\left(\yb-\frac{\wb}{N} \right), \set{\phi\left(3N\left(y_j-\frac{\wb}{N}\right)\right)}_{j\in [d_y]}, \set{\sum_{k<p} g^{\trelu}(x_i,n_i,v_i,k)}_{i \in [d]}\right)
\end{align*}
According to Lemmas \ref{lemma::trapezoid}, \ref{lemma::approx h} and \ref{lemma::product}, the approximation error can be written as 
\begin{align*}
    \epsilon_{\Phi}&=\epsilon_{\text{product}}+(d+d_y+1)C_{12}^{d+d_y}\max \left(\epsilon_{\text{poly}},p\epsilon_{g} \right),~~ \text{where }\\
    C_{12}&=\max_{\norm{\xb}_{\infty}\le C_x\sqrt{\log N}, i \in [d]} \sum_{k<p} g^{\trelu}(x_i,n_i,v_i,k)
\end{align*}
satifies $\log C_{12} =\cO(\log^2 N)$ and $p=\cO(\log N)$. Here $\epsilon_{g}$ represents the uniform approximation error of $g(x,n,v,k)$. Denote $\epsilon_{\star}={s!(d+d_y)^{-s}R^{-s}N^{-(d+d_y)}\epsilon}$. By taking $$\epsilon_{\text{product}}=\frac{\epsilon_{\star}}{2},~~\epsilon_{\text{poly}}=\frac{\epsilon_{\star}}{2(d+d_y+1)C_{12}^{d+d_y}},~~\text{and}~~\epsilon_{g}=\frac{\epsilon_{\star}}{2p(d+d_y+1)C_{12}^{d+d_y}},$$ we ensure that $\epsilon_{\Phi}\le \epsilon_{\star}$. Moreover, we note that the reciprocals of all the error terms ($\epsilon_{\Phi},\epsilon_{g}$, e.t.c.) are in the order of $\exp (\cO(\log^2 N +\log \epsilon^{-1}))$. 
Thus, according to Lemma \ref{lemma::Concat}, we can verify that the network parameters $(W, \kappa, L, K)$ satisfy
\begin{align*}
    & W = {\cO}\left(\log^7 N +\log N \log^3 \epsilon^{-1} \right),
      \kappa =\exp \left({\cO}(\log^4 N+\log^2 \epsilon^{-1})\right),\\
      ~&L = {\cO}(\log^4 N+\log^2 \epsilon^{-1}),~
     K= {\cO}\left(\log^9 N+\log N \log^4 \epsilon^{-1}\right).
\end{align*}

\subsubsection{Construction of $\sb_1^{\trelu}$ in the proof of Proposition \ref{prop::inverse problem}}\label{sec::construct sb1}
According to \eqref{equ::At} and \eqref{equ::Bt},  $\bA(t)\yb+\bB(t)\xb$ can be written as
\begin{align*}
\bA(t)\yb+\bB(t)\xb=-\sum_{i=1}^{m}\frac{\pb_i\pb_i^{\top}\yb+e^{\frac{t}{2}}\mu_i\pb_i\ub_i^{\top}\xb}{\sigma^2+(e^t-1)\lambda_i},
\end{align*}
where $\set{\pb_i}_{i=1}^{m}$ and $\set{\ub_i}_{i=1}^{m}$ are the (first) $m$ row vectors of $\Pb$ and $\Ub$, respectively. To construct a ReLU approximation, we first consider the following functions:
\begin{align*}
    g^{\trelu}_{i}(\xb,\yb,t)=f_{\text{mult}}\left(-\pb_i^{\top}\yb-\mu_i f_{\text{mult}}\left(f_{\alpha}(t),\ub_i^{\top}\xb\right),f_{\text{inv}}\left(\sigma^2+\lambda_i(f_{\alpha^2}(t)-1)\right)\right)\pb_i.
\end{align*}
 Afterward, we sum them up and clip the function value  to construct our target ReLU approximation $\sb_{1}^{\trelu}$, which is given as
 \begin{align*}
     \sb_{1}^{\trelu}=f_{\text{clip},R'}\left(\sum_{i=1}^{m} g^{\trelu}_{i}\right),~~~\text{where}~~~R'=\frac{R\sqrt{(d+d_y)\log N}}{\lambda_{\star}}=\cO\left(\sqrt{\log N}\right).
 \end{align*}
 According to Lemmas \ref{lemma::product}, \ref{lemma::inv} and \ref{lemma::relu alpha}, the approximation error of the entire network can be bounded by
 \begin{align*}
     \norm{\sb_1^{\trelu}(\xb,\yb,t)-\bA(t)\yb-\bB(t)\xb}_{\infty}\le mC_{13}\max\left(C_{14}\epsilon_{\alpha}+\epsilon_{\text{product},2},\epsilon_{\text{inv}}+\frac{\lambda_i\epsilon_{\alpha^2}}{\epsilon_{\text{inv}}^2}\right)+\epsilon_{\text{product,1}},
 \end{align*}
where the constants
 \begin{align*}
     C_{13}=\max\left(\left(\sqrt{d_y}+\mu_i e^{T/2}\sqrt{d}\right)R\sqrt{\log N},\frac{1}{\sigma^2+(e^{t_0}-1)\min_i\lambda_i}\right), C_{14}=\max\left(e^{T},R\sqrt{d\log N}\right)
 \end{align*}
both satisfy $\log C_j=\cO(\log N)$. 
Now, we take
 \begin{align*}
     \epsilon_{\text{product},1}= \frac{N^{-2\beta/d}}{2},~\epsilon_{\text{product},2}=\epsilon_{\text{inv}}=\frac{N^{-2\beta/d}}{4mC_{13}},~\epsilon_{\alpha}=\frac{N^{-2\beta/d}}{4mC_{13}C_{14}}, ~\text{and}~~\epsilon_{\alpha^2}=\frac{\epsilon_{\text{inv}}^3}{\lambda_i}
 \end{align*}
to ensure the error is bounded by $ N^{-2\beta/d}$. Moreover, since the reciprocals of all the error terms ($\epsilon_{\alpha},\epsilon_{\text{inv}}$, e.t.c.) and the upper bound parameters ($C_{13}$ and $C_{14}$) are in the order of $\exp (\cO(\log N))$, the parameters $(M_{t,1}, W_1,\kappa_1,L_1,K_1)$ of the entire network satisfy 
    \begin{align*}
    & \hspace{0.4in} M_{t,1} = \cO\left(\sqrt{\log N}\right),~
     W_1 = {\cO}\left(\log^3 N\right),\\
     & \kappa_1 =\exp \left({\cO}(\log^2 N)\right),~
    L_1 = {\cO}(\log^2 N),~
     K_1= {\cO}\left(\log^4 N\right).
\end{align*}
The proof is complete.